\documentclass{memoir}

\usepackage{geometry}

\usepackage[utf8]{inputenc}
\usepackage[T1]{fontenc}
\usepackage{natbib}
\usepackage{subcaption}
\usepackage{colortbl}
\usepackage{dirtytalk} %
\usepackage{geometry}
\usepackage[skins,theorems]{tcolorbox}

\usepackage{textcomp}

\usepackage[colorinlistoftodos]{todonotes}

\newcommand{\checked}[1]{\textcolor{black}{#1}}

\usepackage{mathtools}
\usepackage{amsthm}
\usepackage{amsmath}
\usepackage{enumitem}
\newtheorem{proposition}{Proposition}[section]

\newtheorem*{axiom*}{Axiom}
\newtheorem{lemma}[proposition]{Lemma}
\newtheorem*{lemma*}{lemma}
\newtheorem{Definition}{Definition}

\newcolumntype{P}[1]{>{\centering\arraybackslash}p{#1}}
\definecolor{Gray}{gray}{0.9}
\definecolor{LightCyan}{rgb}{0.88,1,1}
\usepackage{dsfont}
\usepackage{bm} %
\newtheorem{recommendation}{Recommendation}
\newtheorem{constraint}{Constraint}
\newtheorem{relaxation}{Relaxation}
\newtheorem{desideratum}{Desideratum}

\usepackage{longtable}

\title{Continual Learning: Tackling Catastrophic Forgetting in Deep Neural Networks with Replay Processes}
\author{Timothee Lesort }
\date{May 2020}

\usepackage[english]{babel}
\usepackage{amsmath}
\usepackage{graphicx}
\usepackage{booktabs} 
\usepackage{multirow}           
\usepackage{multicol}

\usepackage{amssymb}
\usepackage{stmaryrd}
\usepackage{mathtools}
\usepackage{multicol}
\usepackage{url}

\usepackage{booktabs} %
\usepackage{tikz} %
\usepackage{array}
\usepackage{float}
\usepackage{wrapfig}
\restylefloat{figure}
\tcbset{highlight math style={enhanced,
  colframe=red,colback=white,arc=0pt,boxrule=1pt}}
\usetikzlibrary{arrows.meta}

\usepackage{verbatim} %
\usepackage[final]{pdfpages}

\usepackage[colorlinks=true]{hyperref}
\setsecnumdepth{subsection}

\begin{document}

\includepdf[page=-, width=\paperwidth, height=\paperheight]{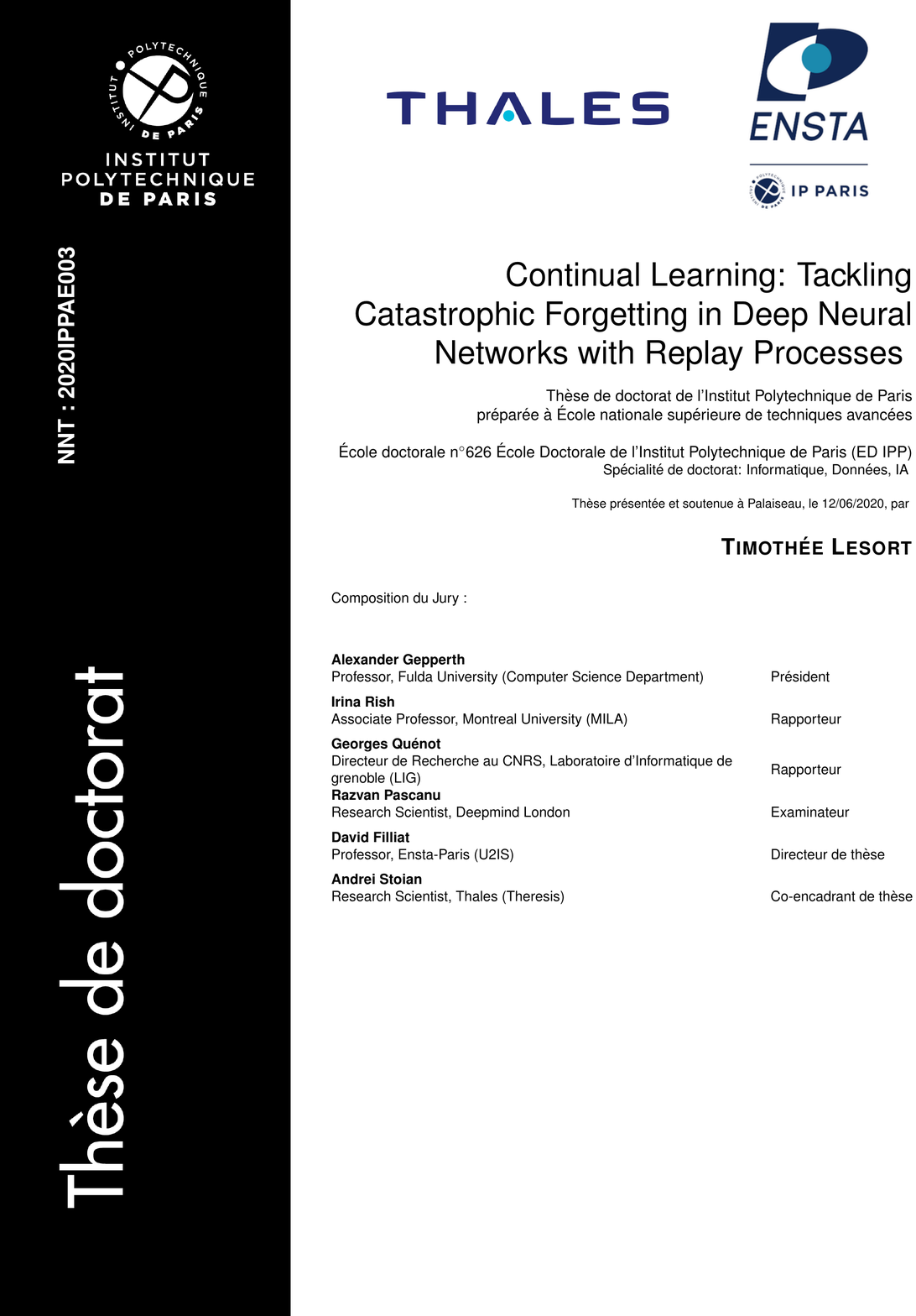}

\newpage

\section*{Remerciements}

J'ai passé trois années pour ce doctorat pendant lesquelles j'ai pu travailler avec beaucoup de liberté sur mon sujet de thèse. J'ai ainsi pu réaliser de nombreuses collaborations ainsi qu'aller à plusieurs conférences. Tout cela m'a permis de m'épanouir dans ma recherche et dans mon sujet et de pouvoir varier mes activités et mes modes de travail. Je suis ainsi reconnaissant à tous les acteurs de cet environnement de travail qui m'a amené à obtenir mon diplôme de doctorat. Par ailleurs, pendant la période de cette thèse, j'ai aussi continué à exercer des activités diverses à côté grâce à ma famille et à mes amis, merci à tous également.

Pour entrer plus dans le détail, tout d'abord j'aimerais remercier mon directeur de thèse, David, qui a été présent pendant toute la durée de ma thèse pour m'aider, me conseiller et m'écouter. Merci pour ta présence, merci pour ta patience.

Je souhaite remercier aussi Andrei et Jean-François qui ont successivement été mes co-directeurs de thèse côté Thales. Merci, d'avoir accepté de m'encadrer et merci pour le temps que vous m'avez consacré.

Par ailleurs, je suis reconnaissant envers mes collègues de Thales qui ont participé chacun à leur manière à faire régner une bonne ambiance au sein du laboratoire, merci Tiago, Thierry, J-E, David, Stéphanie (x2) :), Michael (x3), Cédric, Jean-Yves, Stéphane, Andreina, Yassine, Tom, Rémi, Céline, Louis, mes co-open-space Kévin, Angélique, Astrid, Thomas  ainsi que tous les autres.

Merci également à mes collègues de l'ensta, mes multiples co-bureau, Antonin, Victor, Florence.
Les camarades des trois actes déjeuner-café-baby, Thibault, Olivier, Julien, Alex, Thomas et les autres. Les partenaires de sport: Gabriel, Vyshakh (t'es nul!!), Hugo.
Merci à tous mes co-auteurs, de l'ensta et d'ailleurs, grâce à qui ma thèse a pu avancer plus vite: Natalia (so much work together !), Hugo, Antonin (again), René, Te, Massimo, Vincenzo, Alexander, Florian, Mathieu, Ashley !

Merci à ma famille, merci à mes parents toujours là pour me soutenir, m'aider et m'accueillir, merci à Zaz, Cacou et Céc, une belle friterie toujours prête à partager un bon repas. Merci aussi à Manf notamment pour ces deux formidables neveu et nièce, Loulou et Milou. Merci à mon grand-père Henri, un exemple de ténacité et de vigueur. Ainsi qu'au reste de la famille, notamment Sabine et Claire qui se greffent occasionnellement aux mardis friterie.

Merci à mes amis d'ici et d'ailleurs pour être ce qu'ils sont et avoir été là afin que je puisse penser à d'autres choses que le doctorat!
La TBC : Julaix, Günter, GrandMax, Sylvain, Eddine, Etienne, Jej, Joss et Tang, tous grands amateurs de tarot et de bibines!
Les négligés : Mon p'tit fréro LaKajette (surnom favori s'il en est), Dug mon eternel colloc Munichois et traître parmi les traîtres, mon camarade de chant Driss, Cheucheu, Roro, Pauline et les autres... :)
Les ex-fumistes : Sophie, Loubier, Lucie, Romain, Lucile, Eline, Antho, Quentin, Alexis, PX, Christophe. 
Les autres amis de Paris, ex-Parisiens ou d'ailleurs: Minot, Camille, Emmanuelle, Guidg, Hugo, Louis, Nasta, Olivier, Margaux, Claire, Romdav, Paul, Thomas, le brave Remillieux, Sylvestre, Capucine et tant d'autres.
Les potes de Montréal : Vincent, Florian, César, Alexandre, Thomas, Massimo, Laurent, Theodora. A bientôt, j'espère ;).
Les bons vieux des scouts,  Emile, Balt et Martin, ça fait plaisir qu'on se voit encore parfois.
Les cibeinsois Sean, Marion, Lucie, Mélanie et bien d'autres.
Et enfin Beber et Sylvain, ces personnes respectables et talentueuses, j'espère bien vous voir plus souvent pour travailler efficacement et expérimenter le présent de vérité générale.

Merci également à tous mes relecteurs, ceux de tout temps mais aussi pour ceux qui ont directement participé à ce manuscrit : mon père, ma mère, Zaz, Colin, Cacou, Victor, Vyshakh, Jacko, Günter, Etienne, Capucine, Florence, Andrei et bien sûr, David. 

\newpage

\tableofcontents
\newpage

\listoffigures
\newpage

\listoftables

\newpage

\section*{List of Abbreviation}

\newlist{abbrv}{itemize}{1}
\setlist[abbrv,1]{label=,labelwidth=1in,align=parleft,itemsep=0.1\baselineskip,leftmargin=!}

\begin{abbrv}

\item[CL] Continual Learning
\item[LLL] LifeLong Learning
\item[DL] Deep Learning
\item[ANN] Artificial Neural Network
\item[NN] Neural Network
\item[CNN] Convolutional Neural Network
\item[SGD] Stochastic Gradient Descent
\item[MLP] Multi Layer Perceptron
\item[i.i.d.] Independent and Identically Distributed
\item[MAP] Maximum A posteriori Probability estimate
\item[MLE] Maximum Likelihood Estimation
\item[PCA] Principal Component Analysis
\item[RGB] Red Green and Blue
\item[RL] Reinforcement Learning

\item[LwF] Learning without Forgetting
\item[GEM] Gradient Episodic Memory
\item[iCaRL] Incremental Classifier and Representation Learning
\item[EWC] Elastic Weight Consolidation
 
\end{abbrv}

\newpage

\newpage
\chapter*{Résumé Français}

Les humains apprennent toute leur vie. Ils accumulent des connaissances à partir d'une succession d'expériences d'apprentissage et en mémorisent les aspects essentiels sans les oublier.
Les réseaux de neurones ont des difficultés à apprendre dans de telles conditions. Ils ont en général besoin d'ensembles de données rigoureusement préparés pour apprendre à résoudre des problèmes comme de la classification ou de la régression.
En particulier, lorsqu'ils apprennent sur des séquences d'ensembles de données, les nouvelles expériences leurs font oublier les anciennes.
Ainsi, les réseaux de neurones artificiels sont souvent incapables d'appréhender des scénarios réels tels ceux de robots autonomes apprenant en temps réel à s'adapter à de nouvelles situations et à de nouveaux problèmes \cite{LESORT2019Continual}.

L'apprentissage continu est une branche de l'apprentissage automatique s'attaquant à ce type de scénarios.
Les algorithmes continus sont créés pour apprendre des connaissances, les enrichir et les améliorer au cours d'un curriculum d'expériences d'apprentissage.

Il existe quatre grandes familles d'approches pour l'apprentissage continu. Premièrement, les méthodes à \textbf{architecture dynamique} consistent à faire évoluer l'architecture du réseau de neurones afin que différentes expériences d'apprentissages soient apprises avec différents neurones ou groupes de neurones. Secondement, les méthodes de \textbf{régularisation} évaluent l'importance des neurones ayant déjà été entrainés afin de limiter leurs modifications en conséquence. Il s'agit ainsi d'apprendre de nouvelles tâches préférentiellement avec des neurones jusqu'alors peu ou pas utiles. Troisièmement, les méthodes à \textbf{répétitions de données} consistent à sauvegarder des images représentatives des connaissances apprises et à les rejouer plus tard pour se les remémorer. Le quatrième type de méthode, appelé \textbf{rejeu par génération} utilise un réseau de neurones auxiliaire apprenant à générer les données d'apprentissage actuelles. Ainsi plus tard le réseau auxiliaire pourra être utilisé pour régénérer des données du passé et les remémorer au modèle principal pour éviter qu'il ne les oublie. Nous présentons en détail l'état de l'art de l'apprentissage continu dans le chapitre \ref{chap:2_CL}.

Dans cette thèse, nous proposons d'étudier ces deux dernières méthodes sur des problèmes d'apprentissage sur images. Nous les rassemblont au sein de la famille des méthodes à rejeu de données.
Les méthodes de rejeu de données permettent de trouver un compromis entre l'optimisation de l'objectif d'apprentissage actuel et ceux des experiences d'apprentissage passées.%

Nous montrons que ces méthodes sont prometteuses pour l'apprentissage continu.
Elles permettent la réévaluation des données du passé avec des nouvelles connaissances et de confronter des données issues de différentes expériences. 
Ces caractéristiques confèrent un avantage certain aux méthodes de rejeu de données par rapport aux méthodes à architecture dynamique ou à régularisation, qui peuvent être incapables d'apprendre à différencier des données provenant de différentes expériences d'apprentissage \cite{lesort2019regularization}.

Pour mettre en valeur les avantages de ces méthodes, nous expérimentons les algorithmes à redifussion de données sur des séquences de tâches disjointes. 
Des tâches disjointes sont des tâches d'apprentissage clairement séparées sans intersection dans leur critère d'apprentissage.
 En classification, par exemple, il s'agit d'apprendre à reconnaitre des images provenant de différents ensembles de classes séparés. Le réseau de neurones apprend les uns après les autres chacun des ensembles de classes. Il doit être capable, in fine, de pouvoir reconnaitre une image provenant de n'importe lequel des ensembles et d'identifer sa classe. 
 Ces expérimentations permettent d'évaluer premièrement, la capacité d'apprendre à distinguer des concepts (e.g classes) appris séparément et deuxièmement, la capacité à se souvenir de concepts appris tout au long de la séquence de tâches.

Afin d'étudier au mieux les méchanismes d'apprentissage et d'oubli continus,
les tâches d'apprentissage sont construites à partir d'ensemble de données d'images classiques tels que MNIST \cite{LeCun10}, Fashion-MNIST \cite{Xiao2017} ou CIFAR10 \cite{Krizhevsky09}. Ceux-ci sont des ensembles de données faciles à résoudre en apprentissage profond classique, mais peuvent toujours être ardu à résoudre dans un contexte continu \cite{pfulb2019a}. %
 Nous expérimentons, par ailleur, dans le chapitre \ref{chap:5_DiscoRL} un scénario proche d'une situation réelle en apprenant à un robot à résoudre une séquence de tâche de renforcement.

Nous pouvons ainsi résumer nos contributions de la façon suivante:

\begin{itemize}

\item Nous présentons un aperçut approfondit de l'apprentissage continu (Chapitre \ref{chap:2_CL}). Nous résumons l'état de l'art sur le sujet ainsi que les différents bancs d'expérimentations et méthodes d'évaluations. De plus, nous approfondissons l'exemple de la robotique pour mettre en valeur les potentielles applications de l'apprentissage continu.

\item Nous apportons une preuve théorique des limitations des méthodes dites de régularisation pour l'apprentissage continu. Nous montrons que ces méthodes ne permettent pas d'apprendre à différencier les données provenant de différentes expériences d'apprentissage.

\item Nous réalisons une étude empirique sur l'entrainement des modèles génératifs sur des scénarios d'apprentissage continu et nous introduisons une nouvelle méthode d'évaluation des modèles génératifs : la capacité d'adaptation (Fitting Capacity).

\item Nous expérimentons différentes méthodes de rejeu de données pour l'apprentissage continu. Nous appliquons en particulier ces méthodes aux scénarios de tâches disjointes pour mettre en avant leurs avantages pour l'apprentissage continu.

\end{itemize}

\medskip

Pour résumer, nous démontrons la capacité des méthodes de rejeu de données à apprendre continuellement à travers les paradigmes d'apprentissage non-supervisé (Chapitre \ref{chap:3_CL_GM}), supervisé (Chapitres \ref{chap:2b_Replay} et \ref{chap:4_CL_GR}) et de renforcement (Chapitre \ref{chap:5_DiscoRL}).
Ces experimentations nous permettent de présenter et de mettre en valeur les avantages de ces méthodes et de démontrer leur pouvoir à apprendre certains aspects essentiels d'un curriculum d'expérience d'apprentissage faisant défaut aux méthodes concurrentes.

\chapter{Introduction}
\label{chap:Intro}

\section{Context}

In recent years, machine learning with deep neural networks has significantly improved the state of the art in solving many research and industrial problems.
In vision problems, deep neural networks particularly improve the state of the art in classification and detection. %
In natural language processing, deep neural are nowadays used for search engines or text analysis.
Deep learning also improved reinforcement learning performances. It has made it possible to learn policy and skills in various applications such as video-games, board-games or control.

Robotics is a field which offers significant opportunities for deep learning. Its successes could improve robots cognitive functions such as vision, language processing or exploration. 
Moreover, deep reinforcement learning can help leverage challenging robotics tasks, such as object manipulation or autonomous driving.
 
However, the learning algorithms suffer from many shortcomings. An important restriction of deep learning is the dependence on the quality of the data sets, a clean and well-built dataset being a critical condition to an effective learning process.
In most machine learning algorithms, training data are assumed to be independent and identically distributed (iid), i.e. the data distribution is assumed static. If the data distribution changes while learning, the new data will interfere with current knowledge and erase it. This phenomenon is so dazzling \cite{French99} and it so drastically challenges the algorithms' performance that we call it \say{catastrophic forgetting}. 
This problem has many implications in the way algorithms train neural networks and in the potential application fields for machine learning.

Let us consider, for example, a robot working in an evolving environment and being assigned the goal of manipulating new objects or solving new tasks.
The robot will then need to incrementally learn new knowledge and skills to improve and adapt its behaviour to new situations. 
With classical machine learning techniques, in order to incorporate new knowledge while avoiding catastrophic forgetting, the model will have to re-learn everything from scratch. %
\checked{A robot which needs only the new data to improve and develop knowledge and skills, would be much more efficient in this situation.}

Continual learning (CL) is a branch of machine learning aiming at handling this type of situation and more generally, settings with non-iid data sources. 
It aims at creating machine learning algorithms able to accumulate a set of knowledge learned sequentially.
The general idea behind continual learning is to make algorithms able to learn from a real-life data source. In a natural environment, the learning opportunities are not simultaneously available and need to be processed sequentially. 
Learning from the present data and being able to continue later with new data rather than learning only once for all, seems very appropriate. It opens the possibility to improve the algorithms on a certain task or to make them learn new skills/knowledge without forgetting. 
It also enables transfer between learning experiences.  Previous knowledge acquired may help in learning to solve new problems and new knowledge which may improve the solutions found for past tasks. 
Nevertheless, because of the catastrophic forgetting phenomena, learning continually is a challenging discipline. 
The strategy consisting in saving everything to avoid any forgetting is not satisfying because it would not be scalable in terms of memory and computation. The amount of memory used might grow too fast.
It is therefore important to remember only the essential concepts. Moreover, to deal with catastrophic forgetting, algorithms should identify the potential source of interference between tasks\footnote{The interferences are phenomena happening when different learning criteria conflict.} to come up with a smoother forgetting process.

\begin{figure}[ht]
    \centering
    \includegraphics[scale=0.5]{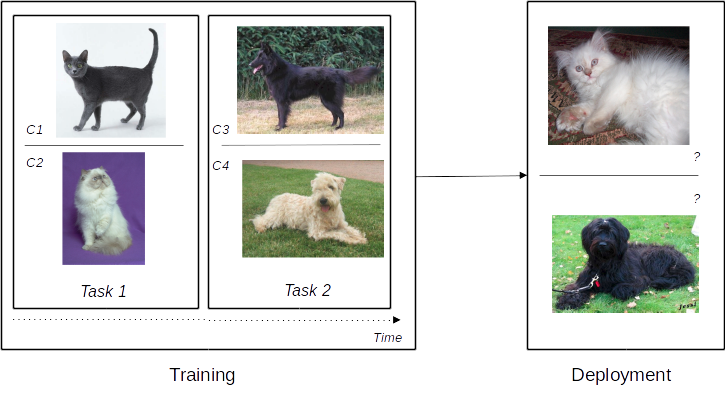}
    \caption[Disjoint tasks illustration.]{Illustration of a disjoint task setting with ImageNet images \cite{krizhevsky12}. There are two tasks learned sequentially, in the first there are two classes, black cats (c1) and white cats (c1), second task is the same but for dogs (c3 vs c4). At the deployment stage, we expect the model to be able to distinguish any class from any other, as white cats from black dogs. Therefore, the model needs both to solve the two tasks and to solve the higher level task which consists of distinguishing classes from various tasks.}
    \label{fig:1_Intro:CL_Tasks}
\end{figure}

\checked{In this thesis, we study specifically learning scenarios where the data is static by parts. Each part is different from the other and is referred to as \say{a task}. This setting is called \textit{disjoint tasks setting} 
and, in classification, it is also called \textit{class-incremental setting}. Each task brings new classes to the learning curriculum and past ones are not available anymore. We show an illustration of class incremental learning in Figure~\ref{fig:1_Intro:CL_Tasks}.
In this type of setting, the forgetting happens only when tasks change. Moreover, as classes are available only in one task, it is convenient to evaluate how the neural network learns and memorizes them. On the other hand, this setting makes it possible to assess if algorithms are able to learn to differentiate classes from different tasks, which is challenging  in continual learning.
Therefore, we study how algorithms are able to deal with this kind of setting as the ability to solve disjoint settings is a necessary condition to be able to deal with real-life settings. }

\section{Contributions and Scope}

As will be detailed later in the thesis, continual learning approaches can be divided into four main categories:
\textbf{Regularization}, \textbf{Dynamics Architectures}, \textbf{Rehearsal} and \textbf{Generative Replay}. In our work, we show that \textit{Regularization} and \textit{Dynamics Architectures} methods have theoretical shortcomings for continual learning and therefore \checked{we focus on studying  applications of replay methods, i.e. \textit{Rehearsal} and \textit{Generative Replay} methods, to categorization and reinforcement learning.}
More precisely,  we study the generative replay and rehearsal methods capacity to learn continually in disjoint settings. %

The contributions of this thesis are: 
\begin{itemize}
\item A global overview of the continual learning research field (Chapter \ref{chap:2_CL}). We present the state of the art in continual learning and introduce classical benchmarks and metrics. 
Moreover, we develop the example of robotics as an application of continual learning solutions.

\item A theoretical proof of the shortcomings of regularization methods for continual learning (Chapter~\ref{chap:2b_Replay}). We show that regularization methods do not provide learning criterion to differentiate data available at different learning experiences \checked{in disjoint tasks settings}. %

\item An empirical study on generative models capabilities in learning continually (Chapter~\ref{chap:3_CL_GM}) with a new evaluation metric: the Fitting Capacity.

\item We experiment with replay methods in continual learning settings
 (Chapters \ref{chap:3_CL_GM}, \ref{chap:4_CL_GR} and \ref{chap:5_DiscoRL}). We study in particular class-incremental settings with supervision free inference. 
 This benchmark highlights the need for replay in continual learning. 
\end{itemize}

This thesis aims at giving an extended perspective on replay methods in continual learning and insights into the advantages of these approaches for continual learning. 
 We apply replay methods to unsupervised, supervised and reinforcement learning to illustrate our statement. 
 
Moreover, we propose an extensive discussion to stress the real requirements of continual learning, to present the advantages of replay methods in continual learning and spread in light the research direction that should be explored to make it progress.

\section{Publications}

Our work has resulted in the following publications:

\subsection{Journals}

\begin{itemize}

\item \cite{LESORT2019Continual} \textbf{Continual Learning for Robotics: Definition, Framework, Learning Strategies, Opportunities and Challenges} (2019) \textit{T Lesort, V Lomonaco, A Stoian, D Maltoni, D Filliat, N D\`iaz-Rodr\`iguez}, Information Fusion, Elsevier, 2019, ISSN 1566-2535, doi: 10.1016/j.inffus.2019.12.004.

\end{itemize}

\subsection{International Conferences}

\begin{itemize}
\item \cite{lesort2018generative} \textbf{Generative Models from the perspective of Continual Learning} (2019) \textit{T Lesort, H Caselles-Dupr\'e, M. Garcia-Ortiz, J-F Goudou, D Filliat}, IJCNN - International Joint Conference on Neural Networks, Budapest, Hungary

\item \cite{lesort2018training} \textbf{Training Discriminative Models to Evaluate Generative Ones} (2019) \textit{T Lesort, A Stoian, J-F Goudou, D Filliat},  Artificial Neural Networks and Machine Learning -- ICANN 2019: Deep Learning, Springer International Publishing, pp 604-619

\item \cite{lesort2018marginal} \textbf{Marginal Replay vs Conditional Replay for Continual Learning} (2019) \textit{T Lesort, A Gepperth, A Stoian, D Filliat}, Artificial Neural Networks and Machine Learning -- ICANN 2019: Deep Learning, Springer International Publishing, pp.466-480

\end{itemize}

\subsection{Workshops in International Conferences}

\begin{itemize}

\item \cite{Kalifou19} \textbf{Continual Reinforcement Learning deployed in Real-life using Policy Distillation and Sim2Real Transfer} (2019) \textit{R Traor\'e*, H Caselles-Dupr\'e*, T Lesort*, T Sun, G Cai, N D\`iaz-Rodr\`iguez, D Filliat}, ICML Workshop on Multi-Task and Lifelong Learning, 2019, Long Beach

\item \cite{Traore19DisCoRL} \textbf{DisCoRL: Continual Reinforcement Learning via Policy Distillation} (2019) \textit{R Traor\'e*, H Caselles-Dupr\'e*, T Lesort*, T Sun, G Cai, N D\`iaz-Rodr\`iguez, D Filliat}, Deep RL Workshop, NIPS 2019, Vancouver

\end{itemize}

\section{Outline}

The organization of the remainder of this manuscript is the following:

\begin{itemize}

\item Chapter \ref{chap:1b_ML} introduces necessary deep learning background to understand the learning processes applied in the thesis experiments.

\item Chapter \ref{chap:2_CL} proposes a global overview of continual learning, its objectives, applications and evaluation.

\item Chapter \ref{chap:2b_Replay} motivates the research in replay methods by pointing out theoretical shortcomings of other methods and by shedding light on replay methods advantages.

\item Chapter \ref{chap:3_CL_GM} presents the generative replay method and evaluate its core component ability: the generative model in a continual context.

\item Chapter \ref{chap:4_CL_GR} experiments the generative replay method in incremental classification tasks sequences.

\item Chapter \ref{chap:5_DiscoRL} brings supplementary results on replay methods by applying rehearsal strategies to a continual multi-task reinforcement learning setting. The resulting algorithm is applied on real robots.

\item Chapter \ref{chap:6_disc} discusses continual learning objectives and use cases, the choices made in the thesis and the traps of continual learning research.

\item Chapter \ref{chap:7_ccl} conclude this 3-year work on continual learning and replay methods and opens research directions for its extension.

\end{itemize}

\newpage
\chapter{Deep Learning Backround:  Principles and Applications}
\label{chap:1b_ML}

Deep learning is a research field that aims at developing learning algorithms. Those algorithms should learn a function that optimizes an objective function on data. In deep learning, this function is implemented as a deep neural network, i.e. a neural network with more than one hidden layer \cite{bengio2009learning, schmidhuber2015deep} (Figure~\ref{fig:1b_ML:DNN}\footnote{Image taken from \url{https://towardsdatascience.com/a-laymans-guide-to-deep-neural-networks-ddcea24847fb}}).

\begin{figure}[ht]
    \centering
    \includegraphics[scale=0.2]{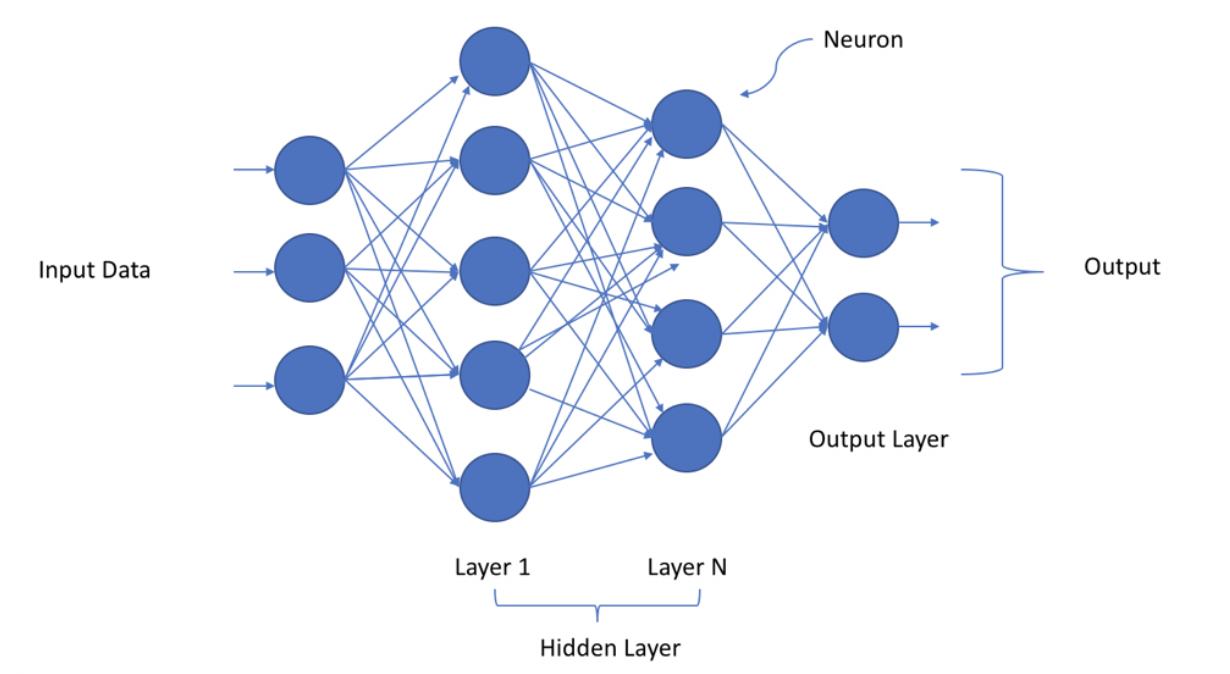}
    \caption[Illustration of a deep neural network (DNN).]{Illustration of a deep neural network (DNN).}
    \label{fig:1b_ML:DNN}
\end{figure}

Deep learning has many applications such as signal processing, language processing or image processing. The scope of this thesis is limited to image processing: we work on algorithms learning from images to understand other images. However, there is no theoretical limitation to transfer results of this thesis in other application fields.

This chapter introduces the basic concepts of classical deep learning in Section \ref{sec:1b_ML:SGD} and its applications in Section \ref{sec:1b_ML:App}. For a more in-depth understanding of the subject we suggest to refer to the book \say{Deep Learning} \cite{Goodfellow2016Deep}. 
We also present the global deep learning pipeline in Section \ref{sec:1b_ML:Pipeline} and introduce in Section \ref{sec:1b_ml:toward_CL}  its constraints leading to continual learning.

\section{Training Deep Neural Networks by Gradient Descent}
\label{sec:1b_ML:SGD}

 We present in this section the simplest method to train deep neural networks: Stochastic Gradient Descent. We also introduce the optimization objective and the libraries dedicated to deep learning.

\subsection{Deep Neural Networks (DNN)}

Deep neural networks (DNN) are artificial neural networks with multiple hidden layers. A layer is composed of a set of neurons connected to neurons from previous layer. They perform a computation and output a single value sent to the next layer. The neurons together form the neural network. A representation of a deep neural structure can be found Fig.~\ref{fig:1b_ML:DNN}.  By combining all the neurons into a coherent ensemble, the neural network should be able to learn complex functions to solve complex problems. 

Mathematically, for a set of $n-1$ input values ${x_1,x_2,..,x_n}$ a neuron will compute the following output:

\begin{equation}
out= \sigma(\sum_{i=1}^{n} x_i\omega_i + b)
\end{equation}
with $\sigma(.)$ a non-linear activation function, $b$ the bias and $\omega_i$ the weights of the neuron. An illustration of a single neuron is presented in Fig.~\ref{fig:1b_ML:neuron}. To train a neural network, we tune the weights (or parameters) and bias of all neurons in order to produce a specific function.

\begin{figure}[h]
    \centering
    \includegraphics[scale=0.5]{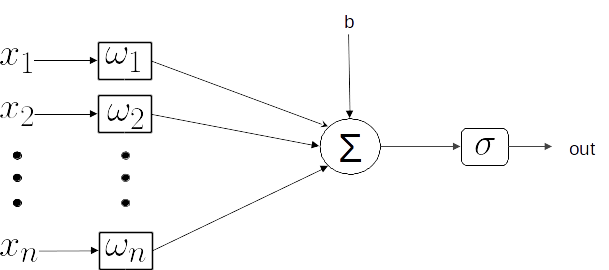}
    \caption[Illustration of an artificial neuron.]{Illustration of an artificial neuron.}
    \label{fig:1b_ML:neuron}
\end{figure}

There are different types of neural networks such as convolutional network or fully connected neural networks. We will present them in Section \ref{sub:1b_ML:Classif}.
In the next section, we will see how to train a DNN.

\subsection{Stochastic Gradient Descent (SGD)}
\label{sub:1b_ML:SGD}

We define the function $f(.)$ implemented by a neural network. $f(.)$ is parametrized by $\theta \in \mathbb{R}^N$ a vector of $N$ real values corresponding to the connection weights and biases of all neurons.
For an input data $x$ we have:
\begin{equation}
\hat{y} = f(x; \theta)
\end{equation}
$\hat{y}$ being the neural network's output.

We assume the dataset composed of pairs $(x, y)$, with $x$ a data point and $y$ its associated expected output.
For each data point $x \in \mathbb{D}$, we can compute the output $\hat{y}=f(x; \theta)$ and the loss $\ell(\hat{y}, y)$. $\ell(\hat{y}, y)$ evaluates the difference between the expected output $y$ and the actual output $\hat{y}$. The loss function is a differentiable function, for example the squared Euclidean distance: 

\begin{equation}
\ell_2(y, \hat{y}) = \parallel y - \hat{y} \parallel^2_2
\end{equation}

The training procedure goal is then to find the best vector $\theta^*$ that minimize the cost function $\ell(.)$ on a dataset $\mathbb{D}$.

\medskip

Deep neural networks are designed such as for each parameter of $\theta$, $\theta_j \in \theta$, we can compute the gradient $\nabla_{\theta_j}$:

\begin{equation}
\nabla_{\theta_j} = \nabla_{\theta_j}(x, y) = \frac{\partial \ell(f(x; \theta),y)}{\partial \theta_j}
\end{equation}

One of the assets of deep neural networks is the efficient back-propagation of the gradient through the model. The gradient can use the chain rule to be transmitted from a layer to another.

\begin{equation}
 \frac{\partial \ell(f(x; \theta),y)}{\partial \theta_j}
 = \frac{\partial \ell(f(x; \theta); \theta_j),y)}{\partial f(x; \theta)} \cdot \frac{\partial f(x; \theta)}{\partial \theta_j}
\label{eq:1b_ML:gradient_propagation}
\end{equation}

Hence, $\frac{\partial \ell(f(x; \theta); \theta_j),y)}{\partial f(x; \theta)}$ can be computed once for all and be used after to compute all the $\nabla_{\theta_j}$.

The gradient is then used to update the value of all $\theta_j$ such that $\ell(\hat{y}, y)$ is minimized.

\begin{equation}
\theta_j \leftarrow \theta_j - \eta \nabla_{\theta_j}
\label{eq:1b_ML:update_rule}
\end{equation}
with $\eta$ the learning rate.

This operation is then repeated for all $(x,y)$ sampled randomly from the dataset, until convergence to a local minima $\theta^*$ of $\ell(f(x; \theta),y)$. This process is called stochastic gradient descent (SGD) \cite{Bottou10large}. It is the simplest method to train a deep neural network by gradient descent. 
Data randomly sampled are called i.i.d. (Identically and Independently Distributed). The i.i.d. assumption on the data distribution is often an essential condition to the success of the training algorithms.

The update rule (eq. \ref{eq:1b_ML:update_rule}) can be modified for a more efficient optimization. Some well known optimization methods are Adagrad \cite{duchi2011adaptive}, Nesterov momentum \cite{sutskever2013importance},  Adam \cite{kingma2014adam}, RMSProp \cite{DauphinVCB15}. They add momentum and acceleration components to the gradient in order to learn faster. 
For the practical applications in this thesis, we mainly used Adam and SGD with momentum to optimize deep neural networks.

\subsection{Overfitting and Generalization}
\label{sub:1b_ML:Generalization}

The optimization process described in Section \ref{sub:1b_ML:SGD} minimizes the loss function on the training data until finding a local minima $\theta^*$:

\begin{equation}
\theta^* = \operatorname*{argmin}_{\theta} \mathbb{E}_{(x,y) \sim \mathbb{D}_{tr}} \ell(f(x;\theta), y)
\label{eq:1b_ML:opt_tr}
\end{equation}
with $\mathbb{D}_{tr}$ the training dataset.

\checked{However, the true objective of deep learning optimization is to make good predictions on never seen data, i.e. to generalize knowledge from training data to new data. 
The ability of making good predictions on unknown data is called \textit{Generalization}. It is measured by computing the loss on a test set $\mathbb{D}_{te}$ never seen by the model. }
If the training loss is very low but the loss on the testing set is high, the model did not learn a good solution to solve the task. This phenomenon is called \textit{overfitting}.
If the loss of the testing set is low then we consider that the model generalized well and the training is successful.

One of the main objectives of machine learning and deep learning is to learn functions that generalize well on new data.
\checked{However, it is important to note that the test set should be similar to the training set. A neural network can not generalize to completely different data.}

\subsection{Deep Learning Programming Frameworks}
\label{sub:1b_ML:Framework}

The training of neural networks is in most situations achieved thanks to programming libraries that are specific to deep learning.
 Those libraries allow to compute efficiently and automatically the gradient for all the parameters and to train neural network faster. Using those libraries makes it also possible to develop code faster and have an easy way to use GPU acceleration for deep neural network training. The most famous deep learning library nowadays are Pytorch \cite{NEURIPS2019_9015}, TensorFlow \cite{tensorflow2015-whitepaper} and Keras \cite{chollet2015} but some years ago caffe \cite{Jia2014Caffe} and Theano \cite{Bastien-Theano-2012} were the most used ones. All of those libraries can be used with python, but some of them have an interface to be used with other programming languages such as C++.

In recent years, those libraries have been developed very intensively, making it possible to find pre-trained models and already implemented architectures, neural layers and optimization processes. Today, they are complete frameworks to develop and train deep neural networks.

In this thesis, all the code to train deep neural networks have been developed in python with the Pytorch framework.

\section{Learning Paradigms}
\label{sec:1b_ML:App}

The training of deep neural networks has been applied to different learning paradigm. These paradigms differs in their supervision signal. Supervised algorithms have a true label for all data point, reinforcement learning algorithms have a sparse label referred to as \textit{reward} and unsupervised algorithms have no label at all.

\subsection{Classification}
\label{sub:1b_ML:Classif}

Images classification (or images recognition) is a typical application of deep learning. It consists of learning to predict the class associated with input data.
In this part, we are interested only in supervised training of deep neural networks for classification. Training by supervision is the most common method for learning classification. 

\subsubsection{History}

In the beginning of the 2010s, deep neural networks helped to make significant progress in the image recognition domain, especially with convolutional neural networks (CNN) architectures \cite{fukushima1980neocognitron} and hardware computation acceleration with graphical processors units (GPU).

The development of GPU hardware contributed to the acceleration of neural networks training. It substantially helped to develop deep neural networks with more layers growing from a few thousands to hundreds of millions parameters in past few years.
Since then, they have been ubiquitous in classification challenges such as PASCAL VOC \cite{Everingham10}, ImageNet \cite{imagenet_cvpr09}, MS COCO \cite{LinMBHPRDZ14} or Open Images \cite{Kuznetsova18}.

Deep neural networks consist of a stack of different neural layers that learn to detect essential features in the data and take decisions. In the early years of classification, feature extractor where hand-engineered, outside the learning algorithm, and only the decision layer was learned. Today, both feature extraction and decision can be learned automatically inside a single neural network. In the next section, we present the important types of neural layers needed for feature extraction and decision making.

Based on these layers, the most famous architectures that helped the development of deep neural networks are LeNet \cite{lecun1998gradient}, AlexNet \citep{krizhevsky12}, Inception \cite{Szegedy15}, VGG \cite{Simonyan15}, ResNet \cite{HeZRS15}. Those models proposed different types of connections between neurons and layers of neurons to help  learning features for image recognition.

\subsubsection{Convolution Layers}

In image classification, the feature extractor is generally composed of a stack of convolution layers \cite{fukushima1980neocognitron}.

\begin{figure}[h]
    \centering
    \includegraphics[scale=0.4]{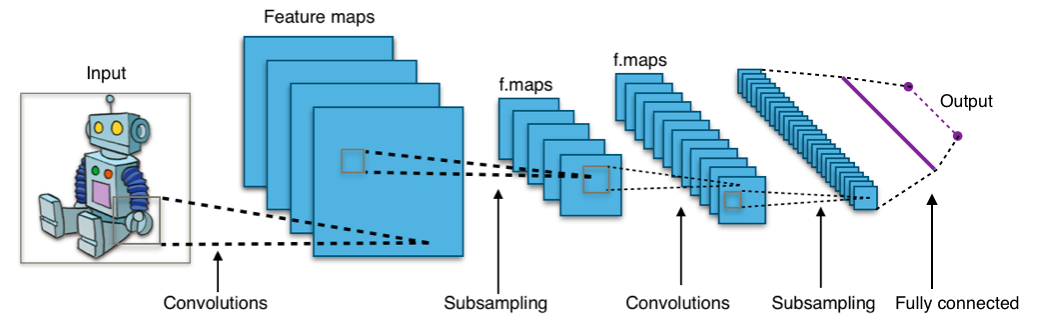}
    \caption[Illustration of an convolution neuronal network.]{Illustration of a multi-layer convolution neural network for image classification. The feature maps contain all the activation output computed with learned filters. The sub-sampling consists of transmitting only a part of the feature map to the next layer.}
    \label{fig:1b_ML:cnn}
\end{figure}

The convolution layers are designed to limit significantly the number of parameters with respect to a fully connected layer (presented in the next section). They are able to capture local dependencies and benefit from the invariance of certain features to learn better, e.g. a car is still a car whatever its position in the images.
The convolution layers are composed of discrete convolution filters (Illustration Figure \ref{fig:1b_ML:cnn}). The goal of each filter is to detect a certain pattern. For a given input, the more the input is close to the feature the higher the output of the convolution will be. 
By stacking convolution layers, the model can detect more and more complex features.
Training the neural network consists of learning the right filters to detect discriminative features allowing them to solve the classification tasks.

The output of a convolution layer is a vector $\bm{h}$ composed of a set of feature map, %
 characteristic of the input features $x$. It is parameterized by the number of filters, their size and how they are applied to the input vector.
$\bm{h}$ is then transmitted to the next layer.

\subsubsection{Fully Connected Layers}

The fully connected (FC) layers have the particularity of connecting all the neurons from one layer to another. Those layer can learn to approximate a high variety of functions, however since they contain a lot of connection, they have a lot of parameters. Their training is then more time consuming and energy consuming than convolution layers.

A FC layer of size $N$ realizes the following function:

\begin{equation}
\bm{o} = W * \bm{h}+ b
\end{equation} 
with $\bm{h}$ the input vector of size $H$, $W$ a weight matrix of size $H*N$ and $\bm{b}$ the bias vector of size $N$.

\subsubsection{Output interpretation and loss}

The output layer is designed to make the right prediction and to be able to compute a gradient that will be back-propagated through the model.

With respect to the learning procedure introduced in Section \ref{sec:1b_ML:SGD}, in the classification case, the expected output $y$ is a label (most of the time an integer) associated with a class of images.
The model (i.e. the neural network) should then, for any image $x$, output a label $\hat{y}$ equal to $y$.

In order to predict such an integer, the classification models have an customized output layer to learn efficiently a solution. Generally, this layer is a fully connected layer which outputs one value per class, the highest value indicating the class selected by the neural network.

Thus, if there are $N$ classes, the output vector $\bm{o}$ is a tensor of float with $N$ values.
The predicted class is then computed as:

\begin{equation}
\hat{y}  = \operatorname*{argmax}_{i \in \llbracket 0, N-1 \rrbracket} (o[i])
\end{equation}
For probabilistic interpretation of the output, it is common to apply a softmax operation to the output. Then, each float value $o[i]$ is transformed to $\sigma(o[i])$ with:

\begin{equation}
\sigma(o[i]) = \frac{e^{o[i]}}{\sum_j e^{o[j]}}
\label{eq:1b_ML:Softmax}
\end{equation}

so that all values are mapped between $0$ and $1$ and they sum to $1$. We will note the resulting tensor $\sigma(\bm{o})$.
We can then compute a loss function to compute a gradient and train the neural network by gradient descent. One example commonly used with the softmax is the negative log-likelihood loss:

\begin{equation}
\ell(\hat{y}, y) = -log(\sigma(o)[y])
\label{eq:1b_ML:NLL}
\end{equation}
with $\hat{y}=\sigma(o)$.%

The gradient descent can then be applied as described in Section \ref{sec:1b_ML:SGD}.

\bigskip

In this thesis, these layers will be directly exploited in Chapter \ref{chap:2b_Replay} and the Chapter \ref{chap:4_CL_GR} dealing with continual learning for classification.
Note that the convolutional and fully connected layers presented in this section are used for reinforcement learning and unsupervised learning as well.

\subsection{Reinforcement Learning}
\label{sub:1b_ML:RL}

Reinforcement Learning is a machine learning paradigm where the goal is to train an agent to perform actions sequences in a particular environment.
The agent should learn a policy which associates the best action to each state of the environment.
It is guided by a reward function providing reward according to policy performance. %
In order to maximize the expected cumulative reward, the agent should explore its environment to discover reward sources and exploit them.

\subsubsection{Training methods}

Most reinforcement learning processes can be described as Markov decision processes (MDPs).  MDPs provide a mathematical framework for modeling decision making in situations where outcomes are partly random and partly under the control of a decision-maker\footnote{Definition taken from \url{https://en.wikipedia.org/wiki/Markov_decision_process}.}. At each time-step $t$, the process is in some state $s_t$, and the decision-maker may choose any action $a_t$ that is available. The process responds at the next time-step by %
 moving into a new state $s_{t+1}$ and giving the decision-maker a corresponding reward $r_t$. This process is illustrated in Figure \ref{fig:1b_ML:RL_env}.

\begin{figure}[h]
    \centering
    \includegraphics[scale=0.8]{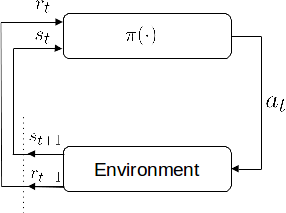}
    \caption[Illustration of reinforcement learning environment.]{Illustration of reinforcement learning environment. With $\pi(\cdot)$ the policy function followed by the agent, $r_t$ the reward at time step $t$, $s_t$ the state of the environment and $a_t$ the action taken by the agent.}
    \label{fig:1b_ML:RL_env}
\end{figure}

The objective of reinforcement learning is to maximize the accumulated reward gathered in a sequence of actions:
\begin{equation}
R_{tot} = \sum^T_{i=1} r_i
\label{eq:1b_ML:sum_reward}
\end{equation}

The total reward $R_{tot}$ is the sum of reward received over all the episode (sequence of actions). A discount factor $\gamma$ can be added to all reward $r_i$ to ponder them.

For each state $s_t$, the neural network function should choose an action to perform. This function is called the policy $\pi(\cdot)$, such as: 
\begin{equation}
p(a|s) = \pi(a,s)
\label{eq:1b_ML:base_RL}
\end{equation}

The expected reward received by an agent following a policy starting at a state $s_t$ is computed by the value function $V^{\pi}(s_t)$:
\begin{equation}
V^{\pi}(s_t) = \mathbb{E} [ \sum^T_{i=1} \gamma^{i-1} r_i ]
\label{eq:1b_ML:value_function}
\end{equation}

The optimal value function $V^*(s_t)$ is the value function for the best possible policy:

\begin{equation}
V^{*}(s_t) = \operatorname*{max}_{\pi} V^{\pi}(s_t)
\label{eq:1b_ML:best_value_function}
\end{equation}

then the best policy $\pi^*$ is:

\begin{equation}
\pi^* = \operatorname*{argmax}_{\pi} V^{\pi}(s_t)
\label{eq:1b_ML:best_policy}
\end{equation}

Many reinforcement learning algorithms rely on good representation of state quality to maximize the expected cumulated reward. Learning the value function is a way of approximating the state quality and learn a good policy.

However, the value function alone does not give directly the right action to realize, it only evaluates the current state. To find the right action to achieve, the Q-function is introduced which evaluates action quality at each state, such as:

\begin{equation}
 V^*(s_t)= \operatorname*{argmax}_{a_t} Q^*(s_t,a_t)
\end{equation}
with $Q^*(s_t,a_t)$ the optimal Q-function. The Q-function makes it possible to introduce the \textit{Bellman equation} \cite{Bellman1957Dynamic} which links values, rewards and Q-functions:

\begin{equation}
Q(s_t,a_t)= R(s_t,a_t) + \gamma  \mathbb{E}[V(s_{t+1})]
\end{equation}
with $R(s_t,a_t)$ the reward received after realizing action $a_t$ at state $s_t$. %

We can then learn a policy that maximizes the reward based on a neural network implementing $Q$. 
A particular method is to learn the optimal $Q$ function which gives directly the best policy. This method is called \textit{Q-Learning} \cite{watkins1992q}. %
There also exist other reinforcement learning training methods types. %

\checked{In the Chapter \ref{chap:5_DiscoRL}, we will use the PPO (Proximal Policy Optimization) algorithm \cite{schulman2017proximal} that belongs to the family of policy gradients methods to directly learn robotics policies. This algorithm is an extension of TRPO (Trust region policy optimization) algorithm \cite{schulman2015trust}. TRPO introduces the trust-region in the policy space. It adds a surrogate loss to constraint the update in the policy space. The goal is to not create drastic changes in the policy and stabilize the learning process. This constraint is implemented as a constraint on the KL divergence between the old and the new policy. KL divergence should not be to high. PPO algorithm does not compute the KL divergence but approximate its action by defining another function that clips the objective function: if the ratio between new policy and old policy is too far from one, the surrogate objective is clipped.
PPO is today a commonly used algorithm for various reinforcement learning applications, we used it for many robotics experiments in \cite{raffin2019decoupling,Raffin18,Kalifou19,Traore19DisCoRL}.}

For more information about reinforcement learning, we link the reader to the \say{Reinforcement Learning: An Introduction} book \cite{Sutton1998}.

\subsubsection{Reward functions}

The reward function defines which behaviour are good or bad for the agent. The reward function can either be sparse and distributes reward only for specific actions (and states) or dense and gives a reward for each action proportionally to the quality of this action.
Dense reward functions make the policy easier to learn but it might be expensive to design. Sparse functions can be cheap and easy to design but it might be very hard to learn the best policy from them. The good reward function is just sparse enough to be cheap to design and easy to use for learning algorithms.

 For example, a reward function could be one point when a robot put a basketball in the basket.
However, the policy that solves those problems might be very difficult to find. 
It might be really hard to get the first reward and exploit it.
 The model needs then to explore the environment to find a potential source of reward. The space to explore might be very large, in the example of basketball, it will be very hard for the algorithms to explore all the possibilities of launching the ball in the basket without some more hints beforehand.
The reward shaping approach is a way to tune the reward function to help the algorithm find a solution. In particular, it aims make reward more frequent during exploration to give more hints of actions' quality. %

\subsubsection{Classical benchmarks}

Reinforcement learning has known recent big success in games such as chess \cite{SilHub18General}, go \cite{Silver2016Mastering} or Dota 2 video games \cite{openai2019dota}. However, commonly used benchmarks are most of the time with simpler robotics settings as in Mujoco \cite{Todorov2012application} or simple video games as Atari \cite{mnih2013playing}.

\begin{figure}[ht]

\begin{subfigure}{0.45\textwidth}

        \includegraphics[scale=0.5]{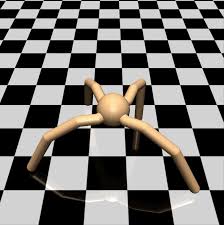}
\centering
    
    \end{subfigure}
    \begin{subfigure}{0.45\textwidth}
\centering
        \includegraphics[scale=0.3]{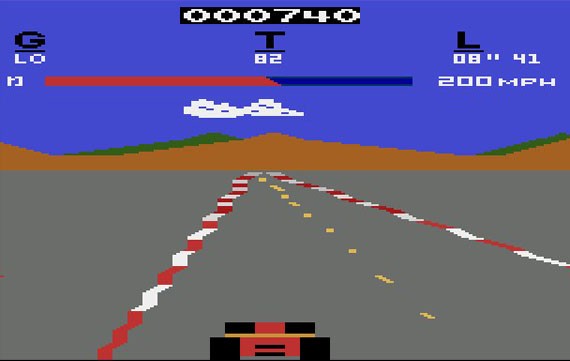}
    \end{subfigure}
    \caption[Illustration of famous RL benchmarks.]{Right: Illustration of the Mujoco environment with the ant task. The ant should walk and move as fast as possible. Left: Illustration of the Atari environment with one of the games. The car should go as fast as possible and stay in the circuit.}
    \label{fig:1b_ML:RL_Images}
\end{figure}

\bigskip

Reinforcement learning is a very appealing subject because it promises that from a simple reward function we can train an algorithm to execute a difficult task. It has many potential applications as robotics \cite{kober2013reinforcement}, autonomous vehicles \cite{sallab2017deep, talpaert2019Exploring} etc. However, in practice, the reinforcement learning algorithms are still very unstable and difficult to train. It remains a very interesting research topic, challenging our understanding of how animals and humans learn and questioning the level of supervision needed to learn a given task.

In this thesis, reinforcement learning paradigm will be applied in Chapter~\ref{chap:5_DiscoRL} in order to learn multiple policies with a robot.

\subsection{Unsupervised learning}
\label{sub:1b_ML:GM}

Unsupervised learning is a wide subject in machine learning. In this thesis we are specifically interested in Generative models, as they are used in continual learning for generative replay  (see Chapter \ref{chap:2_CL}).
They are particular types of neural networks designed to reproduce the input data distribution. We call data distribution, in this context, a theoretical probabilistic distribution that generates the dataset $\mathbb{D}$ and could generate any testing data. The goal is to learn to generate data from this distribution, similar but not identical to the data of the training set. For generative models, the generalization is therefore the capacity to generate novel data points. In this thesis, we focus on generative models for images. 

In recent years, generative models such as BigGAN \cite{brock2018large}, VQ-VAE \cite{van2017neural} or StyleGAN \cite{karras2019style} have shown incredible progress in generating high quality images.
In this section, we introduce two generative models frameworks that led to this progress: variational auto-encoders (VAE) \cite{kingma2013auto, Rezende2014Stochastic} and generative adversarial networks (GAN)  \cite{goodfellow2014generative}. We also introduce a tedious challenge of generative models: the evaluation of generated data.

\subsubsection{Variational Auto-Encoder (VAE)}
Auto-encoders are models that learn to reproduce their input data in their output layer. 
The variational auto-encoder (VAE) \cite{kingma2013auto, Rezende2014Stochastic} framework is a particular kind of auto-encoder (Illustration Fig~\ref{fig:1b_ML:VAE}).
 It is composed generally of an encoder, mapping the input into a latent space and a decoder which learn to regenerate the input image from the latent vector. Those models are useful to learn data compression: if the latent vector is in low dimension, then we can compress input data and decompress it later with the decoder.
The VAE learns to map data into a Gaussian latent space, generally chosen as a univariate normal distribution $\mathcal{N}(0,I)$ (where $I$ is the identity matrix).  The particularity of the latent space comes from the minimization of the Kullback-Leibler (KL) divergence %
 between the distribution of data in the latent space and a prior distribution $\mathcal{N}(0,I)$. The KL divergence \cite{kullback1951} is a measure of the difference between two probability distributions.
The decoder then learns  the inverse mapping from the univariate normal distribution $\mathcal{N}(0,I)$ to the observation space. 
However, since the latent space distribution is the univariate normal distribution, we can sample it without encoding data and generate novel data points.
This characteristic makes the VAE an interesting option for generating new data after training.

\begin{figure}[ht!]
    \centering
    \includegraphics[scale=0.5]{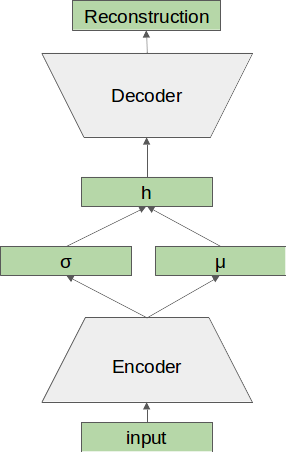}
    \caption[Illustration of the variational auto-encoder (VAE).]{Illustration of the variational auto-encoder (VAE). The encoder computes two vectors from the input data, a mean vector $\mu$ and a standard deviation vector $\sigma$. Then the vector $h$ is sampled from a Gaussian distribution $\mathcal{N}(\mu,\sigma)$ and the decoder output the reconstruction of the input image from $h$.}
    \label{fig:1b_ML:VAE}
\end{figure}

Then, to train the VAE and respecting the prior on the latent distribution, the loss function can be written:

\begin{equation}
\ell = \parallel x - \hat{x} \parallel^2 + D_{KL}( \mathcal{N}(\mu,\sigma), \mathcal{N}(0,I))
\label{eq:1b_ML:VAE}
\end{equation}
with $x$ the input data, $\hat{x}$ the output of the VAE, $\mu$ and $\sigma$ the two VAE latent vectors, $\mathcal{N}$ is a $H$ dimensional (size of the latent dimension) normal distribution parametrized by a mean and a standard deviation vector. %

The Kullback-Leibler between two probabilistic distributions $P$ and $Q$: %

\begin{equation}
D_{KL}(P \parallel Q) = \int_{-\infty}^{\infty} P(x)log\left(\frac{P(x)}{Q(x)}\right) dx
\label{eq:1b_ML:KLD}
\end{equation}

We can note that $D_{KL}(P \parallel Q) \neq D_{KL}(Q \parallel P)$. 

In the case of the VAE the Kullback-Leibler can be simplified to:

\begin{equation}
\forall i \in \llbracket 0, H-1 \rrbracket, ~ D_{KL}( \mathcal{N}(\mu_i,\sigma_i), \mathcal{N}(0,1)) = log\left(\frac{1}{\sigma_i}\right) + \sigma_i^2 + \mu_i^2 - \frac{1}{2}
\label{eq:1b_ML:KLD_simple}
\end{equation}
with $\mu_i$ and $\sigma_i$ the i-th dimension of respectively  $\mu$ and $\sigma$.

\subsubsection{Generative adversarial networks (GAN)}

\begin{figure}[ht!]
    \centering
    \includegraphics[scale=0.5]{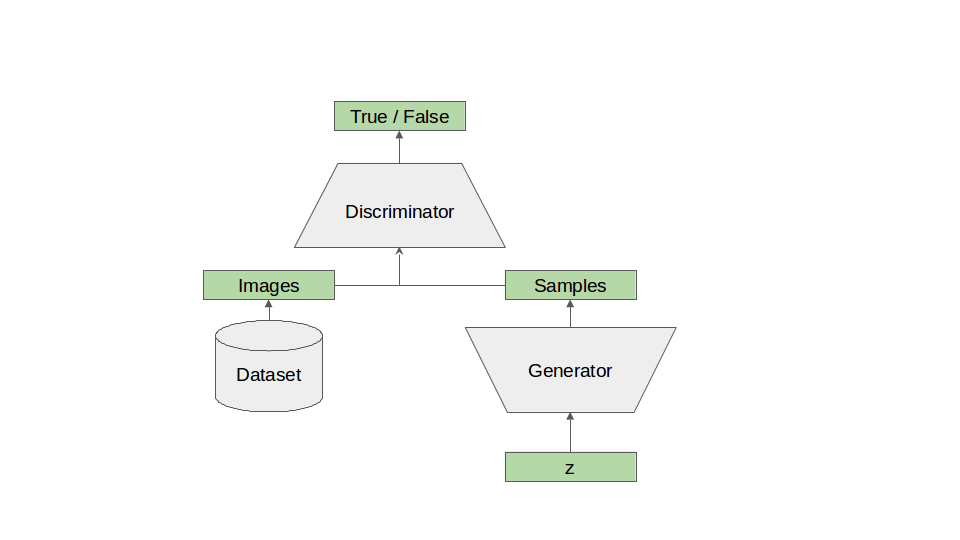}
    \caption[Illustration of the generative adversarial network (GAN).]{Illustration of the generative adversarial network (GAN). An input vector $z$ sampled from a random distribution is given to the generator. The generator output a generated sample. The discriminator receives a mix of true samples and generated samples and predicts if they are true samples or generated one.}
    \label{fig:1b_ML:GAN}
\end{figure}

Generative adversarial network \cite{goodfellow2014generative} is another framework of generative models (Illustration Fig~\ref{fig:1b_ML:GAN}). The learning process is a game between two networks: a generator $G$ learns to produce images from the data distribution $P$ and a discriminator $D$ learns to discriminate between generated and true images. The generator learns to fool the discriminator and the discriminator learns to not be fooled. 
Therefore both $D$ and $G$ have the same loss function $\ell_{GAN}(x, z) $, but $G$ try to minimize it while $D$ try to maximize it:

\begin{equation}
\ell_{GAN}(x, z) = \mathbb{E}_x[log(D(x))] + \mathbb{E}_z[log(1-D(G(z)))]
\label{eq:1b_ML:GAN}
\end{equation}

In this function, $D(x)$ is the discriminator's estimate of the probability that real data instance x is real, $\mathbb{E}_x$ is the expected value over all real data instances, $G(z)$ is the generator's output when given noise $z$, $D(G(z))$ is the discriminator's estimate of the probability that a fake instance is real, $\mathbb{E}_z$ is the expected value over all random inputs to the generator\footnote{Equations and legends are taken from \url{https://developers.google.com/machine-learning/gan/loss}}.

This class of generative models can produce visually realistic samples from diverse datasets but they suffer from instabilities in their training.

First introduced in \citep{goodfellow2014generative}, many follow-up work have extended and improved upon the original model. Among these, GANs have been extended to Conditional-GANs (or CGANs) to support class-conditional generation \citep{mirza2014conditional}, and  numerous of papers \citep{arjovsky2017wasserstein, berthelot2017began,nowozin2016f,gulrajani2017improved} have focused on modifying the objective function \ref{eq:1b_ML:GAN} to stabilize training and improve the generation quality. One of the model we evaluate, the Wasserstein GAN (WGAN) \citep{arjovsky2017wasserstein}, try to address training issues by enforcing a Lipschitz constraint on the discriminator, i.e. they clip the discriminator's gradient to make training more stable.

\bigskip
There exist  many variations of those two frameworks aiming at improving data generation and stability of model with improved losses and architectures. %

\subsubsection{Evaluation methods}

A remaining problem of generative models is the evaluation of the generated samples
because generative models produce images, and it is tedious to have a formal definition of a good image. Furthermore, the expectations on the generated images might be different from one application to another. For example, one might expect images that maximize their reality likelihood while others might expect to maximize their variability.

We present here a partial list of evaluation methods for image generation:

\begin{itemize}

\item{\textbf{Visual Turing Test}} The visual Turing test \cite{Geman3618} is performed by asking humans if images look real or not to assess the generative model quality.

\item{\textbf{Multi-scale structural similarity}}
Multi-scale structural similarity \cite{Wang03} (MS-SIM) is a measurement that gives a way to incorporate image details at different resolutions in order to compare two images. This similarity is generally used in the context of image compression to compare images before and after compression. However, it can be used to estimate the variability of features in generated images \cite{odena2017conditional}. 

\item{\textbf{Inception score}}
One of the most used approaches to evaluate a generative model is Inception Score (IS) \cite{Salimans16,odena2017conditional}. The authors use an inception classifier model pre-trained on ImageNet dataset to evaluate the sample distribution.
They compute the conditional classes distribution $P(Y | X=x)$ at each generated sample $x$ and the general classes distribution $P(Y)$ over the generated dataset.

They proposed the following score:
\begin{equation}
IS(X)=\exp(\mathbb{E}_X [D_{KL} (P(Y | X) \parallel P(Y))]
\label{eq:1b_ML:inception_score}
\end{equation}
where $D_{KL}$ is the Kullback-Leibler divergence. The KL term can be rewritten :

\begin{equation}
  D_{KL} (P(Y | X) \parallel P(Y)) = H(P(Y | X), P(Y)) - H(P(Y|X))
\label{eq:1b_ML:inception_score2}
\end{equation}
where $H(P(Y|X))$ is the entropy of $P(Y|X)$ and $H(P(Y | X), P(Y))$ the cross-entropy between $P(Y | X)$ and $P(Y)$. 

The inception score measures if the inception model predictions gives high confidence in varied classes for the generated data. This relies on the hypothesis that if prediction confidence is high, the input image is good.

\item{\textbf{Frechet Inception Distance}}

Another approach to evaluate generative adversarial networks is the Frechet Inception Distance (FID)~\cite{heusel2017gans}. The FID, as the inception score, is based on features low moment analysis. It compares the mean and the covariance of activations between real data ($\mu$ and $C$) and generated data ($\mu_{gen}$, $C_{gen}$). The activation is taken from an inner layer in a pre-trained inception model. The comparison is done using the Frechet distance (see Eq. \ref{eq:frechet_inception_distance}). The inception model is trained on Imagenet.

\begin{equation}
d^2((\mu,C),(\mu_{gen},C_{gen})) =\parallel \mu - \mu_{gen} \parallel_2^2
+ Tr(C+C_{gen} -  2(C*C_{gen})^{\frac{1}{2}})
\label{eq:frechet_inception_distance}
\end{equation}

FID measures the similarities between the distribution of the generated feature and the distribution of real features. It assumes a Gaussian distribution of features over the dataset.

\item{\textbf{Fitting Capacity}}
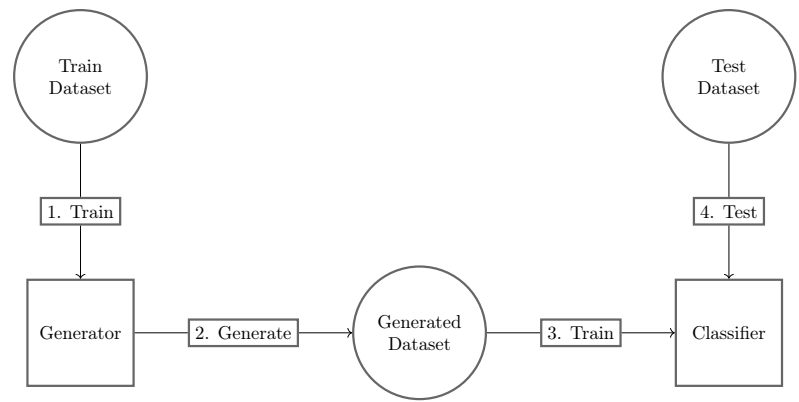
\begin{figure}[h]
\begin{center}
\resizebox{0.7\textwidth}{!}{
\begin{tikzpicture}[scale=0.50,
roundnode/.style={circle, draw=black!60, fill=green!0, very thick, minimum size=25mm},
squarednode/.style={rectangle, draw=black!60, fill=red!0, very thick, minimum size=20mm},
squarednode2/.style={rectangle, draw=black!60, fill=red!0, very thick, minimum size=5mm},
]

\node[squarednode]      (generator)                              {Generator};
\node[squarednode2]      (trainOn)       [above=of generator] {1. Train};
\node[roundnode]        (train)       [above=of trainOn, align=center] {Train \\ Dataset};

\node[squarednode2]        (generate)       [right=of generator] {2. Generate}; 
\node[roundnode]        (generated)       [right=of generate, align=center] {Generated \\ Dataset}; 

\node[squarednode2]      (trainOn2)       [right=of generated] {3. Train};

\node[squarednode]      (classifier)       [right=of trainOn2] {Classifier};
\node[squarednode2]      (testOn)       [above=of classifier] {4. Test};
\node[roundnode]        (test)       [above=of testOn, align=center] {Test \\ Dataset};
 
\draw[-] (train.south) -- (trainOn.north);
\draw[->] (trainOn.south) -- (generator.north);
\draw[-] (generator.east) -- (generate.west);
\draw[->] (generate.east) -- (generated.west);
\draw[-] (generated.east) -- (trainOn2.west);
\draw[->] (trainOn2.east) -- (classifier.west);
\draw[-] (test.south) -- (testOn.north);
\draw[->] (testOn.south) -- (classifier.north);

\end{tikzpicture}

}
\label{fig:shema_methode}
\caption[Illustration of Fitting Capacity method (FiC).]{Illustration of Fitting Capacity method (FiC): 1. Train a generator on real training data, 2. Generate labeled data, 3. Train classifier with the generated data, 4. Evaluate the generator by testing the classifier on the test set composed of real data}
\label{fig:1b_ML_shema_FiC}
\end{center}
\end{figure}

The  \textit{fitting capacity} (FiC) approach is to use labeled generated samples from a generator $G$ (GAN or VAE) to train a classifier and evaluate this classifier afterward on real data \citep{lesort2018training}.  It is  illustrated in figure \ref{fig:1b_ML_shema_FiC}. This estimation of the generative model quality is one of the contributions of this thesis \cite{lesort2018training}. It is presented more in depth in Chapter \ref{chap:3_CL_GM}.

The fitting capacity of $G$ is the test accuracy of a classifier trained with $G$'s samples. It measures the generator's ability to train a classifier that generalizes well on a testing set, i.e the generator's ability to fit the distribution of the testing set. This method aims at evaluating generative models on complex characteristics of data and not only on their features distribution.

\end{itemize}

There exist a lot of different evaluation for generative models as listed and discussed in \cite{borji2018pros}. Anyhow, the best evaluation for a generative model is generally dependent on the future use of the generated data.

\bigskip

In this thesis, the Chapter \ref{chap:3_CL_GM} deals with continual learning for data generation and the \ref{chap:4_CL_GR} take advantage of generative models for continual classification.

\subsection{Classical Benchmarks}
\label{sub:1b_ML:Benchmarks}

We now present some classical machine learning benchmarks (MNIST, Cifar10, ImageNet) and some other benchmarks we will use in the experimental work of this thesis (Fashion MNIST and KMNIST).

\subsubsection{MNIST and MNIST-like datasets}

\begin{figure}[h]
    \centering
    \begin{subfigure}[t]{0.3\linewidth}
        \centering
        \includegraphics[width=0.9\linewidth]{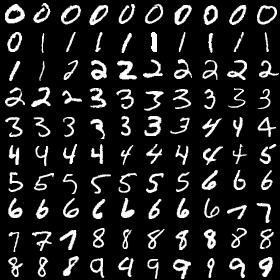}
        \caption{MNIST}
    \end{subfigure}
    \begin{subfigure}[t]{0.3\linewidth}
        \centering
        \includegraphics[width=0.9\linewidth]{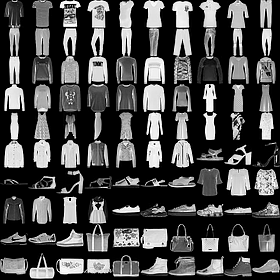}
        \caption{Fashion-MNIST}
    \end{subfigure}
    \begin{subfigure}[t]{0.3\linewidth}
        \centering
        \includegraphics[width=0.9\linewidth]{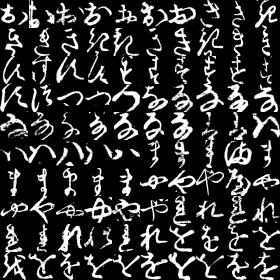}
        \caption{KMNIST}
    \end{subfigure}
    \caption[Samples from MNIST-like datasets.]{Samples from MNIST-like datasets.}
    \label{fig:1b_ML:Mnist_datasets}
\end{figure}

\begin{itemize}

\item \textbf{MNIST}~  \cite{LeCun10} is a common benchmark for computer vision systems and classification problems (Fig.~\ref{fig:1b_ML:Mnist_datasets}). It consists of gray scale 28x28 images of handwritten digits (ten balanced classes representing the digits 0-9). The train, test and validation sets contain 55.000, 10.000 and 5.000 samples, respectively.

\item \textbf{Fashion MNIST}~  \cite{Xiao2017} consists of grayscale 28x28 images of clothes (Fig.~\ref{fig:1b_ML:Mnist_datasets}).   We choose this dataset because it claims to be a \say{more challenging classification task than the simple MNIST digits data  \cite{Xiao2017}} while having the same data dimensions, number of classes, balancing properties and number of samples in train, test and validation sets. 

\item \textbf{KMNIST}~  \cite{Xiao2017} consists of grayscale 28x28 images of Kuzushiji (japanese cursive) (Fig.~\ref{fig:1b_ML:Mnist_datasets}).   As for Fashion-MNIST, we use this dataset because it is a drop in replacement of MNIST dataset and it is a more challenging classification task than the simple MNIST digits data  and adds some diversity in the training data.

\end{itemize}

\subsubsection{Cifar10 / Cifar100}

Cifar10 \cite{Krizhevsky09}  dataset consists of 60000 32x32 colour images in 10 classes, with 6000 images per class. There are 50000 training images and 10000 test images. The classes are completely mutually exclusive. This dataset have been used a lot to design and prototype classification and data generation machine learning algorithms. CIFAR100 is the same dataset with 90 more classes.

\begin{figure}[ht]
    \centering
     \includegraphics[width=0.45\linewidth]{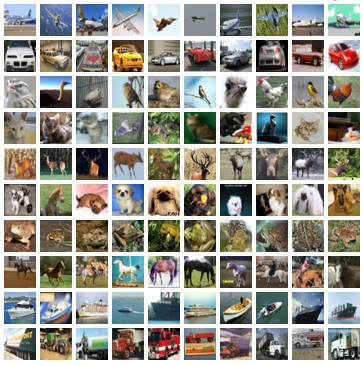}
    \caption[Samples from Cifar10 dataset.]{Samples from Cifar10 dataset.}
    \label{fig:1b_ML:Cifa10_datasets}
\end{figure}

\subsubsection{ImageNet}

ImageNet \cite{krizhevsky12} is a dataset associated with the ILSVRC challenge (ImageNet Large Scale Visual Recognition Challenge). The dataset used in this context is composed of one thousand non-overlapping classes. The images are in color and are from different shapes and resolution but they are often normalized at 224*224 pixels. This dataset is one that leads to the revolution of machine learning for classification. It is also used nowadays for data generation purposes.

\begin{figure}[ht]
    \centering
     \includegraphics[width=\linewidth]{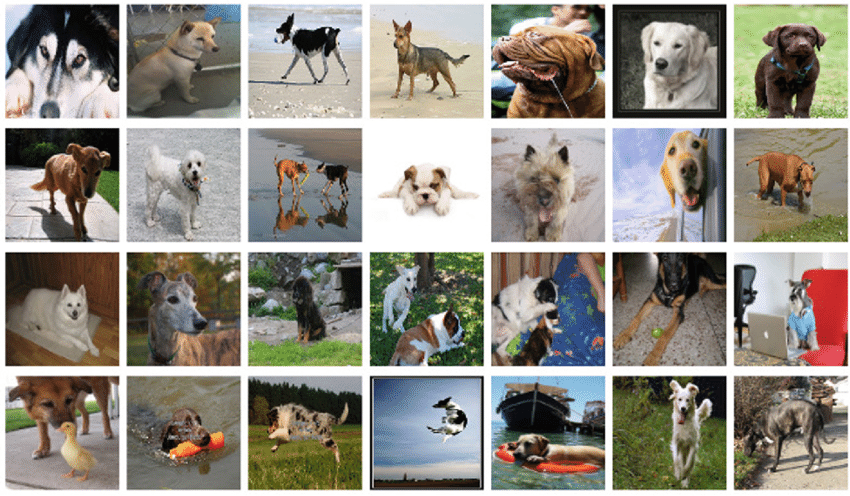}
    \caption[Dog samples from ImageNet dataset.]{Exemples of Imagenet samples with dog pictures.}
    \label{fig:1b_ML:ImageNET_datasets}
\end{figure}

\bigskip

The benchmark listed above are linked to supervised classification, but they have been adapted for benchmarking other type of machine learning such as data generation, few shot-learning or disentanglement.
In this thesis we adapted these classical benchmarks into continual learning benchmarks. Even if some of them can looks almost trivial for nowadays machine learning algorithms they can still be challenging in continual learning settings \cite{pfulb2019a}. In Chapter \ref{chap:2_CL}, we will present the different continual learning benchmarks.

\section{Learning procedure}
\label{sec:1b_ML:Pipeline}

Besides the training optimization process, the success of deep learning algorithms relies significantly on data handling. 
In this section, we present the classical full pipeline necessary to train a neural network.

\subsection{Data Gathering}
\label{sub:1b_ML:Gathering}
First of all, to train a DNN, data needs to be gathered to create a dataset. In classification, the image of the different categories needs to be selected and their label set and verified. Building a good quality dataset is difficult, categories need to be balanced, images should be varied and a trivial solution should not bias the problem if we want to train correctly a neural network. 
For example, one should avoid that to classify cows from birds, it might be possible to only check if the image is mostly blue (because of sky) or green (because of grass) to solve the classification, instead of looking at intrinsic characteristics from cows and birds.
The trivial solutions come from bias into the training set making the learning model believe false causal relationship. A well-designed dataset aims at avoiding misleading correlations in the data.

\subsection{Data Splitting}
\label{sub:1b_ML:Splitting}

Before training the neural network, the dataset is split into three different sets. The training set, used to learn parameters, the validation set, used to select the hyper-parameters and the test set that will be used to evaluate the final performance and verify if the model trained can generalized on new data. It is important to split those set wisely because they should be similar but different enough to be able to measure generalization.

\subsection{Model Architecture}
\label{sub:1b_ML:Archi}

Once the dataset is built, we choose a model architecture. The model architecture consists of the stack of layers, the characteristics of each layer and how they connected to each other.
 The architecture defines implicitly the family of functions we are searching our solution $\theta^*$ in. Architectures are often parametrized by hyper-parameters that can change the model architecture to better fit a learning problem, e.g. the number of layers, the number of convolutional filters by layer, the size of the filters or the padding...

\subsection{Model Initialization}
\label{sub:1b_ML:Init}

Once the architecture is chosen, the model and optimization process should be initialized, i.e. all parameters and hyper-parameters should be set.
The architecture hyper-parameter (HP) described in the previous Subsection \ref{sub:1b_ML:Archi} are generally chosen empirically based on existing algorithms. The optimizer hyper-parameters such as the learning rate or the ones specific to certain optimizers are also chosen empirically. In the frameworks described in Section \ref{sub:1b_ML:Framework}, the optimizer functions are proposed with default values that can be used directly.

Once the hyper-parameters are fixed, we initialize the parameters. In the literature, there are several heuristics to initialize parameters, generally following a random policy.
The most famous one are \textit{Xavier initialization} \cite{glorot2010understanding}, or \textit{Kaiming initialization} \cite{he2015delving}, or a simple normal initialization $\mathbb{N}(0,1)$. %

For example, in the \textit{Xavier} initialization, layers are initialized such as any weight (parameter) $w$ is sampled according to:
\begin{equation}
w \sim \mathcal{U}(-\frac{\sqrt{6}}{\sqrt{n_i + n_{i+1}}},+\frac{\sqrt{6}}{\sqrt{n_i + n_{i+1}}})
\end{equation}
with $\mathcal{U}$ the uniform probabilistic distribution, $n_i$ the incoming network connections and $n_{i+1}$ the out coming network connections from that layer. The bias are initialized to zero.

\subsection{Data Preprocessing}
\label{sub:1b_ML:Prepro}

In order to maximize the learning process chance of converging to a satisfying solution, one could choose to apply specific transformations to the data.
First, the data can be modified to be easier to process by the model. For example by applying an heuristic transformation to improve the saliency of crucial components of data. For example, change the color spectre to remove useless colors. In deep learning, a known heuristic for data preprocessing is to normalize value such as the mean is zero and the standard deviation 1. Another common one is the normalization of input data such that all data have the same mean and standard deviation channel by channel. For example, the default values in pytorch to normalize input for model trained on ImageNet are, respectively for channel RGB $mean=[0.485, 0.456, 0.406]$, $std=[0.229, 0.224, 0.225]$. 
 
To improve the model's generalization performance (see Section \ref{sec:1b_ML:SGD}), we can perform \textit{data augmentation}, i.e. add a random modification to training images. 
Those modifications should change the image without changing what it represents, i.e. without changing essential features. Thus, we expect from the model to learn only important features and ignore the rest.
As discussed in Section \ref{sub:1b_ML:Gathering}, the construction of a dataset can be difficult and produce some bias. The data augmentation can help to correct them artificially.
For example, to detect a car in an image, we expect the model to not give importance to the color (cars can be in any color). If the car's color in the dataset is only white, it will create a bias in the model making it believe that cars are always white. 
To correct this bias we can then artificially change the color of cars.    

\subsection{Parameters Optimization}
\label{sub:1b_ML:Optim}

The model can now be trained by gradient descent and learn the best set of parameters $\theta^*$ as described in Section \ref{sec:1b_ML:SGD}.
$\theta^*$ should optimize the loss on training data but its real objective is to optimize the loss on unknown data. As described in Section \ref{sub:1b_ML:Generalization}, the real application goal is to train a model able to generalize.

\subsection{Hyper-Parameters Optimization}
\label{sub:1b_ML:HP_Optim}

Frequently, the hyper-parameters set at the beginning do not allows to learn the most satisfying solution.
The validation set help then to select the model that has learned and generalizes best.
Indeed, the validation data has not been used for training, the error computed on this set is not biased by the optimization process and can be used to select the best hyper-parameters.
The hyper-parameters selection is a tough problem, contrary to parameters it is not possible to compute the gradient of the hyper-parameters, their selection should then follow empirical heuristics.
Among the possible approaches, grid search consists in searching all the combination of hyper-parameters for specific regularly separated values, but it is also advised to rather search them randomly \cite{bergstra2012random}.

There is a research field aiming at finding automatically hyper-parameters, called auto-machine learning \cite{NIPS2015_5872,elsken2018neural}. The final goal is still to maximize the generalization capacity on never seen data as presented in Section \ref{sub:1b_ML:Generalization}, but these approaches are computationally extremely intensive and are rarely used in practice.

\section{Towards Continual Learning}
\label{sec:1b_ml:toward_CL}

The deep learning pipeline, as presented in the previous section, has shown impressive results in numerous applications such as classification, detection, generation, language processing... However, it imposes to have all the data at the beginning of the learning process. In particular, it assumes  i.i.d. data during training. This assumption is unfortunately difficult to maintain in many situations, for example when data are gathered online. In this section we introduce context in which the pipeline presented in Section \ref{sec:1b_ML:Pipeline} can not be applied rigorously. Then, we will present briefly the implications in the optimization process.

\subsection{Context}
\label{sub:1b_ML:context}

The constraint of collecting all the data before training a neural network is inconvenient in many cases. De facto, a dataset is very likely to be completed with new data, either to increase the variability of existing concepts in data or to add new concepts. In the classification perspective, either to add data to existing classes or to add new classes to the existing ones.

A trivial solution is to train a neural network from scratch with all data every time new data are available.
Nevertheless, training large neural networks can take weeks to months to be trained, then retraining every time can be costly and time-consuming. Especially if new data are regularly available as in a stream of data.
Being able to improve a trained network with new data only would be then quite advantageous.
Another situation, where training from scratch is not possible, is when the data is not available anymore. For example, if a client buys a pre-trained model but has no access to the initial training data. If this client wants to improve the model with its own data, he can not retrain it from scratch since the initial data are not available. It can also happen if the data have not been saved for legal reasons or memory limitations.
From a more general point of view, it would be more convenient to learn from the data available and improve later with new resources. In order to handle such settings, representations should be learned in an online manner \cite{Li17learning}. As data gets discarded and has a limited lifetime, the ability to forget what is not important and retain what matters for the future are the main issues that algorithms should targets and focuses on.

The field of deep learning aiming at solving this problem of data availability is called \textit{Continual Learning}.
In particular, it aims at finding working solutions for agents which learn from an evolving environment and that need to learn continually to adapt to unseen situations and remember already learned solutions to known situations. For examples robots.

Indeed, from a robotics point of view, CL is the machine learning answer to developmental robotics \cite{Lungarella03Developmental}. 
Developmental robotics is the interdisciplinary approach to the autonomous design of behavioral and cognitive capabilities in artificial agents that directly draws inspiration from developmental principles and mechanisms observed in children's natural cognitive systems \cite{Cangelosi18,Lungarella03Developmental}
In this context, CL must consist of a process that learns cumulative skills and that can progressively improve the complexity and the diversity of tasks handled.
Autonomous agents in such settings learn in an open-ended \cite{Doncieux18} manner, but also in a continual way.  Besides CL, crucial components of such developmental approach consist of learning the ability to autonomously generate goals and explore the environment, exploiting intrinsic motivation \cite{Oudeyer07} and computational models of curiosity \cite{Oudeyer18}.

\subsection{Continual Learning procedure}
\label{sub:1b_ML:optim}

We have seen that the classical pipeline for deep learning have limitations in continual learning, in particular, because all the data might not be available at the same time. The pipeline from Section \ref{sec:1b_ML:Pipeline} should therefore be adapted so that data gathering and learning can be simultaneous or interleaved. 

\subsubsection{Optimization}
Concerning the optimization process to train neural networks, the fact that all data are not available at the same time removes the i.i.d. assumption. 
Many deep learning algorithms are then not adapted anymore for such situations.
For most of them, in a continual learning situation, the neural network will automatically adapt to the last data only and forget everything learned on the previous one. This phenomenon is named \say{catastrophic forgetting} \cite{French99}.

\subsubsection{Data gathering}
In fact, gathering data continually is not the essential problem, as theoretically, it is still possible that gathered data are drawn i.i.d. The real problem is the evolution of the underlying data distribution. If this data distribution changes, then \say{catastrophic forgetting} might happen. This phenomenon is called \textit{concept drift} \cite{gepperth2016incremental}. 
In order to find suitable algorithms to avoid catastrophic forgetting, the concept drift needs to be detected and estimated. This process is tedious because the concept drift can be abrupt or very progressive, making it difficult to grasp. 

\subsubsection{Evaluation}
Furthermore, the evaluation process when the data distribution is not static is complicated. As seen in Section \ref{sub:1b_ML:Gathering}, building a dataset to learn should be achieved carefully, there should be a proper balance between classes, a good variability in the data to learn. In the case of continual learning, the evaluation set should be constructed while learning, which might be difficult. Anyhow, it remains essential to have an evaluation set to avoid overfitting.

\subsubsection{Hyper-Parameters Optimization}
Another tedious process in continual learning is the hyper-parameters optimization. In classical deep learning, the hyper-parameters are tuned to make the algorithms learn better its full task. However, for a convenient set of hyper-parameter at a time $t$ it is possible that at time $t+T$ this set is not good anymore, the hyper-parameter should then be modifiable during the training. Moreover, a HP set can be good to learn but not to remember, so at time $t+T$ if the algorithms forget everything it can not come back at time $t$ to re-learn. The HP selection should both be optimized to present data and withstand to potential future concept drifts.

\bigskip

We will not propose a new pipeline for continual learning because many continual learning cases are different and should not be handled in the same way.
Nevertheless, the spirit of the pipeline from Section \ref{sec:1b_ML:Pipeline} should be maintained. The idea of gathering data, preprocessing them, designing a proper model and train it should be kept but more dynamically than in classical deep learning settings. The less impacted step of the original pipeline is \textit{Model Initialization} since it should be achieved only once at the training's beginning.

\section{Conclusion}
\label{sec:1b_ML:ccl}

\checked{In this chapter, we presented the simplest method to train deep neural networks: Stochastic Gradient Descent (SGD). The training is achieved by optimizing the model parameter on a dataset. The training goal is to make a neural network able to make decisions on never seen data, i.e. able to generalized its training data. We described the complete pipeline of deep neural network classical training, from data gathering to hyper-parameters selection and we point out the limitation of this pipeline for real-life applications. Finally, we introduced \textit{continual learning}, a research field that aims at overcoming those limitations.}

In the next chapter, we present a more in-depth overview of continual learning state of the art, key vocabulary and objectives. We also present a framework to frame any continual learning approach, with a set of benchmarks and evaluation metrics.

\newpage
\chapter{Continual Learning}
\label{chap:2_CL}

In the previous chapter, we introduced classical deep learning basic concepts and pipeline and showed their lack of adaptability in practical situations. We also introduced how continual learning aims at solving those shortcomings. %
In this chapter, we present the continual learning research field more extensively. We illustrate the need for continual learning through the lens of robotics. In particular, we stress the need for better practice in research for better transfer from a field to another, as from simulation to robotics. %

This work was the fruit of a collaboration with Vincenzo Lomonaco and Natalia D\'iaz-Rodr\'iguez. It was published in the \say{Information Fusion} Journal \cite{LESORT2019Continual}. The original article has been slightly modified to better fit the thesis thread and updated to add some recent papers.

\section{Introduction}
\label{sec:2_CL:Intro}

\checked{As described in Section \ref{sec:1b_ml:toward_CL},}
machine learning (ML) approaches generally learn from a stream of data randomly sampled from a stationary data distribution. This is often a \textit{sine qua non} condition to learn efficiently, which makes many application scenarios difficult to solve. %

For convenience, we can empirically split the data stream into several  temporally bounded parts called \textit{tasks}.
We can then observe what we learn or forget when learning a new task.
Even if there is no mandatory constraint on a task, a task often refers to a particular period of time within which the data distribution may (but not necessarily) be stationary, and the objective function constant. Tasks can be disjoint or related to each other, in terms of learning objectives, depending on the setting.

We now propose a framework for continual learning. This framework also sets the opportunities for continual learning to have a description frame to present approaches in a clear and systematic way. %
We can summarize the chapter's contributions as following:

\begin{itemize}
\item \checked{We propose an in depth state of the art of continual learning approaches.}
\item \checked{We present a framework to help characterizing continual learning approaches and benchmarks.}
\item \checked{We gather a list of benchmarks and evaluation metrics for continual learning.}
\item \checked{We develop the example of robotics as an application field of continual learning research.}
\end{itemize}

In the sequel, we first present the context and the history of continual learning. Second, we aim at disentangling vocabulary around continual learning to have a clear basis. Third, we introduce our framework as a standard way of presenting CL approaches to help transfer between different fields of continual learning, especially to robotics. 
Fourthly, we present a set of metrics that will help to better understand the quality and shortcomings of every family of approaches.
Finally, we present the specifics and opportunities of continual learning in robotics that make CL so crucial.

\section{Definition of Continual Learning}
\label{sec:2_CL:Def}

Given a potentially unlimited stream of data, a Continual Learning algorithm should learn from a sequence of partial experiences where all data is not available at once.
A non-continual learning setting would then be when the algorithm can have access to all data at once and can process it as desired.
Continual learning algorithms may have to deal with imbalanced or scarce data problems \cite{Sprechmann18}, catastrophic forgetting \cite{French99}, or data distribution shifts \cite{gepperth2016incremental}.

We consider continual learning a synonym of \textit{Incremental Learning} \cite{gepperth2016incremental, rebuffi2017icarl}, \textit{Lifelong Learning} \cite{Chen2018Lifelong, Thrun95} and \textit{Never Ending Learning} \cite{Carlson10, Mitchell15}. 
For the sake of simplicity, for the remainder of the chapter we refer to all Continuous, Incremental and Lifelong learning synonyms as Continual Learning (CL). \checked{However, in the discussion of the thesis (Chapter \ref{chap:6_disc}) we show that we can use them to distinguish different continual learning scenarios.}

In this section we first present the history and motivation of continual learning, then we present several definitions of terms related to CL and, finally, we present challenges addressed by CL in machine learning.

\subsection{History and Motivation}
\label{sub:2_CL:History}

The concept of learning continually from experience has always been present in artificial intelligence and robotics since their birth \cite{turing09, weng01}. However, it is only at the end of the $20^{th}$ century that it has begun to be explored more systematically. Within the machine learning community, the lifelong learning paradigm has been popularized around 1995 by \cite{Thrun95} and \cite{ring94}.%

Between the end of the 90s and the first decade of the $21^{st}$ century, sporadic attention has been devoted to the topic within the supervised, unsupervised and reinforcement learning domains. However, despite the first pioneering attempts and early speculations, research in this area has never been carried out extensively until the recent years \cite{Parisi18review, Chen2018Lifelong}. We argue that this is because there were more complex and fundamental problems to solve and a number of additional constraints:

\begin{itemize}
\item \emph{Lack of systemic approaches}: Machine learning research for the past 20 years has focused on statistical and algorithmic approaches on simple tasks (e.g., tasks where the distribution of data is assumed static). CL typically needs a systems approach that combines multiple components and learning algorithms in complex and dynamic tasks. The complexity of tasks and their multiple uses in continual learning greatly complicates training and evaluation procedures. Disentangling \emph{``static''} learning performance from continual learning side effects is important for the very incremental nature of the research and to facilitate comparison between approaches in this area.

\item \emph{Limited amount of data and computational power}: Digital data is a luxury of the $21^{st}$ century. Before the big data revolution, collecting and processing data was a daunting task. Moreover, the limited amount of computational power available at the time did not allow complex and expensive algorithmic solutions to run effectively, especially in a continual learning setting which undoubtedly makes learning more complex by having to deal with multiple tasks at the same time, as well as having to incorporate the concept of time into the learning process. 
\item \emph{Manually engineered features and ad-hoc solutions}: Before early 2000s and first works on representation learning, creating a machine learning system meant to handcraft features and finding ad-hoc solutions, which may differ significantly depending on the task or domain. Having a general algorithm with a more systematic approach seemed for a long time a very distant goal.
Manually engineered features is also a clear limitation to achieve autonomy, as new tasks need to have the same features or re-engineered ones.
\item \emph{Focus on supervised learning}: creating labelled data is probably the slowest and the most expensive step in most ML systems. This is why learning continuously has been for a long time not a viable and practical option.
\end{itemize}

The relaxation of these constraints, thanks to recent advancements and results in machine learning research, as well as the rapid technological progress witnessed in the last 20 years, have open the door for starting tackling more complex problems such as learning continually. %

We argue that the robotics community, which has always been intrigued by endowing embodied machines with lifelong and open-ended learning \cite{Doncieux18} of new skills and new knowledge, would highly benefit from the recent advances of ML in this area. Robotics applications in unconstrained environments, indeed, have always raised questions out of reach for previous machine learning techniques. On the other hand, CL developed in the context of robotics is involved in understanding the %
role and the impact of the concept of \say{embodiment} in intelligent machines that learn and think like humans. 

Learning, embodiment, and reasoning are presented as the three great families of challenges for robotics in \cite{Sunderhauf18}. We postulate that CL tackles the learning problem, taking into account the importance and constraints of embodiment. At best, CL would also benefit from reasoning in order to maximize the learning process. Thus, continual learning lies in the intersection of crucial robotics challenges.

Though lifelong learning approaches do exist in various ML disciplines (such as evolutionary algorithms for example \cite{Bellas09, Bellas10,Bredeche18,Bellas10cognitive}), we will focus, in the rest of this thesis, on recent continual learning developments in the context of gradient-based neural network and deep learning approaches.
For a more detailed description of many other classic approaches to continual learning with shallow architectures we refer the reader to \cite{Chen2018Lifelong}.

\section{Key vocabulary}
\label{sec:2_CL:vocabulary}

\begin{Definition}
\textbf{Learning objective} 
\label{def:2_CL:learning_objective}
\checked{The learning objective is composed of a data set and a loss function, that has to be optimized. The learning objective change if either the loss or the data change.}
\end{Definition}

\begin{Definition}
\textbf{Task} 
\label{def:2_CL:task}
A task is a learning experience characterized by a unique task label $t$ and its target function $g_{\hat{t}}^*(x) \equiv h^*(x,t=\hat{t})$, i.e., the objective of its learning.
\end{Definition}

\begin{Definition}
\textbf{Task label} 
\label{def:2_CL:task_label}
The task label is a variable that define tasks boundaries. It might be available or not depending on the learning scenario.
\end{Definition}

\begin{Definition}
\textbf{Continuum} 
\label{def:2_CL:continuum}
The continuum is the full learning experience. It is composed by a sequence of tasks.
\end{Definition}

\begin{Definition}
\textbf{Data distribution} 
\label{def:2_CL:data_distribution}
The data distribution, is a theoretical statistical distribution that generate the data. This distribution can be constant or may variate through time.

\end{Definition}

\begin{Definition}
\textbf{Data stream} 
\label{def:2_CL:data_stream}
The data stream is the flux of samples generated by the data distribution. A task is a set of the data stream, the continuum is the full data stream.
\end{Definition}

\begin{Definition}
\textbf{Forgetting} 
\label{def:2_CL:forgetting}
A neural network forget when its performance on a data distribution is decreased by learning on another one. 
\end{Definition}

\begin{Definition}
\textbf{Interferences} 
\label{def:2_CL:Interferences}
\checked{In machine learning, interferences are conflicts between two (or more) objective functions leading to prediction errors. }
\end{Definition}

\begin{Definition}
\textbf{Concept drift} 
\label{def:2_CL:concept_drift}
\checked{The concept drift characterizes the learning objective variations, i.e. when the learning criterion or the data distribution changes.
 When there is no concept drift the learning objective is constant.
 Concept drift may lead to forgetting in neural networks models.}
\end{Definition}

\subsection{Terminology Clarification}

In this section we aim at clarifying the distinction and similarities of continual learning with related topics and terms used in the literature.

\textbf{Online learning:} 
Online learning is a special case of CL \cite{Kaeding16} where updates are done on per single data point basis and therefore, the batch size is one. 
Online learning algorithms are suited to scenarios where information should be processed instantly, either to adapt the model to learn as fast as possible or because data cannot be saved.

\textbf{Few-shot Learning:}  
Few shot learning \cite{Lake11,Fei-Fei06} is the ability to learn to recognize new concepts based on only few samples of them. 
It may be used for continual learning problems when the number of data points is very low.
The extreme case of zero-shot learning %
consists of the ability to detect new %
classes %
while being trained with a disjoint set of classes \cite{Wang19}. %

\textbf{Curriculum Learning: }
Curriculum learning \cite{Bengio09curriculum} is a training process that proposes a sequence of  more and more difficult tasks to a learning algorithm in order to make it able to learn, at last, a generally harder task. 
The sequence of tasks is designed in order to be able to learn the last one.
Both CL and curriculum learning learn on a sequence of tasks (or partial experience). However, in curriculum learning, tasks are chosen in a way that makes possible to learn tasks of different complexity, by taking into account the difficulty of them, while in CL, tasks are not voluntarily chosen nor ordered. 
Furthermore, while the interest of curriculum learning ultimately lies %
into solving the last task, the continual learning objective is to be able to solve all tasks.

\textbf{Meta-learning: }
Meta-learning \cite{Brazdi2008Metalearning} is a learning process that uses meta-data about past experiences, such as hyper-parameters, in order to improve its capacity to learn on new experiences.
It also learns several different tasks; however, its goal is not to learn without forgetting but to progressively improve the learning efficiency while learning on more and more tasks.
It is also called \say{learning to learn}, and it can be used or not in a continual learning setting. 

\textbf{Transfer learning: }
Transfer learning \cite{Pratt93,Finn17, Zhao17} is the ability to use what has been learned from a previous task on a new task. The difference with continual learning is that transfer learning is not concerned about keeping the ability to solve previous tasks. In computer vision, transferring what has been learned from a past environment to new environments would be often referred to as \textit{domain adaptation} \cite{Patel15,Csurka17}.

\textbf{Active Learning: } 
Active learning is a special case of semi-supervised machine learning in which a learning algorithm is able to interactively query the user (or some other information source) to obtain the desired output labels for new data points \cite{Settles09, Burr10}. %
Active learning may be used in CL to query new examples and have control of the data the algorithm has access to.

\subsection{Challenges Addressed by CL} 

In this section we describe the specific problems addressed by continual learning; the kind of problems that arise when data cannot be assumed i.i.d., and when the hypothesis that the data distribution is static is not valid.

\subsubsection{Catastrophic Forgetting}

Catastrophic forgetting \cite{Mccloskey89,French99} refers to the phenomenon of a neural network experiencing performance degradation at previously learned concepts %
when trained sequentially on learning new %
concepts \cite{Mccloskey89}. Since by definition the continual learning setting deals with sequences of classes or tasks, the catastrophic forgetting is an important challenge to be tackled.  
Catastrophic forgetting might also be referred to as \textit{catastrophic interference}. The notion of interference is pertinent since the acquisition of new skills interferes with past skills by modifying important parameters as described in definition \ref{def:2_CL:Interferences}.

\subsubsection{Handling Memories}

One of the main components that distinguishes two CL approaches is the way they handle memories. In order to deal with catastrophic forgetting, each strategy should find a way to remember what may be destroyed by learning future tasks.
Continual learning needs a mechanism to \textit{store} memories of past tasks, which can take very various forms. It is important to note that memories can be saved in different manners: as raw data, as representations, as model weights, regularization matrices, etc.
An efficient memory management strategy should only save important information, as well as be able to transfer knowledge and skills to future tasks.
In practice, it is almost impossible to know what will be important and what could be transferable in the future; a trade off should then be found between the precision of the information saved and the acceptable forgetting.
This trade-off problem is known as the stability/plasticity dilemma \cite{Mermillod13}.

An important challenge inherent to handling memories is to automatically assess them. Learning new tasks may lead to degradation of the memories. As a consequence, the memory process needs mechanisms to evaluate how the memories are degraded, i.e., how it forgets. As no more data and labels from past tasks may be available, this check-up might be very challenging.

\checked{Another challenge is the stability of the learning process, which is crucial to learn and remember. Instability might lead to exploding gradients that would accidentally and permanently erase memories.}

\subsubsection{Detecting Distributional Shifts (concept drift)}

When the distribution is not stationary, a shift into the data stream is observed.
 When there is no external information concerning this shift, the CL model has to detect it, and account for fixing it by itself.
An undetected shift in the data distribution will irrevocably lead to forgetting.
Changes in the data distribution over time are commonly referred to as \textit{concept drift}. This idea is related to online change detection algorithms \cite{Sarkar98,Moens18} or Bayesian surprise \cite{Sun11} in ML. 
Two kinds of concept drift are defined \cite{gepperth2016incremental}:
Virtual and real concept drift.
Virtual concept drift concerns the input distribution only, and can easily occur, e.g., due to imbalanced classes over time. 
Real concept drift, on the contrary, is caused by novelty on data or new classes, and can be detected by its effect, on e.g., classification accuracy.
However, shift may also happen when the task changes. In RL for example an agent may have to solve a new task. Then the shift is not exactly in the data distribution but in the supervision signal.
Regardless of where exactly the shift happened it has to be detected to avoid catastrophic interference with non related skills or knowledge.

\subsection{Learning Paradigms Orthogonal to Continual Learning}

In this section we describe the relationship of continual learning with respect to the main three, generally acknowledged machine learning paradigms introduced in Chapter \ref{chap:1b_ML}: supervised, unsupervised and reinforcement learning.

\subsubsection{Supervised Continual Learning}
\label{subsub:2_CL:sup}

Supervised learning is the machine learning problem of learning from input-output example pairs \cite{Russell09}. \checked{In Chapter \ref{chap:1b_ML}, we introduced supervised learning and its different use cases.}
 
 While the study of continual learning in this context may help disentangling the complexity introduced by algorithms that learn continually, in the context of robotics, the lack of supervision does not allow, most of the time, to apply directly supervised methods.

\subsubsection{Unsupervised Continual Learning}
\label{subsub:2_CL:unsup}

Unsupervised learning refers to machine learning algorithms that do not have labels or rewards to learn from. \checked{In Chapter \ref{chap:1b_ML}, we already introduced unsupervised learning and in particular generative models.}
In the context of robotics, unsupervised continual learning may play an important role in building increasingly robust multi-modal representations over time to be later fine-tuned with an external and very sparse feedback signal from the environment. 
In order to learn robust and adaptive representations with unsupervised learning, the main objective is to find suitable surrogate and meaningful learning signals, as robotics priors \cite{Jonschkowski14, Lesort19}, self-supervised models or curiosity driven techniques.

A particular unsupervised task learned in a continual learning setting is the generation of images. Image generation is achieved by training generative models to reproduce images from a dataset. In a CL setting, the distribution changes over time and the generative model should be able to produce at the end images from the whole distribution. This problem has been studied for various generative models as adversarial models \cite{wu2018memory, lesort2018generative}, variational auto-encoders \cite{nguyen2017variational, ramapuram2017lifelong, achille2018life,Farquhar18, lesort2018generative} and standard auto-encoders \cite{Triki17,Zhou12}.

There is also a different relation between unsupervised learning and CL, since unsupervised models can be used to learn representations from vast amounts of data sources and can then generate such data (cf Section \ref{subsub:2_CL:GR}). This capacity can then be used to perform CL for classification \cite{wu2018incremental, shin2017continual, Triki17, lesort2018marginal} or reinforcement learning tasks \cite{caselles2018continual}.%

\subsubsection{Continual Reinforcement Learning}
\label{subsub:2_CL:reinforcement}

Reinforcement Learning is a machine learning paradigm where the goal is to train an agent to perform actions in a particular environment in order to maximize the expected cumulative reward. As explained in Chapter \ref{chap:1b_ML}, in traditional RL, the world is modeled as a stationary MDP: i.e., fixed dynamics and states that can recur infinitely often \cite{Ring05}. \checked{Chapter \ref{chap:1b_ML} also presented a basic learning process of reinforcement learning.}
Since in general, complex RL environments have no access to all data gathered at once, RL could often be framed as a CL situation.  %
Moreover, RL borrows several tools used in CL models, such as approximating data to an i.i.d. distribution, via either \textit{i)} setting multiple agents or actors to learn in parallel \cite{Mankowitz18}, or \textit{ii)} using a replay buffer (or experience replay \cite{mnih15}), that is equivalent to a particular category of CL (rehearsal, see Section \ref{subsub:2_CL:rehearsal}).
Another link is found in a popular stable method in RL, the TRPO%
 algorithm \cite{schulman2015trust}, %
  which constrains learning by using an estimate of the Fisher information matrix to improve learning continually, in the same way as some CL strategies (e.g., EWC, see Section \ref{subsub:2_CL:penalty}).
Most of Continual Learning approaches in RL have been applied in simulation settings such as Atari games \cite{kirkpatrick2017overcoming}. However, many %
approaches \cite{Traore19DisCoRL, Kalifou19, Bellas10, Bredeche18} also solve use cases on real robots.

\checked{As we can see, the continual learning problems are real shortcoming in reinforcement learning algorithms. Improving continual learning will therefore necessarily help improving reinforcement learning performance.}

\section{A Framework for Continual Learning}
\label{sec:2_CL:framework}

Despite the rapidly growing interest in continual learning and mainly empirical developments of the recent years \cite{Parisi18review}, very little research and effort has been devoted to a common formalization of algorithms that learn continually in dynamic environments. However, the availability of a common ground for thoroughly evaluating and understanding continual learning algorithms is essential to reduce ambiguities, enhancing fair comparisons and ultimately better advancing research in this direction.

\subsection{Setting Description}
\label{sub:2_CL:Setting_Desctiption}

Being able to better compare and evaluate continual learning strategies, while still being general enough to overlook implementation-dependent details over different learning paradigms, becomes essential. This is specially true when targeting deployment of CL paradigms in real-word applications, such as robotics. Nowadays, %
despite the existence of a basic set of shared practices, many are the fundamental questions often overlooked in recent continual learning research. For example, questions about the data availability during training and evaluation, the amount of supervision with respect to the tasks separation and composition, as well as common but biased assumptions on the nature of the data among others. A list of questions of interest we would like to address and report are the following:

\subsection{Questions}

\begin{enumerate}[label=(\alph{*})]

    \item Data Availability

    \begin{itemize}[noitemsep]
        \item $\bm{Q_1}$: \emph{Does some data need to be stored? if yes, how and what for? (e.g. regularization, re-training, validation)?}
        \item $\bm{Q_2}$: \emph{Is the algorithm tuned based on the final performance? I.e. is it possible to go back in time to improve performance?}
        \item $\bm{Q_3}$: \emph{Are data distributions assumed i.i.d. at any point?}
        \item{$\bm{Q_4}$: \emph{Is each task assumed to be encountered only once? }}
    \end{itemize}

    \item Prior Knowledge
    \begin{itemize}[noitemsep]
        \item $\bm{Q_5}$: \emph{Is the continual learning algorithm agnostic with respect to the structure of the training data stream? (e.g. number of classes, numbers of tasks, number of learning objectives...)}

        \item $\bm{Q_6}$: \emph{Does the approach need a pretrained model for the CL setting? If so, what is the new knowledge that needs to be acquired while learning continually?}
    \end{itemize}
    
    \item Memory and Computational Constraints
    \begin{itemize}[noitemsep]
        \item $\bm{Q_7}$: \emph{How much available memory does the algorithm require while learning? Does the memory capacity requirement changes as more tasks are learned?}
        \item $\bm{Q_{8}}$: \emph{Is the continual learning algorithm constrained in terms of computational overhead for each learning experience? Does the computational overhead increase over the task sequence? }
        \item $\bm{Q_{9}}$: \emph{Is the continual learning algorithm agnostic with respect to the data type? (e.g. images, video, text,...)} 
        \item $\bm{Q_{10}}$: \emph{Is the continual learning algorithm able to handle situations where there is not enough time to learn?}
    \end{itemize}
    
    \item Amount/Type of Supervision
    
    \begin{itemize}[noitemsep]

        \item $\bm{Q_{11}}$: \emph{In the presence of multiple tasks, is the task label available to the algorithm during the training phase? And during evaluation?}
        \item $\bm{Q_{12}}$: \emph{Are all the data labeled? or only the first training set? Can the user provide sparse label/feedback (e.g. active learning) to correct the system errors?}
        
    \end{itemize}

    \item Performance Expectation
    
    \begin{itemize}[noitemsep]

        \item $\bm{Q_{13}}$: \emph{What is expected from the algorithm to remember at the end of the full stream? Is it acceptable to forget somehow, when task, context or supervision change?}
    \end{itemize}
    
\end{enumerate}

To summarize these questions, in any new CL algorithm proposition, it is fundamental to clearly describe the data stream, its use, the algorithm functioning, its assumed 
prior knowledge, and its requirements in terms of supervision, memory and computation. \checked{We will answer those questions in the discussion chapter of the thesis (Chapter \ref{chap:6_disc})}.

We will now propose a comprehensive and detailed framework to help distinguish and disentangle different approaches in different continual learning settings and help answer these questions.

Early theoretical attempts to formalize the CL paradigm are found in \cite{Ring05} as a combination between reinforcement learning and inductive transfer. More general framework approaches include the one on non i.i.d. tasks of \cite{Pentina15}. As in \cite{Pentina15}, we assume CL is tackling a probably approximately correct (PAC) learnable problem in the approximation of a target hypothesis $h^*$ as well as learning from a sequence of non i.i.d. training sets. Our framework could also be seen as a generalization of the one proposed in \cite{Lopez-Paz17}, where learning happens continuously through a \textit{continuum} of data and a \say{task supervised signal} $t$ may be provided along with each training example.

\subsection{Framework Definitions}
\label{sub:2_CL:Framework_Definitions}

In continual learning data can be conveniently seen as drawn from a sequence of distributions $D_i$, and thus the need to redefine a CL framework taking into account this important property is defined as follows.

\begin{Definition}
\textbf{Continual Distributions and Training Sets} 

In Continual Learning, $\mathcal{D}$ is a potentially infinite sequence of unknown distributions $\mathcal{D} = \{D_1, \dots, D_N\}$ over $X \times Y$, with $X$ and $Y$ input and output random variables, respectively. At time $i$ a training set $Tr_i$ containing one or more observations is provided by $D_i$ to the algorithm.
\end{Definition}

As the framework hereby proposed is supposed to be general enough to cover the orthogonal and classical unsupervised, supervised and reinforcement learning approaches, $Tr_i$, as better detailed in Definition \ref{def:2_CL:cla}, is a collection of training observations/data samples that act as signal of the joint distribution to be learned.

\begin{Definition}
\textbf{Task} 
\label{def:2_CL:task2}

A task is a learning experience characterized by a unique task label $t$ and its target function $g_{\hat{t}}^*(x) \equiv h^*(x,t=\hat{t})$, i.e., the objective of its learning.
\end{Definition}

It is important to note that the tasks are just an abstract representation of a learning experience represented by a task label. This label helps to split the full learning experience into smaller learning pieces. However, there is not necessarily a bijective correspondence between data distributions and tasks.

\begin{Definition}
\label{def:2_CL:cla}
\textbf{Continual Learning Algorithm}
Given $h^*$ as the general target function (i.e. our ideal prediction model), %
and a task label $t$, %
a continual learning algorithm $A^{CL}$ is an algorithm with the following signature: 
\begin{equation}
	\forall D_i \in \mathcal{D}, \hspace{20pt} A^{CL}_i:\ \ <h_{i-1}, Tr_i, M_{i-1}, t_i>  \rightarrow <h_i, M_i> 
\end{equation}

Where:
\begin{itemize}
	\item $h_i$ is the current hypothesis at timestep $i$, or, practically speaking, the parametric model learned continually.
	\item $M_i$ is an external memory where we can store previous training examples or partial computation not directly related to the parametrization of the model.
	\item $t_i$ is a task label, that can be used to disentangle tasks and customize the hypothesis parameters. For simplicity, we can assume $N$ as the number of tasks, one for each $Tr_i$.
    \item $Tr_i$ is the training set of examples. %
Each $Tr_i$ is composed of a number of examples $e_j^i$ with $j \in [1,\dots,m]$. Each example $e^{i}_j = <x^{i}_j, y^{i}_j>$, where $y^{i}$ is the feedback signal and can be the optimal hypothesis $h^*(x,  t)$ (i.e., exact label $y^{i}_j$ in supervised learning), or any real tensor (from which we can estimate $h^*(x, t)$, such as a reward $r^{i}_j$ in RL). 
\end{itemize}
\end{Definition}

It is worth pointing out that each $D_i$, can be considered as a stationary distribution. However, this framework setting allows to accommodate continual learning approaches where examples can also be assumed to be drawn non i.i.d. from each $D_i$ over $X \times Y$, as in \cite{gepperth2016incremental,Hayes18NewMetrics}.

\begin{Definition}
\textbf{Continual Learning scenarios}
\label{def:2_CL:scen} 
A CL scenario is a specific CL setting in which the sequence of $N$ task labels respects a certain ``task structure'' over time. Based on the proposed framework, we can define three different common scenarios:
\begin{itemize}
	\item \textsf{Single-Incremental-Task (SIT)}: $t_1 = t_2 = \dots = t_N $.
    \item \textsf{Multi-Task (MT)}: $ \forall i,j \in [1,.., n]^2, i\neq j \implies t_i \neq t_j$.
    \item \textsf{Multi-Incremental-Task (MIT)}: $\exists\ i,j,k:\ t_i = t_j$ and $t_j \neq t_k $.
\end{itemize}
\end{Definition}

Table \ref{tab:2_CL:batch-examples} illustrates an example to clarify the definition of SIT, MT and MIT.

An example of Single-Incremental-Task (SIT) scenario is an ordinary classification task between cats and dogs, %
where the distribution changes through time. First, there may only be input images of white dogs and white cats, and later only black dogs and black cats. Therefore, while learning to distinguish black cats from black dogs the algorithm should not forget to differentiate white cats from white dogs. The task is always the same, but the concept drift might lead to forgetting.

However, in a classification setting, a Multi-Task (MT) scenario would first consist of learning cats versus dogs, and later cars versus bikes, without forgetting. The task label changes when the classes change, and the algorithm can use this information to maximize its continual learning performance.
The Multi-Incremental-Task (MIT) is the scenario where the same task can happen several times in the sequence of tasks.%

\begin{table}[ht]
\centering
\caption[Illustration of continual learning scenarios.]{Illustration of continual learning scenarios. Sequential task labels (corresponding to different distribution $D_i \in \mathcal{D}$) to reflect differences among CL categorization w.r.t. number and unicity of tasks for SIT, MT and MIT. 
Notice that a MIT setting requires relaxing the constraint definition of SIT but also relaxing the constraint definition of MT, i.e., it corresponds to the case where not all the tasks are considered having the same \textit{ID}, and not all the task are considered distinct.}
\label{tab:2_CL:batch-examples}
\begin{tabular}{|l|c|c|c|} %
\hline
\textbf{Task ID/Session} %
&   \multicolumn{3}{|c|}{\textbf{CL settings}}  \\\hline
\textbf{Task ID} & \textbf{SIT} & \textbf{MT} & \textbf{MIT} \\\hline\hline
$t_1$ & 0 & 1 & 0   \\\hline  
$t_2$& 0 & 2 & 1   \\\hline 
$t_3$ & 0 & 3 & 0   \\\hline 
... & ... & ... & ...   \\\hline 
$t_i$ & 0 & i & ...   \\\hline 
\end{tabular}
\end{table}

In any learning problem (be it classification, RL or unsupervised learning), the ability to adapt to new concepts to be learned (from the PAC ML framework \cite{Valiant84}),
as well as new instances of each concept, should be accounted. This is the objective of the next definition where we formally set three different settings an algorithm is required to manage, as they can have very high impact on the algorithm performance.

\begin{Definition}
\textbf{Task label and concept drift scenarios}
\label{def:2_CL:label} 
The task label can specify different assumptions made in a continual learning scenario.
We can define three main categories of task label assumptions regarding concept drift:

\begin{itemize}
	\item \textsf{No task label}: Changes in the distribution are not signaled by any task label. The task is always the same (equivalent to SIT scenario).
    \item \textsf{Sparse task label}: Changes in the distribution are sparsely signaled by the task label. There are several tasks but changes in distribution may as well happen inside a task.
    \item \textsf{Task label oracle}: Every change in the data distribution is signaled by the task label, which is given.  %
\end{itemize}
\end{Definition}

We illustrate the different scenarios in Figure \ref{fig:2_CL:concept_drift}.

\begin{figure}[ht]
    \centering
    \includegraphics[width=0.6\textwidth]{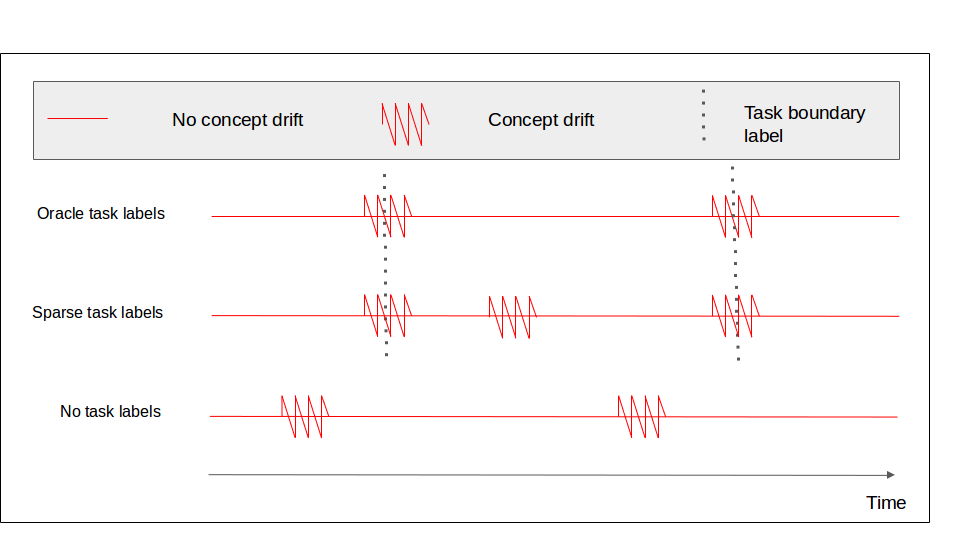}
    \caption[Task label and concept drift: illustration of the different scenarios.]{Task label and concept drift: illustration of the different scenarios.}
    \label{fig:2_CL:concept_drift}
\end{figure}

\begin{Definition}
\textbf{Availability of task label.} When a task label is provided, %
it is worth distinguishing among two different cases:
\begin{itemize}
	\item \textsf{Learning labels}: Task labels are provided for learning only. At test time, inference should be done without knowing from which task a data point is coming from. 
    \item \textsf{Permanent labels}: The task labels are provided for learning, and it is assumed they will also be provided at test time for inference.
\end{itemize}
\end{Definition}

\checked{The need for a task label is a need for supervision. It might be legit for training since it helps significantly the learning process. However, it might be a shortcoming in terms of autonomy at test time (for inference).}

\begin{Definition}
\label{def:2_CL:cut}
\textbf{Content Update Type}
The nature of the data samples or observations contained in each $Tr_i$ can be conveniently framed in three different categories: 
\begin{itemize}
    \item \textsf{New Instances (NI)}: Data samples or observations contained in the training set at time-step $i$ %
    relate to the same dependent variable $Y$ used in the past. %
	\item \textsf{New Concepts (NC)}: Data samples or observations contained in the training set at time-step $i$ %
    relate to a new dependent variable $Y$ to be learned from the model. %
    \item \textsf{New Instances and New Concepts (NIC)}: Data samples or observations contained in the training set at time-step $i$ %
    relate to both, already encountered dependent variables, and new ones ($Y$). %
\end{itemize}
\end{Definition}

In order to exemplify the concept of \emph{Content Update Type} defined in Definition \ref{def:2_CL:cut}, let us recover the aforementioned example of classification. If an algorithm learns the \textit{cat vs dogs} classification problem on a dataset and then new images of cat vs dogs are provided to the algorithm, we are then in a \textit{New Instances} case (NI), we have new data but no new concepts. If the new instances were of different classes (e.g. cars vs bikes) we then would face the \textit{New Concepts} case (NC). The new instances and new concepts case would then have been a mix of both new images of known and new classes.

If a CL algorithm uses a network pretrained on a dataset, the features of such dataset will need to be accounted for as one more task or the same, depending on the distribution of new instances and new classes according to definitions \ref{def:2_CL:scen} and \ref{def:2_CL:cut}. In other words, using a pretrained model is similar to assume there is a task already learned by the model, and the new learning experiences of the algorithm are just a continuum of learning curricula. If there is any intersection between the pretraining and the new tasks, it should be reported in the setting description. The pretraining effect can then be estimated with the metrics proposed in Section \ref{sub:2_CL:metrics}.

\subsection{Constraints}
\label{sub:2_CL:Constraincts}

\checked{In this section, we phrase constraints that are advised to be respected to make small experiments scalable to bigger ones.}

\begin{constraint}
For every step in time, the number of current examples contained in the memory is lower than the total number of previously seen examples\footnote{I.e., if we could fit all previous examples in memory $M$, it would become a problem of scarce interest for the CL community, given that re-training the entire model $h_i$ from scratch would be always possible \cite{Kaeding16}}.
\end{constraint}

\begin{constraint}
Memory and computation for each iteration step $i$ are bounded. Given two functions $ops()$ and $mem()$ that compute the number of operations and memory occupation required by $A_i^{CL}$, two reasonably small values 
\emph{max\_ops} and \emph{max\_mem} should exist, such that, for each $i$, $ops(A_i^{CL}) < max\_ops$ and $mem(h_{i-1},M_{i-1}) < max\_mem$.%
\end{constraint}

$max\_ops$ and $max\_mem$ are the max throughput, in number of operations, and the max memory capacity of the system running $A_i^{CL}$. Having a memory and computational bounds for each iteration $i$ is an important constraint for a continual learning algorithm. The reason is that the number of training sets $Tr_i$ can potentially be unlimited, and thus, computation and memory should not be proportional to the number of hypothesis updates $h_i$ over time. %
A finite upper bound should exist and be considered, especially with $n \rightarrow \infty$.

\subsection{Relaxation and desiderata}
\label{sub:2_CL:Desiderata}

Given the difficult setting and the additional constraints imposed by Continual Learning with respect to the classic ``static'' setting, many researchers in the recent literature have proposed new CL strategies in slightly relaxed \cite{Rusu16progressive,kirkpatrick2017overcoming, Mallya18Piggyback, Lopez-Paz17} yet reasonable settings: 

\begin{relaxation}
\textsf{Memory relaxation}: Removes the fixed memory bound constraint over $ops()$ and $mem()$. 
\end{relaxation}

\begin{relaxation}
\textsf{Computation relaxation}: Removes the fixed computational bound constraint $ops(h_i) < max\_ops$.
\end{relaxation}

In both cases we assume that for practical applications, either a finite (and reasonable) number of tasks $N$ are encountered or that we are transigent with the memorization of past tasks, hence, for many settings with a generous memory and computational bounds, many continual learning strategies that, in terms of complexity and memory usage, grow somehow proportional to the number of training sets $Tr_i$ may still be a viable option, especially if they can guarantee better performance. 

Having defined a formal framework for CL, we can therefore highlight a number of desiderata:

\begin{desideratum}
\textsf{Online Continual Learning}:
Limit the size of each training set, moving towards online learning so that $|Tr_i| = 1$.
\end{desideratum}

Being able to learn without storing any raw data would mean a large step towards continual learning. In fact, getting rid of storing raw data means that the learning algorithm is able to extract information from the current task that may be not only useful and accurate for the current task, but also transferable for the future.
 
In our biological counterparts, namely the brain, a system-level consolidation process is often thought to take place, where memories are encoded, stored and than retrieved for rehearsal purposes \cite{delvenne09}. However, the idea of storing high-dimensional perceptual data appears impractical given the incredible amount of information flowing into our brain every day from our multi-modal senses. Being able to process data online as well, is an important desideratum especially for reducing adaptation time and operational memory usage in an embedded or robotics setting.
\checked{At least, unprocessed data should be selected to keep only essential information.}

\begin{desideratum}
\textsf{Task indicator free Continual Learning}: %
Learning continually without help of an external signal $t$ indicating the current task, in particular at test time, is strongly desirable.
\end{desideratum}

\begin{desideratum}
\textsf{Be ready for the future}:
Prepare the model to be robust and provided with good representation for handling future learning experiences. Solving only the current task only is probably not sufficient.
\end{desideratum}

\section{State of the art}
\label{sec:2_CL:SOTA} 

\checked{Continual learning can be classified in four families of approaches. Each family has its own method to manage memories and learning without forgetting in non-iid settings. }

\subsection{Regularization}
\label{subsub:2_CL:regu}

Regularization is a process of introducing additional information in order to prevent overfitting \cite{bhlmann2011statistics}. In the context of Continual Learning, the model should not overfit a new problem because it would make it forget it's previous skills.
The regularization approaches in continual learning consist of modifying the update of weights when learning in order to keep memory of previous knowledge.

\subsubsection{Penalty Computing}
\label{subsub:2_CL:penalty}

Basic regularization techniques that could be used for CL are weight sparsification, dropout \cite{Goodfellow13}, and early stopping \cite{Maltoni18}. 
These simple regularization techniques reduce the chance of weights being modified, and thus decrease the probability of forgetting.
More complex methods consist of searching for important weights inside the models and protect them afterwards to prevent forgetting.
The Fisher matrix can be used to estimate the importance of weights and produce an adapted regularization as for Elastic Weight Consolidation (EWC) approach \cite{kirkpatrick2017overcoming}. For efficiency purpose, EWC only use the diagonal of the Fisher matrix to estimate importance.  \cite{Ritter18} proposes an alternative to get a better estimation of the Fisher matrix using the Kronecker factorization.
EWC approach needs to have clear task delimitation to compute Fisher matrix at the end of the task, but  Synaptic Intelligence (SI) \cite{Zenke17} extended the method in an online learning fashion to relax this constraint.
\cite{lee2017overcoming} propose to use a regularization method called \textit{incremental moment matching} to overcome catastrophic forgetting. This method saves the moment posterior distribution of neural networks weights from past tasks and uses it to regularize learning of a new task. Two different declinations of this method are proposed: one with the use of first order moment \textit{IMM-mean} and one with second order moment \textit{IMM-mode}.

Another method to apply regularization for continual learning is the use of \textit{Conceptor} \cite{Jaeger14, He18}. Conceptor are memory mechanism that store learned patterns and representation. They are used to guide the gradient of the loss function to prevent forgetting and then favor modification for some weights and penalize others. 

The regularization methods have been shown to be efficient in reinforcement learning \cite{kirkpatrick2017overcoming}, classification \cite{kirkpatrick2017overcoming,Ritter18, Zenke17,He18} and also generative models \cite{nguyen2017variational, seff2017continual}. A limitation is that after several tasks the model may saturate because of a too high regularization, and finding a good trade-off between regularization that allows learning without forgetting may be hard.
Some regularization methods are described more in depth in Chapter \ref{chap:2b_Replay}.

\subsubsection{Knowledge Distillation}
\label{subsub:2_CL:distillation}

Distillation techniques were introduced by \cite{hinton2015distilling} in order to transfer knowledge from neural network A to neural network B. The idea is that after A has learned to solve a task, we want B to share this skill with A. We then forward the same input to both A and B and impose B to have the same output as A. Distillation should be more efficient than retraining B because A produces a soft-target that helps B to learn faster. In order to apply this method for continual learning, after network A learned to solve the first task, and while B is learning the second one, we distill knowledge from A to B. In the end, B should be able to solve both tasks. This and related methods have been used in various approaches \cite{wu2018memory,schwarz2018progress, Furlanello16, Rusu16distillation,Kalifou19,Traore19DisCoRL,Dhar19,michieli2019knowledge}. A drawback of distillation is that it generally needs to preserve a reservoir of persistent data learned for each task in order to apply distillation from a teacher model to a student model. %
Distillation can also be used to transfer policy learning from one model to another \cite{Rusu16distillation}.

\subsection{Dynamic Architecture}
\label{sub:2_CL:archi}

The architecture of learning models has a strong influence on how they learn.
One approach to CL is to modify dynamically the architecture of a model to make it learn new concepts or skills without interfering with old ones. We present two types of dynamic architectures. First, when the changes in the architecture are explicit; and second, when changes are implicit architectural changes by freezing weights. We also present an important architectural approach to CL: dual memory models.

\subsubsection{Explicit Architecture Modification}
\label{subsub:2_CL:expl}

Explicit dynamics architecture gather all methods that add, clone or save parts of parameters of the models to avoid catastrophic forgetting.

Progressive neural networks \cite{Rusu16progressive} is one of the first approaches within this paradigm for deep neural networks. 
For each new task to be learned, a new model is created connected to all past ones. The goal of this new model is to learn the new task by using what was already learned by previous models, and so develop the new skills needed.
At test time, the method needs to inject data into each neural network previously created, and needs to know the task index to pick the right output.
Because the weights are used to connect neural networks together, the growth of parameters is quadratic w.r.t. the number of tasks. This growth is generally to be prevented. %
Instead, layers may be dynamically expanded in a single network without the need of re-training or freezing previously learned parameters, improving model capacity over time \cite{wang17}.

Another type of dynamic architecture strategy consists of  dynamically adding neurons for new tasks. As an example, output layers can be added in order not to change output parameters from previous tasks as in LWF approach \cite{Li17learning}. 
This method ensures that the output layer will not be modified; however, as the feature extraction layers are shared between tasks, some parameters risk to be modified and forgotten. In addition, at test time, the method needs the task label.

It is worth mentioning that we consider as \textit{dynamic architecture}, those approaches that adapt their architecture specifically with the aim of not forgetting, while similar mechanisms can be used for other purposes\footnote{If the architecture is changed without this objective, it is not considered as part of the CL approach. As an example, when new classes are available, we might choose to make the output size grow to handle these, without making it as a way to not forget.}.

\subsubsection{Implicit Architecture Modification}
\label{subsub:2_CL:freeze}

Implicit architecture modification is the use of model adaptation for continual learning without modifying its architecture. This adaptation is typically achieved by inactivating some learning units or by changing the forward pass path. 

We categorize the fact of dynamically freezing weights as an implicit dynamic architecture approach. It is implicit because the architecture of the model does not change; however, the model's capacity is necessarily affected.

Freezing weights consist of choosing some weights at the end of a task that will no more change in the future. The backward pass will not be able to tune them anymore; however, they can still be used in the forward pass. This method ensures that these weights will not be modified, and tries to keep enough free parameters to learn in the future \cite{mallya2018packnet, Mallya18Piggyback, serra2018overcoming}.  The difficulty lies in freezing enough weights to remember, but not too much to still be able to learn new skills.
The way weight freezing is implemented in PackNet \cite{mallya2018packnet}, Piggypack \cite{Mallya18Piggyback}  or HAT \cite{serra2018overcoming} is by defining a special mask for each task that is used to both protect weights when new tasks are learned, and to define which weights to use at inference time for a given task. 
The use of masks to freeze important weights can be referred to as hard attention process \cite{serra2018overcoming}. 
Weight freezing can also be used to keep the decision boundary of the output unchanged \cite{Jung16}. 

An alternative to a weight freezing when tasks change is to define a dynamics path inside the model in order to use a specific path for a specific task and not modify already learned weights. This is the idea exploited in \textit{PathNet} \cite{fernando2017pathnet}.

The use of implicit architecture modifications is not incompatible with explicit architecture modification as it is shown in \cite{Mallya18Piggyback, serra2018overcoming}.

\subsubsection{Dual Architectures}
\label{subsub:2_CL:dual}

Dual approaches characterize architectures that are split into two models. One model is used in order to learn the current task and should be easily adaptable, while the second model is used as a memory of past experiences. 
This approach can be linked to interactions between the hippocampus and neocortex to avoid catastrophic interference in mammals \cite{Mcclelland95}. The stable network plays the role of the neocortex, and the flexible one plays the role of hippocampus \cite{Furlanello16,gepperth2016incremental, gepperth16bio, Maltoni18}. 

The use of dual architecture is explicit in many bio-inspired approaches such as \cite{Furlanello16, gepperth16bio, Parisi18, Sprechmann18,kemker18fearnet}. Dual architectures are extended in \cite{Sprechmann18} with the addition of an embedding model, and then, continual learning happens in the embedding space. The dual architecture can also be extended to more than two components, as in FearNet \cite{kemker18fearnet}, which takes inspiration from the basolateral amygdala from the brain to add a third component that is able to choose between the flexible and the stable memory for recall.

\subsection{Rehearsal}
\label{subsub:2_CL:rehearsal}

Rehearsal approaches gather all methods that save samples as memory of past tasks either raw or processed.

These samples are used to maintain knowledge about the past in the model. Ideally, those samples are carefully chosen in order to be representative of past tasks; by default, they can be randomly chosen.

The initial strategy is to save the representative samples and incorporate them in the new training set \cite{rebuffi2017icarl,Hayes18MemoryEfficient, lesort2018generative}. In the third article samples are chosen randomly for continual learning of generative models but in \cite{rebuffi2017icarl,Hayes18MemoryEfficient} the set is carefully sorted in order to keep the most representative samples into a coreset. This process allows to dynamically adapt the weights of the feature extractor and strengthen the network connections for memories already learned without forcing to keep previous weights. 

However, the coreset can also be used for regularization purpose and not just to be replayed from time to time along with new data in the learning process. 

For example, the coreset can be used for distillation  in \cite{Robins95}, in A-LTM (Active Long Term Memory Networks) \cite{Furlanello16} 
\checked{and in DisCorl for continual reinforcement learning \cite{Traore19DisCoRL}}.
 They can also be useful to regularize the gradient when learning new tasks as in GEM (Gradient Episodic Memory)   \cite{Lopez-Paz17}, A-GEM (Averaged Gradient Episodic Memory) \cite{Chaudhry19} and \cite{Aljundi2019Gradient}. Coresets have also been used to regularize the continual learning of a generative model in the CloGAN approach \cite{Rios19}.
In a bayesian learning setting the coreset can be incorporated into the prior to regularize learning update as in \cite{nguyen2017variational}. The authors experimented the use of a coreset to create a variational continual learning model (VCL). 

\checked{The coresest should be built in order to maintain knowledge and to be able to learn to distinguish old data from new ones. Moreover, the saved samples should not be used extensively in order to avoid overfitting them. In \cite{Aljundi2019Online}, the authors save samples but only use samples creating interference with current tasks. This makes it possible to use only appropriate memories.}

The disadvantage of rehearsal approaches is the utilization of a separate memory for raw and unprocessed data which is a vanilla way of saving knowledge that does not respect data privacy. Nevertheless it ensures that the memories are not degraded through time.
\checked{In order to mitigate this problem, some methods aimed at saving data point representation for rehearsal purpose using deep neural networks latent representation \cite{Belouadah2018DeeSIL, caccia2019online, pellegrini2019latent}.}

 \checked{Another problem could arise in rehearsal methods; since memory is limited, there could be a large unbalance between data from previous tasks and new data. However, this problem can be partially addressed by rescaling weights of learning criterion \cite{belouadah2020scail, Hou_2019_CVPR, Wu19Large}.}

\checked{Pseudo-rehearsal \cite{Robins95,Wang18}, a similar method to rehearsal, is is an approach to continual learning that generates pseudo-input as a memory of past tasks. The pseudo-inputs are data points generated by gradient descent on the classification models to match a randomly sampled output. 
Pseudo-rehearsal has been abusively used as a synonym to generative replay described in the next section in the continual learning literature. However, the fundamental difference is that the pseudo-inputs depend only on the classifier knowledge and are not produced by a generative model.}

\subsection{Generative Replay}
\label{subsub:2_CL:GR}

Instead of modeling the past from few samples as it is done in \textit{Rehearsal} approaches, \textit{Generative Replay} approaches train generative models on the data distribution.
Therefore, they are able to afterwards sample data from past experience when learning on new data. By learning on current data and artificially generated past data, they ensure that the knowledge and skills from the past is not forgotten.
These methods have also been associated with the term  \textit{Intrinsic Replay} \cite{Draelos17NeurogenesisDL}.
They could be understood as methods that perform \textit{regeneration} of samples or internal states, and thus, they can be associated with model-based learning, where the model learns the data distribution of past experiences.
The generative models is generally a GAN \cite{goodfellow2014generative} as in \cite{wu2018incremental, lesort2018generative, shin2017continual} or an auto-encoder as in \cite{Draelos17NeurogenesisDL, kemker18fearnet, caselles2018continual, kamra2017deep}.
 
A classical method %
implementing a generative replay normally makes use of dual models \cite{kamra2017deep, shin2017continual, wu2018incremental, Farquhar18, kemker18fearnet}.
One frozen model generates samples from past experiences and another learns to generate and classify current samples in addition to the regenerated ones. When a task is over, we replace the frozen model by the current one, %
freeze it, and initialize a new model to learn next task.

Generative Replay models can be categorized into two different approaches: "Marginal Replay" and "Conditional Replay" \cite{lesort2018marginal}. Techniques using \textit{Marginal Replay} make use of standard generative models, while \textit{Conditional Replay} are a particular case of the former where the generative model is conditional. Conditional models can generate data from a specific condition, e.g. a class or a task. In continual learning, it allows then to choose from which past learning experience we want to generate. It is important for example to balance data in generated replay \cite{lesort2018marginal}.

While most of the Generative Replay based approaches are meant to solve classification tasks \cite{kemker18fearnet, kamra2017deep, shin2017continual, wu2018incremental, Rios19}, some models use it for unsupervised learning \cite{lesort2018generative, wu2018memory} or reinforcement learning \cite{caselles2018continual}.

\subsection{Hybrid Approaches}
\label{subsub:2_CL:hybrid}

Memorization processed described are not incompatible, therefore they can be combined.
Most CL approaches have an implicit dual architecture strategy, as they always have a slow learning and a fast learning mechanisms to learn continually.
For example, in rehearsal approaches the stable model role is played by a memory that stores samples, in generative replay approaches a generative model plays the role of stable model, in some regularization approaches the stable model is played by the Fisher matrix which saves important weights.

Moreover, most of continual learning approaches do not rely on a single strategy to tackle catastrophic forgetting.  As stated in previous sections, each approach offers advantages and disadvantages, but most of the times, combining strategies allows to find the best solutions.
We summarize in Table \ref{tab:2_CL:classification} and Figure \ref{fig:2_CL:2_Venn} the different approaches cited and the strategies they propose.

\newcolumntype{M}[1]{>{\centering\arraybackslash}m{#1}}

\begin{longtable}{|l |M{20mm} M{20mm} M{20mm} M{20mm}|}
\caption[Classification of continual learning  main strategies]{\label{tab:2_CL:classification}Classification of continual learning  main strategies}\\

\hline
\multirow{2}*{\textbf{\footnotesize{References}}} &
\textbf{\footnotesize{Regularization}}&
\textbf{\footnotesize{Rehearsal}}  &
\textbf{\footnotesize{Architectural}}  &
\textbf{\footnotesize{Generative-Replay}} \\\hline\hline

\endfirsthead
\multicolumn{4}{c}%
{\tablename\ \thetable\ -- \footnotesize{Continued from previous page}} \\
\hline
\multirow{2}*{\textbf{\footnotesize{References}}} &
\textbf{\footnotesize{Regularization}}&
\textbf{\footnotesize{Rehearsal}}  &
\textbf{\footnotesize{Architectural}}  &
\textbf{\footnotesize{Generative-Replay}} \\\hline\hline
\hline
\endhead
\hline \multicolumn{4}{r}{\footnotesize{Continued on next page}} \\
\endfoot
\endlastfoot

\rowcolor{gray!20}
\footnotesize{\cite{Zhou12}}& 
& %
& %
\checkmark & %
\\ %

\footnotesize{\cite{Goodfellow13}} & 
  \checkmark &  %
& %
& %
\\ %

\rowcolor{gray!20}
\footnotesize{\cite{lyubova15}}  & 
 & %
\checkmark & %
& %
 \\ %

\footnotesize{\cite{Rusu16distillation}}  & 
  \checkmark & %
& %
& %
 \\ %

\rowcolor{gray!20}
\footnotesize{\cite{Camoriano16}} & 
\checkmark &%
\checkmark & %
& %
\\ %

\footnotesize{\cite{Furlanello16}} & 
  \checkmark & %
& %
  \checkmark & %
\\ %

\rowcolor{gray!20}
\footnotesize{\cite{Li17learning} (LwF)}& 
  \checkmark & %
& %
  \checkmark & %
\\ %

\footnotesize{\cite{Rusu16progressive} (PNN)} & 
& %
& %
  \checkmark & %
\\ %

\rowcolor{gray!20}
\footnotesize{\cite{Jung16}} & 
  \checkmark & %
& %
  \checkmark & %
\\ %

\footnotesize{\cite{aljundi2017expertGate}}  & 
& %
& %
\checkmark & %
\\ %

\rowcolor{gray!20}
\footnotesize{\cite{rebuffi2017icarl} (Icarl)}  & 
  \checkmark & %
  \checkmark & %
& %
\\

\footnotesize{\cite{kirkpatrick2017overcoming} (EWC) } & 
  \checkmark  & %
& %
& %
\\ %

\rowcolor{gray!20}
\footnotesize{\cite{fernando2017pathnet}}  & 
& %
& %
\checkmark & %
\\

\footnotesize{\cite{lee2017overcoming}}  & 
  \checkmark & %
& %
& %
\\ %

\rowcolor{gray!20}
\footnotesize{\cite{Lee17Lifelong} } & 
  \checkmark & %
& %
& %
\\ %

\footnotesize{\cite{Triki17}}  & 
  \checkmark & %
& %
& %
\\ %

\rowcolor{gray!20}
\footnotesize{\cite{seff2017continual} } &  
  \checkmark & %
& %
& %
\\ %

\footnotesize{\cite{shin2017continual} (DGR)} & 
& %
& %
& %
  \checkmark \\ %

\rowcolor{gray!20}
\footnotesize{\cite{Velez17}}  & 
  \checkmark & %
& %
& %
\\ %

\footnotesize{\cite{Lopez-Paz17} (GEM)}  & 
  \checkmark & %
  \checkmark & %
& %
\\ %

\rowcolor{gray!20}
\footnotesize{\cite{Zenke17} (SI)} & 
  \checkmark & %
& %
& %
\\

\footnotesize{\cite{nguyen2017variational} (VCL)} & 
\checkmark & %
\checkmark & %
\checkmark  & %
\\ %

\rowcolor{gray!20}
\footnotesize{\cite{ramapuram2017lifelong}}  & 
  \checkmark & %
& %
& %
  \checkmark \\ %

\footnotesize{\cite{mallya2018packnet} } & 
& %
& %
  \checkmark & %
\\ %

\rowcolor{gray!20}
\footnotesize{\cite{kamra2017deep}}  & 
& %
& %
& %
  \checkmark \\ %

\footnotesize{\cite{Draelos17NeurogenesisDL}}  & 
& %
& %
& %
  \checkmark \\ %

\rowcolor{gray!20}
\footnotesize{\cite{serra2018overcoming} }& 
  \checkmark & %
& %
& %
\\ %

\footnotesize{\cite{Mallya18Piggyback}}  & 
& %
& %
  \checkmark & %
\\ %

\rowcolor{gray!20}
\footnotesize{\cite{Parisi18} } (GDM)& 
  \checkmark  & %
& %
  \checkmark & %
  \checkmark \\ %
 
\footnotesize{\cite{He18}}  & 
  \checkmark & %
& %
  \checkmark & %
 \\ %
 
\rowcolor{gray!20}
\footnotesize{\cite{Hayes18MemoryEfficient}}  & 
 & %
\checkmark  & %
  & %
 \\ %
 
\footnotesize{\cite{wu2018incremental}}  & 
& %
  \checkmark & %
& %
  \checkmark  \\

\rowcolor{gray!20}
\footnotesize{\cite{Ritter18}}  & 
  \checkmark & %
& %
& %
\\

\footnotesize{\cite{schwarz2018progress}} &
& %
  \checkmark & %
& %
 \\ %

\rowcolor{gray!20}
\footnotesize{\cite{Maltoni18}} &
  \checkmark& %
 & %
  \checkmark& %
\\ %

\footnotesize{\cite{achille2018life}}  & 
& %
& %
  \checkmark & %
  \checkmark \\ %

\rowcolor{gray!20}
\footnotesize{\cite{wu2018memory} (MeRGAN)}  &  %
  \checkmark & %
& %
& %
  \checkmark \\ %
  
\footnotesize{\cite{Dhar19}} & 
\checkmark & %
& %
& %
\\ %

\rowcolor{gray!20}
\footnotesize{\cite{lesort2018generative}}  & 
& %
& %
& %
  \checkmark \\
  
\footnotesize{\cite{castro2018end}}  & 
& %
 \checkmark & %
& %
 \\

  \rowcolor{gray!20}
\footnotesize{\cite{caselles2018continual}} & 
& %
& %
& %
  \checkmark \\ %

\footnotesize{\cite{riemer2018learning} (MER)} & 
 \checkmark & %
  \checkmark   & %
 & %
  \\ %

\rowcolor{gray!20}
\footnotesize{\cite{Rios19} (CloGAN)} & 
\checkmark  & %
  \checkmark   & %
 & %
 \checkmark  \\ %

\footnotesize{\cite{lesort2018marginal} } & 
& %
& %
& %
  \checkmark \\ %

\rowcolor{gray!20}
\footnotesize{\cite{Sprechmann18}}  & 
& %
  \checkmark  & %
  \checkmark  & %
 \\ %

\footnotesize{\cite{Hayes18MemoryEfficient} (ExStream)} & 
 & %
  \checkmark  & %
    & %
 \\ %

\rowcolor{gray!20}
\footnotesize{\cite{Belouadah2018DeeSIL} (DeeSIL)} & 
 & %
  \checkmark  & %
    & %
 \\ %

\footnotesize{\cite{kemker18fearnet} (FearNet)} & 
& %
  & %
  \checkmark  & %
 \checkmark  \\ %
 
\rowcolor{gray!20}
\footnotesize{\cite{Chaudhry19} } & 
  \checkmark & %
  \checkmark & %
& %
\\ %

\footnotesize{\cite{Wu19Large}}  & 
 & %
  \checkmark & %
& %
\\ %
 
\rowcolor{gray!20}
\footnotesize{\cite{Kalifou19}}  & 
  \checkmark & %
\checkmark  & %
& %
    \\%

\footnotesize{\cite{Aljundi2019Online}}  & 
  \checkmark & %
\checkmark  & %
& %
\checkmark     \\ %

\rowcolor{gray!20}
\footnotesize{\cite{michieli2019knowledge}}  & 
  \checkmark & %
\checkmark  & %
& %
   \\ %

\footnotesize{\cite{caccia2019online} }  & 
  & %
\checkmark  & %
& %
     \\ %
 
\rowcolor{gray!20}
\footnotesize{\cite{Hou_2019_CVPR}}   & 
& %
\checkmark  & %
& %
     \\ %

\footnotesize{\cite{Traore19DisCoRL}}  & 
  \checkmark & %
\checkmark  & %
& %
    \\ %
    
\rowcolor{gray!20}
\footnotesize{\cite{Aljundi2019Gradient}}  & 
  \checkmark & %
\checkmark  & %
& %
    \\ %
   
\rowcolor{gray!20}
\footnotesize{\cite{belouadah2020scail} } & 
\checkmark  & %
& %
& 
\\\hline %
\end{longtable}

\begin{figure}[ht]
  \centering
  \includegraphics[width=0.6\textwidth]{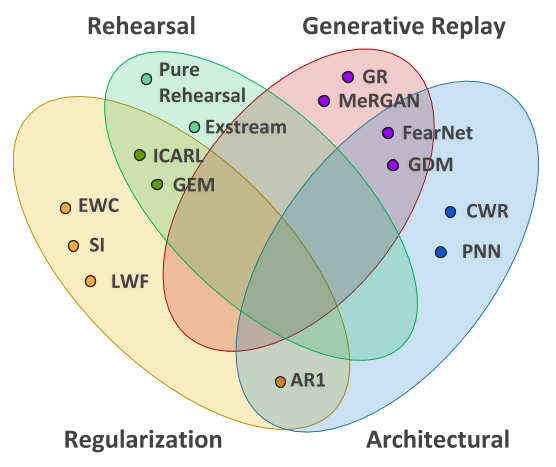}
  \caption[Venn diagram of some of the most popular CL strategies.]{Venn diagram of some of the most popular CL strategies w.r.t the four approaches illustrated in Section \ref{sec:2_CL:SOTA}: CWR \cite{Lomonaco17}, PNN \cite{Rusu16progressive}, EWC \cite{kirkpatrick2017overcoming}, SI \cite{Zenke17}, LWF \cite{Li17learning}, ICARL \cite{rebuffi2017icarl}, GEM \cite{Lopez-Paz17}, FearNet \cite{kemker18fearnet}, GDM \cite{Parisi18}, ExStream \cite{Hayes18MemoryEfficient}, Pure Rehearsal, GR \cite{shin2017continual}, MeRGAN \cite{wu2018memory} and AR1 \cite{Maltoni18}. Rehearsal and Generative Replay upper categories can be seen as a subset of replay strategies. Better viewed in color.}
  \label{fig:2_CL:2_Venn}
\end{figure}

\section{Evaluation}
\label{sec:2_CL:eval}

Before applying CL solutions to autonomous agents, they should be experimented and evaluated in simulation or toy examples. It is crucial to have a set of good evaluation metrics and benchmarks to assess if the approaches are scalable to real problems or may not solve harder ones.
\checked{It is important to distinguish the evaluation of the performances and the evaluation of algorithms.}
\checked{In research, the evaluation should be able to make predictions of success in applications, while in applications, the evaluation only assesses if the results are satisfying. Therefore, in research, we assess the method while in an application we should assess the result.}
 In this section we summarize existing evaluation methods and benchmarks and highlight some of them we believe worth using when targeting the deployment of practical CL applications.

\subsection{Benchmarks}
\label{sub:2_CL:benchmarks}

\begin{table}[!t]

\caption[Continual learning's benchmarks and environments]{Benchmarks and environments for continual learning. For each resource, paper use cases in the NI, NC and NIC scenarios are reported.} 
\label{tab:2_CL:benchmarks} 
\begin{tabular}{|p{50mm}|P{6mm}P{6mm}P{9mm}|P{60mm}|}
\hline
\textbf{Benchmark} & \textbf{NI}  & \textbf{NC} & \textbf{NIC} & \textbf{Use Cases} \\ %
\hline
\hline 

\rowcolor{gray!20}
Split MNIST/Fashion MNIST & 
 &  %
 \checkmark & 
 & %
\cite{lesort2018marginal, lesort2018generative, He18, Rios19} \\

Rotation MNIST & 
 \checkmark &  %
 & 
 & %
 \cite{Lopez-Paz17,lesort2018marginal, riemer2018learning}\\ %

\rowcolor{gray!20}
Permutation MNIST& 
 \checkmark &  %
 &
 & %
 \cite{Goodfellow13,kirkpatrick2017overcoming,fernando2017pathnet,shin2017continual,Zenke17,lesort2018marginal,He18, riemer2018learning}\\

iCIFAR10/100 & 
& %
\checkmark  & %
 & %
  \cite{rebuffi2017icarl, Maltoni18,castro2018end, kemker18fearnet, belouadah2020scail, Wu19Large, Hou_2019_CVPR} \\

\rowcolor{gray!20}
SVHN & 
  & %
 \checkmark  &
 & %
 \cite{Kemker2017Measuring, seff2017continual, Rios19}
 \\

CUB200 & 
\checkmark & %
 &
 & %
\cite{Lee17Lifelong, Hayes18MemoryEfficient} \\

\rowcolor{gray!20}
CORe50 & 
 \checkmark & %
 \checkmark  &
 \checkmark  & %
\cite{Lomonaco17, Parisi18,Maltoni18, Hayes18MemoryEfficient} \\

iCubWorld28 & \checkmark & & & \cite{Pasquale15,lomonaco16, Hayes18MemoryEfficient} \\

\rowcolor{gray!20}
iCubWorld-Transformation  & & \checkmark & & \cite{Pasquale16,camoriano17a} \\

LSUN & 
& %
 \checkmark &
 & %
\cite{wu2018memory} \\

\rowcolor{gray!20}
ImageNet & 
& %
 \checkmark  &
 & %
 \cite{rebuffi2017icarl, Mallya18Piggyback, castro2018end, Belouadah2018DeeSIL, Wu19Large, Hou_2019_CVPR, belouadah2020scail}\\
 
 Omniglot & & \checkmark & & \cite{lake2015human,schwarz2018progress}\\

\rowcolor{gray!20}
Pascal VOC & & \checkmark & & \cite{michieli2019knowledge,shmelkov2017incremental,michieli2019knowledge}\\

 Atari & \checkmark
& &
 & %
\cite{Rusu16distillation, kirkpatrick2017overcoming, schwarz2018progress}\\

\rowcolor{gray!20}
RNN CL benchmark & & & \checkmark  & \cite{Sodhani18}\\  %

CRLMaze (based on VizDoom) & \checkmark & & & \cite{lomonaco2019continual}\\

\rowcolor{gray!20}
 DeepMind Lab &  \checkmark & & & \cite{Mankowitz18}\\

MNIST Fellowship & \checkmark & \checkmark & & \cite{lesort2019regularization}\\

\hline

\end{tabular}
\end{table}

In continual learning, the difficulty of learning on a sequence of tasks is first of all dependant on the difficulty of each of the tasks separately. 
If a task is difficult to learn, a model will have to deeply modify its weights. If those weights contain knowledge from previous tasks, there is a high probability they will be degraded.
On the other hand, the risk of forgetting is also dependant on the likelihood of tasks occurring.
Indeed, after learning a task $T_t$, it is easier for a neural network to learn a radically different task $T_{t+1}$ without forgetting, than learning a task $T_{t+1}$ with similarities to $T_t$ \cite{Farquhar18}.

There are several kinds of similarities in a sequence of tasks:
\begin{itemize}
  \item Similarities in learning objectives: They occur when the objective is similar from task to task. For example, in a classification setting, when the same classes are used from one task to another (e.g. Permuted MNIST), or in RL, the same tasks need to be achieved in different environments.
  \item Similarities in features: the features from task to task are the same or very similar (e.g. Rotation MNIST).
\end{itemize}

Beyond the similarity among tasks and the learnability of each task, the availability of data is primordial to evaluate the difficulty of a benchmark. 
For convenience, most of the classical benchmarks assume that each task is available long enough to learn a satisfying solution.
Nevertheless, even when there is no constraint on the time to learn a task, data from the past cannot be available again in the future. %
In several approaches, past data is used for model selection, however using the performance obtained on task $T_t$ to fine-tune a model that will learn on $T_0$ violates temporal causality \cite{pfulb2019a}. Data might be saved for later use as in rehearsal approaches, but this must be done before moving on to the next task.

Most CL benchmarks are adapted from others fields, for instance:
\begin{itemize}
  \item \textbf{Classification}: MNIST \cite{LeCun10}, Fashion-MNIST \cite{Xiao2017}, KMNIST \cite{clanuwat2018deep}, CIFAR10/100 \cite{Krizhevsky09}, Street View House Numbers (SVHN) \cite{Netzer11}, CUB200 \cite{Welinder10}, LSUN  \cite{Yu15}, MNIST-Fellowship \cite{Lesort19Data}, ImageNet \cite{krizhevsky12}, Omniglot \cite{lake2015human} or Pascal VOC \cite{Everingham15} (object detection and segmentation).
  \item \textbf{Reinforcement Learning}: Arcade Learning Environment (ALE) \cite{Bellemare13} for Atari games, SURREAL \cite{Fan18} for robot manipulation and RoboTurk for robotic skill learning through imitation \cite{Mandlekar18}, \textit{CRLMaze} extension of VizDoom \cite{lomonaco2019continual} and DeepMind Lab \cite{Mankowitz18}. 
  \item \textbf{Generative models}: Datasets that prevail in this domain are the same as those used in classification tasks.

\end{itemize}

These datasets are then split, artificially modified (e.g., with image rotations or permutation of pixels) or concatenated together to create sequences of tasks and build a continual learning setting. As an example, permuted MNIST \cite{kirkpatrick2017overcoming} and rotated MNIST \cite{Lopez-Paz17} are continual learning datasets artificially created from MNIST.

Another possible continual learning scenario is the use of
naturally non i.i.d. datasets (e.g. NICO \cite{He19}) or learning sequentially different datasets either on the same input space \cite{lee2017overcoming, serra2018overcoming} or in a multi-modal fashion \cite{Kemker2017Measuring}.
However, only few datasets, such as CORe50 \cite{Lomonaco17}, OpenLORIS-Object \cite{she2019openlorisobject}, OpenLORIS-Scene \cite{shi2019ready} or MNIST-Fellowship \cite{Lesort19Data}, are specifically built with continual learning in mind.

In robotics, numerous datasets are often recorded in a online fashion through video. Therefore, they are suitable to evaluate continual learning algorithms. As an example, those proposed by \cite{Pasquale15,Pasquale16, azagra17} are composed of sequences of images captured during robotics object manipulation; they are used for classification and detection algorithms. 
A summary of the main datasets and examples of their applications can be found in Table  \ref{tab:2_CL:benchmarks}.

\checked{For the remainder of the thesis, we experiment with quite easy datasets, such as MNIST, Fashion-MNIST or KMNSIT (presented in Chapter \ref{chap:1b_ML}), to focus more on continual learning problems rather than on learning problems. Solving hard problems is important to know if approaches are scalable but for prototyping purposes, it is not always necessary. We discuss this point furthermore in Chapter \ref{chap:6_disc}.}

\subsection{Metrics}
\label{sub:2_CL:metrics}

\checked{Following the algorithm evaluation on a benchmark, we should make sure that the evaluation criteria are rigorous and cover the representative aspects of algorithm capacities.} %
\checked{A thorough research evaluation should report more than just the final accuracy.}
 We should also evaluate how fast it learns and forgets, if the algorithm is able to transfer knowledge from one task to another, and if the algorithm is stable and efficient while learning. In this section we gather a set of metrics to evaluate a CL approach.

For evaluation purpose, we assume access to series of test sets $Te_i$. The aim is to assess and disentangle the performance of our hypothesis $h_i$ as well as to evaluate if it is representative of the knowledge that should be learned by the corresponding training batch $Tr_i$.

For instance, one example of such evaluation is one of the first metrics proposed for CL \cite{Hayes18NewMetrics}; it consists of an overall performance $\mathcal{M}$ in a supervised classification setting. It is based on the relative performance of an incrementally trained algorithm with respect to an offline trained algorithm (which has access to all the data at once). In our notation, $\mathcal{M}$ is: %

\begin{equation}
\mathcal{M} = \frac{1}{N}\sum_{i=1}^{N}\frac{R_{i,i}}{R_{i,i}^{C}}.
\end{equation}  %

Where %
$N$ is the number of tasks encountered, $R^{C}_{i,j}$ is the potentially best accuracy we can have on $Te^C_i$ if the model was trained with all data at once, i.e. on $Tr^C_{i}$ (the accumulation of training sets $Tr^{C}_{t}$ from t=0 to t=i). $Te^C_{i}$ is the accumulation of all test sets $Te^{C}_{t}$ from $t=0$ to $t=i$.  %
 $\mathcal{M}$ = 1 indicates identical performance to an off-line cumulative setting; an $\mathcal{M}$ larger than one is possible when the offline model is worse than trained in a CL paradigm.

In \cite{serra2018overcoming}, instead, the authors try to directly model forgetting with the proposed \emph{forgetting ratio} metric $\rho$ after learning $i$ tasks, defined as:

\begin{equation}
\rho^{j\leq i} = \frac{1}{N} \sum_{i}^{N} \sum_{j}^{N} \left(\frac{R_{ij}-R_{j}^R}{R_{ij}^C -R_{j}^R}-1\right)
\end{equation}
Where, $R_{j}^R$ is the accuracy of a random stratified classifier using the class information of task $j$.

Always in the same sequential setting, in \cite{Lopez-Paz17} other three metrics are proposed: \emph{Average Accuracy} (ACC), \emph{Backward Transfer} (BWT), and \emph{Forward Transfer} (FWT). In this case, after the model finishes learning about the training batch $Tr_i$, its performance is evaluated on all (even future) test batches $Te_j$.

The larger these metrics, the better the model. 
The metrics are extended for more fine grained, generic evaluation \cite{Diaz18} so that the original accuracy \cite{Lopez-Paz17} (as well as BWT and FWT) can account for performance at \emph{every timestep in time}. Average Accuracy is defined as:

\begin{equation}
ACC = \frac{\sum_{i=1}^{N}\sum_{j=1}^{i} R_{i,j}}{\frac{N(N+1)}{2}}
\end{equation}
where $R \in \mathds{R}^{N \times N}$ is the training-test accuracy matrix that contains in each entry $R_{i,j}$ the test classification accuracy of the model on task $t_j$ after observing the last sample from task $t_i$%
, \checked{Average Accuracy (ACC) considers all accuracies of training set $Tr_i$ and test set $Te_j$ by considering the diagonal elements of  $R$, as well as all elements below it} %
(i.e., averages $R_{i,j}$ where $i>=j$ see Table \ref{tab:2_CL:acc-matrix-r}).

\checked{It should not be confused with the classical accuracy $A$.}%
\begin{equation}
A = \frac{\sum_{i=1}^{N} R_{i,i}}{N}
\label{eq:2_CL:Average_Accuracy}
\end{equation}
\checked{$A$ computes the average performance at time $t$ and does not take into account the previous performance evolution. We can compare $A$ and $ACC$ to get some more insight of the algorithm behaviour.
If $A$ > $ACC$ then the mean performance on past tasks improved through the learning of new tasks, if $A$ < $ACC$ then the mean performance decreased and the algorithm forgot and, finally, if $A$ = $ACC$, then either the algorithm stays very stable or the progress and forgetting compensate each others.}

\begin{table}[ht]
\centering
\caption[Illustration of accuracy matrix $R$]{Illustration of accuracy matrix $R$: elements accounted to compute ACC (white \& cyan), BWT (cyan), and FWT (gray). $R^{*} = R_{ii}$, \textbf{$Tr_i$} = training, \textbf{$Te_i$}= test tasks.}
\label{tab:2_CL:acc-matrix-r}
\begin{tabular}{l|ccc} 
\textit{R}  & \textbf{$Te_1$} & \textbf{$Te_2$} & \textbf{$Te_3$} \\
\specialrule{.1em}{.05em}{.05em} 
\textbf{$Tr_1$} & $R_{1,1}$ & \cellcolor{Gray}$R_{1,2}$ & \cellcolor{Gray}$R_{1,3}$  \\ 
\textbf{$Tr_2$} & \cellcolor{LightCyan}$R_{2,1}$ &$R_{2,2}$ & \cellcolor{Gray}$R_{2,3}$  \\
\textbf{$Tr_3$} & \cellcolor{LightCyan}$R_{3,1}$ & \cellcolor{LightCyan}$R_{3,2}$ & $R_{3,3}$ \\
\specialrule{.1em}{.05em}{.05em} 
\end{tabular}
\end{table}

Backward Transfer (BWT) %
measures the influence that learning a task has on the performance on previous tasks. %
It is defined as the accuracy computed on $Te_i$ right after learning $Tr_i$ as well as at the end of the last task on the same test set (see %
 Table \ref{tab:2_CL:acc-matrix-r} in light cyan).
\begin{equation}
BWT = \frac{\sum_{i = 2}^{N}\sum_{j = 1}^{i-1}(R_{i,j} - R_{j,j})}{\frac{N(N-1)}{2}}
\end{equation}
The original BWT \cite{chaudhry2018riemannian,Lopez-Paz17}
is extended %
into two terms to distinguish among two semantically different concepts (so that, as the rest of metrics, is to be maximized and in [0,1]). 
\begin{equation}
REM = 1- |min (BWT, 0)|
\end{equation}
i.e., \textit{Remembering}, and (the originally positive) BWT, %
i.e., improvement over time, \textit{Positive Backward Transfer}: %
\begin{equation}
BWT^{+} =  max (BWT, 0)
\end{equation}

Likewise, the FWT redefined to account for the dynamics of CL at each timestep is 
\begin{equation}
FWT = \frac{\sum_{i=1}^{j-1}\sum_{j=1}^{N} R_{i,j}}{\frac{N(N-1)}{2}}  %
\end{equation}

FWT accounts for the train-test accuracy entries $R_{i,j}$ above the principal diagonal of $R$, excluding it (see elements accounted in Table \ref{tab:2_CL:acc-matrix-r} in light gray). Forward transfer can occur when the model is able to perform \textit{zero-shot} learning.

\checked{FWT and BWT can be interpreted as trans-task generalization capacity of an algorithm. It is important to note that FWT and BWT give no insight into the algorithms assets if not compared to another algorithm. It is not possible to easily disentangle the generalization performance from the similarity of tasks.}

A Learning Curve Area (LCA) ($\in [0,1]$) metric to quantify the learning speed by a CL strategy is proposed in \cite{Chaudhry19}. It uses the $b$-shot performance (where $b$ is the mini-batch number) after being trained for all the $N$ tasks: 

\begin{equation}
Z_{b} =  \frac{1}{N} \sum_{i=1}^{N}a_{i,b,i}  %
\end{equation}
where %
$a_{i,k,j} \in [0,1]$
is the accuracy evaluated on the test set of task $j$ after the model has been trained with the $k$-th mini-batch of task $i$. This amount is equivalent to previous accuracy matrix entry $R_{ij}$ but at a lower granularity of a batch level. $a_{i,k,j}$ is used to define a forgetting measure $\in [-1, 1]$ that quantifies the drop in accuracy on previous tasks \cite{chaudhry2018riemannian}. $f^k_j$ is the forgetting on task $j$ after the model is trained with all mini-batches up to task $k$:

\begin{equation} 
f_j^{k} = \max_{l \in {1,..,k-1}} a_{l, B_l, j} - a_{k, B_k, j}
\end{equation} 
where $B_i$ is all mini-batches corresponding to training dataset of task $k$ ($\mathcal{D}_k$).

$LCA_{\beta}$ is the area of the convergence curve $Z_b$ during training as a function of $b \in [0, \beta]$:
\begin{equation}
LCA_{\beta} =  \frac{1}{\beta + 1} \int_{0}^{\beta}Z_{b}db = \frac{1}{\beta +1}\sum_{b=0}^{\beta}Z_{b}
\end{equation}
The interpretation of LCA is intuitive: an $LCA_{0}$ is the average 0-shot performance (FWT), and $LCA_{\beta}$ is the area under the $Z_{b}$ curve, which is high if the 0-shot performance is good and if the learner learns quickly. LCA aims at disambiguating the performance of models that may have the same $Z_{b}$ or $A_T$, but very different $LCA_{\beta}$ because despite both eventually obtaining the same final accuracy, one may learn much faster than the other.

While forgetting and knowledge transfer could be quantified and evaluated in various ways, as argued in \cite{Farquhar18,Hayes18NewMetrics,Kemker2017Measuring}, these may not suffice for a robust evaluation of CL strategies. For example, in order to better understand the different properties of each strategy in different conditions, especially for embedded systems and robotics, it would be interesting to keep track and unambiguously determine the amount of computation and memory resources exploited. In this context, the metrics proposed in \cite{Lopez-Paz17} are extended in \cite{Diaz18} to unify in a common evaluation framework different infrastructural and operational metrics. 
Other practical metrics included are Continual Memory Size (CMS) and Computational Efficiency (CE). We briefly describe them next. 

The memory size of model $h_i$ is quantified in terms of parameters $\theta$ at each task $i$, $Mem(\theta_i)$ and the eventual external memory to save data  $M_{Data_i}$; with the idea that it should not grow too rapidly with respect to the size of the model that learned the first task, $Mem(\theta_1)$:
\begin{equation}
MS = min(1, \frac{\sum_{i = 1}^{N}\frac{Mem(\theta_1)}{Mem(\theta_i)+M_{Data_i}}}{N})
\end{equation}

\checked{To compute the $MS$ metric, it is important to note that $M_{Data_i}$ may contain data to remember and evaluation data needed for hyper-parameters selection.}

A metric that bounds the Computational efficiency (CE) by the number of %
operations for training set $Tr_i$ is defined as: %
\begin{equation} %
CE = min(1, \frac{\sum_{i = 1}^{N} \frac{\mathit{Ops}\uparrow\downarrow(Tr_i) \cdot \varepsilon}{1+\mathit{Ops}(Tr_i)}}{N})
\end{equation}
where $Ops(Tr_i)$ is the number of (mul-adds) operations needed to learn $Tr_i$, $Ops\uparrow\downarrow$($Tr_i$) are the operations required to do one forward and one backward (backprop) pass on $Tr_i$, and $\varepsilon$ is a scaling factor (associated to the nr of epochs needed to learn $Tr_i$). %
Overall \textbf{$CL_{score}$} and \textbf{$CL_{stability}$} metrics are also finally proposed \cite{Diaz18} in order to aggregate different criteria to be maximized that allow to rank CL strategies.

In future evaluation scenarios, particularly in robotics, stability is another important property that should be evaluated since in many robotic tasks and safety-critical conditions, potential abrupt performance drifts would be a major concern when learning continuously.
The metrics presented here can also be combined to assess higher-level capabilities. As an example, if we are to assess the \textit{scalability} of a CL algorithm, one could use a weighted average of \textit{MS}, and \textit{CE}.

The metrics presented in a supervised classification context \cite{Diaz18} can also be generalized with different performance measure $P$, instead of \checked{task accuracy $R_{i,j}$, and used in the same way in the metrics proposed previously. For example, it could be used for reinforcement and unsupervised learning.} For instance, they can be extended to RL; the underlying performance metric is, instead of accuracy, the accumulated reward on test episodes. In general in RL, cumulative reward plots over time are common norm to evaluate policy learning algorithms. Extra performance metrics in RL tasks will very much depend on the task being assessed, the reward function, and other evaluation metrics that act as evaluation \textit{proxies}, as it is common in semi/unsupervised learning settings.

The evaluation of generative models in any setting is challenging.
Fr\'{e}chet Inception Score (FID) \cite{heusel2017gans} is a common metric that compares features from generated data and true data.
Inception Score (IS) \cite{Salimans16} has also been widely used as a proxy to evaluate the quality of generative models. It measures if the class of generated samples are varied by making use of a model trained on ImageNet. One shortcoming of these scores is that they may be maximized by over-fitting generative models.
Another evaluation method is using generated data to train a classifier and evaluate its accuracy on a test set of true data \cite{lesort2018training}. The test accuracy, called Fitting Capacity (FC) gives a proxy on the quality of the generated data. 
Fitting Capacity and Fr\'{e}chet Inception Score were used in a CL setting in \cite{lesort2018marginal, lesort2018generative}.\\
More methods for evaluating generative models are described and assessed more in depth in \cite{borji2018pros, Jiwoong18}; however, they have never been used in a CL setting.
In any case, the need for real data is mandatory in most evaluation schemes. In a CL setting, evaluating the generation of data from past tasks may need to violate the data availability assumption. The different metrics for generative models may then be useful tools for example for evaluating generative replay methods; however, they have to be manipulated carefully to be incorporated into the continual learning spirit.

\bigskip

\checked{For the remainder of this thesis, we are particularly focus on the final accuracy performance (eq. \ref{eq:2_CL:Average_Accuracy}). %
Even if the other metrics are interesting to evaluate an algorithm, the most valuable results are the final performance.  We discuss more this point in Chapter \ref{chap:6_disc}. }
\checked{Moreover, we mostly study disjoint settings and in such setting BWT learning and FWT learning cannot happen, so the corresponding metrics are not relevant. Finally, reporting the experiments computational costs would be clearly interesting, but would require a higher number of experiments than what has been possible to achieve within this work.}

\section{Applications : Continual Learning for Robotics}
\label{sec:2_CL:robotics}

In the previous section we listed and described the different existing types of strategies to tackle continual learning. In this section, we will present real use cases of CL with an emphasis on robotics applications. First, we present why continual learning is crucial for robotics, and then, the challenges that robotics face in CL tasks. Finally, we present concrete robotic applications with potential insights to draw from CL. %

\subsection{Opportunities for Continual Learning in Robotics }
\label{sub:2_CL:opportunities}

A robot is an agent that interacts with the real world. It means that it cannot go back in time to improve what it has learn in the past. %
These particularities of robotic platforms make them a natural playground for CL algorithms.
Furthermore, robots suffer from several constraints in terms of power or memory, which CL intends to optimize, in the way it addresses learning problems.
On the other hand, robots have rich information about their experiences. They are in control of their interaction with the environment, which may help them understanding the concept of causality, and extracting knowledge from different kinds of sensors (images, sound, depth...). This rich information helps machines to produce strong representations which are crucial for a well performing CL algorithm \cite{Lesort18}.

We could almost conclude that CL is born for robotics, and it may be true; however, today most of CL approaches are not robotics related and rather focus on experiments on image processing or simulated environments. The next section will present the challenges that make CL difficult to apply in robotic environments.

\subsection{Challenges of Continual Learning in Robotics} 
\label{sub:2_CL:challenges}

\subsubsection{Robotics Hardware}

The first challenge to deal with when doing any experiment with robots is the hardware.
Robots are known to be unstable and fragile. Robot failures are one of the main restrictions for researchers to propose new approaches on robotics tasks. They add unavoidable delay in any experiment and are expensive to fix.
Moreover, if the failure is not hardware but software, since it is not possible to reset the state of the robot automatically, manual help is often needed, e.g., to put back the robot in his starting position or recover it from an irrecoverable state.
Furthermore, most of the time building or buying a robot is itself quite costly.
Once the robot is correctly working, one new problem arises, which is its autonomy in terms of energy. This aspect is also a main difficulty to deal with when experiments need to be set. It is difficult to program long experiments without manually recharging the robot and making sure that it will not stop by a lack of power supply or failure.
Lastly, robots are embedded platforms and, consequently, have limited memory and computation resources, which should be carefully managed to avoid overflow.

The difficulties of using robots in experiments explain why there are so few approaches of continual learning with robots in the literature. In the next section, we will see how robotic environments challenge continual learning algorithms. 

\subsubsection{Data Sampling}
\label{sub:2_CL:sampling}

When a robot needs to learn a task in a known or unknown environment, it must collect its own training data in the real world \cite{Wong16}. Data serves as the basis for environment exploration and comprehension.
This problematic is exactly the same as the one met by RL algorithms \cite{Sutton98}. 
In infants, a crucial component of lifelong learning is the ability to autonomously generate goals and explore the environment driven by intrinsic motivation \cite{Oudeyer07, Cangelosi18}. 
Self-supervised approaches \cite{Pinto15, Levine16, Wong16, Shelhamer16} also help to automatically explore environments. 
Curiosity \cite{pathak18largescale} and self-supervision \cite{Doersch17} allow to search for new experiences (or data) and build a base of knowledge useful to achieve current or future tasks via transfer learning \cite{Parisi18}.
As an example, manipulation tasks \cite{Kim19} such as grasping \cite{Pinto15}, reaching \cite{raffin2019decoupling, Colas18Curious}, pushing buttons \cite{Lesort19}, throwing \cite{Stulp14,Kim19} or stacking \cite{Colas18Curious} objects (cubes, balls...) are common complex tasks
built on comprehensive sets of experiments. %

Data gathered in this way can then be used on the fly in an online learning process or stored for later processing.
However, in order to improve learning algorithms the need for annotations or external help is crucial. In the next subsection we will describe the particular needs for annotations in robotics. 

\subsubsection{Data Labelling}
\label{subsub:2_CL:labeling}

As seen in previous section, gathering a varied set of raw data is already a difficult task. However, using it and understanding it is even more tedious. In this section, we detail different needs for labelling that autonomous agents such as robots need.
First of all, to understand its environment, a robot will need to apprehend the objects that compose it. To do so, the robot will need at some point that an external expert assesses that the object representation learned is good. This is the first kind of label the robot will need, i.e., object labelling \cite{Collet15, Craye18}.
Second, if we want the robot to perform a certain task, it will need to get information about the goals we gave it and also about what it should not do. This is generally done by a reward function that defines credit assignment \cite{Minsky61}, 
or it can also be defined internally by more abstract rules such as self-supervision \cite{Gopnik01, Smith05}, intrinsic motivation or curiosity \cite{Oudeyer07, Schmidhuber10} as in \cite{Forestier2017intrinsically,Colas18, Craye18, Laversanne-Finot18}.
Third, the robot should know when the task changes, and what task it should try to perform. This process consists of labelling the task; and the label is called the task identifier \cite{Lopez-Paz17}.

All these types of labels are not mandatory, but they drastically help and impact the learning process. 
The downside of labelling is that it is expensive and time consuming, which slows down the learning algorithms. To tackle those two problems, CL needs to find efficient solutions that can make the best out of the available labels for learning. 

The specific fields that aim at answering these questions are few-shot learning \cite{Lake11,Fei-Fei06} and active learning \cite{Burr10}. The former tries to grasp a concept from very few data points. Active learning aims at identifying and selecting the most needed labels in order to maximize learning. 
By combining optimization procedures in learning from few instances and minimizing the needs for labels, the field of robotics could be more suitable for leveraging continual learning settings in the real world. Furthermore, efficiency in learning reduces the risks of forgetting and degrading memories.

\subsubsection{Learning Algorithms Stability}
\label{subsub:2_CL:learning_algo}

In continual learning, algorithms face several learning experiences in a row. From each learning experience, some memory should be saved to later prevent for not forgetting. The stability of learning algorithms is then crucial: if only one learning experience fails, the whole process may be corrupted. Moreover, if we respect the continual learning causality, we cannot go back one or several tasks earlier in time in order to fix an current problem. The corruption of one learning experience can lead to the corruption of memories and then to the model degradation when learning later tasks. The needs for robust mechanisms to validate or reject results of a learning algorithm is key to keep sane memories and weights; however, the instability of deep learning models must also be addressed to overcome this drawback. As an example, generative models are powerful tools for continual learning but their instability may make them unsuitable for this kind of setting \cite{lesort2018generative}. Reinforcement learning algorithms are also known to be unstable and unpredictable, which is disastrous for continual learning. %

\subsection{Applications Fields}
\label{sub:2_CL:applications}

\begin{figure}[ht]
  \centering
  \includegraphics[width=0.6\textwidth]{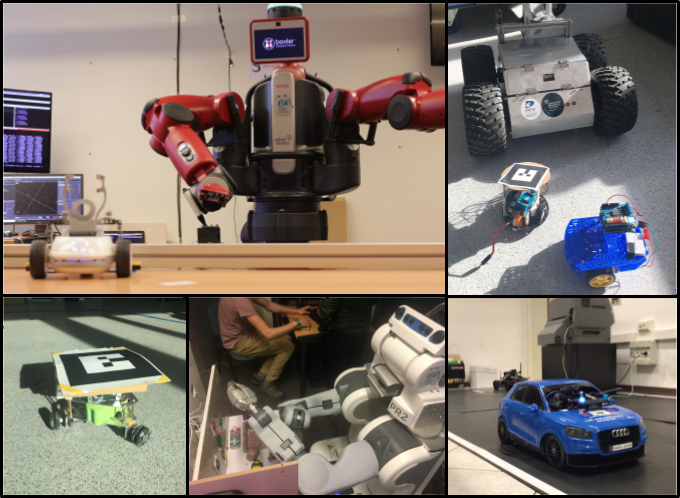}
  \caption[Illustation for robotics tasks]{Sample tasks tested for unsupervised open-ended learning \cite{raffin2019decoupling,Doncieux18} and continual learning settings \cite{Kalifou19} in a couple of robotics labs, among others, from the DREAM project.}
  \label{fig:2_CL:robots}
\end{figure}

Real-word applications of continual learning are virtually unlimited. In fact, any learning algorithm that needs to deal with the real world will face a non i.i.d. data stream. This as well happens for autonomous robots that learn new manipulation tasks, for exploration policies, as well as for autonomous vehicles that need to learn and adapt to new circumstances %
\cite{Bojarski16,Codevilla17, Jaritz18,Rhinehart18}. Non-static settings are also faced by algorithms that learn how to predict trends based on data streams from internet user activities, e.g., among others, for advertisement or finance. This problem is likewise confronted when an already trained algorithms needs to acquire new knowledge without forgetting, e.g., recognize new classes for classification, anomaly detection, etc. However, in this section we focus on specific continual learning use cases on robotics.

\subsubsection{Perception}
\label{subsub:2_CL:ap-classif}

While the world of perception is a multi-faceted topic at the very center of every application on autonomous sytems, the vast majority of CL algorithms in the literature are evaluated on object recognition tasks. Most models, indeed, are evaluated on datasets including static or moving objects. This is motivated by the fact that before any further action or policy, an autonomous agent (or robot) needs to identify the different component of its environment. In the case of classification, the robot may be pre-trained from an initial dataset. However, in any environment the robot would probably need to learn new objects from the new domain, and new variants (different poses, lighting, aspect) of already learned objects should be leveraged to improve its recognition \cite{lyubova15} capabilities. CL is crucial to deal with such dynamic scenarios. Initial progresses in this area have been proposed in \cite{Thrun95, Pasquale15, lomonaco16, camoriano17a, Lomonaco17}. Concrete Continual learning approaches to object segmentation can be found in \cite{michieli2019knowledge, michieli2019knowledge}, and in object detection in \cite{shmelkov2017incremental}.

Visual saliency for semantic segmentation and unsupervised object detection are other equally important applications in the context of perception which have been recently explored under continual learning and robotics settings \cite{Craye15}. RL-IAC (RL Intelligent Adaptive Curiosity), in particular, explores to learn visual saliency incrementally \cite{Craye18} with an articulated autonomous exploration technique based on curiosity to efficiently and continually learn a saliency model in a complex robotics environment tested in the real-world.

A classic problem in robotics within inherently continual learning settings are simultaneous localization and mapping (SLAM) \cite{cavallari2017fly} and navigation \cite{Thrun95}. In \cite{Thrun95}, using a HERO-2000 mobile robot with a radar sensor a continual learning algorithm based on explanation-based neural network learning (EBNN) is proposed to perform room mapping and navigation. Action models in EBNN \textit{explain} (in terms of previous experiences) and analyze observations to transfer task-independent (navigation) knowledge via predicting collisions and their prediction certainty. %

\subsubsection{Reinforcement Learning}
\label{subsub:2_CL:ap-reinforcement}

In reinforcement learning the data distribution is dependent on the actions taken by the controlled agent. 
Therefore, since the actions taken are not random, data is not i.i.d. and the data distribution is not stationary.
In the context of reinforcement learning similar techniques to those proposed in CL are often adopted in order to learn over a data distribution which is approximately stationary. An example of such techniques is the use of a external memory for rehearsal purposes, also know as \textit{experience} or \textit{memory replay} buffer \cite{Lin92,Schaul15,Hayes18MemoryEfficient}.

The first challenge for RL is the extraction of representations to understand and compact what is relevant from the input data \cite{Lesort18}.
Continual learning of state representations for RL is intrinsically close to unsupervised learning or representation learning for classification; the methods used in both cases may then be very similar or worth leveraging across.%

The second RL challenge is learning a policy to solve a specific task. The CL challenge of policy learning is different because it often should take into account both state and context. Context is usually handled with recurrent neural networks, and this kind of model is not yet %
been studied extensively in CL; one example is in \cite{Sodhani18}. 
Different robot manipulation tasks such as grasping and reaching objects that are used as benchmarks can be seen in Fig. \ref{fig:2_CL:robots} and, for instance, in state representation learning for robotics goal-based tasks \cite{raffin2019decoupling,Kalifou19}. 
These two challenges face continual learning problems, to learn representations and to learn policies from non stationary data distributions. However, it is worth distinguishing among both problems because learning and transfer between tasks are different challenges. Two tasks may need similar representations with different policies, while similar policies may require dissimilar representations.

In the context of robotics, fewer RL approaches have been proposed than in video-games or simulation settings. In particular, this is due to the low data efficiency of RL algorithms \cite{raffin2019decoupling}. 
We can still note several approaches that successfully tackle this problem, either in an end-to-end manner \cite{Kalashnikov18, Pinto15}, or by splitting the two challenges to address them separately, i.e., by first learning a state representation \cite{Lesort18} and later performing policy learning \cite{Finn15, Hoof16, Mattner12, Agrawal16, Duan17, Jonschkowski14}.
Nevertheless, a solution to this problem is to learn the policy in simulation and transfer it to deploy it in a real world robot \cite{Bojarski16, Rusu16sim2real, Gandhi17, Kalifou19}.

\subsubsection{Model-based Learning}
\label{subsubsap-ec:ers} 
Smoothly moving and interacting with always different, unpredictable environments, while building a consistent model of the external world, is one of the holy grails of robotics. For many years, researchers in this area have tried to propose robust and general enough sensory-motor solutions to complex problems such as navigation or object grasping. However, as it appears to be also true for humans, there will always be an environment or situation in which our biased model of the world fails and would require a further adaptation

Online (inverse dynamics) learning has also been applied in robotics, but generally not using deep learning. In \cite{Romeres16,Camoriano16}, the inverse and semiparametric dynamics of an iCub humanoid robot is learned in an incremental manner. This means both parametric modelling (based on rigid body dynamics equations) and nonparametric modelling (using incremental kernel methods) are used.
In \cite{Romeres18} it is shown that derivative-free models outperform numerical differentiation schemes in online settings when applied to non parametric (e.g. Gaussian processes with kernel function) model structures. %

In the pioneering work by \cite{Thrun95}, a model of both the external world and the robot itself is incrementally learned through reinforcement learning in complex navigation tasks on a real robot. However, incrementally and autonomously building a causal model of the external world, still remains a poorly explored topic in the context of robotics. Nevertheless, as it has been shown in recent RL literature, a model-based approach may be of fundamental importance in the real-world where millions of trials and errors are not always conceivable.

\section{Conclusion} 
\label{sec:2_CL:ccl}

Several notions appear to be crucial to clearly describe learning algorithms in CL settings, to compare them fairly enough and to transfer them from simulation to real autonomous systems and robotics.
First of all, identifying the exact problem we want to solve, and what are the existing constraints is mandatory. The framework we introduce in Section \ref{sec:2_CL:framework} should be an help in achieving the characterization of these settings. This formal step helps finding the proper approach to use and identifying similarities with other settings.
Secondly, in the same spirit of defining what we want to learn, it is important to define the level of supervision we are able to give to our learning algorithm. For example, if we can give it the task label, or some kind of information about the structure of the input data stream (number of classes, type of data distribution, number of instances of each task, etc.). This point is also discussed in our proposed framework (Section  \ref{sec:2_CL:framework}).
Finally, it is important to exactly clarify what is the expected performance of the algorithm. The set of metrics and benchmarks gathered in Section \ref{sec:2_CL:eval} should help defining and articulating the dimension of evaluation for important properties worth considering in the development of embodied agents that learn continually.

To summarize, in this chapter, we proposed a generalized framework to hold a variety of CL strategies and make easier the connection between machine learning and robotics in continual learning settings.
We reviewed the state of the art in continual learning and illustrated how to use the proposed framework to present various approaches. We also discussed benchmarks and evaluation techniques currently being used in continual learning algorithms.
Machine learning and robotics are fields undergoing an aggressive development period. We believe that pushing them forward to find formalization solutions to facilitate transfer between both fields is critical in order to understand each other, and make them profit from each other's successes. \checked{Moreover, the example of robotics for continual learning is well representative of the challenge addressed for the autonomy of learning agents.}

\checked{In the following chapters of this manuscript, we will refer to this chapter for state of the art description, continual learning scenarios, benchmarks, and evaluation metrics. But some content will also be summarized in the context of each chapter.}

\newpage
\chapter{Supervision Free Inference in Continual Learning}
\label{chap:2b_Replay}

\checked{In the previous chapter, we presented extensively the continual learning research field and illustrate its application potential through the lens of robotics.}
\checked{In this chapter, we discuss the need for supervision during the inference in some continual learning approaches and the practical limitations following from this requirement.}

\section{Introduction}

\checked{In continual learning, some approaches such as \textit{dynamic architectures} need to know from which task a data point is coming from in order to perform an inference. Indeed, since they use different inference paths or different models for prediction they need the task label to choose which one to use.} 
For regularization methods, it is often assumed that the task label is needed only at training time. However, in this chapter, we show that in class-incremental settings, the approach can not distinguish classes from different tasks. Therefore, the task label is necessary at test time.

The class-incremental settings we study is considered iid by parts. Each iid part is referred to as a \textit{task} and the data distribution changes are signalled by a task label. Each task contains different classes. %
In continual learning, this setting is called a class-incremental or disjoint-task scenario, as introduced in Chapter \ref{chap:2_CL}.
\checked{%
It consists of learning sets of classes incrementally. Each task is composed of new classes. As the training ends, the model should classify data from all classes correctly.}

In this chapter, the setting considers the task label as provided for training but not for inference.
Then, without task labels for inference, the model needs to both learn the discrimination of intra-task classes and the inter-task classes discrimination (i.e. distinctions between classes from different tasks). 
On the contrary, if the task label was available for inference, only the discrimination of intra-task classes needs to be learned. The discrimination upon different tasks is given by the task label.
Learning without access to task labels at test time is then much more complex, since it needs to discriminate data that are not available at the same time in the data stream.
We study in particular a widely used approach for continual learning: regularization. We show that in the classical setting of class-incremental tasks, this approach has theoretical limitations and can not be used alone. Indeed, it can not distinguish classes from different tasks.

\checked{The contributions of this chapter are:}

\begin{itemize}
\item \checked{We prove theoretical shortcomings of regularization based approaches in class-incremental settings.}
\item \checked{We propose an existing alternative to circumvent this problem.}
\end{itemize}

We believe this chapter presents important results for a better understanding of CL which will help practitioners to choose the appropriate approach for practical settings.

\section{Background}

Let's first summarize the information from Chapter \ref{chap:2_CL} relevant to this chapter. In continual learning, algorithms protect knowledge from catastrophic forgetting by saving them into a memory. The memory should be able to incorporate new knowledge and protect them from modification.

In continual learning, we distinguish four types of memorization approaches:

\begin{itemize}
\item \textbf{Dynamic architecture:} The neural networks create new weights automatically that will learn new tasks. Trained weights are frozen to protect memories \cite{Rusu16progressive, fernando2017pathnet, Li17learning}. In this case, the memory is composed of the old weights that are not modified anymore.

\item \textbf{Rehearsal:} In order to maintain knowledge from past learning experiences, the algorithms  save a subset of training data as memory
\cite{rebuffi2017icarl, nguyen2017variational, Aljundi2019Online, Belouadah2018DeeSIL, Kalifou19, Wu19Large, Hou_2019_CVPR, caccia2019online, Traore19DisCoRL}.
    \item \textbf{Generative Replay:} Instead of saving samples, this method learns generative models that will produce 
    artificial samples as memory of past learning experiences \cite{shin2017continual, lesort2018generative, wu2018memory, lesort2018marginal}. 
    \item \textbf{Regularization:} Regularization defines a loss that will constrain weight updates to retain knowledge from previous tasks \cite{kirkpatrick2017overcoming, Zenke17, Maltoni18}, or  distill knowledge \cite{hinton2015distilling}  from old models to a new one to remember past learning experiences \cite{Li17learning, schwarz2018progress}.

\end{itemize}
As presented in Section \ref{subsub:2_CL:hybrid}, many approaches use combinations of these families to allow better memorization. %
The effectiveness of these approaches is related to the use of the task label defined in Section \ref{sec:2_CL:vocabulary}. %
The task label $t$ is an abstract representation built to help continual algorithms to learn. It is designed to give information about the current task and notify if the task changes.%
$t$ is typically a simple integer indexing the tasks in the learning curricula.
The different use cases of task labels are described in Section \ref{sub:2_CL:Framework_Definitions}. %
Families of approaches have different dependencies to the task label.
For example, \textit{dynamic architecture} is an effective approach but it needs the task label at test time for inference.
Unfortunately, this necessity of supervision at test time is not desirable in most continual learning settings. Rehearsal and Generative Replay methods generally need the task label for learning but not for inference.
\checked{As mentioned in the chapter's introduction, for regularization methods, it is often assumed that the task label is needed only at training time. %
 To challenge this belief, we study the case of class-incremental settings. We would like to demonstrate that regularization methods can not distinguish classes from different tasks and therefore, that the task label is also necessary at test time to expect a good prediction.}

\section{Regularization Approach}
\label{sec:2b_Replay:approach}

\checked{In the class incremental setting presented in chapter's introduction, we would like to demonstrate that regularization does not help learn the discrimination between tasks.}
For example, if a first task is to discriminate white cats vs black cats and the second is the same with dogs, a regularization based method does not provide the learning criteria to learn features to distinguish white dogs from white cats.

\subsection{Formalism}

Let's first introduce a more detailed formalization of regularization. We assume that the data stream is composed of $N$ disjoint tasks learned sequentially one by one (with $N>=2$).
Task $t$ is noted $T_t$ and $\mathbb{D}_{t}$ is the associated dataset. 
The task label $t$ is a simple integer indicating the task index.
We refer to the full sequence of tasks as the continuum, noted $C_N$.
The dataset combining all data until task $t$ is noted $\mathbb{C}_{t}$. 
While learning task $T_t$, the algorithm has access to data from $\mathbb{D}_{t}$ only.

We study a disjoint set of classification tasks where classes of each task only appear in this task and never again.
We assume at least two classes per task (otherwise a classifier cannot learn). 

 Let $f$ be a function parametrized by $\bm{\theta}$ that implements the neural network's model.
At each task $t$ the model learns an optimal set of parameters $\bm{\theta}^*_t$ optimizing the task loss $\ell_{\mathbb{D}_t}(\cdot)$.
Since we are in a continual learning setting, $\bm{\theta}^{\ast}_t$ should also be an optima for all tasks $T_{t'}$, $ \forall t' \in \llbracket 0, t \rrbracket$. 

We consider the class-incremental setting with no test label. It means that an optima $\bm{\theta}_1^\ast$ for $T_1$ 
 is a set of parameters which at test time will, for any data point $x$ from $\mathbb{D}_0 \cup \mathbb{D}_1$, classify correctly without knowing if $x$ comes from $T_0$ or $T_1$. 
Therefore, in our continual learning setting, the loss to optimize when learning a given task $t$ is augmented  with a remembering loss:
\begin{equation}
\ell_{\mathbb{C}_t}(f(x; \bm{\theta}), y) = \ell_{\mathbb{D}_t}(f(x; \bm{\theta}), y) + \lambda \Omega(C_{t-1})
\label{eq:continual_loss}
\end{equation}
where $\ell_{\mathbb{C}_t}(.)$ is the continual loss, $\ell_{\mathbb{D}_t}(.)$ is the current task loss, $\Omega(C_{t-1})$ is the remembering loss with $C_{t-1}$ represents past tasks, $\lambda$ is the importance parameter.

\subsection{Problem}
\label{sub:pb}

In continual learning, the regularization approach is to define $\Omega(\cdot)$ as a regularization term to maintain knowledge from $C_{t-1}$ %
in the parameters $\bm{\theta}$ such as while learning a new task $T_t$, $f(x; \bm{\theta}^{\ast}_{t-1}) \approx f(x; \bm{\theta})$, $\forall x \in \mathbb{C}_{t-1}$.
In other words, it aims at keeping $\ell_{\mathbb{C}_{t-1}}(f(x; \bm{\theta}), y)$ low $\forall x \in \mathbb{C}_{t-1}$ while learning $T_t$. 

The regularization term $\Omega_{t-1}$ act as a memory of $\bm{\theta}^{\ast}_{t-1}$. 
This memory term depends on the learned parameters $\bm{\theta}^{\ast}_{t-1}$, on $\ell_{\mathbb{C}_{t-1}}$ the loss computed on $T_{t-1}$ and the current parameters $\bm{\theta}$. 
$\Omega_{t-1}$ memorizes the optimal state of the model at $T_{t-1}$ and generally the importance of each parameter with regard to the loss $\ell_{\mathbb{C}_{t-1}}$. 
We note $\Omega_{\mathbb{C}_{t-1}}$ the regularization term memorizing optimal parameters for all past tasks.

When learning the task $T_t$, the loss to optimize is then:
\begin{equation}
\ell_{\mathbb{C}_t}(f(x; \bm{\theta}), y) = 
\ell_{\mathbb{D}_t}(f(x; \bm{\theta}), y)
+ 
\lambda \Omega_{\mathbb{C}_{t-1}}(\bm{\theta}^{\ast}_{t-1},\ell_{\mathbb{C}_{t-1}}, \bm{\theta})
\label{eq:update_general}
\end{equation}
Eq. \ref{eq:update_general} is similar to eq. \ref{eq:continual_loss} but in this case the function $\Omega(\cdot)$ is a regularization term depending on past optimal parameters $\bm{\theta}^{\ast}_{t-1}$, loss on previous tasks $\ell_{\mathbb{C}_{t-1}}$ and the vector of current model parameters $\bm{\theta}$ only. 
It could be for example a matrix pondering weights importance in previous tasks 
\cite{kirkpatrick2017overcoming, Ritter18, Zenke17}.

\subsection{Regularization methods}

To illustrate the previous section, we present several well known regularization methods in our formalism.

\subsubsection{Elastic Weight Consolidation}
(EWC) \cite{kirkpatrick2017overcoming} is one of the most famous regularization approaches for continual learning.
The loss augmented with a regularization term is at task $t$: 
\begin{equation}
\ell_{\mathbb{C}_t}(\bm{\theta}) = \ell_{\mathbb{D}_t}(f(x; \bm{\theta}), y) + \frac{\lambda}{2} * F_{t-1} (\bm{\theta}^*_{t-1} - \bm{\theta})^2
\label{eq:update_ewc}
\end{equation}
We can then by identification, extract our function $\Omega_t(\bm{\theta}^*, \ell_D, \bm{\theta})$
\begin{equation}
\Omega_t(\bm{\theta}^*, \ell_{\mathbb{C}_{t-1}}, \bm{\theta}) =  \frac{1}{2} * F_{t-1} (\bm{\theta}^*_{t-1} - \bm{\theta})^2
\label{eq:ewc}
\end{equation}
$F_t$ is a tensor of size $card(\bm{\theta})^2$, specific to task $t$, characterizing the importance of each parameter $\theta_k$. $F_t$ is computed at the end of each task and will protect important parameters to learn without forgetting.
In EWC, the $F_t$ tensor is implemented as a diagonal approximation of the Fisher Information Matrix:
\begin{equation}
F_t =  \mathbb{E}_{(x,y) \in \mathbb{D}_t} \left[ \left( 
\frac{ \partial log~p(\hat{y}) }{\partial \bm{\theta}}   \right)^2
\right]
\label{eq:fisher_matrix}
\end{equation}
 where $\hat{y} \sim P(f(x; \bm{\theta}))$. 
 The diagonal approximation allows to save only $card(\bm{\theta})$ values in $F_t$.

\subsubsection{K-FAC Fisher approximation} 
The K-FAC Fisher approximation \cite{Ritter18} is very similar to EWC but approximates the Fisher matrices with a Kronecker factorization (K-FAC) \cite{martens2015optimizing} to improve the expressiveness of the posterior over the diagonal approximation. However, the Kronecker factorization saves more values than the diagonal approximation.

\subsubsection{Incremental Moment Matching}
Incremental Moment Matching (IMM) \cite{lee2017overcoming}  proposes two regularization approaches for continual learning which differ in the computation of the mean $\theta_{0:t}$ and the variance $\sigma_{0:t}$ of the parameters on all tasks.

The idea is to regularize parameters such that the moments of their posterior distributions are matched in an incremental way. It means that each parameter is approximated as a normal distribution and their mean or standard deviation should match from one task to another. This regularization, on the parameters' low-order moments, helps to protect the model from forgetting. 

\begin{itemize}

\item Mean based Incremental Moment Matching (mean-IMM)
\begin{equation}
\theta_{0:t} =  \sum_{i=0}^{t} \alpha_i \bm{\theta}^*_i  \quad\mathrm{and}\quad  \sigma_{0:t} =  \sum_{i=0}^{t} \alpha_i (\sigma_i +(\bm{\theta}^*_i - \theta_{0:t})^2)
\label{eq:mu_mean_imm}
\end{equation}
$\alpha_i$ are importance hyper-parameters to balance past task weight into the loss function. They sum up to one.

\item Mode based Incremental Moment Matching (mode-IMM)
\begin{equation}
\theta_{0:t} =  \sigma_{0:t} \cdot \sum_{i=0}^{t} (\alpha_i \sigma_i^{-1} \bm{\theta}^*_i) \quad\mathrm{and}\quad \sigma_{0:t} =  (\sum_{i=0}^{t} \alpha_i \sigma_i^{-1})^{-1}
\label{eq:mu_mode_imm}
\end{equation}
 $\sigma_{i}$ is computed as the Fisher matrix (eq. \ref{eq:fisher_matrix}) at task $i$.

\end{itemize} 
 
 Then at task $t$, with $\theta_{0:t-1}$ and $\sigma_{0:t-1}$ we can compute:
\begin{equation} 
\Omega_t(\bm{\theta}^*, \ell_{\mathbb{C}_{t-1}}, \bm{\theta}) = \frac{1}{2}  \sigma_{0:t-1} (\theta_{0:t-1} - \bm{\theta})^2
\label{eq:omega_imm}
\end{equation}
\subsubsection{Synaptic Intelligence}

The original idea of Synaptic Intelligence approach (SI) \cite{Zenke17} is to imitate synapse biological activity. Therefore, each synapse accumulates task relevant information over time, and exploits this information to rapidly store new memories without forgetting old ones. 
In this approach, we can identify $\Omega_t$ as:
\begin{equation}
\Omega_t(\bm{\theta}^*, \ell_{\mathbb{C}_{t-1}}, \bm{\theta})  =  M_t (\bm{\theta}^*_{t-1} - \bm{\theta})^2
\label{eq:synaptic_intelligence}
\end{equation}
$M_t$ is a tensor of size $card(\bm{\theta})$ specific to task $t$ characterizing the importance of each parameter $\theta_k$ over the all past tasks such as:
\begin{equation}
M_t = \sum_{0<i<t} \frac{m_i}{\Delta_i^2 + \xi}
\label{eq:synaptic_intelligence_Weights}
\end{equation}
$M_t$ is the sum over $m_i$ which characterizes the importance of each parameter on task $i$, with $\Delta_i = \bm{\theta}^*_i-\bm{\theta}^*_{i-1}$. $\xi$ is a supplementary parameter to avoid null discriminator.
\begin{equation}
m_i = \int^{T_i}_{T_{i-1}} \nabla_{\bm{\theta}} \delta_{\bm{\theta}}(t) dt
\label{eq:synaptic_intelligence_weights}
\end{equation}
With $\delta_{\bm{\theta}}(t)$ the parameter update at time step $t$.

\section{Propositions}
In this section, we present the proposition concerning the shortcomings of regularization methods in class-incremental settings. We also present  preliminary definitions and lemmas to prepare for the proposition and we illustrate the proposition with practical examples.

\subsection{Preliminary Definition / Lemma}

\begin{Definition}\textbf{Linear separability}\\
Let $S$ and $S'$ be two sets of points in an n-dimensional Euclidean space. $S$ and $S'$ are linearly separable if there exists $n+1$ real numbers $\omega_1, \omega_2,...,\omega_n, k$ such that $\forall x\in S$,  $\sum_{i=1}^n \omega_i x_i > k$ and $\forall x \in S'$,  $\sum_{i=1}^n \omega_i x_i < k$
\end{Definition}
where $x_i$ the $i$-th component of x. This means that two classes are linearly separable in an embedded space if there exists a hyper-plane separating both classes of data points.

It can be written, $\forall x \in S$ and $\forall x' \in  S'$.
\begin{equation}
(\bm{q} \cdot \bm{x} + q_0) \cdot (\bm{q} \cdot \bm{x'} + q_0) < 0
\label{eq:linear_sep}
\end{equation}
with $\bm{q}=[\omega_1, \omega_2,...,\omega_]$ and $q_0=-k$ respectively the  normal vector

\begin{figure}[h]
    \centering
    \includegraphics[width=0.3\linewidth]{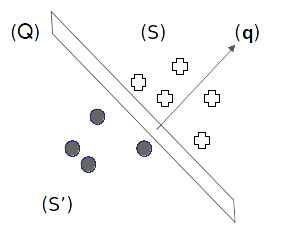}
    \caption[Illustration of a decision boundary learned.]{Illustration of a decision boundary learned between two sets of points.}
    \label{fig:2b_Replay:hyperplan}
\end{figure} and position vector of a hyper-plane $\mathcal{Q}$.

In the case of learning a binary classification with linear model, the model is the best hyper-plane separating two dataset as in Fig.~\ref{fig:2b_Replay:hyperplan}. %
As soon as eq. \ref{eq:linear_sep} can be solved, then it is possible to define a function $f(x,\theta)$ and a loss $\ell(.)$ to learn a hyper-plane that will separate $S$ and $S'$ perfectly.

\begin{Definition}\textbf{Interferences}\\
In machine learning, interferences are conflicts between two (or more) objective functions leading to prediction errors. 
\end{Definition} 

As such, optimizing one of the objective function increases the error on the other one. In continual learning, interferences happen often after a drift in the data distribution. The loss on previous data is increased with the optimization of the loss for the new data leading to interferences and catastrophic forgetting.

\checked{We would like to present the following general lemma, as a preparatory step for the later proposition presented in this chapter:}

\begin{lemma}
$\forall (S, S')$ bounded set of discrete points in $R^n$ and linearly separable by a hyper-plane $\mathcal{Q}$. For any algorithm, it is impossible to assess $\mathcal{Q}$ as a separation hyper-plane without access to
$S'$ set.
\label{lem:boundary}
\end{lemma}

\begin{proof}
Let $S$ and $S'$ be two bounded and linearly separable set of discrete points in $R^n$. %
 Let $\mathcal{Q}$ be a potential linear separation between $S$ and $S'$.
The hyper-plane $\mathcal{Q}$ can not be assessed as a linear separation between $S$ and $S'$ if there exists at least one hyper-plane indistinguishable from $\mathcal{Q}$ and which is not a separation boundary between $S$ and $S'$.

Let $\mathcal{P}$ be a hyper-plane, defined as a normal vector $\bm{p}$ and position vector $p_0$. $\mathcal{P}$ is a separation boundary between $S$ and $S'$ if all the point of $S$ are on one side of $\mathcal{P}$ and all point of $S'$ are on the other side. It can be formalized as follows:

\medskip

$\forall \bm{x} \in S ~ \& ~ \forall \bm{x'} \in S' $:

\begin{equation}
(\bm{p} \cdot \bm{x} + p_0) \cdot (\bm{p} \cdot \bm{x'} + p_0) < 0
\label{eq:linear_sep2}
\end{equation}
where $<\cdot>$ is the scalar product.

Without the access to $S'$, eq. \ref{eq:linear_sep2} can not be evaluated to verify that $S$ and $S'$ are each entirely on different side of the $\mathcal{P}$.

However, we can evaluate if all the point of $S$ are on the same side of $\mathcal{P}$. By definition if all the point sof $S$ are above $\mathcal{P}$ then: %

\medskip

$\forall \bm{x} \in S$
\begin{equation}
(\bm{p} \cdot \bm{x} + p_0) > 0
\label{eq:upper_set}
\end{equation}

If all the point are under $\mathcal{P}$ then:
\begin{equation}
(\bm{p} \cdot \bm{x} + p_0) < 0
\label{eq:under_set}
\end{equation}

And if neither eq. \ref{eq:upper_set} nor eq. \ref{eq:under_set} are verified then all the points of $S$ are not on the same side of $\mathcal{P}$.
Finally, we can merge both \ref{eq:upper_set} and eq. \ref{eq:under_set} and verify only:

\medskip

$\forall \bm{x} \in S$
\begin{equation}
sign{(\bm{p} \cdot \bm{x} + p_0)}  =  constant
\label{eq:sign}
\end{equation}
where $sign(.)$ is the function which returns the sign of any real value. 

The lemma \ref{lem:boundary} is proven if $\exists$ $\mathcal{P}$ such as eq. \ref{eq:sign} is true but not eq. \ref{eq:linear_sep2}, because $\mathcal{P}$ would not be a linear separation of $S$ and $S'$ and would not be distinguishable from $\mathcal{Q}$ without access to $S'$. 

Now, we will build an hyper-plan $P$ that respect eq. \ref{eq:sign} and not eq. \ref{eq:linear_sep2}.
We know that $S$ is bounded, then it has both upper and lower bounds in all the direction of $R^n$.
If eq. \ref{eq:sign} is respected, then $\mathcal{Q}$ is a bound of $S$ in the direction of its normal vector $\bm{q}$. If we move  $\mathcal{Q}$ along the direction of $\bm{q}$ (i.e. if we change the position vector $q_0$ ), we can find at least one other plane $\mathcal{P}$ respecting eq. \ref{eq:sign}: the opposing bound of $S$ along the direction $\bm{q}$.

Since, $\mathcal{P}$ and $\mathcal{Q}$ are two opposing bounds of $S$ in the same direction $\bm{q}$, then: 

\medskip

$\forall x \in S$
\begin{equation}
sign{(\bm{p} \cdot \bm{x} + p_0)}  \neq  sign{(\bm{q} \cdot \bm{x} + q_0)}
\label{eq:sign2}
\end{equation}

If $\mathcal{Q}$ is a lowerbound of $S$ in the direction $\bm{q}$ and an upperbound of $S'$ in the same direction then, a lowerbound of $S'$ in the direction $\bm{q}$ is a lowerbound of $S$ in the same direction and  an upperbound of $S$ in the direction $\bm{q}$ is an upperbound of $S'$ in the same direction. (We leave the demonstration to the reader).

Therefore, $\mathcal{Q}$ and $\mathcal{P}$ are both upperbounds or both lowerbounds of $S'$ in the direction of $\bm{q}$:

\medskip

$\forall x' \in S'$:
\begin{equation}
sign{(\bm{p} \cdot \bm{x'} + p_0)}  =  sign{(\bm{q} \cdot \bm{x'} + q_0)}
\label{eq:sign3}
\end{equation}

Then with \ref{eq:sign2} and eq. \ref{eq:sign3}:

\begin{equation}
(\bm{p} \cdot \bm{x} + b) \cdot (\bm{p} \cdot \bm{x'} + b) > 0
\label{eq:linear_sep3}
\end{equation}

Consequently, from eq \ref{eq:sign} and eq \ref{eq:linear_sep3}, $\exists$ a hyper-plane $\mathcal{P}$ which respects eq. \ref{eq:sign} and not eq \ref{eq:linear_sep2}, $\mathcal{P}$ is indistinguishable from $\mathcal{Q}$ and is not a separation boundary between $S$ and $S'$. 

\end{proof}

Let's summarize this demonstration in a more insightful way. For any bounded set of points $S$, there is an infinite number of linearly separable set of points. Thus, there exists an infinite number of potential separating hyper-planes. If the second set of points $S'$ is not known, then it is not possible to choose among the infinite number of potential separating hyper-plane which one is a good one. And even if one is chosen, there is no way to tell if it is better or not than another.

In the context of machine learning, this demonstration says that without an assessment criterion for a classification problem, it is not possible to learn a viable solution. Hence, we can not optimize the parameters. For binary classification, the lemma \ref{lem:boundary} can be interpreted as: \say{The decision boundary between two classes can not be assessed nor learned if there is no access to data from both simultaneously}.

\checked{After presenting the lemma about finding linear separation between sets of points, we would like to add another one concerning finding a projection making two mixed sets of points linearly separable. This lemma will also help in the proposition proof presented in this chapter.}

\begin{lemma}
 $\forall (S, S')$ two bounded datasets not linearly separable.
 For any algorithm, it is impossible to assess a function $g(.)$ as a projection of $S$ and $S'$ into a space were they are linearly separable without access to
 $S'$ set.
\label{lem:features}
\end{lemma}

\begin{proof}

$g(.)$ is a projection of $S$ and $S'$ into a space where they are linearly separable means:

\medskip

$\forall \bm{x} \in S ~ \& ~ \forall \bm{x'} \in S'$, then $g(x)$ and $g(x')$ respect eq. \ref{eq:linear_sep2}.

\medskip

Without access to $S'$ this condition can not be verified. However, we can verify eq. \ref{eq:sign} with $g(x)$.

The lemma 3.4 is proven if $\forall \bm{x} \in S ~ \& ~ \forall \bm{x'} \in S'$, $\exists$ a projection $f$, that respect eq. \ref{eq:sign} with $f(x)$ but not eq. \ref{eq:linear_sep2} with $f(x)$ and $f(x')$, because then $f$ and $g$ are indistinguishable without access to $S'$.

Let $f$ be the identity function, $\forall z \in \mathbb{R}$ $f(z)=z$.
We define $S_f$ and $S'_f$, the set of point $S$ and $S'$ after projection by $f$. Since $f$ is the identity function, $S$ and $S'$ are respectively identical to $S_f$ and $S'_f$.
Since $S$ is bounded, $S_f$ is also bounded. Hence there exists a hyper-plane $\mathcal{P}$ that verify eq. \ref{eq:sign} with $f(x)$ $\forall x \in S$. By hypothesis, $S$ and $S'$ are not linearly separable so $S_f$ and $S'_f$ is also not linearly separable. Then $\exists !$ hyper-plane $\mathcal{P}$ which respect eq. \ref{eq:linear_sep2} with $f(x)$ and $f(x')$.

Thus, $f$ exists and therefore it is impossible to assess any function as a projection of $S$ and $S'$ into a space were they are linearly separable without %
$S'$ set.

\end{proof}

In a more insightful way, for any bounded set of points, there is an infinite number of projections of the initial set of point in a space where it could be linearly separable from another set of points. Then, if you don't know the second set of points $S'$ you can not choose among the infinite number of potential projections which one is a good one. And if you ever choose one, you have no way to tell if it is better or not than another.

In the context of binary classification, the previous lemma can be interpreted as: \say{Two classes representation cannot be disentangled if there is no access to data from both simultaneously}.

In those lemma, the concept of \say{not having access to} a certain dataset can both be applicable to not being able to sample data point from the distributions and to not have a model of the dataset. 
It can be generalized to not having access to any representative data distribution of a dataset.

\subsection{Shortcomings in class-incremental tasks}
\label{sub:prop}

We would like to prove now that in incremental-class tasks, it is not possible to learn to discriminate classes from different tasks using only a regularization based memory. 
The main point is that, to correctly learn to discriminate classes over different tasks, the model needs access to both data distributions simultaneously. 

In regularization methods, the memory only characterizes the model and the important parameters as explained in Section \ref{sub:pb}. This memorization gives insight on some past data characteristics but it is not a model of their distributions globally.%
If we take again the cat/dog example, a model that needs to discriminate white cats from black cats will learn to discriminate black features from white features and this can be saved in $\Omega$ but $\Omega$ will not save the full characteristics of a cat because the model never has to learn it.

\begin{proposition}
While learning a sequence of disjoint classification tasks, if the memory $\Omega$ of the past tasks is only dependent on trained weights and learning criterion of previous task and does not model the past distribution, it is not possible to learn new tasks without interference.
\label{prop:regularization}
\end{proposition}

\begin{proof}

We are here in the context of learning with a deep neural network. We can decompose the model into a non-linear feature extractor $g(\cdot)$ and a linear output layer to predict a class $y=argmax (\sigma(A \cdot g(x) + b)$:

The projection $g(.)$ ensures the linear separability of classes and the output layer learns the separation.
The output layer allows for each class $i$ to learn a hyper-plane $A[:,i]$ with bias $b[i]$ that separate all classes from the class $i$, such as:
$\forall i \in \llbracket 1, N \rrbracket$ %
\begin{equation}
\forall (x,y) \in \mathbb{D}_t,  \operatorname*{argmax}_{i} (A[:,i]h + b[i]) = y
\label{eq:output_class}
\end{equation}
With $h = g(x)$. However, if we look at the classes independently the model should only require that:
$\forall i \in \llbracket 1, N \rrbracket$,
\begin{equation}
\forall (x,y) \in \mathbb{D}_t, 
\begin{cases} 
    \mathrm{if} ~ y=i, A[:,i]h + b[i] > 0 \\
    \mathrm{if} ~ y \neq i, A[:,i]h + b[i] < 0 \\
\end{cases}
\label{eq:hyper-plane}
\end{equation}
We now study how to learn new tasks $T_t$ for $0<t<N$. There are two different cases, \textit{first} if $g(\cdot)$ is already a good projection for $T_t$, i.e. classes are already linearly separable in the embedded space. \textit{Secondly}, if $g(\cdot)$ needs to be adapted, i.e. classes are not yet linearly separable in the embedded space and new features need to be learned by $g(\cdot)$ to fix it. \checked{Those two cases are illustrated with simple examples in the next section}.%
We call features the intrinsic characteristics of data that a model needs to detect to distinguish a class from another.
\begin{itemize}
    \item \textbf{First case:} \textit{Classes are linearly separable }\\
For all task $\forall (t) \in \llbracket 1, N-1 \rrbracket$, $T_{t-1}$ happens before $T_t$ in the continuum.
Since we are in a regularization setting, at task $T_t$, we have access to $\Omega_{t-1}$ which contains classification information from previous tasks including $T_{t-1}$. However, by hypothesis, $\Omega_{t-1}$ does not model the data distribution from $T_{t-1}$ and therefore it does not model data distribution from $T_{t-1}$ classes.

$T_t$ and $T_{t-1}$ classes images are a bounded set of points and $T_{t-1}$ points are not accessible,
consequently following lemma \ref{lem:boundary}, it is impossible to assess a boundary between any classes from $T_t$ and any classes $T_{t-1}$ even if by hypothesis this boundary exists. Therefore, we can not learn to discriminate classes from different tasks in this case.

    \item \textbf{Second case:} \textit{$g(\cdot)$ needs to be updated}.\\
Let $\delta_{t-1}$ be the set features already learned by $g_{t-1}(\cdot)$ the feature extractor from previous task.
$\Omega_{t-1}$ should keep $\delta_{t-1}$ unchanged while learning $T_t$.
Either, $\delta_{t-1}$ allows to solve $T_t$ and we are in \textit{first case}, or a new set of features $\delta_t$ needs to be learned while learning $T_t$.
In the second case, the set $\delta_t$ contains features to solve $T_t$, but features $\delta_{t-1:t}$ that distinguish classes from $T_{t-1}$ to classes from $T_t$ should also be learned. Then two cases raise, $\delta_{t-1:t} \not\subset \delta_t$ or $\delta_{t-1:t} \subset \delta_t$.

\begin{itemize}
\item if $\delta_{t-1:t} \not\subset \delta_t$, then supplementary features need to be learned to $\delta_{t-1}$ and $\delta_t$.
$T_t$ and $T_{t-1}$ classes images are a bounded set of points not linearly separable and since $\Omega_{t-1}$ does not give access to $T_{t-1}$ data points, from lemma \ref{lem:features} we can not assess a projection that put classes from $T_t$ and classes $T_{t-1}$  into a linearly separable space, i.e. we can not learn the set of features $\delta_{t-1:t}$ to discriminate $T_{t-1}$ from $T_t$ and solve the continual problem.

\item $\delta_{t-1:t} \subset \delta_t$ is possible but %
there is no way to project data from $T_{t-1}$ in the new latent space since they are no more accessible. Therefore, at $t$ we can not know if $\delta_{t-1:t} \subset \delta_t$ and which features of $\delta_t$ are in $\delta_{t-1:t}$.

\end{itemize}
\end{itemize}

To conclude, it is not possible to learn proper boundaries between classes in different tasks should be the feature space is already adapted to it or not. 
 There will be in any way conflict between losses leading to interference in the decision boundaries either because classes are not linearly separable or because a separation hyper-plane can't be found.
The regularization methods can not discriminate classes from different tasks and they are then not suited to class-incremental settings.
A simple trick used in some regularization approaches to compensate this shortcomings is to use the task label for inference, it gives a simple way to distinguish tasks from each other. However, it assumes the algorithms rely on a supervision signal for inference.

\end{proof}

We can note that
proposition \ref{prop:regularization}, still holds if tasks are only partially disjoint, i.e. only some classes appears only once in the continual curriculum.

Indeed, in partially disjoint settings, several classes are never in the same task.
If we define two set of disjoint classes $Y$ and $Y'$, that will never be in the same task. The demonstration of proposition \ref{prop:regularization} can be applied on $Y$ and $Y'$.
Then, classes $Y$ and $Y'$ will suffer from interference showing a shortcoming of regularization methods for this case too.
Therefore, if there is a class-incremental setting hidden into another setting, the regularization approach will not be able to solve it perfectly either. We could note that in many applications there are latent class-incremental problem to address in the learning curriculum. We mention some applications in Section \ref{sub:applications}.

\subsection{Practical examples}
\label{sub:practical_examples}

To illustrate the proposition from Section \ref{sub:prop}, we present two insightful examples of regularization limitations.

\subsubsection{The Task Separability Problem}

In the first case of proposition \ref{prop:regularization} proof, we already have a perfect feature extractor.
Classes are already linearly separable and only the output layer needs to be learned continually. %

If we have only two classes in the first task, the model will learn one hyper-plane $\mathcal{Q}_0$ separating the instances of these two classes (See Figure \ref{fig:2b_Replay:simple_continual}). 
For the second task, we have two new classes and a regularization protecting $\mathcal{Q}_0$. Then, we can learn a hyper-plane $\mathcal{Q}_1$ that separates our two new classes.
In the end, we have learned the hyper-planes $\mathcal{Q}_0$ and $\mathcal{Q}_1$ to distinguish classes from $T_0$ and classes from $T_1$. But none of those hyper-planes helps to discriminate $T_0$ classes from $T_1$ classes, as illustrated Figure \ref{fig:2b_Replay:simple_continual}. This will lead to error in the neural networks predictions.

\begin{figure}[h]
    \centering
    \begin{subfigure}[t]{0.3\linewidth}
        \centering
        \includegraphics[width=0.9\linewidth]{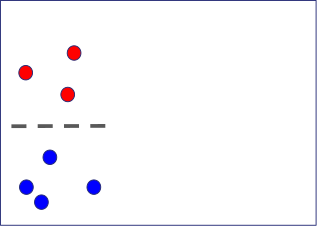}
    \end{subfigure}
    \begin{subfigure}[t]{0.3\linewidth}
        \centering
        \includegraphics[width=0.9\linewidth]{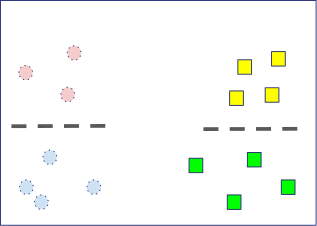}
    \end{subfigure}
    \caption[Simple case of continual learning classification in a multi-task setting.]{Simple case of continual learning classification in a multi-task setting. Left, the task T0: learning a hyper-plane splitting two classes (red and blue dots). Right, the task  T1: learning a line splitting two classes (yellow and green squares) while remembering $T_0$ models without remembering $T_0$ data (pale red and blue dots).}
    \label{fig:2b_Replay:simple_continual}
\end{figure}

\subsubsection{The Latent Features Problem}

In the second case of Proposition \ref{prop:regularization} proof, the feature extractor needs to be updated to learn new features extractors.

If we have only two classes in the first task, the model will learn to separate classes instances into two groups with the features extractor $g_0$ and one hyper-plan $\mathcal{Q}_0$ separating the two classes instances (See Figure \ref{fig:2b_Replay:T0}). 

For the second task, we have two new classes and a regularization protecting $\mathcal{Q}_0$ and $g_0$. Then, we can learn a features extractor $g_1$ to disentangle new class instances in the latent space and a hyper-plane $\mathcal{Q}_1$ that separates them.
In the end, we can disentangle classes from $T_0$ and classes from $T_1$ and we have two hyper-planes $\mathcal{Q}_0$ and $\mathcal{Q}_1$ to distinguish classes from $T_0$ and classes from $T_1$. 
But we can not disentangle  $T_0$ classes from $T_1$ classes and none of the learned hyper-planes helps to discriminate $T_0$ classes from $T_1$ classes (See Fig.~ \ref{fig:2b_Replay:T1}). It leads to errors in the neural network predictions. %
At test time, it will not be possible for the model to discriminate between classes correctly.

\begin{figure}
    \centering
    \begin{subfigure}[t]{0.3\linewidth}
        \includegraphics[width=0.9\linewidth]{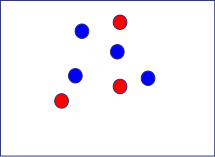}
    \end{subfigure}%
    ~ 
    \begin{subfigure}[t]{0.3\linewidth}
        \centering
        \includegraphics[width=0.9\linewidth]{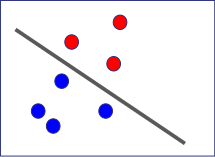}
    \end{subfigure}
    \caption[Case of representation overlapping in one task classification setting.]{$\mathbb{D}_0$ feature space before learning $T_0$ (Left), $\mathbb{D}_0$ feature space after learning $T_0$ with a possible decision boundary (Right). Data points are shown by blue and red dots. The line (right part) is the model learned to separate data into the feature space.}
    \label{fig:2b_Replay:T0}
\end{figure}

\begin{figure}
        \centering
        \begin{subfigure}[t]{0.3\linewidth}
        \centering
        \includegraphics[width=0.9\linewidth]{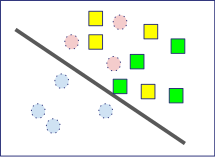}
    \end{subfigure}%
    ~ 
    \begin{subfigure}[t]{0.3\linewidth}
        \centering
        \includegraphics[width=0.9\linewidth]{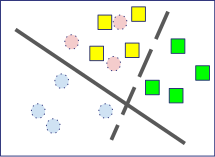}
    \end{subfigure}
    \caption[Case of representation overlapping in multi-tasks classification setting.]{Case of representation overlapping while continual learning classification in a multi-task setting. At task $T_1$, feature space of $\mathbb{D}_1$ before learning $T_1$ (Left), Feature space of $\mathbb{D}_1$ after learning $T_1$ with a possible decision boundary (Right). New data are plotted as yellow and green squares and old data that are not available anymore to learn are shown with pale red and blue dots. } 
    \label{fig:2b_Replay:T1}
\end{figure}

However, with the task label for inference, we could potentially perfectly use $g_0$, $g_1$, $\mathcal{Q}_0$ and $\mathcal{Q}_1$ to make good predictions.
Nevertheless, assuming that the task label is available for prediction is a strong assumption in continual learning and involves a need of supervision at test time.

\section{Experiments}
\label{sec:exp}

To illustrate the concrete effects of the limitations presented earlier, we propose the dataset \say{MNIST-Fellowship} for our experiments. %
 This dataset is composed of three datasets (Fig. \ref{fig:2b_Replay:MNIST_Fellowship}): MNIST \cite{LeCun10}, Fashion-MNIST \cite{Xiao2017} and KMNIST \cite{clanuwat2018deep}, each composed of 10 classes, those datasets should be learned sequentially one by one. We choose this dataset because it gathered three easy datasets for prototyping machine learning algorithms but solving those three quite different datasets is still harder than solving only one.

\begin{figure}
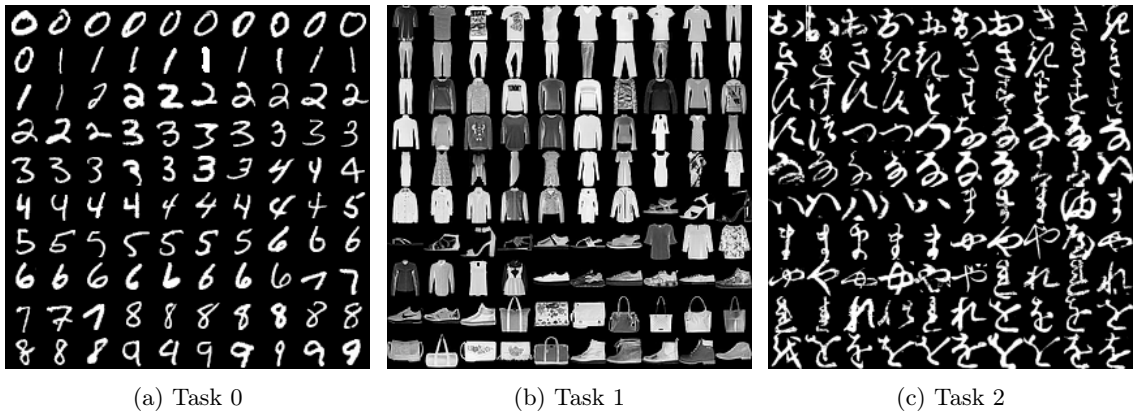

    \centering
    \begin{subfigure}[t]{0.32\linewidth}
        \centering
        \includegraphics[width=\linewidth]{Files/2b_Replay_Better/Samples/mnist_fellowship_task_0.png}
        \caption{Task 0}
    \end{subfigure}%
    ~ 
    \begin{subfigure}[t]{0.32\linewidth}
        \centering
        \includegraphics[width=\linewidth]{Files/2b_Replay_Better/Samples/mnist_fellowship_task_1.png}
        \caption{Task 1}
    \end{subfigure}%
    ~ 
    \begin{subfigure}[t]{0.32\linewidth}
        \centering
        \includegraphics[width=\linewidth]{Files/2b_Replay_Better/Samples/mnist_fellowship_task_2.png}
        \caption{Task 2}
    \end{subfigure}
    \caption[Illustration of the MNIST-Fellowship dataset's three tasks.]{The three tasks of the MNIST-Fellowship dataset.}
    \label{fig:2b_Replay:MNIST_Fellowship}
\end{figure}

Our goal is to illustrate the limitation of regularization based methods in disjoint settings. In particular that they can not distinguish classes from different tasks. We would like also to show that the shortcomming happen both in the output layer and in the feature extractor. Thus, we propose three different settings with the \textit{MNIST-Fellowship} dataset.

 \begin{itemize}
    \item \textbf{1. Disjoint setting}: all tasks have different classes (i.e. from 0 to 29). 
    
     \item \textbf{2. Joint setting}: all tasks have the same classes ( i.e. from 0 to 9) but different data. 
     
     \item \textbf{3. Disjoint setting with test label}: All tasks have different classes but at inference, we know from which task a data-point is coming from.
\end{itemize} 

First setting (disjoint with no test label), is the hardest because all classes need to be discriminated from all the others. The second setting (joint) is a bit easier because we don't need to discriminate task from each other but the model needs to use the same output layer for all task which can produce interferences. Theoretically, the second setting requires only the feature extractor to be learned.
The last setting (disjoint with test label) is the easiest, classes from different tasks don't need to be compared and the output layer is different for each task. \checked{The model used is presented in Table \ref{tab:2b_Replay:hyperparams}.}

\begin{table}[h]
\centering

  \caption[Model architecture for regularization methods evaluation.]{Model architecture, convolution have 5*5 kernel size, maxpool have 2*2 kernel size. Parameters not mentioned are default parameters in Pytorch library \cite{NEURIPS2019_9015} (in torch.nn). BS is for batch size, which is 64. All layers are initialized with Xavier init method \cite{glorot2010understanding}.}
  \label{tab:2b_Replay:hyperparams}
  \begin{tabular}{ccccccc}
    \hline 
    Layer Name &
    Layer Type &
    Input Size&
    Output Size\\
    \hline
    
    Conv1     & %
    ReLu(MaxPool2d(Conv2d(input))) & %
    BS*1*28*28 & %
    BS*10*12*12 \\ %
    \hline
    
    Conv2     & %
    ReLu(MaxPool2d(Conv2d(input))) & %
    BS*10*14*14  & %
    BS*20*4*4 \\ %
    \hline
    
    FC1     & %
    ReLu(Linear(input)) & %
    BS*320 & %
    BS*50 \\ %
    \hline
    
    FC2     & %
    functional.log\_softmax(Linear(input)) & %
    BS*50 & %
    BS*10 \\ %
    \hline
    \hline
\end{tabular}

\end{table}

With those three settings, We present two different experiments, a first one comparing disjoint settings with and without a label for inference. The goal is to bring to light that regularization fails in disjoint settings if the task label is not provided.
Secondly, we experiment with the joint setting, to show that even if the feature extractor only needs to be learned the approach still struggles to learn continually and forget significantly.   %

We present EWC results with diagonal Fisher Matrix \cite{kirkpatrick2017overcoming} and with Kronecker Factorization of the Fisher matrix \cite{Ritter18}. We add an expert model which learned on the full dataset at once and a baseline model who learned continually without any memorization process. All models are trained with stochastic gradient descent with a learning rate of $0.01$ and a momentum of $0.9$. Even if continual learning does not support a-posteriori hyper-parameter selection, for fairness in comparison, the parameter lambda has been tuned. The best lambda upon $[0.1;1;2;5;10;20;100;1000]$ is selected for each model. Then the model is trained on 5 different seeds.

The first experiment (Fig. \ref{fig:2b_Replay:test_label}), shows that \textit{regularization} methods performances are significantly reduced when there is no test label in the disjoint settings. The experiment also shows that without labels for inference the model forgets almost instantaneously the first task when switching to the second one. 
Those results support the proposition \ref{prop:regularization}. Indeed, the low performance of regularization methods without test labels in disjoint settings illustrates the output layer shortcomings in continual learning (task separability problem, Section \ref{sub:practical_examples}).
 
In Experiment 2 (Fig. \ref{fig:2b_Replay:merge}), since the classes are the same in all tasks, only the feature extractor needs to be learned continually. The low performance of the proposed models illustrates the shortcomings in the continual learning of the feature extractor (the latent features problem, Section \ref{sub:practical_examples}).
\begin{figure}[ht]
    \centering
    \includegraphics[width=0.43\linewidth]{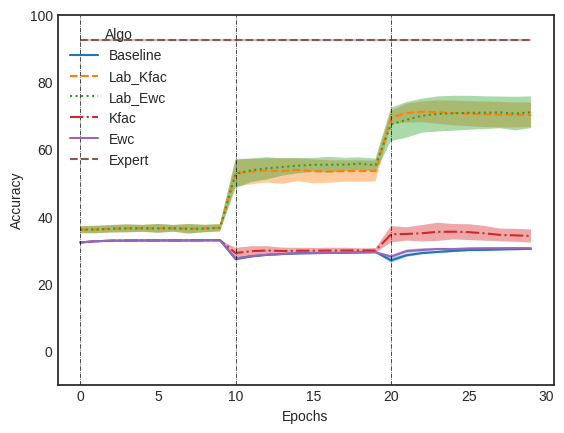}
    ~ 
    \includegraphics[width=0.43\linewidth]{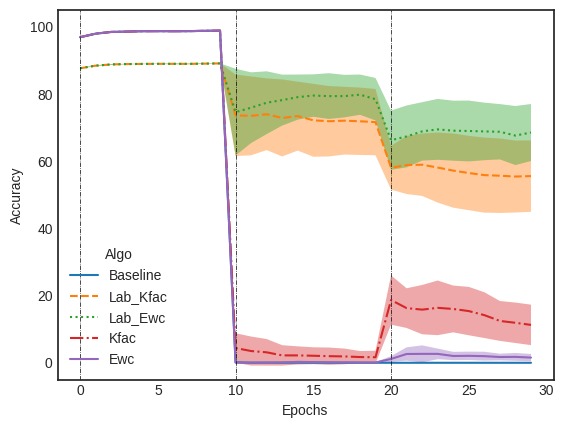}
    \caption[Experiment on disjoint classes without test label vs test label.]{Experiment on disjoint classes without test label vs test label. Left, the mean accuracy of all 3 tasks, vertical dashed line are task transitions. Right, accuracy on the first task. Legends with `Lab\_' indicate experiments with task labels for test. 
    The expert model is trained with i.i.d. data from all task and the baseline model is finetuned on each new task without any continual process.} 
    \label{fig:2b_Replay:test_label}
\end{figure}

\begin{figure}[ht]
    \centering
    \includegraphics[width=0.43\linewidth]{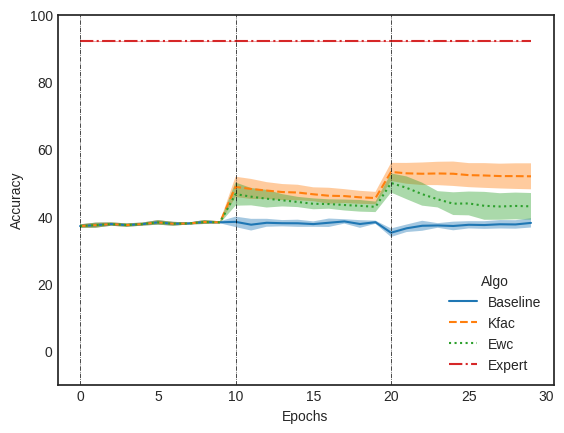}
    ~
    \includegraphics[width=0.43\linewidth]{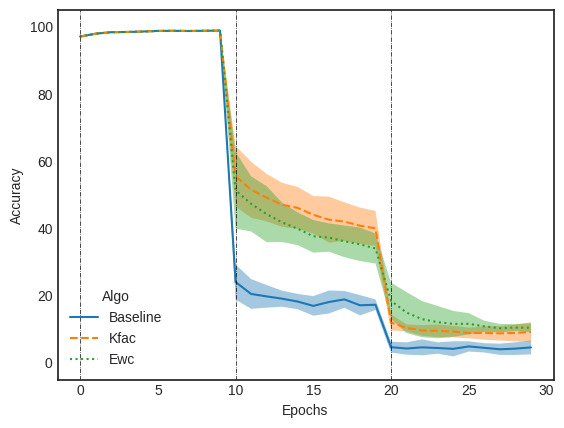}
    \caption[Experiments with joint classes.]{Experiments with joint classes.  Left, the mean accuracy of all 3 tasks, vertical dashed line are task transitions. Right, accuracy on the first task.} 
    \label{fig:2b_Replay:merge}
\end{figure}

These two experiments show that learning continually with regularization is only efficient in the setting with task label and maintains performance on task 0. The two other settings seem to either have interference in the output layer and in the feature extractor.

\section{Discussion}

The goal of this chapter is to propose a theoretical approach to the shortcomings of regularization methods in class-incremental settings. Regularization methods might have great characteristics for continual learning under certain conditions, but it is important to know their limitations to use the best of their capabilities. 
Regularization shortcomings could be offset with a replay method such as \textit{rehearsal} or \textit{generative replay}. \checked{We discuss in this section the potential applications where the shortcomings described have an impact, we also discuss the relationship between learning representation and memorizing. We finish by introducing how the replay methods could overcome the shortcomings described here.}

\subsection{Applications}
\label{sub:applications}

In this section, we point out supplementary shortcomings of regularization in other types of learning situations, namely a classification task with one only class and multi-task continual reinforcement learning. We also use proposition \ref{prop:regularization} for the case of pre-trained models.

\subsubsection{Learning from one class only}

A classification task with only one class might look foolish, however, in a sequence of tasks with varying number of classes, it makes more sense and it seems evident that a CL algorithm should be able to handle this situation.
Nevertheless, a classification neural network needs at least two classes to learn discriminative parameters.
Hence, in a one-class task, the model learns no useful parameters, a regularization term can then a fortiori not protect any knowledge.
As noted in \cite{lesort2018marginal}, the regularization method is not suited for such setting. It is worth noting that in a real-life settings it is mandatory to be able to learn only one concept at a time.

\subsubsection{Multi-task Continual Reinforcement Learning}

Results from Section \ref{sub:prop} can also be generalized to continual multi-tasks reinforcement learning settings \cite{Traore19DisCoRL}.
In this setting, a model has to learn several policies sequentially to solve different tasks.
At test time with no task label, the model needs both to be able to run the policies correctly but to infer which policy to run.
However, since policies are learned separately inferring which one to run is equivalent to a class incremental task. Therefore, following proposition \ref{prop:regularization}, the regularization based method will not be able to learn the implicit classification correctly.
Hence, in continual multi-tasks RL a regularization method alone will fail if task label is not enabled at test time.

\subsection{Using pre-trained models for continual learning}

We showed in Section \ref{sub:prop} that, in a class incremental classification scenario, regularization methods are not sufficient to learn continually. In the case of a pre-trained classification model on $N$ classes that we want to train on new classes without forgetting, if the training data are not available for some reasons, then we don't even have a \textit{regularization term} $\Omega$ to protect some. Following the proposition \ref{prop:regularization} and a fortiori without the regularization term, the model will forget past knowledge while learning new classes. 

Using pre-trained models can be useful to converge faster to a new task solution but it will undoubtedly forget what it has learn previously.

\subsection{Representation Learning and Memorization}

We presented the theoretical impossibility to distinguish past classes from current classes based only on regularization. Primarily, because we are unable to find the good decision boundaries in the output layer and because we can not learn features to disentangle the past from the present.
But also because the model representation overfits the task. The classifier only optimizes the current learning criterion, therefore data representations are restricted to it. Those representations could be insufficient to memorize the learning experiences correctly for future tasks, as in the cat/dog tasks described in the Section \ref{sec:2b_Replay:approach}. %
In any way, to maximize discrimination between tasks with no test label it is mandatory to have good memorization of past tasks. This memorization can be performed either by modelling their data distribution with generative models or samples or by adding surrogate losses that help the model to learn general representations of past tasks.
Memorization is intrinsically linked to representation. Hence, adding surrogate loss to improve the learned representation would a priori improve memorization and consequently continual learning.

\subsection{Toward Replay Methods}

\checked{In this chapter, we assume that a training \say{task label} is available,
indicating each time a drift happens in the data distribution while learning. These settings make learning easier than when the drifts are not signalled in any way. In case of task label unavailability, it is even more important to have a robust and resilient memorization process to detect data distribution drift and to deal with catastrophic forgetting.}

\checked{To be easily applicable in real-life scenarios,  algorithms should eventually have access to task labels for learning but not for inference. A family of algorithms that fit this requirement is the family of replay methods. It gathers rehearsal methods and generative replay described in Chapter \ref{chap:2_CL}.}
\checked{%
 Replay methods emulate the simultaneous apparition of two concepts making it possible to learn good decision boundaries.
 Replay hence enables the possibility to reinterpret data from past tasks with the current state of knowledge and belief. 
 They can learn task where there is only one concept at a time.
 The memorization process is agnostic to the past task learning criterion. It is driven by high-level rules. 
 Moreover, the method is agnostic to the number of tasks, the model solving the current task and the task label at test time. 
 Finally, the memory is easy to control, if in any way we want to forget some knowledge, the algorithms just have to stop replaying them and the model will forget automatically.}

\checked{Some additional assets can be attributed to rehearsal and generative replay separately:}

\begin{itemize}
\item \checked{\textbf{Rehearsal}: As long as the samples are saved in the memory, they can not be forgotten and there is no risk of damaging memories. Rehearsal is very suited for online learning, it can save samples as soon as they have been received. }

\item \checked{\textbf{Generative Replay:} The generative replay uses a learning criterion to learn the data distribution different from the task criterion. Therefore, it enables the possibility of a supplementary generalization than the generalization of the task learning criterion. Moreover, the generative model synthesizes the information and offers an automatic compression process for memorization.}
\end{itemize}

\checked{The replay methods look, therefore, very appealing for research in continual learning. They look adapted, robust and resilient enough to deal with various continual learning settings.}

\section{Conclusion}

Regularization is a widespread method for continual learning.
However, we proved that for class-incremental classification, no regularization method can learn alone to discriminate classes from different tasks.
At test time, this shortcoming makes them dependent on the task label for prediction. This need for supervision during inference restricts significantly the application scenario possible for regularization methods alone.

The class-incremental scenario is a specific benchmark measuring the ability of algorithms to learn sequentially different classes. However, being unable to deal with this setting implies that in a more complex learning environment, all sub-tasks interpretable as class-incremental will be ignored by the algorithm. 

\checked{This shortcoming in regularization methods is one of the motivations to study replay methods for continual learning in the following chapters of this thesis. 
In particular, in the next chapter, we will study the capacity of generative models to learn continually. This study allows to determine the best conditions to use a generative model for generative replay.}

\newpage
\chapter{Learning Continually a Generative Models}
\label{chap:3_CL_GM}

\checked{In the previous chapter, we have highlighted the shortcomings of regularization methods for continual learning in the case of disjoint classification. We also stressed the benefits of replay for continual learning. The use of rehearsal is a known method for continual learning and the recent improvements of generative models motivate us to explore generative replay methods. But first, we wanted to evaluate the capacity of generative models to learn continually. In this chapter, we lead an empirical study of generative models in the context of disjoint task generation.}

\checked{This work has been realized in collaboration with Hugo Caselles{-}Dupr\'{e} and led to the publication \say{Generative Models from the perspective of Continual Learning} published at IJCNN 2019 \cite{lesort2018generative}. The original article has been slightly modified and extended to better fit the thesis and include figures and results cut by article size restriction. }\checked{Moreover, we incorporate some results of the paper \say{Training Discriminative Models to Evaluate Generative Ones} \cite{lesort2018training} in the background section. This second paper has been published at ICANN 20019.}

\section{Introduction / Motivation}

In this chapter, we focus on generative models in Continual Learning scenarios. Previous work on CL has mainly focused on classification tasks \citep{kirkpatrick2017overcoming, rebuffi2017icarl,shin2017continual, schwarz2018progress} with approaches such as \textit{regularization}, \textit{rehearsal} and \textit{architectural} strategies, as described in Chapter \ref{chap:2_CL}.
However, discriminative and generative models strongly differ in their architecture and learning objective (classify images vs generating images). Several methods developed for discriminative models are thus not directly extendable to the generative setting. 
Once they are trained, generative models can be used as memory of the past for learning continually especially in reinforcement learning and classification. For example in  \citep{shin2017continual}, successful CL strategies with generative models have been used, via sample generation as detailed in the next section, to continually train discriminative models.
Hence, studying the viability and success/failure modes of CL strategies for generative models is an important step towards a better understanding of generative models but also Continual Learning as a whole.

We conduct a comparative study of generative models with different CL strategies. 
In our experiments, we sequentially learn generation tasks. We perform ten disjoint tasks, using commonly used benchmarks for CL: MNIST \citep{lecun1998gradient}, Fashion MNIST \citep{Xiao2017} and CIFAR10 \citep{Krizhevsky09}. In each task, the model gets a training set from one new class, and should learn to generate data from this class without forgetting what it learned in previous tasks, see Fig. \ref{fig:3_CL_GM:task_explained} for an example with disjoint tasks on MNIST.

\begin{figure}
    \centering
    \begin{subfigure}[b]{0.9\textwidth}
        \includegraphics[width=\textwidth]{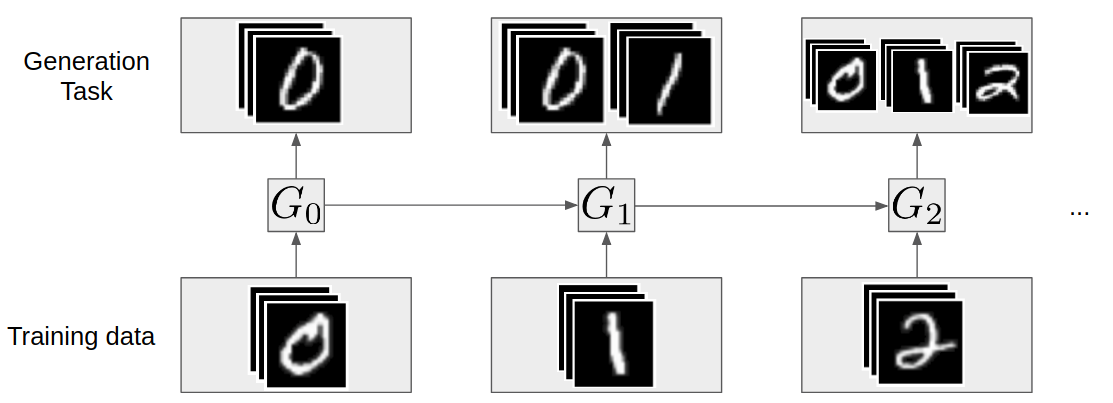}

    \end{subfigure}
    \caption[Illustration of the disjoint setting considered ]{Illustration of the disjoint setting considered. At task $i$ the training set includes images belonging to category $i$, and the task is to generate samples from all previously seen categories. Here MNIST is used as a visual example, we experiment in the same way Fashion MNIST and CIFAR10.}
    \label{fig:3_CL_GM:task_explained}
\end{figure}

We evaluate several generative models: Variational Auto-Encoders (VAEs)\citep{kingma2013auto}, Generative Adversarial Networks (GANs) \citep{goodfellow2014generative}, their conditional variant (CVAE ans CGAN), Wasserstein GANs (WGANs)\citep{arjovsky2017wasserstein} and Wasserstein GANs Gradient Penalty (WGAN-GP)\citep{gulrajani2017improved}.%
\checked{We introduced those models in Section \ref{sub:1b_ML:GM}.}

 We compare results on approaches taken from CL in a classification setting: \textit{finetuning}, \textit{rehearsal}, \textit{regularization} and \textit{generative replay}. \textit{Generative replay} consists of using generated samples to maintain knowledge from previous tasks. All CL approaches are applicable to both variational and adversarial frameworks. We evaluate with two quantitative metrics, Fr\'echet Inception Distance \citep{heusel2017gans} and fitting capacity (FiC) \citep{lesort2018training}, as well as visualization.
Also, we discuss the data availability and scalability of CL strategies.

Our contributions are:

\begin{itemize}
    \item We propose the Fitting Capacity, an evaluation method for generative models.
    \item Evaluating a wide range of generative models in a Continual Learning setting.
    \item Highlight success/failure modes of combinations of generative models and CL approaches.
    \item Comparing, in a CL setting, two evaluation metrics of generative models.
\end{itemize}

We describe related work in Section \ref{sec:3_CL_GM:background_CL_GM}, and our approach in Section \ref{sec:3_CL_GM:approach_CL_GM}. We explain the experimental setup that implements our approach in Section \ref{sec:3_CL_GM:exp_CL_GM}. Finally, we present our results and discussion in Section \ref{sec:3_CL_GM:results_CL_GM} and \ref{sec:3_CL_GM:discussion_CL_GM}, before concluding in Section \ref{sec:3_CL_GM:ccl_CL_GM}.

\section{Background}
\label{sec:3_CL_GM:background_CL_GM}

\checked{In this section, we will re-introduce the Fitting Capacity with some results from the original paper. The fitting capacity is a contribution from this thesis, we proposed to have a general evaluation of generative models.
Then, we will present a brief state of the art on continual learning and generative models.}

\subsection{Fitting Capacity}
\label{sub:3_CL_GM:FC}

\checked{As presented in Chapter \ref{chap:1b_ML}, the Fitting Capacity (FiC) is a method to evaluate generative models. It measures the accuracy of a classifier trained on generated data to estimate the quality of the generator.
It does not directly assess the realistic characteristics of the generated data but rather if their content and variability contain enough information to classify real data. 
This method makes it possible to take into account the complex characteristics of the generated samples and not only the distribution of their features. In this background section, we present a summary of \citep{lesort2018training} results to illustrate FiC evaluation.}

\medskip

\checked{The process for the fitting capacity is the following:}

\begin{enumerate}
\item Train a conditional generative model over a train set
\item Generate a dataset $D_{gen}$
\item Train a discriminative model (a classifier) over $D_{gen}$
\item Iterate the process for several generative models and compare the accuracy of the classifiers on the test set.
\end{enumerate}

\begin{figure}[ht]
\centering
\includegraphics[width=0.5\textwidth]{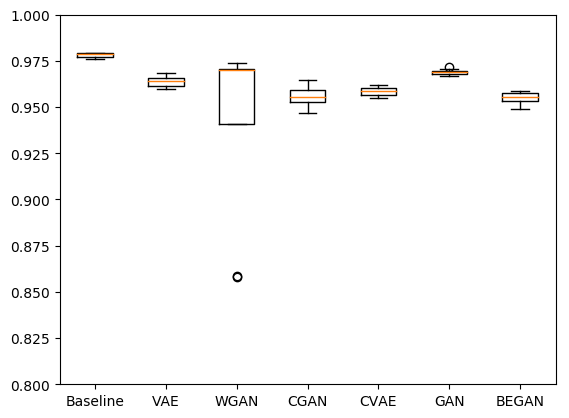}
\caption[Fitting capacity evaluation of various generative models on MNIST.]{Fitting capacity evaluation of various generative models on MNIST. Each model has been trained on 8 different seeds to create plot bars. NB: the figure is zoomed to better appreciate the difference between models but consequently an outlier of the GAN model is hidden and is around 0.1.}
\label{fig:3_CL_GR:barplot_FiC_MNIST}
\end{figure} 

\begin{figure}[ht]
\centering
\includegraphics[width=0.5\textwidth]{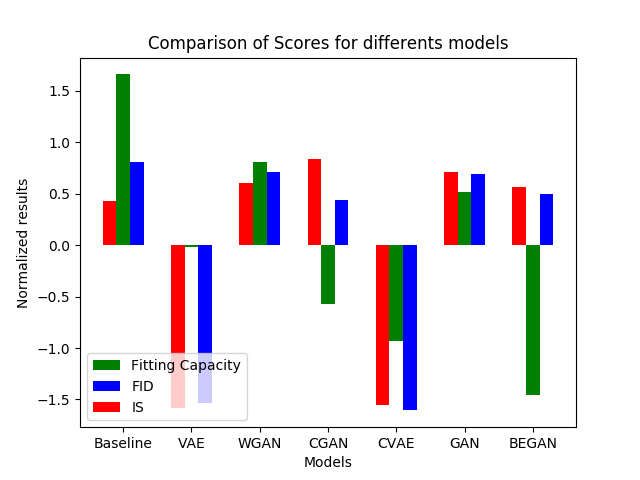}
\caption[Comparison between generative model evaluations.]{Comparison between normalized evaluation with inception score (IS), Frechet Inception Distance (FID) and Fitting Capacity (FiC). The normalization is achieved such as the sum of each metric over all the generative models is 0 with 1 standard deviation.}
\label{fig:3_CL_GR:GM_Eval_MNIST}
\end{figure} 

The method makes it possible to compare generative models trained on similar datasets. In Figure \ref{fig:3_CL_GR:barplot_FiC_MNIST}, we can find the FiC of various generative models trained on MNIST and appreciate the difference among models.
We can also evaluate the relative difference between metrics in Figure \ref{fig:3_CL_GR:GM_Eval_MNIST}. This figure highlights the differences by normalizing performance on the best models. The baseline is a classifier trained on real data and not on generated ones.
We can see that the FiC discriminates real data the most from generated ones and that FiC results do not correlate with other metrics results.

The Fitting Capacity is time-consuming since it necessitates to train a classifier for evaluation. Nevertheless, we believe that the fitting capacity is a good evaluation to estimate the intrinsic characteristics of generated data.

The outcome of the  fitting capacity evaluation of generative models in article \cite{lesort2018training}, was that GAN or WGAN models were producing the best samples but that VAE and CGAN were more stable and had a better mean accuracy. Moreover, conclusion was that inception score and Frechet inception distance do not make possible to predict if the data is good enough for downstream tasks. 
Similar evaluation methods also came out with that conclusion \cite{santurkar2017classification, shmelkov2018good, Dat2019Classifier, ravuri2019classification}.

\medskip

In this chapter, we will use the Fitting Capacity to evaluate generative models trained continually.

\subsection{State of the art}
\label{sub:3_CL_GM:sota_CL_GM}

As said in the introduction of this chapter, discriminative and generative models do not share the same learning objective (classify data vs generate data) and 
architecture (top down vs bottom up). For this reason, CL strategies for discriminative models are usually not directly transferable to generative models. Continual Learning in the context of generative models remains largely unexplored compared to CL for discriminative models. %
\checked{In this section, we present the state of the art of continual learning specifically for generative models.}

\medskip

Among notable previous work, \cite{seff2017continual} successfully apply EWC on the generator of Conditional-GANs (CGANS), after observing that applying the same regularization scheme to a classic GAN leads to catastrophic forgetting. However, their work is based on a scenario where two classes are presented first, and then unique classes come sequentially, e.g the first task is composed of 0 and 1 digits of MNIST dataset, and then is presented with only one digit at a time on the following tasks. This is likely due to the failure of CGANs on single digits, which we observe in our experiments. %
Moreover, the method is only shown to work on CGANs.
Another method for generative Continual Learning is Variational Continual Learning (VCL) \citep{nguyen2017variational}, which adapts variational inference to a continual setting. They exploit the online update from one task to another %
inspired from Bayes' rule.
They successfully experiment on a single-task scenario.  However, they experiment only on VAEs.
Moreover, since they use a multi-head architecture, they use specific weights for each task, which requires task index for inference.
A second method experimented on VAEs is to use a student-teacher method where the student learns the current task while the teacher retains knowledge \citep{ramapuram2017lifelong}. %
Finally, VASE \citep{achille2018life} is a third method, also experimented only on VAEs, which allocates spare representational capacity to new knowledge, while protecting previously learned representations from catastrophic forgetting by using snapshots (i.e. weights) of previous model.

\medskip

A different approach, introduced by \cite{shin2017continual} is an adaptation of the \textit{generative replay} method. It is applicable to both adversarial and variational frameworks. 
It uses two generative models: one which acts as a memory, capable of generating all past tasks, and one that learns to generate data from all past tasks and the current task.
It has mainly been used as a method for Continual Learning of discriminative models \citep{shin2017continual, Shah18}.
Recently, \cite{wu2018memory} authors have developed a similar approach called Memory Replay GANs, where they use Generative Replay combined to replay alignment, a distillation scheme that transfers previous knowledge from a conditional generator to the current one. However, they note that this method leads to mode collapse because it could favor learning to generate few class instances rather than a wider range of class instances. 

\section{Approach}
\label{sec:3_CL_GM:approach_CL_GM}

Typical previous works on Continual Learning for generative models focus on presenting a novel CL technique and comparing it to previous approaches, on one type of generative model (e.g. GAN or VAE).

On the contrary, we focus on looking for the best generative model and CL strategy association.
For now, empirical evaluation remains the only way to find the best performing combinations. Hence, 
we compare several existing CL strategies on a wide variety of generative models with the objective of finding the most suited generative model for Continual Learning.

\medskip 

\checked{In this chapter, we use two evaluation methods: the fitting capacity (FiC) \citep{lesort2018training} and the FID\citep{heusel2017gans}.}
We believe that using these two metrics is complementary. FID is a commonly used metric based solely on the distribution of images features. In order to have a complementary evaluation,
we use the fitting capacity, which evaluate samples on a classification criterion rather than features distribution.
For unconditional models, we used an adaptation of the FiC where data are labelled by an expert network trained on the dataset.

For all the progress made in quantitative metrics for evaluating generative models
\citep{borji2018pros}, qualitative evaluation remains a widely used and informative method. While visualizing samples provides a instantaneous detection of failure, it does not provide a way to compare two well-performing models. It is not a rigorous evaluation and it may be misleading when evaluating sample variability.

\section{Experiments}
\label{sec:3_CL_GM:exp_CL_GM}

\checked{In this section, we present the experiments conducted to evaluate the generative models and different training strategies.}

\subsection{Experimental setup}
\label{sub:setup_CL_GM}

We now describe our experimental setup: data, tasks, and evaluated approaches. 
The code is available online\footnote{https://github.com/TLESORT/Generative\_Continual\_Learning}.

Our main experiments use 10 sequential tasks created using the MNIST, Fashion MNIST and CIFAR10 datasets. \checked{This setting is referenced to as disjoint classes setting or class-incremental setting.}
For each dataset, we define 10 sequential tasks, one task consists of learning to generate a new class and all the previous ones (See Fig. \ref{fig:3_CL_GM:task_explained} for an example on MNIST). Both evaluations, FID and fitting capacity of generative models, are computed at the end of each task.

\subsection{Strategies for continual learning}
\label{sub:3_CL_GM:baseline_CL_GM}

We focus on strategies that are usable in both the variational and adversarial frameworks. We use 3 different strategies for Continual Learning of generative models, that we compare to 3 baselines. Our experiments are done on 8 seeds with 50 epochs per tasks for MNIST and Fashion MNIST using Adam \citep{kingma2014adam} for optimization.
For CIFAR10, we experimented with the CL strategy that performed the best on the previous datatsets.

The first baseline is Fine-tuning, which consists of classical training, ignoring catastrophic forgetting. It is essentially a lower bound of the performance. Our other baselines are two upper bounds: Upperbound Data, for which one generative model is trained on joint data from all past tasks, and Upperbound Model, for which one separate generator is trained for each task.

For Continual Learning strategies, we first use a vanilla Rehearsal method, where we keep a fixed number of samples of each observed task, and add those samples to the training set of the current generative model. We balance the resulting dataset by copying the saved samples so that each class has the same number of samples. The number of samples selected, here $10$, is motivated by the results in Fig. \ref{fig:3_CL_GR:Classifier-MNIST} and \ref{fig:3_CL_GR:Classifier-fashion}, where we show that $10$ samples per class is enough to get a satisfactory but not maximal validation accuracy for a classification task on MNIST and Fashion MNIST.

\begin{figure}[ht]
       \centering
    \begin{subfigure}{0.45\textwidth}
        \includegraphics[width=\textwidth]{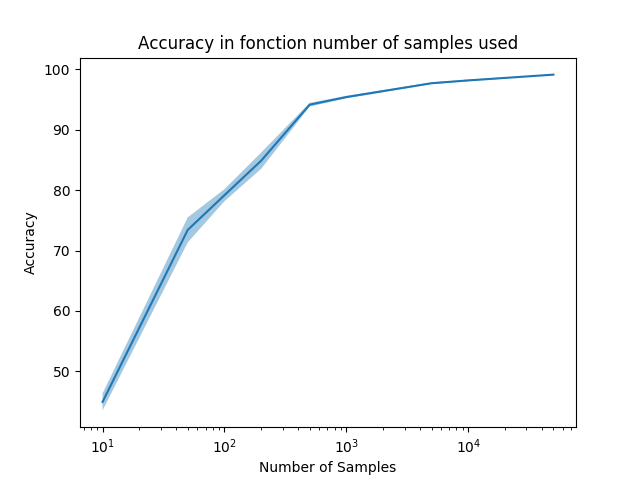}
        \caption{MNIST}
        \label{fig:3_CL_GR:Classifier-MNIST}
    \end{subfigure}
   \begin{subfigure}{0.45\textwidth}
       \includegraphics[width=\textwidth]{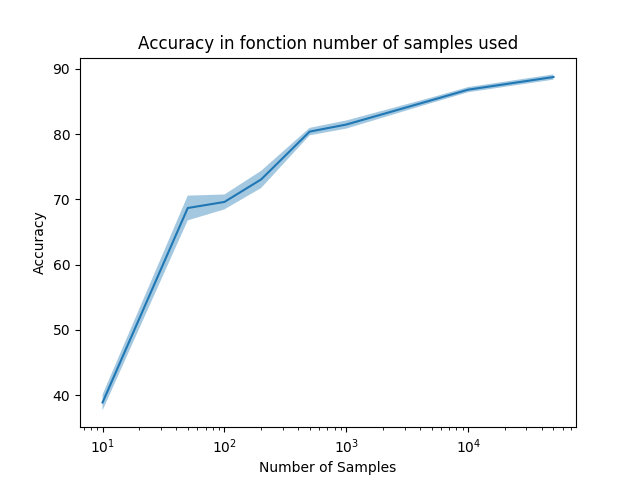}
        \caption{fashion-MNIST}
        \label{fig:3_CL_GR:Classifier-fashion}
   \end{subfigure}
\caption[Test set classification accuracy on MNIST and Fashion-MNIST]{Test set classification accuracy as a function of number of training samples, on MNIST (Left) and Fashion-MNIST (Right). Those figures make possible to estimate the minimal number of samples needed to achieve a high test accuracy. Furthermore, by comparing against the fitting capacity we can estimate how many different images of the dataset a generator can produce.}
\end{figure}

\checked{As they use the same test set, the fitting capacity (FiC) and the original accuracy with $10$ samples per task can be compared. A higher FiC shows that the memory prevents catastrophic forgetting. Equal FiC means overfitting of the saved samples and lower FiC means that the generator failed to even memorize these samples.}

We also experiment with EWC. We follow
the method described by \cite{seff2017continual} for GANs, i.e. the penalty is applied only on the generator's weights
, and for VAEs we apply the penalty on both the encoder and decoder.
As tasks are sequentially presented, we choose to update the diagonal of the Fisher information matrix by cumulatively adding the new one to the previous one. 
The last method is Generative Replay, described in Section \ref{sub:3_CL_GM:sota_CL_GM}.
Generative replay is a dual-model approach where a \say{frozen} generative model $G_{t-1}$ is used to sample from previously learned distributions and a \say{current} generative model $G_{t}$ is used to learn the current distribution and $G_{t-1}$ distribution. When a task is over, $G_{t}$ becomes the \say{frozen} model for training $G_{t+1}$. 

\checked{The classification model used for experiments with MNIST and Fashion-MNIST is described in table \ref{tab:3_CL_GM:hyperparams}. The model is slightly different from the one in Chapter \ref{chap:2b_Replay}, indeed we added the batch-normalization layer \cite{Ioffe2015batch} to boost the model training speed and performance.}

\begin{table}[ht]
\centering

  \caption[Classifier architecture for Fitting Capacity computation]{Classifier architecture, convolution have 5*5 kernel size, maxpool have 2*2 kernel size. Parameters not mentioned are default parameters in Pytorch library \cite{NEURIPS2019_9015} (in torch.nn). BS is for batch size, which is 64. All layers are initialized with Xavier init method \cite{glorot2010understanding}.}
  \label{tab:3_CL_GM:hyperparams}
  \begin{tabular}{ccccccc}
    \hline 
    Layer Name &
    Layer Type &
    Input Size&
    Output Size\\
    \hline
    
    Conv1     & %
    ReLu(MaxPool2d(Conv2d(input))) & %
    BS*1*28*28 & %
    BS*10*12*12 \\ %
    \hline
    
    Conv2     & %
    ReLu(MaxPool2d(Conv2d(Dropout(input,p=0.5)))) & %
    BS*10*14*14  & %
    BS*20*4*4 \\ %
    \hline
    
    FC1     & %
    ReLu(Linear(BatchNorm1d(input))) & %
    BS*320 & %
    BS*50 \\ %
    \hline
    
    FC2     & %
    functional.log\_softmax(Linear(input)) & %
    BS*50 & %
    BS*10 \\ %
    \hline
    \hline
\end{tabular}

\end{table}

\checked{If we describe the experiments as proposed in the framework presented in Chapter \ref{chap:2_CL}, then we are in an unsupervised learning setting, multi-task scenario, with 10 disjoint tasks with iid data, with an integer oracle task label for training but not for testing (learning labels). The content update of this setting, is a new concepts type (NC).}

\checked{Concerning the approaches experimented, the growth of memory is less than linear and the growth of computation is bounded by liner growth.}

\checked{For memory, the generative replay has only one generative models as a memory and the regularization model only retains a set of weights and a Fisher matrix. It makes the growth in memory less than linear. For regularization, the sample number grows linearly but the memory of the model itself stays constant. Therefore the growth in memory is also less than linear.}

\checked{The model upperbound method, which saves one mode per task has a linear growth in memory, the upperbound data has also a linear growth since it saves all data (and because tasks have the same size).}

\medskip
\checked{For computation, the growth of computation of generative replay is close to linear. The number of samples to generate increases linearly with the number of tasks. For rehearsal, the growth is also almost linear. In regularization, the growth is almost constant after the second task.}

\checked{The upperbound model has constant computation and the upperbound data has linear growth of computation since the number of data grows linearly.}

\section{Results}
\label{sec:3_CL_GM:results_CL_GM}

\begin{figure}[ht]
    \centering
    \begin{subfigure}[b]{0.45\textwidth}
        \includegraphics[width=\textwidth]{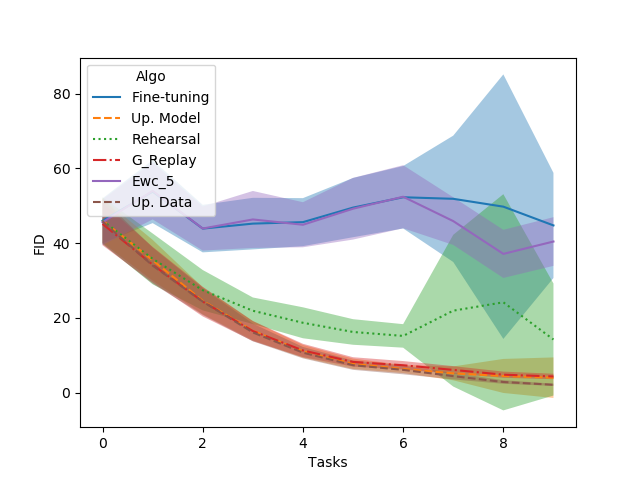}
    \end{subfigure}
       \begin{subfigure}[b]{0.45\textwidth}
        \includegraphics[width=\textwidth]{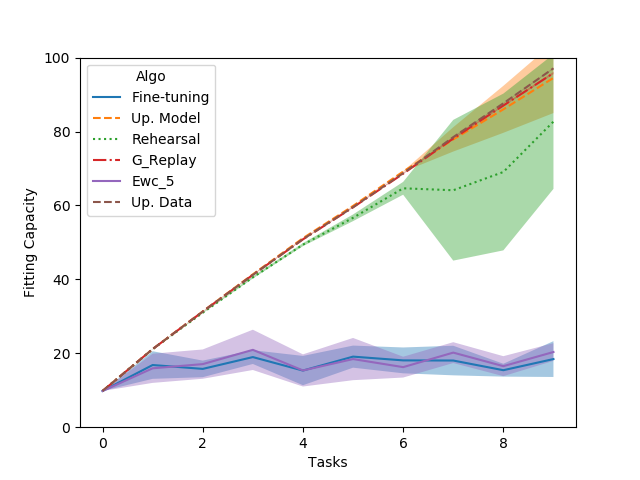}
    \end{subfigure}
   \caption[Comparison between FID results and Fitting Capacity results]{Comparison, averaged over 8 seeds, between FID results (left, lower is better) and Fitting Capacity results (right, higher is better) with GAN trained on MNIST.}
   \label{fig:3_CL_GM:FID_FT_figure}
\end{figure} 

The figures we report show the evolution of the metrics through tasks. Both FID and FiC are computed on the test set. A well-performing model should increase its FiC and decrease its FID. We observe a strong correlation between the FiC and FID (see Fig. \ref{fig:3_CL_GM:FID_FT_figure} for an example on GAN for MNIST).

Nevertheless, FiC results are more stable: over the 8 random seeds we used, the standard deviations are less significant than in the FID results. 
For that reason, we focus our interpretation on the FiC results.

\begin{figure}[!ht]
    \centering
    \begin{subfigure}[b]{0.9\textwidth}
        \includegraphics[width=\textwidth]{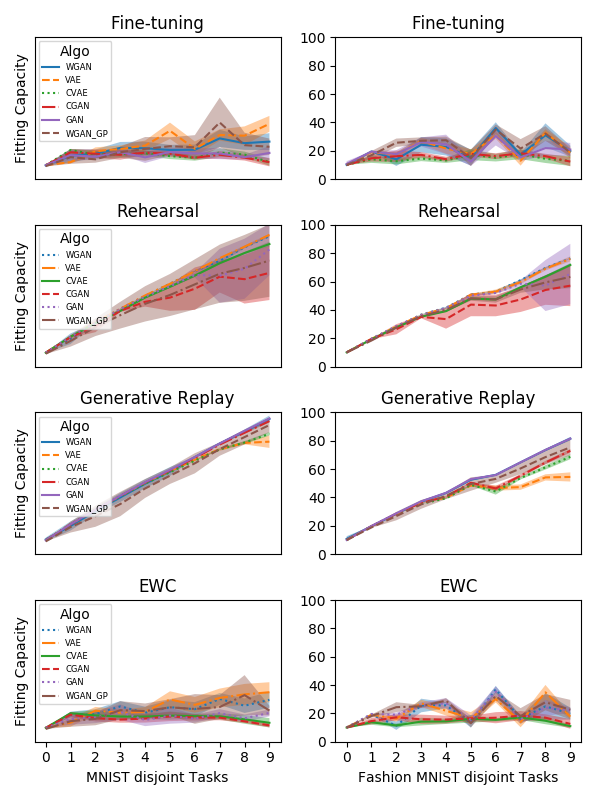}
    \end{subfigure}
   \caption[Fitting Capacity metric evaluation of VAE, CVAE, GAN, CGAN and WGAN]{Means and standard deviations over 8 seeds of Fitting Capacity metric evaluation of VAE, CVAE, GAN, CGAN and WGAN. The four considered CL strategies are: Fine Tuning, Generative Replay, Rehearsal and EWC. The setting is 10 disjoint tasks on MNIST and Fashion MNIST.}
   \label{fig:3_CL_GR:accuracy_per_task}
\end{figure}

\begin{table*}
\centering

  \caption[Results table on the 10 disjoint tasks setting]{Mean and standard deviations for Fitting Capacity (in \%) metric evaluation on last task of 10 disjoint tasks setting, on MNIST and Fashion MNIST, over 8 seeds. }
  
\resizebox{\textwidth}{!}{
  \begin{tabular}{cccccccc}
    \hline 
    Strategy&Dataset&GAN&CGAN&WGAN&WGAN-GP&VAE&CVAE\\ 
    \hline
    Fine-tuning&
    MNIST&
    $18.43$\mbox{\scriptsize{$\pm4.85$}}& %
    $11.93$\mbox{\scriptsize{$\pm2.97$}}&
    $23.17$\mbox{\scriptsize{$\pm5.66$}}&
    $22.79$\mbox{\scriptsize{$\pm5.75$}}&
    $38.98$\mbox{\scriptsize{$\pm5.57$}}&
    $11.96$\mbox{\scriptsize{$\pm2.56$}}\\ %
    
    \rowcolor{gray!20}EWC&
    -&
    $20.34$\mbox{\scriptsize{$\pm2.39$}}& %
    $11.53$\mbox{\scriptsize{$\pm1.42$}}&
    $29.57$\mbox{\scriptsize{$\pm5.59$}}&
    $22.00$\mbox{\scriptsize{$\pm3.39$}}&
    $34.93$\mbox{\scriptsize{$\pm7.06$}}&
    $13.37$\mbox{\scriptsize{$\pm3.28$}}\\ %
    
    Rehearsal&
    - &
    $82.69$\mbox{\scriptsize{$\pm18.21$}}& %
    $66.14$\mbox{\scriptsize{$\pm19.2$}}& %
    $92.05$\mbox{\scriptsize{$\pm0.64$}}& %
    $74.79$\mbox{\scriptsize{$\pm25.25$}}& %
    $92.99$\mbox{\scriptsize{$\pm0.64$}}& %
    $86.47$\mbox{\scriptsize{$\pm1.69$}}\\ %
    
    \rowcolor{gray!20}Generative Replay&
    -&
    $\textbf{95.81}$\mbox{\scriptsize{$\pm0.31$}}& %
    $93.89$\mbox{\scriptsize{$\pm0.35$}}&
    $95.41$\mbox{\scriptsize{$\pm2.41$}}&
    $91.12$\mbox{\scriptsize{$\pm5.09$}}&
    $79.38$\mbox{\scriptsize{$\pm4.40$}}&
    $84.95$\mbox{\scriptsize{$\pm1.24$}}\\ %
    
    Upperbound Model&
    -&
    $94.50$\mbox{\scriptsize{$\pm9.51$}}& %
    $96.84$\mbox{\scriptsize{$\pm3.22$}}&
    $95.72$\mbox{\scriptsize{$\pm6.93$}}&
    $79.41$\mbox{\scriptsize{$\pm27.85$}}&
    $97.82$\mbox{\scriptsize{$\pm0.17$}}&
    $97.89$\mbox{\scriptsize{$\pm0.12$}}\\ %
    
    \rowcolor{gray!20}Upperbound Data&
    -&
    $97.10$\mbox{\scriptsize{$\pm0.13$}}& %
    $96.65$\mbox{\scriptsize{$\pm0.21$}}&
    $96.76$\mbox{\scriptsize{$\pm0.29$}}&
    $84.79$\mbox{\scriptsize{$\pm27.76$}}&
    $96.88$\mbox{\scriptsize{$\pm0.27$}}&
    $96.17$\mbox{\scriptsize{$\pm0.19$}}\\ %
    
    \hline
    Fine-tuning&
    Fashion MNIST &
    $20.82$\mbox{\scriptsize{$\pm4.69$}}& %
    $12.30$\mbox{\scriptsize{$\pm3.33$}}&
    $19.68$\mbox{\scriptsize{$\pm3.92$}}&
    $18.75$\mbox{\scriptsize{$\pm2.58$}}&
    $18.60$\mbox{\scriptsize{$\pm4.24$}}&
    $12.82$\mbox{\scriptsize{$\pm3.55$}}\\ %
    
    \rowcolor{gray!20}EWC&
    -&
    $22.22$\mbox{\scriptsize{$\pm2.03$}}& %
    $12.58$\mbox{\scriptsize{$\pm3.48$}}&
    $19.81$\mbox{\scriptsize{$\pm4.18$}}&
    $22.63$\mbox{\scriptsize{$\pm6.91$}}&
    $17.70$\mbox{\scriptsize{$\pm1.83$}}&
    $11.00$\mbox{\scriptsize{$\pm1.16$}}\\ %
    
    Rehearsal&
    - &
    $65.34$\mbox{\scriptsize{$\pm21.3$}}& %
    $57.12$\mbox{\scriptsize{$\pm14.4$}}& %
    $76.32$\mbox{\scriptsize{$\pm0.33$}}& %
    $63.28$\mbox{\scriptsize{$\pm7.9$}}& %
    $76.03$\mbox{\scriptsize{$\pm1.77$}}& %
    $71.73$\mbox{\scriptsize{$\pm1.29$}}\\ %
    
    \rowcolor{gray!20}Generative Replay&
    - &
    $\textbf{81.52}$\mbox{\scriptsize{$\pm0.87$}}& %
    $72.98$\mbox{\scriptsize{$\pm1.22$}}&
    $81.50$\mbox{\scriptsize{$\pm1.26$}}&
    $75.37$\mbox{\scriptsize{$\pm5.49$}}&
    $54.49$\mbox{\scriptsize{$\pm3.24$}}&
    $68.70$\mbox{\scriptsize{$\pm1.71$}}\\ %
    
    Upperbound Model&
    -&
    $77.93$\mbox{\scriptsize{$\pm15.07$}}& %
    $80.96$\mbox{\scriptsize{$\pm0.69$}}&
    $73.20$\mbox{\scriptsize{$\pm5.63$}}&
    $65.5$\mbox{\scriptsize{$\pm2.69$}}&
    $78.64$\mbox{\scriptsize{$\pm1.36$}}&
    $79.15$\mbox{\scriptsize{$\pm0.96$}}\\ %
    
    \rowcolor{gray!20}Upperbound Data&
    -&
    $83.27$\mbox{\scriptsize{$\pm0.41$}}& %
    $80.09$\mbox{\scriptsize{$\pm0.94$}}&
    $83.29$\mbox{\scriptsize{$\pm0.52$}}&
    $81.5$\mbox{\scriptsize{$\pm0.50$}}&
    $80.21$\mbox{\scriptsize{$\pm0.79$}}&
    $79.51$\mbox{\scriptsize{$\pm0.55$}}\\ %

    \hline
\end{tabular}}
\label{tab:3_CL_GM:results}
\end{table*}

\clearpage

\subsection{MNIST and Fashion MNIST results}
\label{sub:mnist_CL_GM}

\subsubsection{Main results}
Our main results with fitting capacity (FiC) are displayed in Fig. \ref{fig:3_CL_GR:accuracy_per_task} and Table \ref{tab:3_CL_GM:results}. 
The best combination was Generative Replay + GAN with a mean FiC of $95.81\%$ on MNIST and $81.52\%$ on Fashion MNIST.
We observe that, for the adversarial framework, Generative Replay outperforms other approaches by a significant margin. However, for the variational framework, the Rehearsal approach was the best performing.
The Rehearsal approach worked quite well but is unsatisfactory for CGAN and WGAN-GP.
Indeed, the FiC is lower than the accuracy of a classifier trained on 10 samples per classes (see Fig. \ref{fig:3_CL_GR:Classifier-MNIST} and \ref{fig:3_CL_GR:Classifier-fashion}).

\medskip

In our setting, EWC is not able to overcome catastrophic forgetting and performs as well as the naive Fine-tuning baseline
which is contradictory with the results of \cite{seff2017continual} who found EWC successful in a slightly different setting. We replicated their result in a setting where there are two classes by tasks, showing the strong effect of task definition (Illustration Figure \ref{fig:3_CL_GM:cgan_numtask5_sample}). %

In \cite{seff2017continual} authors already found that EWC did not work with non-conditional models but showed successful results with conditional models (i.e. CGANs). The difference comes from the experimental setting. In \cite{seff2017continual}, the training sequence starts by a task with two classes. Hence, when CGAN is trained it is possible for the Fisher Matrix to understand the influence of the class-index input vector $c$. In our setting, since there is only one class at the first task, the Fisher matrix does not capture the importance of the class-index input vector $c$. Hence, as for non-conditional models, the Fisher Matrix is not able to protect weights appropriately and at the end of the second task the model has forgotten the first task. Moreover, since the generator forgot what it learned at the first task, it is only capable of generating samples of only one class. Then, the Fisher Matrix will still not capture the influence of $c$ until the end of the sequence.

Moreover, we show that even by starting with two classes, when there is only one class for the second task, the Fisher matrix is not able to protect the class from the second task in the third task (see Figure \ref{fig:3_CL_GM:seffbar}).

\begin{figure}
    \centering
    \includegraphics[width=0.25\textwidth]{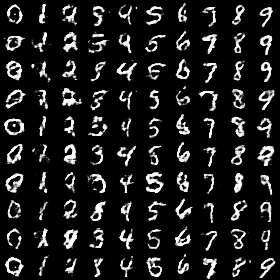}
    \caption[EWC results with CGAN]{CGAN augmented with EWC. MNIST samples after 5 sequential tasks of 2 digits each. Catastrophic forgetting is avoided.}
    \label{fig:3_CL_GM:cgan_numtask5_sample}
\end{figure}

\begin{figure}
    \centering
    \begin{subfigure}[b]{0.30\textwidth}
        \includegraphics[width=\textwidth]{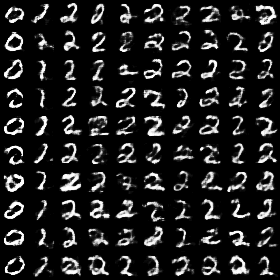}
        \caption{Task 2}
    \end{subfigure}
    \begin{subfigure}[b]{0.30\textwidth}
        \includegraphics[width=\textwidth]{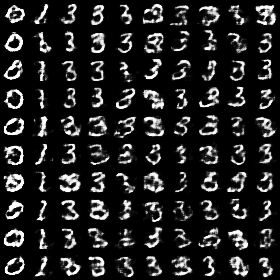}
        \caption{Task 3}
    \end{subfigure}
    \begin{subfigure}[b]{0.30\textwidth}
        \includegraphics[width=\textwidth]{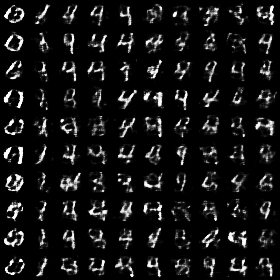}
        \caption{Task 4}
    \end{subfigure}
    
    \caption[Reproduction of EWC experiments from \citep{seff2017continual}]{Reproduction of EWC experiments from \citep{seff2017continual} with four tasks. First task with 0 and 1 digits, then digits of 2 for task 2, digits of 3 for task 3 etc.
    The first task results are not shown in the figure but the generated images where accurate 0 and 1 as expected.
     When task contains only one class, the Fisher information matrix cannot capture the importance of the class-index input vector because it is always fixed to one class. This problem makes the learning setting similar to a non-conditional models one which is known to not work \citep{seff2017continual}. As a consequence 0 and 1 are well protected when following classes are not. Samples to illustrate that the method works if there are several classes at each tasks are in Fig.~\ref{fig:3_CL_GM:cgan_numtask5_sample}}
    \label{fig:3_CL_GM:seffbar}
\end{figure}

\bigskip

Our results do not highlight a clear distinction between conditional and unconditional models. However, adversarial methods perform significantly better than variational methods. GANs variants are able to produce better, sharper quality and variety of samples. 
Hence, adversarial methods seem more viable for CL.
Samples for each model can be visualized on figures \ref{fig:3_CL_GM:samples_mnist} and figure \ref{fig:3_CL_GM:samples_fashion}.
We can link the accuracy from \ref{fig:3_CL_GR:Classifier-MNIST} and \ref{fig:3_CL_GR:Classifier-fashion} to the fitting capacity results. As an example,  we can estimate that GAN with Generative Replay is equivalent for both datasets to a memory of approximately $100$ samples per class.

\begin{figure}[ht]
    \centering
    \includegraphics[width=1\textwidth]{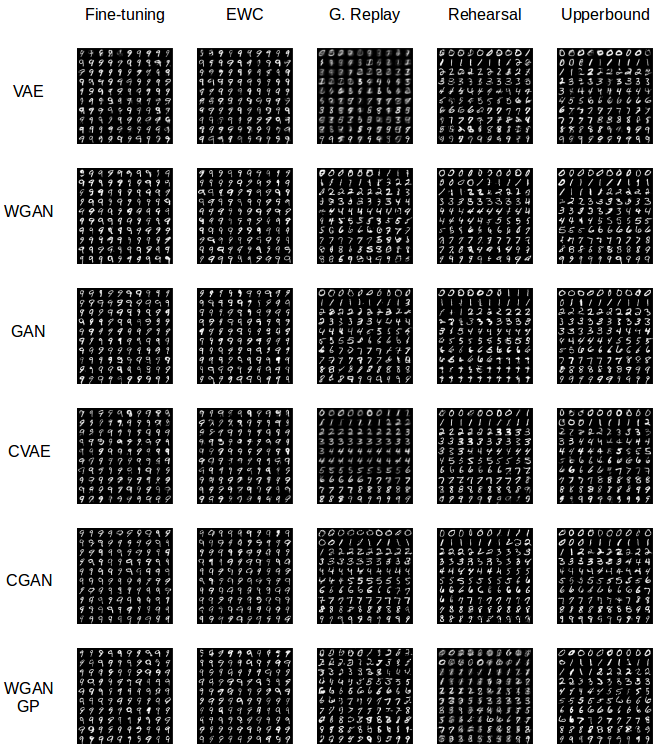}
    \caption[Generated samples with MNIST]{For each method and model, images sampled after the 10 sequential tasks with MNIST.}
\label{fig:3_CL_GM:samples_mnist}
\end{figure}

\begin{figure}[ht]
    \centering
    \includegraphics[width=1\textwidth]{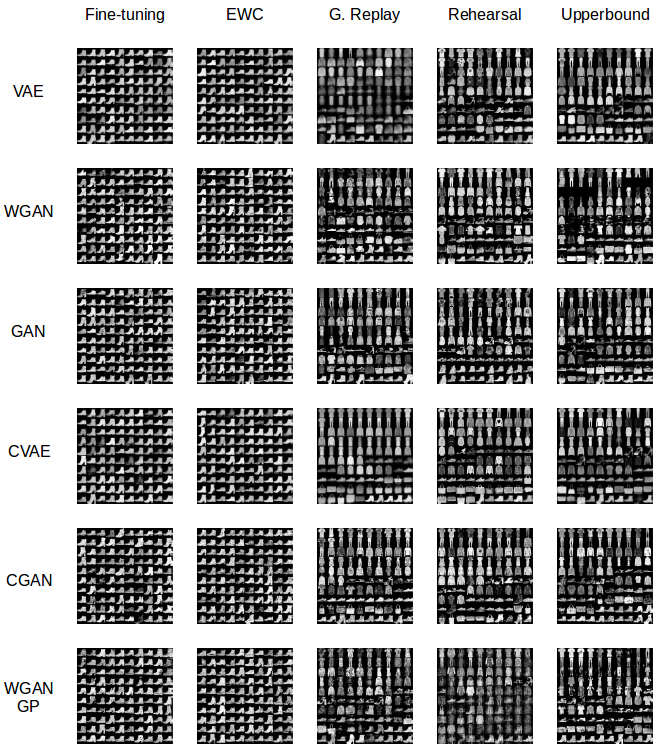}
    \caption[Generated Samples with Fashion-MNIST]{For each method and model, images sampled after the 10 sequential tasks with Fashion-MNIST.}
\label{fig:3_CL_GM:samples_fashion}
\end{figure}

\clearpage

\subsubsection{Corollary results}

Catastrophic forgetting can be visualized in Fig.\ref{fig:3_CL_GM:grid}. Each square's column represents the task index and each row the class, the color indicates the fitting capacity (FiC). Yellow squares show a high FiC, blue ones show a low FiC. We can visualize both the performance of VAE and GAN but also the performance evolution for each class. For Generative Replay, at the end of the task sequence, VAE decreases its performance in several classes but GAN does not. For Rehearsal, we observe the opposite.

\begin{figure}[ht]
   
       \centering
    \begin{subfigure}[b]{0.20\textwidth}
        \includegraphics[width=\textwidth]{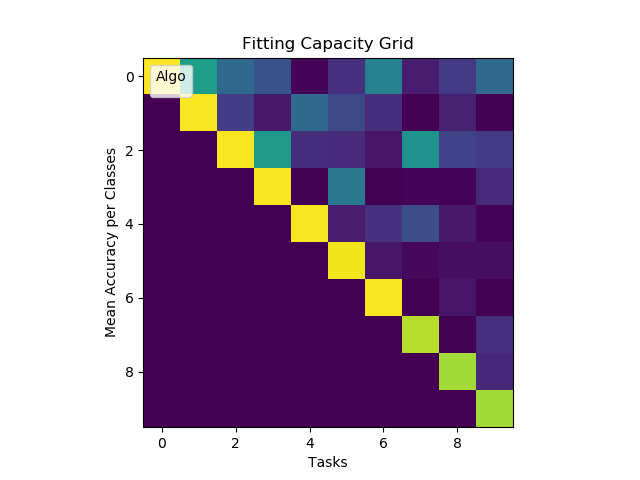}
        \label{fig:3_CL_GR:mnist_disjoint_Baseline_GAN_grid}
    \end{subfigure}
    \begin{subfigure}[b]{0.20\textwidth}
        \includegraphics[width=\textwidth]{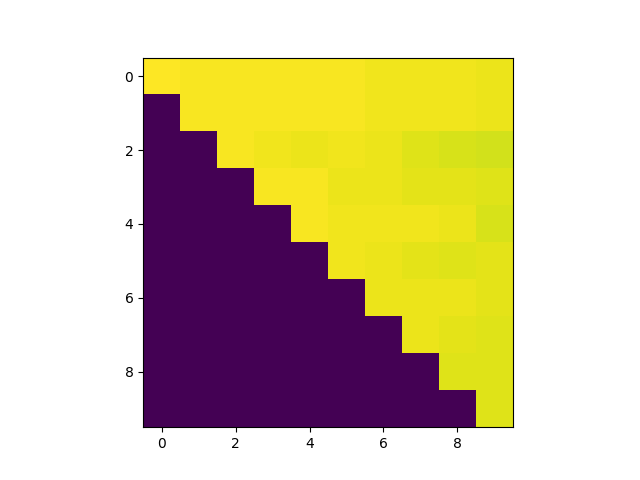}
        \label{fig:3_CL_GR:mnist_disjoint_Generative_Transfer_GAN_grid}
    \end{subfigure}
        \begin{subfigure}[b]{0.20\textwidth}
        \includegraphics[width=\textwidth]{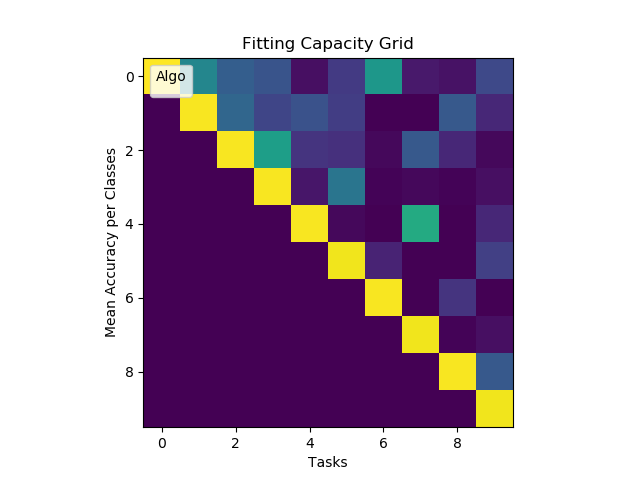}
        \label{fig:3_CL_GR:mnist_disjoint_Ewc_5_GAN_grid}
    \end{subfigure}
    \begin{subfigure}[b]{0.20\textwidth}
        \includegraphics[width=\textwidth]{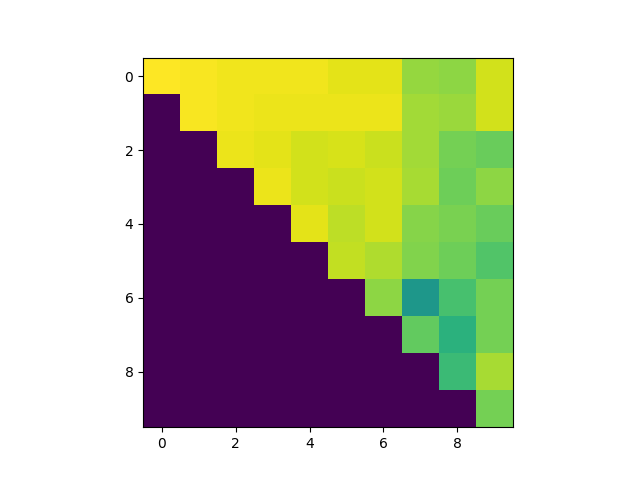}
        \label{fig:3_CL_GR:mnist_disjoint_Reharsal_balanced_GAN_grid}
    \end{subfigure}

          \centering
    \begin{subfigure}[b]{0.20\textwidth}
        \includegraphics[width=\textwidth]{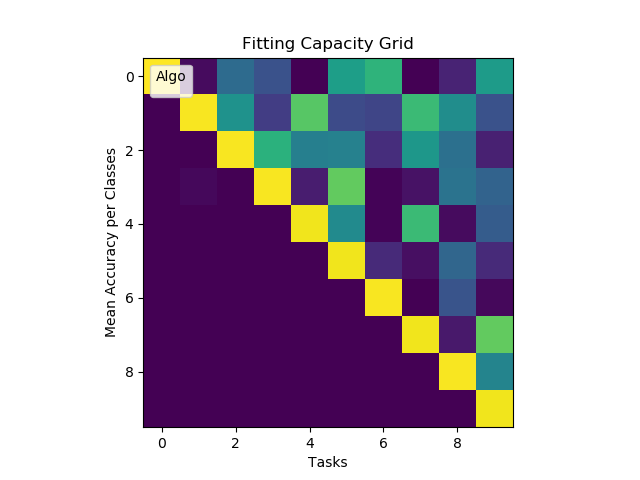}
        \caption{Fine-tuning}
        \label{fig:3_CL_GR:mnist_disjoint_Baseline_VAE_grid}
    \end{subfigure}
    \begin{subfigure}[b]{0.2\textwidth}
        \includegraphics[width=\textwidth]{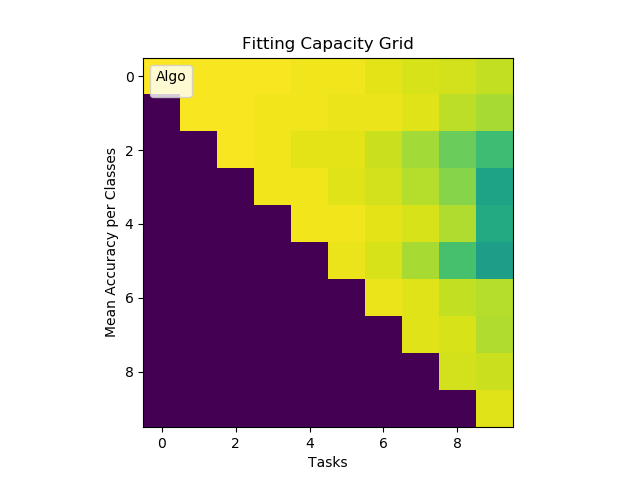}
        \caption{G. Replay}
        \label{fig:3_CL_GR:mnist_disjoint_Generative_Transfer_VAE_grid}
    \end{subfigure}
        \begin{subfigure}[b]{0.2\textwidth}
        \includegraphics[width=\textwidth]{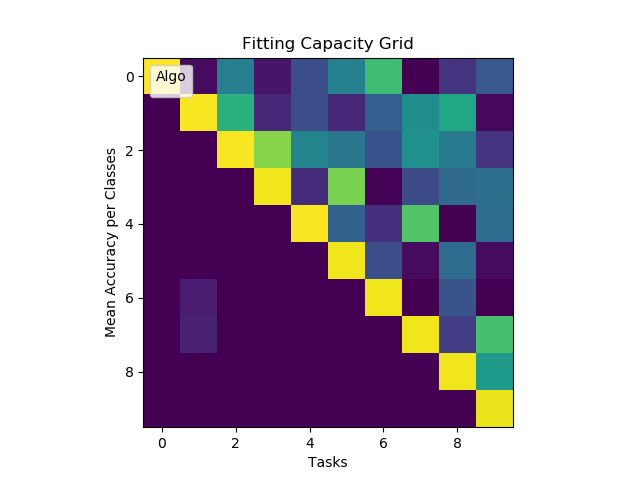}
        \caption{EWC}
        \label{fig:3_CL_GR:mnist_disjoint_Ewc_5_VAE_grid}
    \end{subfigure}
    \begin{subfigure}[b]{0.2\textwidth}
        \includegraphics[width=\textwidth]{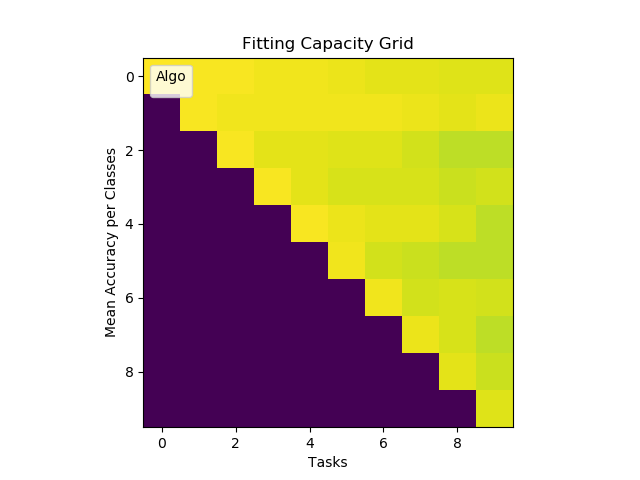}
        \caption{Rehearsal}
        \label{fig:3_CL_GR:mnist_disjoint_Reharsal_balanced_VAE_grid}
    \end{subfigure}

   \caption[Fitting Capacity results for GAN and VAE on MNIST]{Fitting Capacity results for GAN (top) and VAE (bottom) on MNIST. Each square is the accuracy on one class for one task. Abscissa is the task index (left: 0 , right: 9) and orderly is the class index (top: 0, down: 9). The accuracy is proportional to the color (dark blue : $0\%$, yellow $100\%$)}
   \label{fig:3_CL_GM:grid}
\end{figure}

Concerning the high performance of original GAN and WGAN with Generative Replay, those models benefit from their samples quality and their stability. In comparison, samples from CGAN and WGAN-GP are moisier and samples from VAE and CVAE blurrier.
However, in the Rehearsal approach, GANs-based models seem much less stable (See variance in Table \ref{tab:3_CL_GM:results} and Figure \ref{fig:3_CL_GR:accuracy_per_task}). 
In this setting, as the dataset is made of few data samples, the discriminative task is almost trivial for the discriminator and easy to overfit
which makes training harder for the generator.
In opposition, VAE-based models are particularly effective and stable in the Rehearsal setting (See Fig. \ref{fig:3_CL_GM:grid}). Indeed, their learning objective (pixel-wise error) is not disturbed by a low samples variability and their probabilistic hidden variables make them less prone to overfit.

However, the FiC of Fine-tuning and EWC in Table \ref{tab:3_CL_GM:results} is higher than expected for unconditional models. As the generator is only able to produce samples from the last task, the FiC should be near $10\%$. This is a downside of using an expert for annotation before computing the FiC. Fuzzy samples can be wrongly annotated, which can artificially increase the labels variability and thus the FiC of low performing models, e.g., VAE with Fine-tuning. However, this results stay lower than the FiC of well performing models.

\medskip

Incidentally, an important side result is that the fitting capacity (FiC) of conditional generative models is comparable to results of Continual Learning classification. Our maximal performance in this setting is with CGAN: $94.7\%$ on MNIST and $75.44\%$ on Fashion MNIST.
In a similar setting with two sequential tasks, which is arguably easier than our setting (one with digits from 0,1,2,3,4 and another with 5,6,7,8,9), \cite{He18} authors achieve a performance of $94.91\%$.
This shows that using generative models for CL could be a competitive tool in a classification scenario.
It is worth noting that we did not compare our results of unconditional models FiC with classification state of the art. Indeed, in this case, the FiC is based on an annotation from an expert not trained in a continual setting. The comparison would then not be fair.

\subsection{CIFAR10 results}
\label{sub:cifar_CL_GM}

\begin{figure}[ht]
    \centering
    \includegraphics[width=0.5\textwidth]{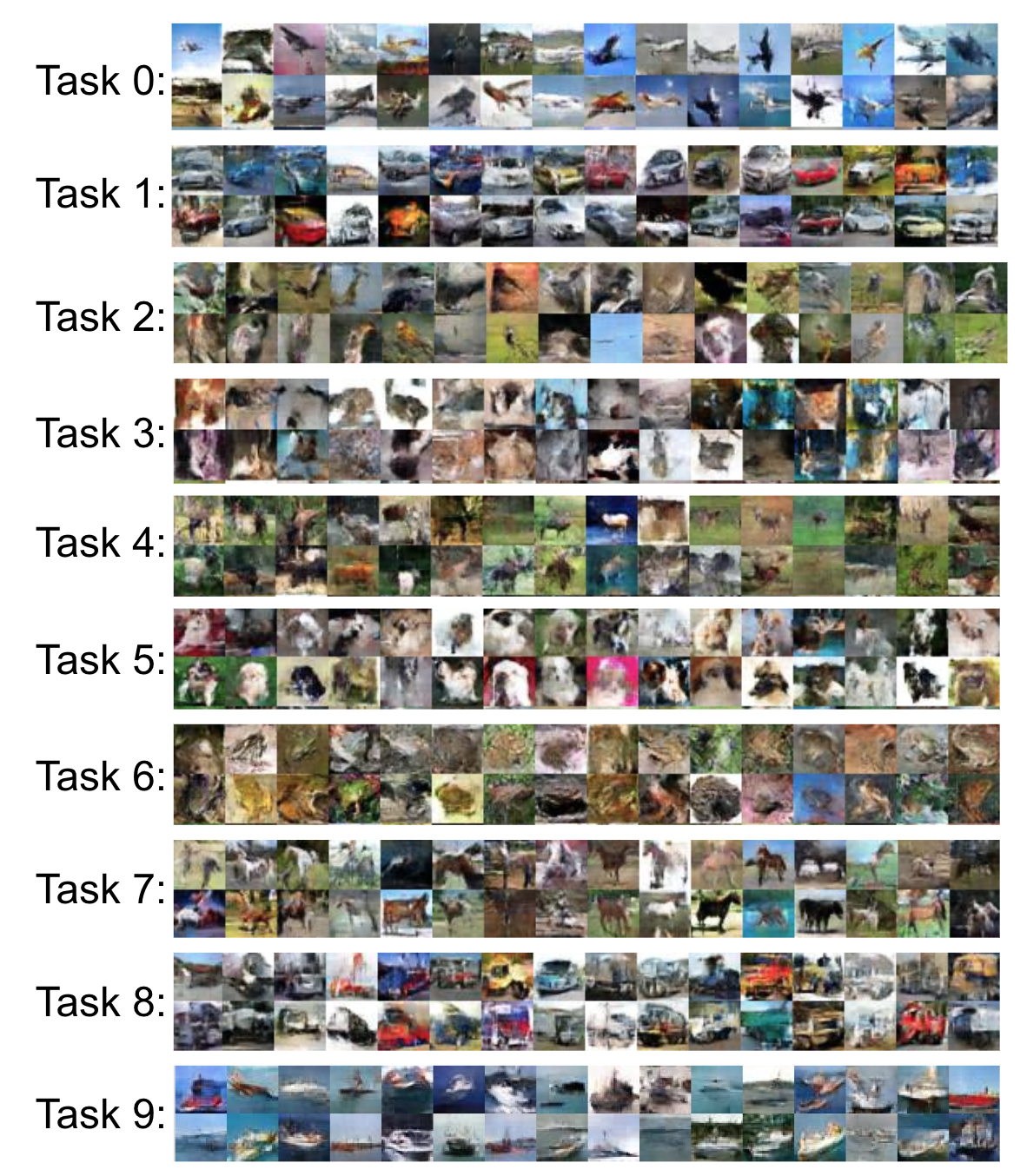}
    \caption[Generated samples with CIFAR10]{Images sampled from a WGAN-GP at the end of each tasks with CIFAR10. \checked{The images are similar to real images in their features but still not very likely to be real images.}}
\label{fig:3_CL_GM:samples_cifar}
\end{figure}

In this experiment, we selected the best performing CL methods on MNIST and Fashion MNIST, Generative Replay and Rehearsal, and tested them on the more challenging CIFAR10 dataset. We compared the two methods to naive Fine-tuning, and to Upperbound Model (one generator for each class). The setting remains the same, one task for each category, for which the aim is to avoid forgetting of previously seen categories. We selected WGAN-GP because it produced the most visually satisfying samples on CIFAR10 (Figure \ref{fig:3_CL_GM:samples_cifar}). \checked{We can note that visual assessment allows detecting failing training even if it is not accurate to distinguish the best models.}

\begin{figure}[ht]
    \centering
    \includegraphics[width=1\textwidth]{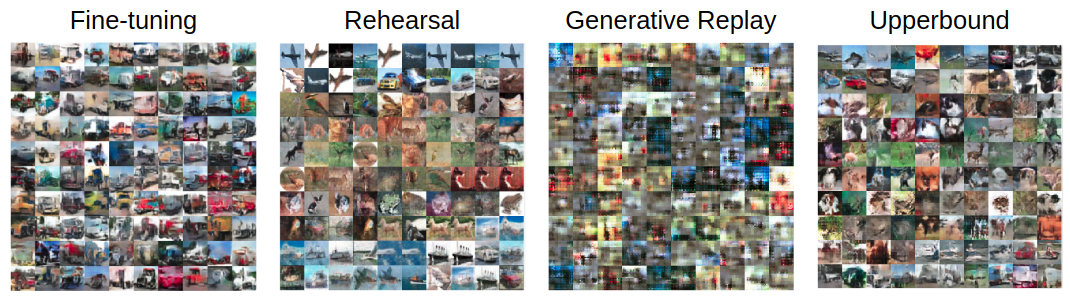}
    \caption[Generated samples on CIFAR10 with each method.]{For each method, images sampled after the 10 sequential tasks with WGAN-GP trained on CIFAR10.}
\label{fig:3_CL_GM:samples_cifar_final}
\end{figure}

\begin{figure}
    \centering
        \includegraphics[width=0.9\textwidth]{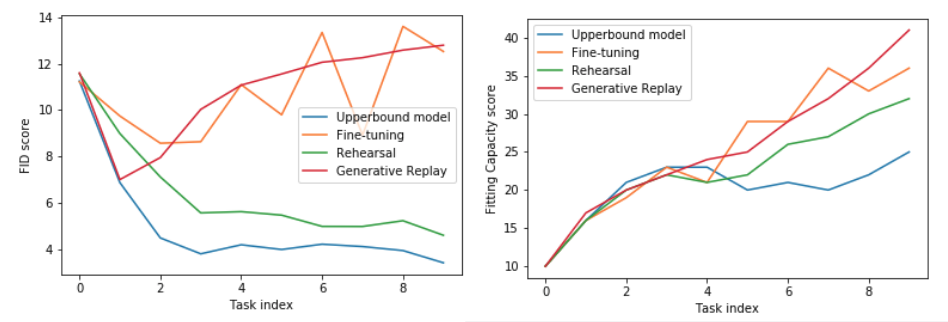}
    \caption[FiC and FID on CL tasks with WGAN\_GP, on CIFAR10.]{\checked{Fitting capacity and FID score of Continual Learning methods applied to WGAN\_GP, on CIFAR10. The fitting capacity (FiC) (right figure) is too low to be significant. Within this range of fitting capacities, we cannot make any comparison between models. 
     FiC is more suited for models which already work quite well. It explains probably why fine-tuning seems to work better than rehearsal in the CIFAR10 experiments.}}
\label{fig:3_CL_GM:cifar_gr}
\end{figure}

FID and fitting capacity (FiC) curves throughout training are provided in Fig. \ref{fig:3_CL_GM:cifar_gr}, and generated samples after the 10 sequential tasks are displayed in Fig. \ref{fig:3_CL_GM:samples_cifar_final}, where we display images sampled after the 10 sequential tasks. The FiC results show that all four methods fail to generate images that allow to learn a classifier that performs well on real CIFAR10 test data.
As stated for MNIST and Fashion MNIST, with non-conditional models, when the FiC is low, it can been artificially increased by automatic annotation which makes the difference between curves not significant in this case.
Naive Fine-tuning catastrophically forgets previous tasks, as expected. Rehearsal does not yield satisfactory results. While the FID score shows improvement at each new task, visualization clearly shows that the generator copies samples in memory, and suffers from mode collapse. This confirms our intuition that Rehearsal overfits to the few samples kept in memory. Generative Replay fails; since the dataset is composed of real-life images, the generation task is much harder to complete. 
At task 0, the generator is able to produce images that roughly resemble samples of the category, here planes. As tasks are presented, minor generation errors accumulated and snowballed into the result in task 9 (see Fig. \ref{fig:3_CL_GM:samples_cifar_final}): samples are blurry and categories are indistinguishable. As a consequence, the FID improves at the beginning of the training sequence, and then deteriorates at each new task. We also trained the same model separately on each task, and while the result might look visually satisfactory (see Upperbound in Fig. \ref{fig:3_CL_GM:samples_cifar_final}), the quantitative metrics show that generation quality is not excellent.

\bigskip

These negative results show that training a generative model on a sequential task scenario is not equivalent to successfully training a generative model on all data or each category, and that state-of-the-art generative models struggle on real-life image datasets like CIFAR10. Designing a CL strategy for these type of datasets remains a challenge.

\section{Discussion}
\label{sec:3_CL_GM:discussion_CL_GM}

Besides the quantitative results and visual evaluation of the generated samples, the evaluated strategies have, by design, specific characteristics relevant to CL that we discuss in this section. 

Rehearsal violates the data availability assumption that might be required in CL scenarios, by recording part of the samples. Furthermore, the risk of overfitting is high when only few samples represent a task, as shown in the CIFAR10 results. EWC and Generative Replay respect this assumption. EWC has the advantage of not requiring any computational overload during training, but this comes at the cost of computing the Fisher information matrix, and storing its values as well as a copy of previous parameters. The memory needed for EWC to save information from the past is twice the size of the model which may be expensive in comparison to rehearsal methods.
\checked{As an example, the model we used to solve MNIST classification problem, is 95.5 kBytes, it is the same memory space as approximately 120 images. We saw that with 10 images per classes the rehearsal is quite effective.}
 Nevertheless, with Rehearsal and Generative Replay, the model has more and more samples to learn from at each new task, which makes training computation cost increase at each new tasks. 

\bigskip

Finally, we want to highlight the importance of using metrics  that are sensitive to mode collapse (like the one we used). For example, \citep{wu2018memory} proposes a metric to evaluate CL for conditional generative models. For a given label $l$, they sample images from the generator conditioned on $l$ and feed it to a pre-trained classifier. If the predicted label of the classifier matches $l$, then it is considered correct. In our experiment we find that it gives a clear advantage to rehearsal methods. As the generator may overfit the few samples kept in memory and then maximizes the evaluation proposed by \cite{wu2018incremental}, while not producing diverse samples. 
Nevertheless, even if their metric is unable to detect mode collapse or overfitting, it can efficiently expose catastrophic forgetting in conditional models.

\begin{figure}[ht]
    \centering
    \includegraphics[width=1\textwidth]{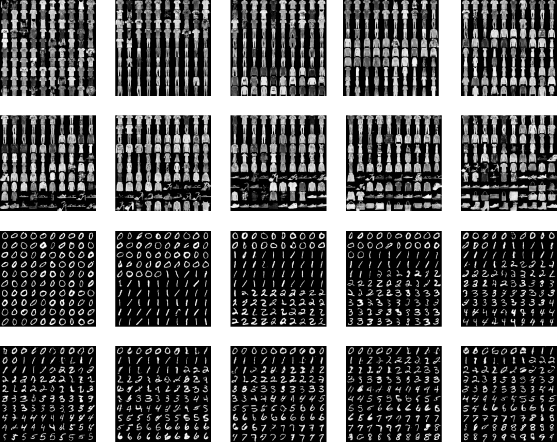}
    \caption{Samples at each task for the best working solution, GAN with generative replay.}
\label{fig:CL_GM:best_samples}
\end{figure}

\section{Conclusion}
\label{sec:3_CL_GM:ccl_CL_GM}

In this chapter, we experimented with the viability and effectiveness of generative models on Continual Learning (CL) settings.
We evaluated the considered approaches on commonly used datasets for CL, with two quantitative metrics. 
Our experiments indicate that on MNIST and Fashion MNIST, the original GAN combined to the Generative Replay method is particularly effective (samples at each task can be visualized Figure \ref{fig:CL_GM:best_samples}).
This method avoids catastrophic forgetting by using the generator as a memory to sample from the previous tasks and hence maintains past knowledge.
Furthermore, we shed light on how generative models can learn continually with various methods and present successful combinations.
We also revealed that generative models do not perform well enough on CIFAR10 to learn continually.
Since generation errors accumulate, they are not usable into a complex continual setting.
 
\checked{In the next chapter, we will study the use of generative models for learning continually on classification tasks.}

\newpage
\chapter{Generative Replay for Classification}
\label{chap:4_CL_GR}

In the previous chapter, we evaluate different generative models' ability to learn continually.
In this chapter, we use the results of the previous section for a classification task. In particular, we study the use of generative replay for classification and the advantage of conditional models over marginal models for continual classification.

This work was done in collaboration with Alexander Geppert from Fulda University (Germany), it led to the publication of \say{Marginal Replay vs Conditional Replay for Continual Learning} at ICANN 2019 \cite{lesort2018marginal}. The original article has been slightly modified and extended to better fit the thesis and include figures and results cut by article size restriction.

\section{Introduction}

\checked{As seen in the previous chapter, the generative replay uses a generative model to generate data from past learning experiences to remember them.}
The \textit{generative replay} approach is in principle model-agnostic. Indeed, it can be performed with a variety of machine learning models such as decision trees, support vector machines (SVMs) or deep neural networks (DNNs). \checked{It is also the case more generally to every replay method.} %

A downside of generative replay and similar approaches, such as rehearsal, is that the time complexity of adapting to a new task is not constant but depends on the number of preceding tasks that should be replayed. Or, conversely, if continual learning should be  performed at constant time complexity, only a fixed amount of samples can be generated, and thus there will be forgetting, although it won't be catastrophic.

\medskip

As in the previous chapter, we decided to investigate two different types of generative models: Generative adversarial networks (GAN) and variational auto-encoder (VAE). On one hand GAN are known to generate samples of high quality but on the other hand VAE directly maximize the likelihood of the learned distribution while training. It was then interesting to experiment both of them to compare their performances.

\medskip

This chapter proposes and evaluates a particular method for performing replay using DNNs, termed \say{conditional replay}, which is similar in spirit to  \cite{shin2017continual} but presents important conceptual improvements (see next section). The main advantage of conditional replay is that samples can be generated conditionally, i.e., based on a provided label. Thus, labels for generated samples need not be inferred in a separate step as other replay-based approaches, e.g., \cite{shin2017continual}, which we term \textit{marginal replay} approaches. Since inferring the label of a generated sample inevitably requires the application of a possibly less-than-perfect classifier, avoiding this step conceivably reduces the margin for error in complex continual learning tasks.

The original contributions of this chapter can be summarized as follows: 
\begin{itemize}
    \item \textbf{Conditional replay as a method for continual classification learning} We experimentally establish the advantages of conditional replay in the field of continual learning by comparing conditional and marginal replay models on a common set of benchmarks. 
    \item \textbf{Improvement of marginal learning} We furthermore propose an improvement of marginal replay as proposed in  \cite{shin2017continual} by using generative adversarial networks
    \item \textbf{New experimental benchmarks for generative replay strategies} To measure the merit of these proposals, we use two experimental settings that have not been previously considered for benchmarking generative replay: rotations and permutations. In addition, we promote the ''10-class-disjoint'' task as an important benchmark for continual learning as it is impossible to solve for purely discriminative methods (at no time, samples from different classes are provided for training so no discrimination can happen). 
    \item \textbf{Comparison of generative replay to EWC} We show the principled advantage that generative replay techniques have with respect to regularization methods like EWC in a \say{one class per task} setting, which is a very common setting in practice and in which discriminatively trained models strongly tend to assign the same class label to every sample regardless of content. \checked{Those experiences correlate with what was discussed in Chapter \ref{chap:2b_Replay}.} %
\end{itemize}

\checked{The chapter organization is the following: first, in Section \ref{sec:4_CL_GR:Backgroud}, we present some background material about generative replay state of the art. Second, in Section \ref{sec:4_CL_GR:approach}, we describe the methods used as well as the benchmarks. Third, in Section \ref{sec:4_CL_GR:results_CL_GR}, we present the experiments conducted to compare conditional replay and marginal replay. Then, in Section \ref{sec:4_CL_GR:discussion_CL_GR} we show and discuss our results. Finally, in Section \ref{sec:4_CL_GR:conclusion} we conclude about the advantages and limitations of the experimented methods.}

\section{Background}
\label{sec:4_CL_GR:Backgroud}

\checked{In this section, we present a summary of the state of the art from Chapter \ref{chap:2_CL} oriented on generative replay for continual learning classification and introduce advanced methods for generative replay with improved sampling strategies.}

\subsection{State of the art}

\checked{As already presented in Chapter \ref{chap:2_CL} and \ref{chap:3_CL_GM}, generative replay consists of training a generative model on the current task to replay it later. It benefits from the task learning criterion and the generation learning criterion to learn without forgetting in various types of tasks.}

Concerning recent advances in generative replay improving upon  \cite{shin2017continual},
several works propose the use of generative models in continual learning of classification tasks \cite{kamra2017deep, wu2018incremental, wu2018memory, Shah18, Rios19}. 
\checked{Those approaches generally improves vanilla approach \cite{shin2017continual} with supplementary losses as with distillation loss \cite{hinton2015distilling}.
 They also experiment with harder datasets than MNIST such as LSUN \cite{Yu15} in \cite{wu2018memory}.}
 However, their results do not provide comparison between different types of generative models.
 
 \checked{Our results from \cite{lesort2018generative} (Chapter \ref{chap:3_CL_GM}) offer more insights on which generative models suit continual learning the best. In this chapter, we propose to build on those results to search for best generative replay method for classification.}

Generally, each approach to continual learning has its advantages and disadvantages:
\begin{itemize}
\item dynamic architecture methods suffer from little to no interference between present and past knowledge as usually different networks or sub-networks are allocated to different learning tasks. The problem with this approach is that, on the one hand, model complexity is not constant, and more seriously, that the task from which a sample is coming from must be known at inference time in order to select the appropriate (sub-)network.

\item regularization approaches are very diverse: in general, their advantage is simplicity and (often) a constant-time/memory behaviour w.r.t. the number of tasks. However, \checked{they have some theoretical shortcomings in class-incremental learning (See Chapter \ref{chap:2b_Replay}),}
 the impact of the regularizer on continual learning performance is difficult to understand, and several parameters need to be tuned whose significance is unclear (i.e., the strengths of the regularization terms).

\item generative replay show very good and robust continual learning performance, although time complexity of learning depends on the number of previous tasks for current generative replay methods. In addition, the storage of weights for a sufficiently powerful generator may prove very memory-consuming, so this approach cannot be used in all settings.
\end{itemize}

\checked{Hence, the choice of the generative model is crucial to correctly learn continually. The way generative models are used is also particularly important to maximize the algorithm's performance and make the best out of a trained generative model. We detail, in next section, the importance of sampling to improve results.}

\subsection{Sampling generative models}

\checked{A well trained generative model can generate original data from the training data distribution.
The generation process is generally the following, a seed $z$ is sampled from a random distribution, typically a normal or uniform distribution. $z$ it then fed to the generative model as input and the generative model outputs a data point. The objective is that for two different $z$ the model generates two different data points.}

\checked{The sampling can also be generated under conditions. The generative model can be trained to generate different types of data depending on a defined flag. For example, the flag can indicate from which class the model should generate data or from which task. The data is still generated randomly from $z$ but a supplementary input $c$ is given, conditioning the data point generated. This is the kind of sampling we use for \textit{conditional replay}.}

\checked{The sampling can also be adapted to maximize a certain criterion, like variety or similarity. In continual learning, it can be tuned to find the samples creating the maximum of interferences to use them and solve a conflict between learning criterion \cite{Aljundi2019Online}. Finding the best sampling method is an interesting research field. However, in this chapter, the samples were selected uniformly as a vanilla method for fair comparison purposes.}

\section{Approach}
\label{sec:4_CL_GR:approach}

\begin{figure}
    \centering
    \includegraphics*[viewport=0in 0in 4.8in 1.3in,width=0.39\textwidth]{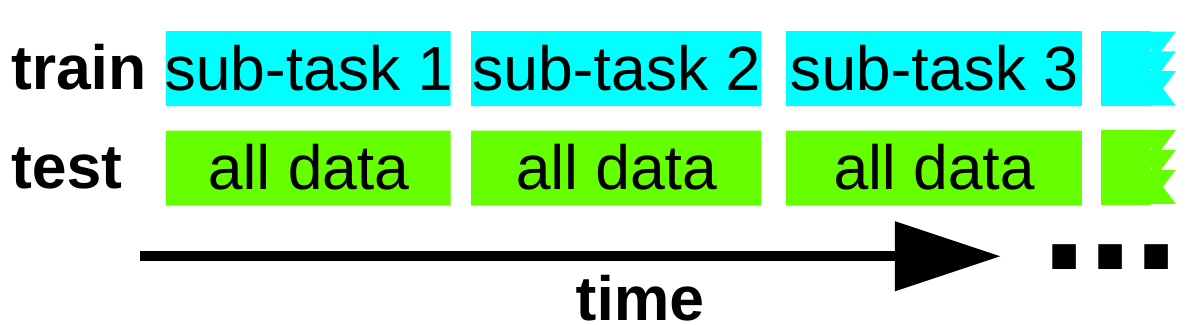}
    \includegraphics[width=0.40\textwidth]{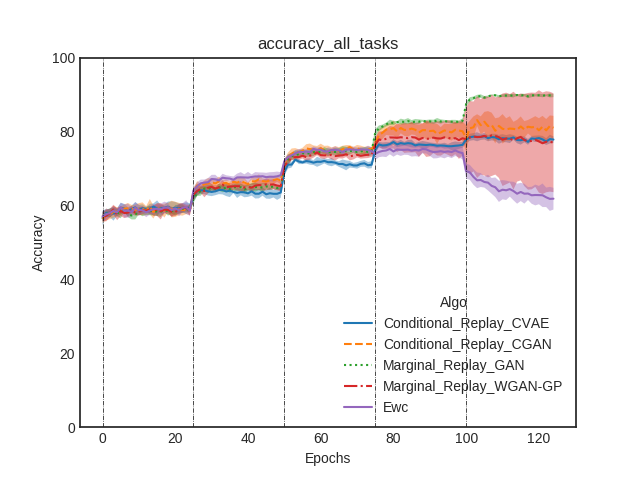}
    \caption[Illustration of a multi-tasks learning continuum.]{Illustration of a multi-tasks learning continuum. Left: The problem setting of continual learning as investigated in this chapter. DNN models are trained one after the other on a sequence of tasks (of which three are shown here), and are continuously evaluated on a test set consisting of the union of all task test sets. This gives rise to results as shown exemplarily on the right-hand side of the figure, i.e., plots of test set accuracy over time for different models, where boundaries between tasks (5 in this case) are indicated by vertical lines.
    \label{fig:4CL_GR:my_label}
    }
\end{figure}

In this chapter, \say{learning a task} denotes learning on a classification problem that is composed of two or more tasks presented sequentially to the neural network model. Fig.~\ref{fig:4CL_GR:my_label} offers a visualization of the problem setting.
 Here, the tasks are constructed from two standard visual classification benchmarks: MNIST and Fashion MNIST, either by dividing available classes into several tasks, or by performing per-sample image processing operations that are identical within, and different between, tasks. All continual learning models are then trained and evaluated in an identical fashion on all tasks, and performances are compared by a simple visual inspection of classification accuracy plots.

\subsection{Continual learning tasks}
All tasks are constructed from the underlying MNIST and FashionMNIST benchmarks, 
so the number of samples in train and test sets for each task depend on the precise way of constructing them, as described below.
\par

\begin{figure}
    \centering
    \begin{subfigure}[b]{0.15\textwidth}
        \includegraphics[width=\textwidth]{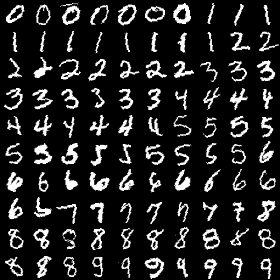}
        \caption{task 0}
        \label{fig:4_CL_GR:mnist_rotations_0}
    \end{subfigure}
    \begin{subfigure}[b]{0.15\textwidth}
        \includegraphics[width=\textwidth]{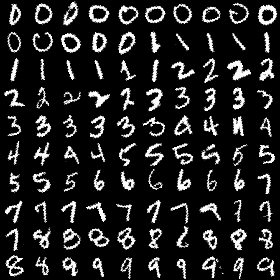}
        \caption{task 1}
        \label{fig:4_CL_GR:mnist_rotations_1}
    \end{subfigure}
        \begin{subfigure}[b]{0.15\textwidth}
        \includegraphics[width=\textwidth]{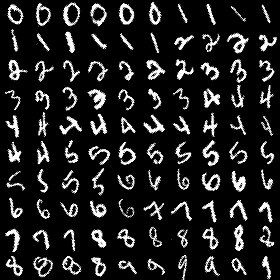}
        \caption{task 2}
        \label{fig:4_CL_GR:mnist_rotations_2}
    \end{subfigure}
    \begin{subfigure}[b]{0.15\textwidth}
        \includegraphics[width=\textwidth]{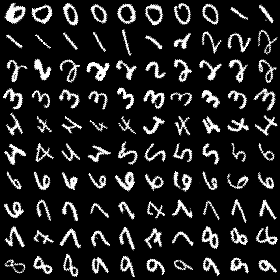}
        \caption{task 3}
        \label{fig:4_CL_GR:mnist_rotations_3}
    \end{subfigure}
    \begin{subfigure}[b]{0.15\textwidth}
        \includegraphics[width=\textwidth]{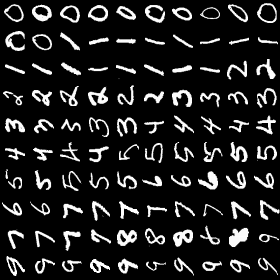}
        \caption{task 4}
        \label{fig:4_CL_GR:mnist_rotations_4}
    \end{subfigure}
  \caption{MNIST training data for rotation tasks.}
  \label{fig:4_CL_GR:rotations_samples}
\end{figure}

\begin{figure}
    \centering
  \includegraphics[width=0.9\textwidth]{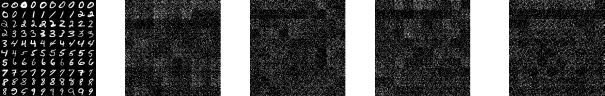}
  \caption{MNIST training data for permutation-type tasks.}
  \label{fig:4_CL_GR:mnist_permutation_train}
\end{figure}

\begin{figure}
    \centering
    \begin{subfigure}[b]{0.15\textwidth}
        \includegraphics[width=\textwidth]{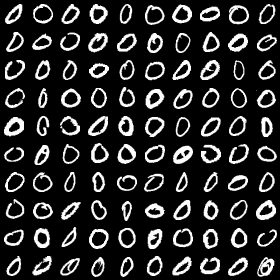}
        \caption{Task 0 }
        \label{fig:4_CL_GR:mnist_disjoint_0}
    \end{subfigure}
    \begin{subfigure}[b]{0.15\textwidth}
        \includegraphics[width=\textwidth]{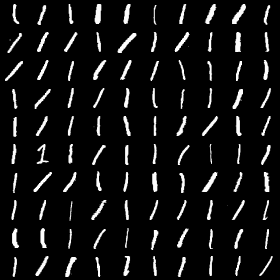}
        \caption{Task 1 }
        \label{fig:4_CL_GR:mnist_disjoint_1}
    \end{subfigure}
        \begin{subfigure}[b]{0.15\textwidth}
        \includegraphics[width=\textwidth]{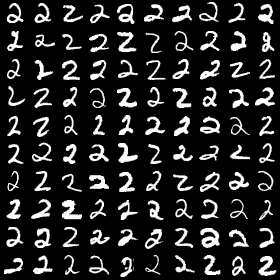}
        \caption{Task 2 }
        \label{fig:4_CL_GR:mnist_disjoint_2}
    \end{subfigure}
    \begin{subfigure}[b]{0.15\textwidth}
        \includegraphics[width=\textwidth]{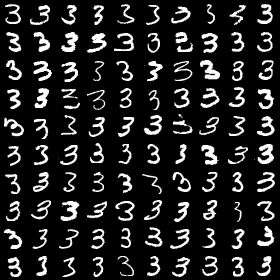}
        \caption{Task 3 }
        \label{fig:4_CL_GR:mnist_disjoint_3}
    \end{subfigure}
    \begin{subfigure}[b]{0.15\textwidth}
        \includegraphics[width=\textwidth]{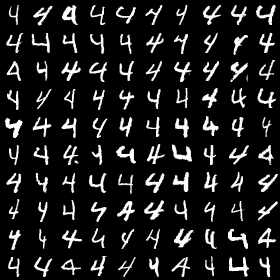}
        \caption{Task 4 }
        \label{fig:4_CL_GR:mnist_disjoint_4}
    \end{subfigure}
    
        \begin{subfigure}[b]{0.15\textwidth}
        \includegraphics[width=\textwidth]{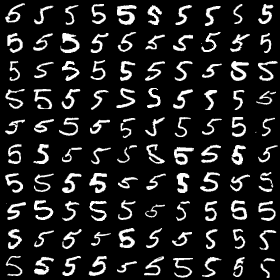}
        \caption{Task 5 }
        \label{fig:4_CL_GR:mnist_disjoint_5}
    \end{subfigure}
    \begin{subfigure}[b]{0.15\textwidth}
        \includegraphics[width=\textwidth]{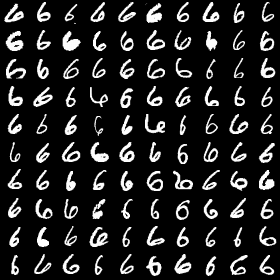}
        \caption{Task 6 }
        \label{fig:4_CL_GR:mnist_disjoint_6}
    \end{subfigure}
        \begin{subfigure}[b]{0.15\textwidth}
        \includegraphics[width=\textwidth]{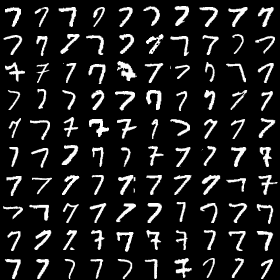}
        \caption{Task 7 }
        \label{fig:4_CL_GR:mnist_disjoint_7}
    \end{subfigure}
    \begin{subfigure}[b]{0.15\textwidth}
        \includegraphics[width=\textwidth]{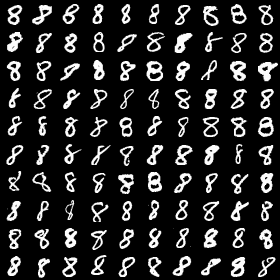}
        \caption{Task 8}
        \label{fig:4_CL_GR:mnist_disjoint_8}
    \end{subfigure}
    \begin{subfigure}[b]{0.15\textwidth}
        \includegraphics[width=\textwidth]{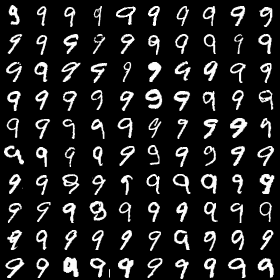}
        \caption{Task 9 }
        \label{fig:4_CL_GR:mnist_disjoint_9}
    \end{subfigure}
   
   \caption[Samples of MNIST training data for the disjoint tasks]{Samples of MNIST training data for the disjoint tasks. Each task adds one more visual class, a principle which carries over identically to FashionMNIST.}
   \label{fig:4_CL_GR:mnist_disjoint}
\end{figure}

\begin{itemize}

\item \textbf{Rotations}~New tasks are generated by choosing a random rotation angle $\beta \in [0,\pi/2]$ and then performing a 2D in-plane rotation on all samples of the original benchmark. As both benchmarks we use contain samples of 28x28 pixels, no information loss is introduced by this procedure. We limit rotation angles to $\pi/2$ because larger rotations could mix MNIST classes like 6 and 9. Each task in rotation-based tasks contains all 10 classes of the underlying benchmark, leading to 55.000 and 10.000 samples, respectively, in the train and test sets of each task.

\item \textbf{Permutations}~New tasks are generated by defining a random pixel permutation scheme, and then applying it to each data sample of the original benchmark. Each task in permutation-based tasks contains all 10 classes of the underlying benchmark, leading to 55.000 and 10.000 samples, respectively, in the train and test sets of each task.

\item \textbf{Disjoint classes}~For each benchmark, this task has as many tasks as there are classes in the benchmark. Each task contains the samples of a single class, i.e., roughly 6.000 samples in the train set and 1.000 samples in the test set. As the classes are balanced for both benchmarks, this does not unduly favor certain classes. This task presents a substantial challenge for machine learning methods since a normal DNN would, for each task, learn to map all samples to a single class label irrespective of content. Selective discrimination between any two classes is hard to obtain except if replay is involved, because then a classifier actually \say{sees} samples from different classes at the same time.

\end{itemize}

In this chapter, we will experiment with three different scenarios: 10 disjoint tasks (Fig.~\ref{fig:4_CL_GR:mnist_disjoint}), five permutation tasks (See Fig.~\ref{fig:4_CL_GR:mnist_permutation_train} and Fig.~\ref{fig:4_CL_GR:mnist_permutation}) and five rotation tasks (See Fig.~\ref{fig:4_CL_GR:rotations_samples}).

\begin{figure}
    \centering
    \includegraphics[width=\textwidth]{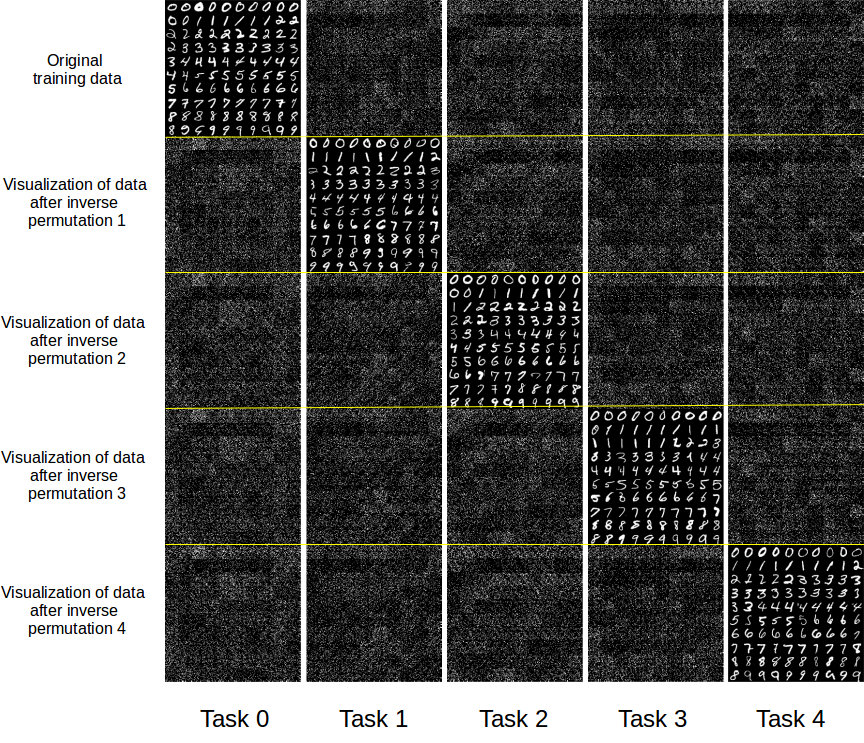}
   
   \caption[Samples of MNIST training data for the permutations tasks]{Visualization of training data for the MNIST permutations continual learning setting. %
The first line shows each task original training data. The other lines show the effect of all tasks inverse permutations applied to those data.
Obviously, the inverse permutation $i$, results in original data at task $i$.
 This figure should help the interpretation of Figure \ref{fig:4_CL_GR:mnist_permutation_results}}
   \label{fig:4_CL_GR:mnist_permutation}
\end{figure}

\clearpage

\subsection{Models}

\begin{table*}[ht!]
\centering

  \caption[Hyperparameters table for MNIST and Fashion MNIST]{Hyperparameters for MNIST and Fashion MNIST all models (all CL settings have the same training hyper parameters with Adam)}
  \label{tab:4_CL_GR:hyperparams}
  \begin{tabular}{ccccccc}
    \hline 
    Method &
    Epochs&
    LR Classifier&
    LR Generator&
    beta1&
    beta2&
    Batch Size\\
    \hline
    
    Marginal Replay     & %
    25 & %
    0.01& %
    2e-4 & %
    5e-1& %
    0.999& %
    64 \\ %
    \hline
    
    Conditional Replay     & %
    25 & %
    0.01& %
    2e-4 & %
    5e-1& %
    0.999& %
    64 \\ %
    \hline
    
    EWC     & %
    25 & %
    0.01& %
    -& %
    5e-1& %
    0.999& %
    64\\ %
    \hline

    \hline
\end{tabular}

\end{table*}

 As a reference implementation, we use a fully-connected network (2 hidden layers with 200 neurons each) with ReLU activation function. No batch normalization or dropout is performed. All other training parameters are described in Tab. \ref{tab:4_CL_GR:hyperparams}.
In this chapter, we compare a number of deep learning models: unless otherwise stated, we 
employ the Rectified Linear Unit (ReLU) transfer function, cross-entropy loss for classifier training, and the Adam optimizer.

\begin{figure}
\begin{subfigure}[b]{0.49\textwidth}
    \centering
  \includegraphics[width=0.9\textwidth]{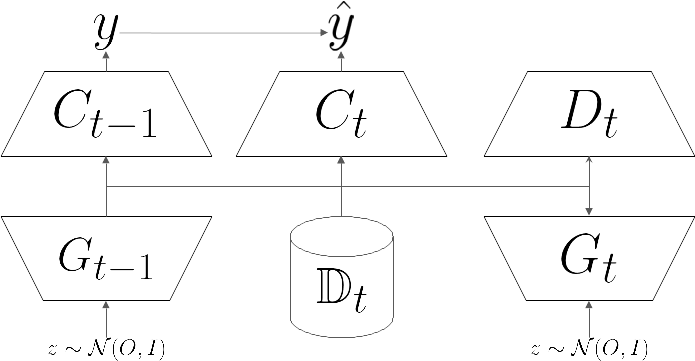}
  \caption{Generative replay}
  \label{fig:4_CL_GR:Generative_Replay}
\end{subfigure} 
\begin{subfigure}[b]{0.49\textwidth}
    \centering
  \includegraphics[width=0.9\textwidth]{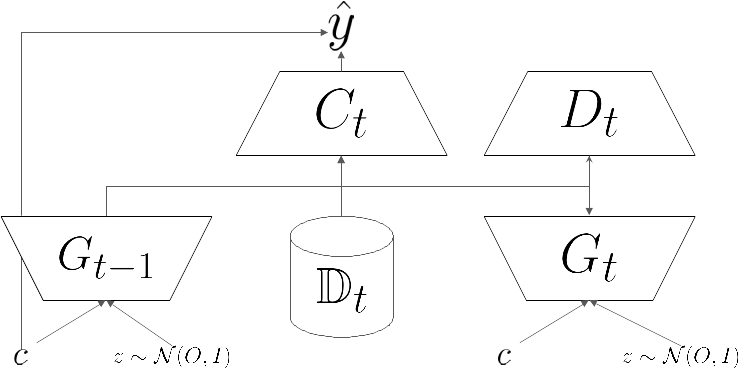}
  \caption{conditional replay}
  \label{fig:4_CL_GR:Conditional_Replay}
\end{subfigure} 

\caption[Illustration of Generative Replay and Conditional Replay for continual learning.]{\checked{Illustration of Generative Replay and Conditional Replay for continual learning. The representation is done with a GAN architecture for the generative model but it could be adapted for another one. 
The Generative replay method train the current generative model $G_t$ and classifier $C_t$ with a mixture of the current dataset $\mathbb{D}_t$ and generated data from $G_{t-1}$ and associated label $y$ given by $C_{t-1}$.
 At the end of the task only $G_t$ and $C_t$ are kept frozen as the memory of the past. For Conditional Replay, the label is imposed to the generative model, therefore, we don't need an old classifier $C_{t-1}$ but only $G_{t-1}$. At the end of the task, only $G_{t-1}$ is kept frozen. Conditional Replay is then lighter to realize.}}
\label{fig:4_CL_GR:G_Replay}
\end{figure}

\begin{itemize}

\item \textbf{EWC} We re-implemented the algorithm described in  \cite{kirkpatrick2017overcoming}, choosing
two hidden layers with 200 neurons each.

\item \textbf{Marginal replay}~ {In the context of classification, the \textit{marginal replay}  \cite{lesort2018generative, shin2017continual, wu2018memory} method works as follows: 
For each task $t$, there is a dataset $D_t$, a classifier $C_t$, a generator $G_t$ and a memory of past samples composed of a generator $G_{t-1}$ and a classifier $C_{t-1}$. The latter two allow the generation of artificial samples $D_{t-1}$ from previous tasks.
Then, by training $C_t$ and $G_t$ on $D_t$ and $D_{t-1}$, the model can learn the new task $t$ without forgetting old ones.}
At the end of the task, $C_t$ and $G_t$ are frozen and replace $C_{t-1}$ and $G_{t-1}$ (see Fig.~\ref{fig:4_CL_GR:Generative_Replay}).
In the default setting, we use the generator for marginal replay in a way that ensures a balanced distribution of classes from past tasks $D_{t-1}$, see also Fig.~\ref{fig:4_CL_GR:distribution}. This is achieved by  choosing a predetermined number of samples $N$ to be added for all tasks t, and letting the generator produce $tN$ previous samples at task $t$. Thus, 
the number of generated samples increases linearly over time. We choose to evaluate two different models for the generator: WGAN-GP as used in  \cite{shin2017continual} and the original GAN model  \cite{goodfellow2014generative} since it is a competitive baseline \cite{lesort2018training}.

\item \textbf{Conditional replay}~ The conditional replay method is derived from \textit{marginal replay}: instead of saving a classifier and a generator, the algorithm only saves a generator that can generate conditionally (for a certain class).
Hence, for each task $t$, there is a dataset $D_t$, a classifier $C_t$ and two generators $G_t$ and $G_{t-1}$ (see Fig.~\ref{fig:4_CL_GR:Conditional_Replay}).
The goal of $G_{t-1}$ is to generate data from all the previous tasks during the training on the new task. Since data is generated conditionally, samples automatically have a label and do not require a frozen classifier. We follow the same strategy as for marginal replay (previous paragraph) for choosing the number of generated samples at each task. However, conditional replay does not require this: it can, in principle, keep the number of generated samples constant for each task since it is trivially possible to generate a balanced distribution of $\frac{N}{t}$ samples per class, from $t$
different classes, via conditional sample generation.
$C_t$ and $G_t$ learn from generated data $D_{t-1}$ and $D_t$. At the end of a task $t$, $C_t$ is able to classify data from the current and previous tasks, and $G_t$ is able to sample from them also. We choose to use two different popular conditional models: CGAN described in  \cite{mirza2014conditional} and CVAE \cite{kihyuk2015learning}.

\end{itemize}

\section{Experiments}

\begin{figure}[ht]
    \centering
    \includegraphics[width=0.44\textwidth]{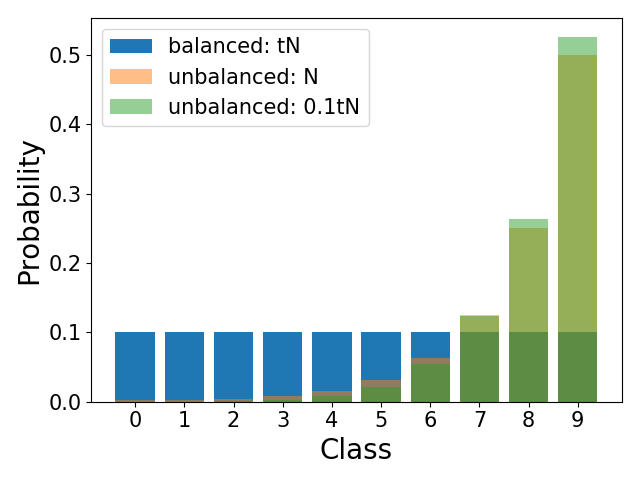}
    \caption[Illustration of classes imbalances for different replay strategies.]{Illustration of classes imbalances for different replay strategies. Why marginal replay must linearly increase the number of generated samples: distribution of classes produced by the generator of a marginal replay strategy after sequential training of 10 tasks (of 1 class each). This essentially corresponds to the \say{disjoint} type of tasks. Shown are three cases: \say{\textit{balanced}: $tN$} (blue bars) where $tN$ samples are generated for each task $t$, \say{unbalanced: $N$} (orange bars) where the number of generated samples is constant and equal to the number of newly trained samples $N$ for each task, and \say{unbalanced: $0.1 tN$} where $0.1tN$ samples are generated. We observe that, in order to ensure a balanced distribution of classes, the number of generated samples must be re-scaled, or, in other words, must increase linearly with the number of tasks.
    }
    \label{fig:4_CL_GR:distribution}
\end{figure}

We conduct experiments using all models and tasks described in the previous section. Each class (regardless of the task) is presented for 25 epochs, 
Results are presented either based on the time-varying classification accuracy on the \textit{whole} test set, or on the class (from the test set) that was presented first. In the first case, accuracy should ideally increase over time and reach its 
maximum after the last class has been presented. In the second case, accuracy should be stable if the model does not forgot or decrease over time, reflecting that some information about the first class is forgotten. 
We distinguish two major experimental goals or questions: 
\begin{itemize}
    \item Establishing the performance of the newly proposed methods (marginal replay with GAN, conditional replay with CGAN or CVAE) w.r.t. the state of the art. To this effect, we conduct experiments that increase the number of generated samples over time in a way that ensures an effectively balanced class distribution (see Fig.~\ref{fig:4_CL_GR:distribution}). We do this both for marginal and conditional replay in order to ensure a fair comparison, although technically conditional replay can generate balanced distribution even with a constant number of generated samples.
    \item Demonstrating the advantages of conditional w.r.t. marginal replay strategies, especially when only few samples can be generated, thus obtaining a skewed class distribution for marginal replay (see Fig.~\ref{fig:4_CL_GR:distribution}). 
\end{itemize} 

Results shedding light on the first question are presented in Fig.~\ref{fig:4_CL_GR:all_task_accuracy} (showing classification accuracy on whole test set over time, see Fig.~\ref{fig:4_CL_GR:first} for accuracy on first task), whereas the second question is addressed in Fig.~\ref{fig:4_CL_GR:bal} for the disjoint task.

\checked{On the other hand, we can also analyze the stability performances of the different models. Results from Fig.~\ref{fig:4_CL_GR:all_task_accuracy} presenting accuracy performance on the different tasks sequence can be compared with the FID results of generative models Fig.~\ref{fig:4_CL_GR:CL_GR_FID}. We can see a high similarity between the performance of generative models (FID) and the performance of the continual learning approach.}

\medskip

\checked{If we describe the experiments as proposed in the framework presented in Chapter \ref{chap:2_CL}, then we are in a supervised learning setting with multi-task scenario. They are 10 disjoint tasks with iid data or 5 joint tasks with iid data (for rotations and permutations tasks), with an integer oracle task label for training but not for testing (learning labels). The content update is a new concepts type (NC) for disjoint tasks and new instances for the others. The growth of used memory is less than linear and the growth of computation cost is bounded by linear growth as for generative replay described in Chapter \ref{chap:3_CL_GM}.}

\section{Results}
\label{sec:4_CL_GR:results_CL_GR}

\begin{figure}[h!]
    \centering
    \begin{subfigure}[b]{0.45\textwidth}
        \includegraphics[width=\textwidth]{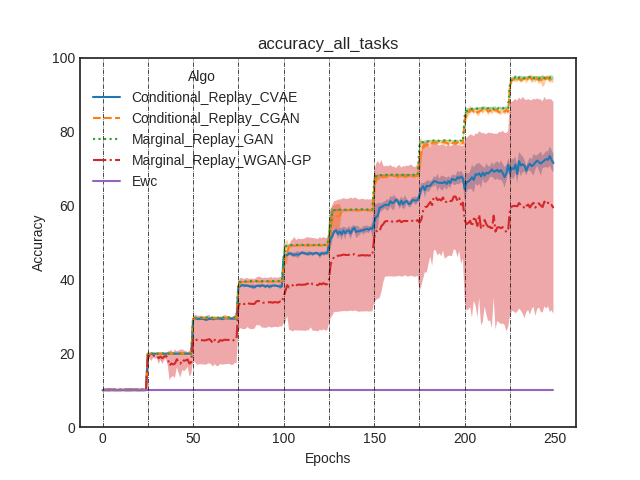}
        \caption{accuracy for MNIST disjoint task}
   \end{subfigure}
    \begin{subfigure}[b]{0.45\textwidth}
        \includegraphics[width=\textwidth]{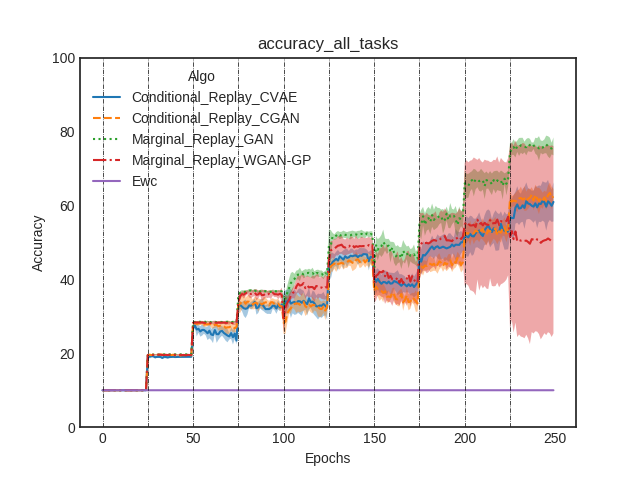}
        \caption{accuracy for Fashion MNIST disjoint task}
   \end{subfigure}
   
       \centering
    \begin{subfigure}[b]{0.45\textwidth}
        \includegraphics[width=\textwidth]{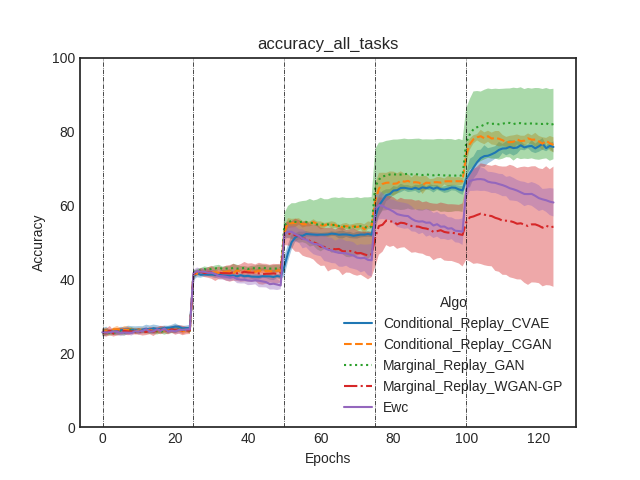}
        \caption{accuracy for MNIST permutation task}
    \end{subfigure}
   \begin{subfigure}[b]{0.45\textwidth}
       \includegraphics[width=\textwidth]{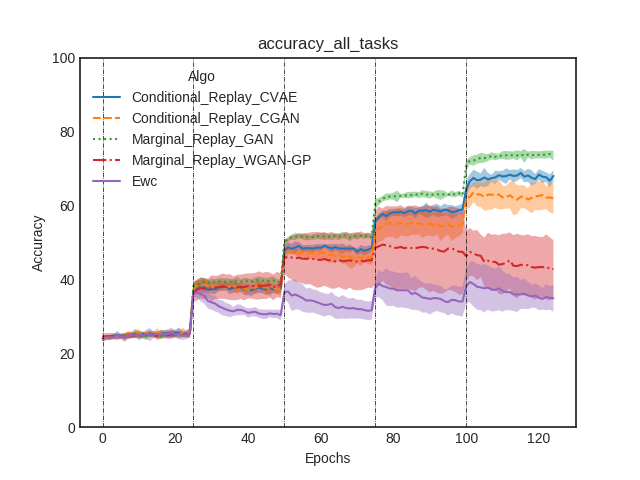}
        \caption{accuracy for Fashion MNIST permutation task}
   \end{subfigure}

       \centering
    \begin{subfigure}[b]{0.45\textwidth}
        \includegraphics[width=\textwidth]{Files/4_CL_GR/Figures/mnist_rotations_all_task_accuracy}
        \caption{accuracy for MNIST rotation task}
    \end{subfigure}
   \begin{subfigure}[b]{0.45\textwidth}
       \includegraphics[width=\textwidth]{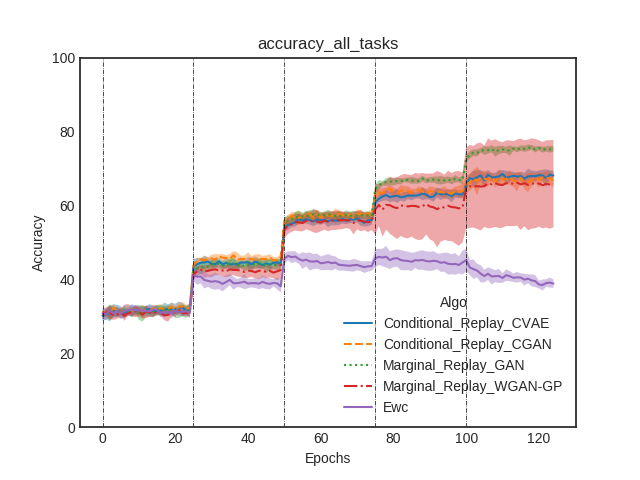}
        \caption{accuracy for Fashion MNIST rotation task}
   \end{subfigure}

   \caption[Test set accuracies during the training on different tasks]{Test set accuracies during the training on different tasks, shown for all tasks (indicated by dotted lines).}
    \label{fig:4_CL_GR:all_task_accuracy}
\end{figure}

From the experiments described in the previous section, we can state the following principal findings:

\subsection{Replay methods outperform EWC}
 As can be observed from Fig.~\ref{fig:4_CL_GR:all_task_accuracy}, the %
novel methods we propose (marginal replay with GAN and WGAN-GP, conditional replay with CGAN and conditional replay with CVAE) outperform EWC, on all tasks, sometimes by a large margin.
Particular attention should be given to the performance of EWC: while generally acceptable for rotation and permutation tasks, it completely fails for the disjoint task. This is due to the fact that there is only one class in each task, making EWC try to map all samples to the currently presented class label regardless of input, since no replay is available to include samples from previous tasks.%

\subsection{Marginal replay with GAN outperforms WGAN-GP}
The clear advantage of GAN over WGAN-GP is the higher stability of the generative models.
This is not only observable in the accuracy (see Fig.~\ref{fig:4_CL_GR:all_task_accuracy}), but also when measuring performance on the first task only during the course of 
continual learning (see Fig.~\ref{fig:4_CL_GR:first}) and in the generative model FID (see Fig.~\ref{fig:4_CL_GR:CL_GR_FID}).
\checked{We show some samples of GAN for the permutation tasks in Fig.~\ref{fig:4_CL_GR:mnist_permutation_results} (Fig.~\ref{fig:4_CL_GR:mnist_permutation} helps to understand Fig.~\ref{fig:4_CL_GR:mnist_permutation_results} \checked{by showing how  original training data look like and how they look like after applying each inverse permutation. It makes it easier to see that a single trained GAN can generate correct data in all tasks whatever the permutation.}). 
The permutation task is the hardest one for a generative model but we can see that the GAN is able to successfully generate samples from each tasks. }

\begin{figure}[t!]
    \centering
    \begin{subfigure}[b]{0.45\textwidth}
        \includegraphics*[width=\textwidth]{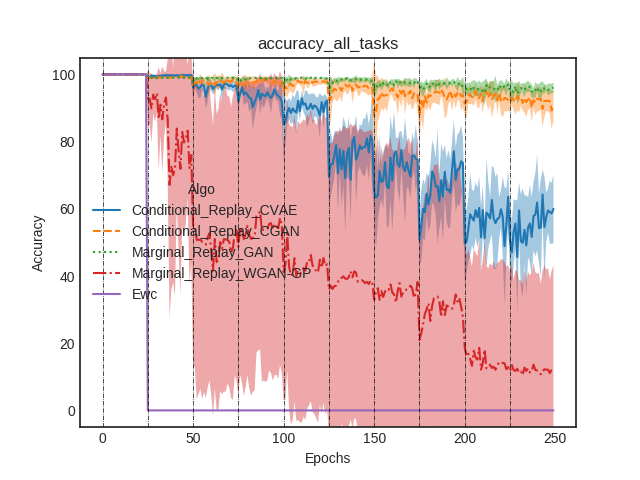}
        \caption{MNIST: disjoint task}
        \label{fig:mnist_disjoint_first_task_accuracy}
    \end{subfigure}
   \begin{subfigure}[b]{0.45\textwidth}
       \includegraphics[width=\textwidth]{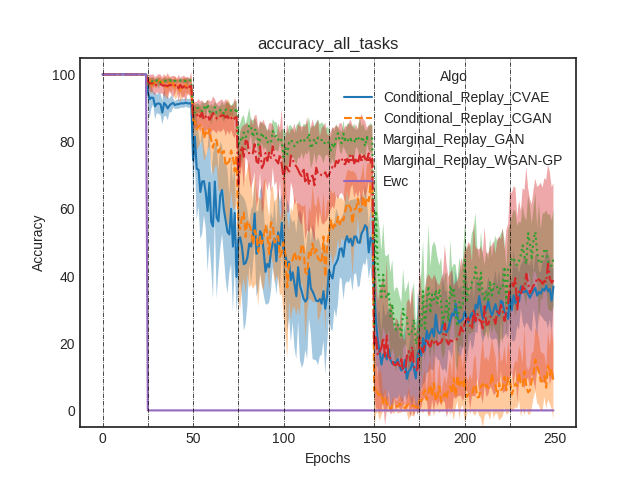}
        \caption{Fashion MNIST: disjoint task}
        \label{fig:fashion_disjoint_first_task_accuracy}
   \end{subfigure}

    \centering
    \begin{subfigure}[b]{0.45\textwidth}
        \includegraphics*[width=\textwidth]{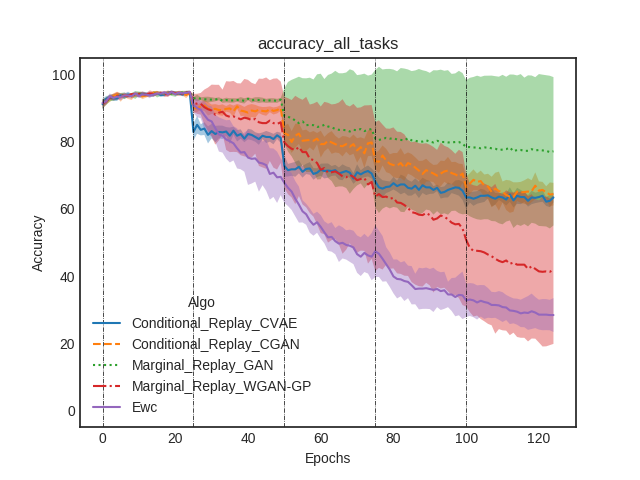}
        \caption{MNIST: permutation task}
        \label{fig:mnist_permutations_first_task_accuracy}
    \end{subfigure}
   \begin{subfigure}[b]{0.45\textwidth}
       \includegraphics[width=\textwidth]{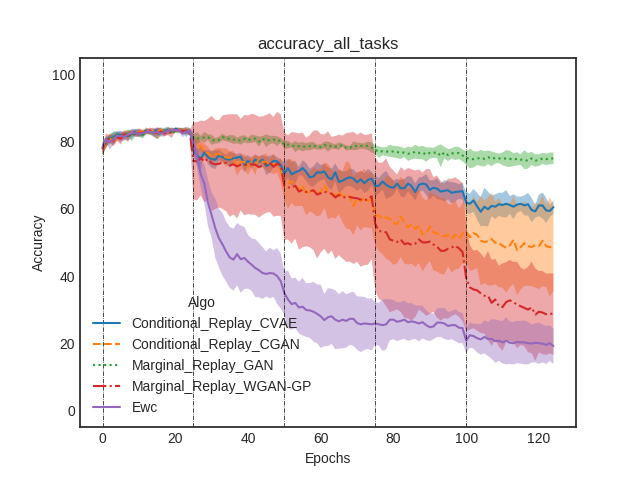}
        \caption{Fashion MNIST: permutation task}
        \label{fig:fashion_permutations_first_task_accuracy}
   \end{subfigure}

    \centering
    \begin{subfigure}[b]{0.45\textwidth}
        \includegraphics*[width=\textwidth]{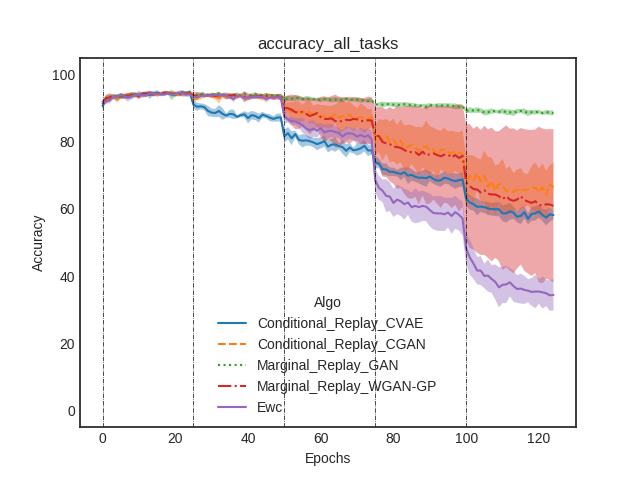}
        \caption{MNIST: rotation task}
        \label{fig:mnist_rotations_first_task_accuracy}
    \end{subfigure}
   \begin{subfigure}[b]{0.45\textwidth}
       \includegraphics[width=\textwidth]{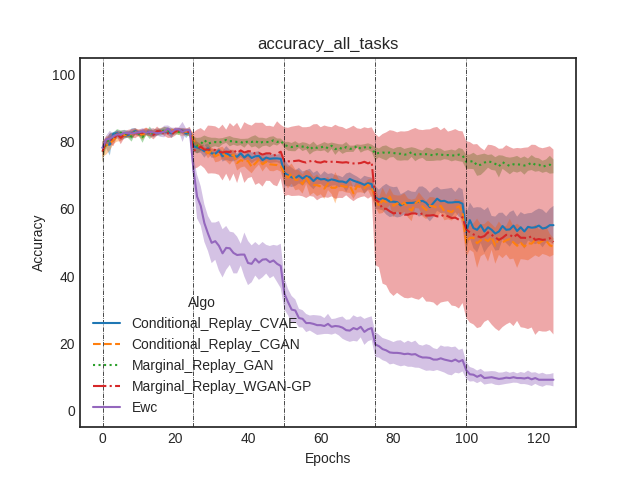}
        \caption{Fashion MNIST: rotation task}
        \label{fig:fashion_rotations_first_task_accuracy}
   \end{subfigure}
   \caption[Comparison of the accuracy of each approach on the first task.]{
   Comparison of the accuracy of each approach on the first task. This is another, very intuitive measure of how much is forgotten during continual learning. Means and standard deviations computed over 8 seeds.}
\label{fig:4_CL_GR:first}
\end{figure}

\begin{figure}[ht]
    \centering
    \begin{subfigure}[b]{0.45\textwidth}
        \includegraphics*[width=\textwidth]{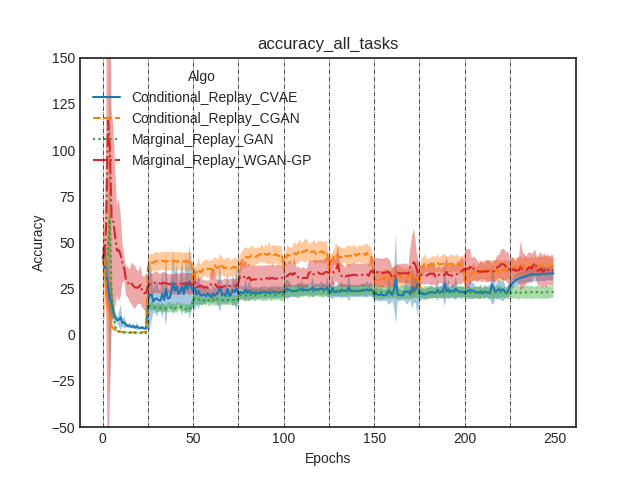}
        \caption{MNIST: disjoint task}
        \label{fig:mnist_disjoint_FID}
    \end{subfigure}
   \begin{subfigure}[b]{0.45\textwidth}
       \includegraphics[width=\textwidth]{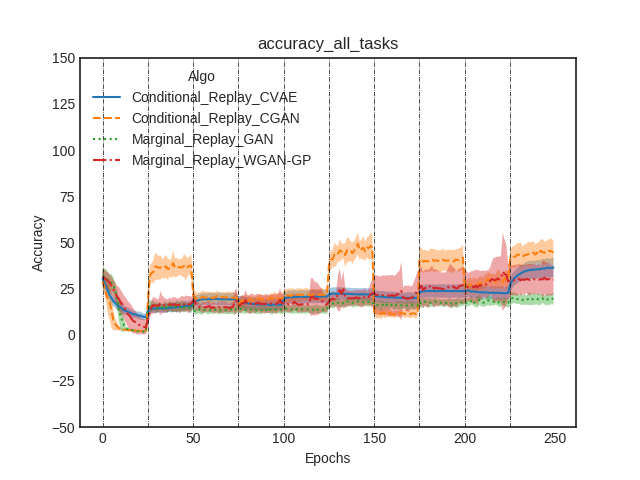}
        \caption{Fashion MNIST: disjoint task}
        \label{fig:fashion_disjoint_FID}
   \end{subfigure}
    
    \centering
    \begin{subfigure}[b]{0.45\textwidth}
        \includegraphics*[width=\textwidth]{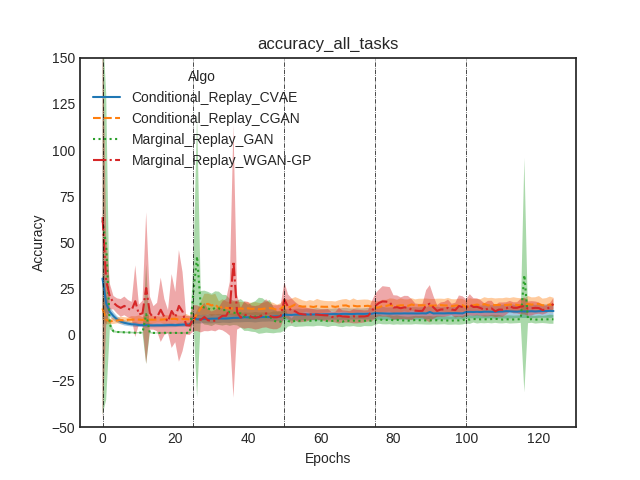}
        \caption{MNIST: permutation task}
        \label{fig:mnist_permutations_FID}
    \end{subfigure}
   \begin{subfigure}[b]{0.45\textwidth}
       \includegraphics[width=\textwidth]{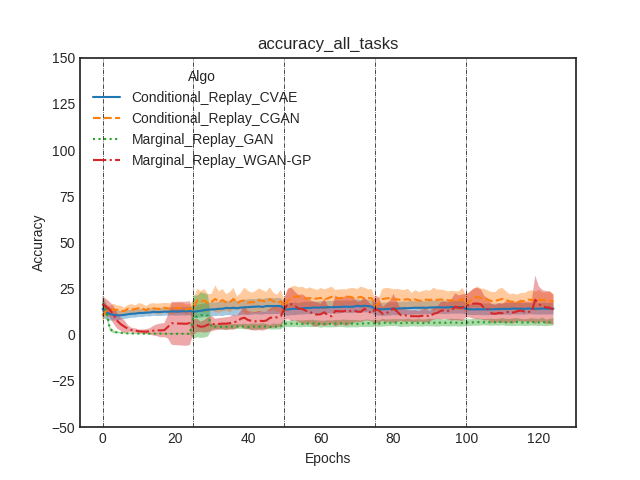}
        \caption{Fashion MNIST: permutation task}
        \label{fig:fashion_permutations_FID}
   \end{subfigure}
    
    \centering
    \begin{subfigure}[b]{0.45\textwidth}
        \includegraphics*[width=\textwidth]{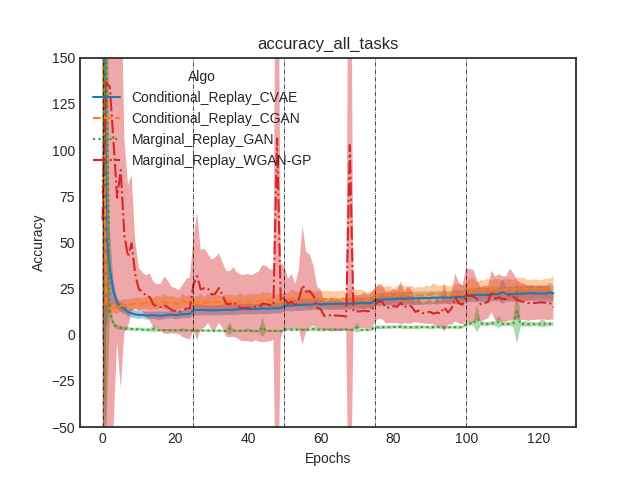}
        \caption{MNIST: rotation task}
        \label{fig:mnist_rotations_FID}
    \end{subfigure}
   \begin{subfigure}[b]{0.45\textwidth}
       \includegraphics[width=\textwidth]{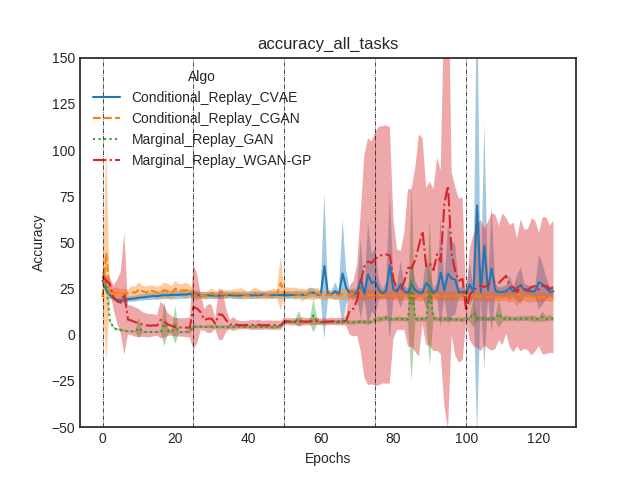}
        \caption{Fashion MNIST: rotation task}
        \label{fig:fashion_rotations_FID}
   \end{subfigure}
   
   \caption[Comparison of the FID of each approach's generative models]{Comparison of the FID of each approach's generative models. Means and standard deviations computed over 8 seeds.}
   \label{fig:4_CL_GR:CL_GR_FID}
\end{figure}

\begin{figure}
    \centering
    
    \includegraphics[width=\textwidth]{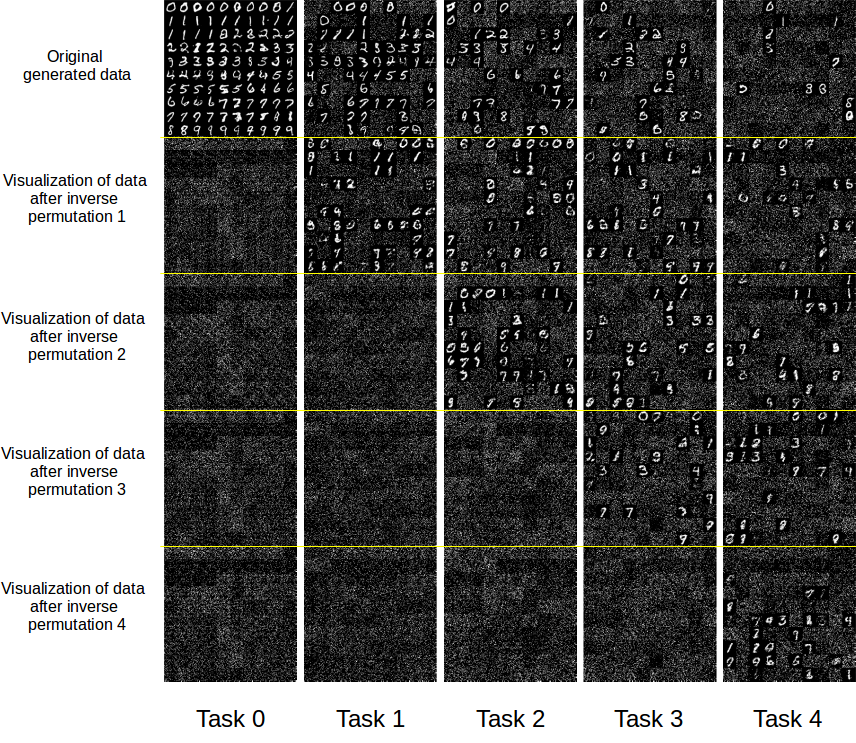}
   
   \caption[Generated samples of MNIST for the permutations tasks]{Visualization of data generated during the training of marginal replay + GAN on the MNIST permutation tasks. %
   The first line shows each task generated data. The other lines show the effect of all tasks inverse permutations applied to those data as in Fig.~\ref{fig:4_CL_GR:mnist_permutation}.
     We observe that at each task the generative model can generate data in all previous permutation spaces. Then, we have a stable retention behaviour as the number of achieved tasks increases, while data from the current task is learned successfully as well.}
   \label{fig:4_CL_GR:mnist_permutation_results}
\end{figure}

\clearpage

\subsection{Conditional replay can be run at constant time complexity}
A very important point in favour of conditional replay is run-time complexity, as expressed by the number of samples that need to be generated each time a new task is trained. Since the generators in marginal replay strategies generate samples regardless of class, the distribution of classes will be proportional to the distribution of classes during the last training of the generator, which leads to an unbalanced class distribution over time, with the oldest classes being strongly under-represented (see Fig.~\ref{fig:4_CL_GR:distribution}). This is
avoided by increasing the number of generated samples over time for marginal replay, leading to a balanced class distribution (see also Fig.~\ref{fig:4_CL_GR:distribution}) while vastly increasing the number of samples. Conditional replay, on the other hand, 
can selectively generate samples from a defined class, thus constructing a class-balanced dataset without needing to increase the number of generated samples over time. In the interest of accuracy, it can of course make sense to
increase the number of generated samples over time, just as for marginal replay. This, however, is a deliberate choice and not something required by conditional replay itself.

\subsection{Marginal replay vs Conditional Replay performances}
From Fig.~\ref{fig:4_CL_GR:all_task_accuracy}, it can be observed that
marginal replay outperforms conditional replay by a small margin. This comes at the price of having to generate a large number of samples, which will become unfeasible if many classes are involved in the retraining.

\medskip

The results of Fig.~\ref{fig:4_CL_GR:bal} show that conditional replay is superior to marginal replay when generating fewer samples at each task (more precisely: $0.1tN$ samples instead of $tN$, for task $t$ and number of new samples per task N). This can be understood quite easily: since we generate only $0.1tN$ samples instead of $tN$ samples at each task, marginal replay produces an unbalanced class distribution, see Fig.~\ref{fig:4_CL_GR:distribution}, which strongly impairs classification performance. This is a principal advantage that conditional replay has over marginal replay: generating balanced class distributions while having much more control over the number of generated samples.

\begin{figure}[h]
    \centering
    \begin{subfigure}[b]{0.45\textwidth}
       \includegraphics[width=\textwidth]{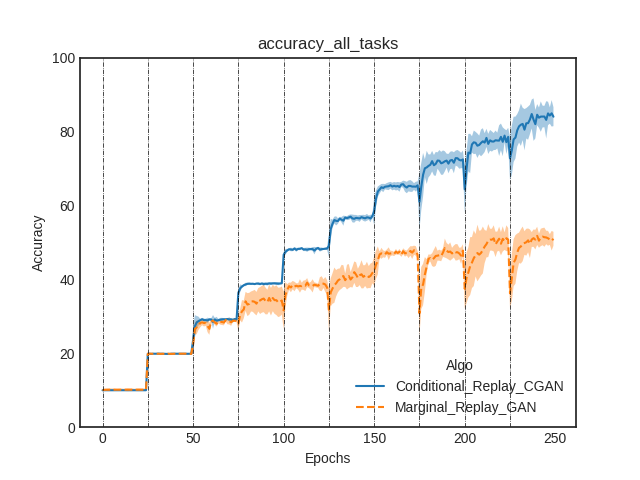}
        \caption{Unbalanced MNIST Disjoint}
        \label{fig:unbalanced_mnist}
    \end{subfigure}
   \begin{subfigure}[b]{0.45\textwidth}
       \includegraphics[width=\textwidth]{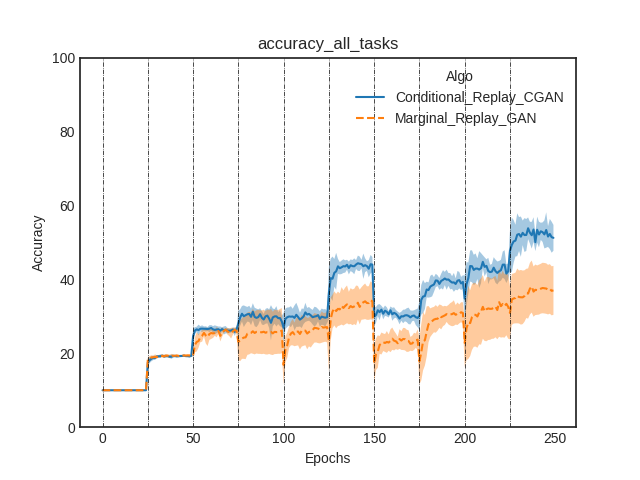}
        \caption{Unbalanced Fashion Disjoint}
        \label{fig:unbalanced_fashion}
   \end{subfigure}
   
       \centering
    \begin{subfigure}[b]{0.45\textwidth}
       \includegraphics[width=\textwidth]{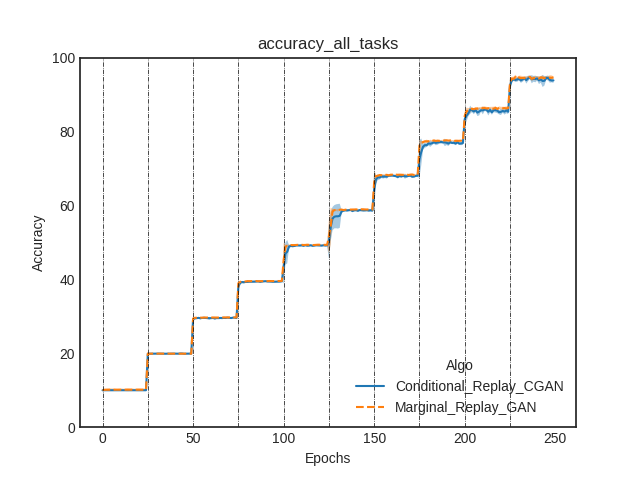}
        \caption{Balanced MNIST Disjoint}
        \label{fig:balanced_mnist}
    \end{subfigure}
   \begin{subfigure}[b]{0.45\textwidth}
       \includegraphics[width=\textwidth]{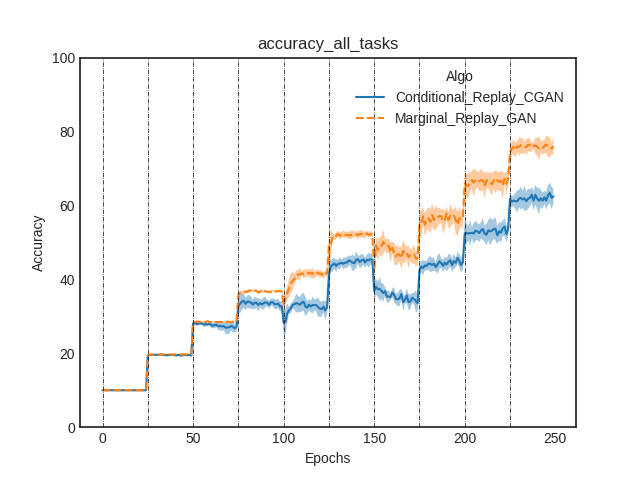}
        \caption{Balanced Fashion Disjoint}
        \label{fig:balanced_fashion}
   \end{subfigure}
   
   \caption[Comparison between conditional and marginal replay accuracy.]{Comparison between conditional and marginal replay accuracy.
   We compare final accuracy when the ratio between size of old task and size of new task is 1 (balanced) or 1/10 (unbalanced, factor was chosen empirically).}
   \label{fig:4_CL_GR:bal}
\end{figure} 

\clearpage

\section{Discussion}
\label{sec:4_CL_GR:discussion_CL_GR}

\checked{In this chapter, we have seen that to ensure the best performance of generative replay, and in replay methods in general, the most straight forward method is to replay samples while taking care  of balancing the instance of classes. %
}

\medskip

Even if balancing the classes offer the best performing results, it is not satisfying in term of computation, indeed since we should balance classes, the amount of computation needs grows linearly with the number of past classes to learn sequentially. It becomes then interesting to look for a solution to reduce the amount of replay needed to reduce the total computational cost. In our results, we show that conditional models are more likely to learn continually if the number of samples is reduced. It would be interesting to also study the impact of balancing the loss instead of the replay to see if we could also improve the computational cost.
It might also be interesting to sample more data but processing only a subset of them that could maximize the memorization as proposed in \cite{Aljundi2019Online} or find smarter methods of sampling.

\medskip

While one might argue that MNIST and FashionMNIST are too simple for a meaningful evaluation, this holds only for non-continual learning scenarios. In fact, recent articles \cite{pfulb2019a} show that MNIST-related tasks are still a major obstacle for most current approaches to continual learning under realistic conditions. So, while we agree that MNIST and FashionMNIST are not suitable benchmarks in general anymore, we must stress the difficulty of MNIST-related tasks in continual learning, thus making these benchmarks very suitable indeed in this particular context. The use of intrinsically more complex benchmarks, such as CIFAR,SVHN or ImageNet is at present not really possible since generative methods are not really good enough for replaying these data \cite{lesort2018generative} (see or experiments with CIFAR in Chapter \ref{chap:3_CL_GM}).

\medskip 
 
An interesting point is that the disjoint type tasks pose enormous problems to conventional machine learning architectures, and therefore represent a very useful benchmark for continual learning algorithms. \checked{As pointed in the Chapter \ref{chap:2b_Replay},} if each task contains a single visual class, training them one after the other will induce no between-class discrimination at all since 
every training step just \say{sees} a single class. %
Replay-based methods nicely bridge this
gap, allowing continual learning while allowing between-class discrimination. In our opinion, any application-relevant algorithm for continual learning therefore must include some form of experience replay.

\section{Conclusion}
\label{sec:4_CL_GR:conclusion}

In this chapter, we have proposed several ways of performing continual learning with replay-based models and empirically demonstrated (on novel benchmarks) their merit w.r.t. the state of the art, represented by the EWC method. A principal conclusion of this chapter is that conditional replay methods show strong promise because they have competitive performance, 
and they impose less restrictions on their use in applications. Most notably, they
can be used at constant time complexity, meaning that the number of generated samples 
does not need to increase over time, which would be problematic in applications with many tasks and real-time constraints.

\checked{Nevertheless, as noted in Chapter \ref{chap:3_CL_GM} one of the biggest limitations of generative replay is the difficulty to train generative models with real images. Generative models are generally very long to train and suffer from high instability.}

Ultimately, the goal of our research is to come up with replay-based models where the effort spent on replaying past knowledge is small compared to the effort of training with new samples.
\checked{Then, to minimize the amount of replay necessary the model needs also to limit catastrophic forgetting without replay.}

\newpage
\chapter{Replay for Policy Distillation}
\label{chap:5_DiscoRL}

\checked{In the previous chapters, we studied generative replay methods for continual learning on two aspects: the continual learning process of generative models and the generative replays applied to classification.}
\checked{We have seen the successes and failures of these uses of generative models for continual learning.}
\checked{ In this section, we study the rehearsal method for reinforcement learning. We apply a replay method to train a robot to learn several tasks sequentially and test them in a real-life setting.}

\checked{This work is a collaboration with Ren{\'{e}} Traor\'{e}, Hugo Caselles{-}Dupr{\'{e}}, Te Sun and Guanghang Cai. It has been published at NeurIPS deep Reinforcement Learning workshop as \say{DisCoRL: Continual Reinforcement Learning via Policy Distillation} \cite{Traore19DisCoRL}. The original article has been slightly modified to include additional figures and results.}

\section{Introduction}

An autonomous agent should be able to acquire and exploit knowledge through its entire life. In a sequence of learning experiences, it should therefore be able to build representations and skills online and be able to reactivate and reuse them later.
In this chapter, we focus on a setting where an agent learn skills sequentially using a single model, and use them afterwards.

This challenge is partially addressed by a sub-domain of machine learning called multi-task learning. Multi-task learning \cite{Caruana97} studies how to optimize several problems \textit{simultaneously} with a single model. However, when those problem can not be optimized at the same time, but sequentially, we identify the learning setting as a continual learning problem.
In this chapter, we address a continual learning problem of reinforcement learning (RL). In this setting, each learning experience is called task, and each task solution is called policy.

\begin{figure}[t]
    \centering
    \includegraphics[width=0.8\textwidth]{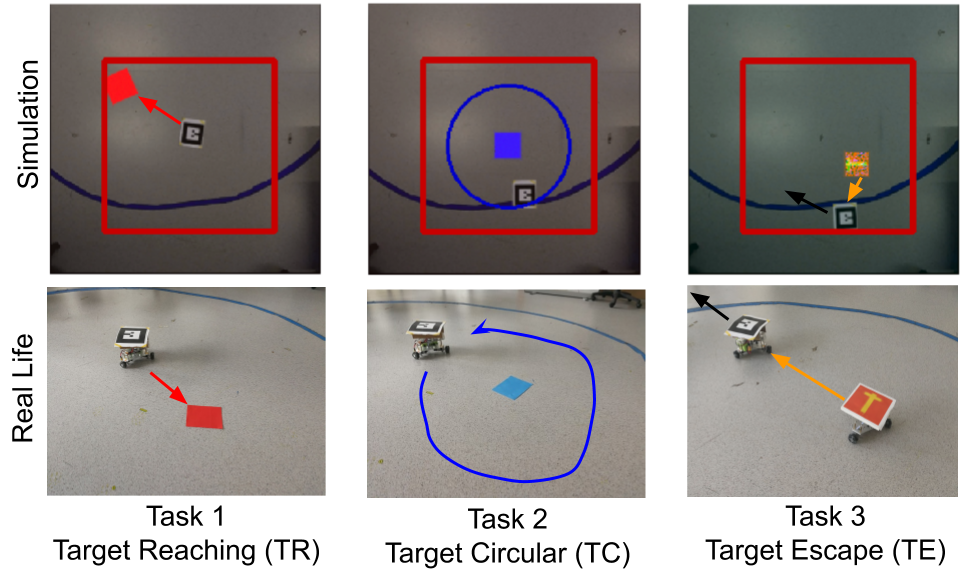}
        \caption[Image of the three robotics tasks.]{Image of the three tasks, in simulation (top) and in real life (bottom) sequentially experienced. Learning is performed in simulation, real life is only used at test time. }
    \label{fig:5_DiscoRL:real-life-tasks}
\end{figure}

To propose a learning setting compatible with a real autonomous agent, we define three simulated robotics tasks to solve sequentially. At each task, the agent need to learn a policy based on a reward function and a RL algorithm.

As presented in Section \ref{sub:1b_ML:RL}, RL is a popular framework to learn robot controllers that also has to face the continual learning (CL) %
challenges.
In order to learn a multi-task reinforcement learning policy continually %
, we use a method called policy distillation \cite{Rusu16distillation} that allows to transfer several policies learned sequentially into one in a single model.
To validate our approach, we evaluate the final results on the three simulated learning settings but also in a real life settings matching the simulation (Figure \ref{fig:5_DiscoRL:real-life-tasks}). 
It is important to note that, at test time, the agent does not have access to a task label to determine which policy to run, and thus, it needs to figure it out by itself from its observations. \checked{As discussed in the previous chapters, it is an important feature for the autonomy of decision making.}

Our contribution are:

\begin{itemize}

\item We propose \textbf{DisCoRL} (\textit{Distillation for Continual Reinforcement learning}): a modular, effective and scalable pipeline for continual RL. This pipeline uses policy distillation for learning without forgetting \cite{French99}%
, without access to previous environments, and without task labels. Our results show that the method is efficient and learns policies transferable into real-life scenarios.

\item \checked{We explore various sampling strategies for policy distillation and compare their performances for task transfer.}

\end{itemize}

The chapter is structured as follows: Section \ref{sec:5_DiscoRL:Background} introduces related work,  Section \ref{sec:5_DiscoRL:Approach} details the methods utilized, Section \ref{sec:5_DiscoRL:experiments} describes the robotics setting and  tasks, Section \ref{sec:5_DiscoRL:Results} presents the experiments performed, and Section \ref{sec:5_DiscoRL:discussion} concludes with future insights from our experiments.

\section{Background}
\label{sec:5_DiscoRL:Background}

\checked{In this section, we present some background material about multi-task reinforcement learning as well as reinforcement learning in continual learning and robotics.}

\subsection{Multi-task RL}
The objective of Multi-task learning (MTL) \cite{Caruana97} is to learn several tasks simultaneously; generally by training tasks in parallel with an unique model.
Therefore, multi-task RL aims at constructing one single policy that can solve a number of different tasks. Note how in classification this problem is quite simple, as data from all tasks just have to be shuffled randomly and can then be learned all together at once. However, in RL environments, data is sampled on sequences that can not be shuffled randomly with all other environments because the environments are not accessible simultaneously. Learning multiple tasks at once is thus more complicated.

Policy distillation \cite{Rusu16distillation} can be used to merge different policies into one single module/network.  This approach uses two models, a trained policy (the teachers) to annotate data with soft-labels%
, and a model to learn from the former (the student). The student is trained in a supervised manner with the soft-labels. The soft-annotation %
helps the student to learn faster than the teacher did~\cite{Furlanello18}.
Policy distillation can then be used to learn several policies separately and simultaneously, and distill them into a single model as in the \textit{distral} algorithm \cite{teh2017distral}. In our approach, we also use distillation but we do not keep the teacher model. We just label a set of data and then we don't need to keep the teacher model anymore and we can delete it to save memory space. %
Furthermore, tasks are learned sequentially, and not simultaneously.
Other approaches such as SAC-X \cite{riedmiller2018learning} or HER \cite{andrychowicz2017hindsight} take advantage of Multi-task RL by learning auxiliary tasks in order to help learning a main task. 
This approach is extended in the CURIOUS algorithm \cite{Colas18Curious} which selects tasks to be learned that improve an absolute learning progress metric the most.

\subsection{Continual Learning}
\checked{We present here a brief recall of Chapter \ref{chap:2_CL} on continual learning for reinforcement learning.}

Continual learning (CL) is the ability of a model to learn new skills without forgetting previous knowledge. It is in many ways similar to multi-task learning but with the difference past task can not be accessed directly.
In our context, it means learning several tasks sequentially and being able to solve any of the learned tasks at the end of the sequence. 

In continual RL, several approaches have been proposed, such as the use of \textit{Progressive Nets} in~\cite{Rusu16sim2real}, \textit{Elastic Weight Consolidation (EWC)} \cite{kirkpatrick2017overcoming}, \textit{Progress And Compress} (P\&C) \cite{schwarz2018progress}, or \textit{CRL-Unsup} \cite{lomonaco2019continual}. 
However, they either need a task indicator at test time to choose which policies  to run or, they have some hyper-parameter difficult to tune during a continual learning training, such as the importance of the Fisher information matrix in EWC. Our method does not add any new hyper-parameter to tune during the sequence of tasks and does not need a task label at test time.

\subsection{RL in Robotics}
\checked{As discussed in Chapter \ref{chap:2_CL},} applying RL to real-life scenarios such as robotics is still a major challenge. 

One of the main problem in this setting is that sampling data, and a fortiori learning, is costly. Therefore sample efficiency and stability in learning are highly valuable.
One common approach to reduce training cost, is training policies in simulation and then deploying them in real-life hoping that they will successfully transfer, considering the gap in complexity between simulation and the real world. Such approaches are termed \textit{Sim2Real} \cite{Golemo19}, and have been successfully applied \cite{christiano2016transfer, matas2018sim} in different scenarios. One of these approaches is Domain Randomization \cite{tobin2017domain}, which we use in this chapter. This technique trains policies in numerous simulations that are randomly different from each other (different background, colors, etc.). Using this technique, the transfer to real life is easier.

Another method we also exploit is to first learn a state representation \cite{Lesort18} to compress the observation into a low dimensional embedding and secondly learn the policy on top of this representation. This method helps to improve sample efficiency and stability of RL algorithms \cite{raffin2019decoupling} and thus can make them directly applicable in real life.

Others have tried to train a policy directly on real robots, facing the hurdle of the lack of sample efficiency of RL algorithms. SAC-X \cite{riedmiller2018learning} is one example that takes advantage of multi-task learning to improve efficiency, by simultaneously learning the policy and a set of auxiliary tasks to explore its observation space - in search for sparse rewards of the externally defined target task.

In the literature, most approaches focus on the single-task or simultaneous multi-task scenario. In this chapter, we attempt to train a policy on several tasks sequentially and deploy it in real life by combining policy distillation, training in simulation and state representation learning.

\section{Approach}
\label{sec:5_DiscoRL:Approach}

In this section we present our approach towards continual reinforcement learning for a sequence of vision based tasks. We assume that observations visually allow to recognize the current task from other tasks. We first explain how we learn a single task by combining state representation learning (SRL) \cite{Lesort18} and reinforcement learning (RL), then how each task is incorporated in the continual learning pipeline. Finally, we present how we evaluate the full pipeline. 

\subsection{Learning one task}
\label{subsec:5_DiscoRL:oneTask}

\begin{figure*}
\centering
\includegraphics[width=0.7\textwidth]{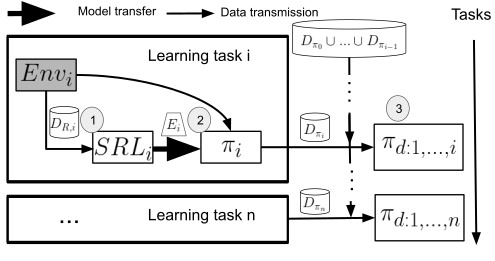}
\caption[Overview of our full pipeline for Continual multi-tasks Reinforcement Learning.]{Overview of our full pipeline for Continual Reinforcement Learning. White cylinders are for datasets, gray squares for environments, and white squares for learning algorithms, whose name corresponds to the model trained. 
Each task $i$ is learned sequentially and independently by first generating a dataset $D_{R,i}$ with a random policy to train a state representation with an encoder $E_i$ with an SRL method (1), then we use $E_i$ and the environment to learn a policy $\pi_i$ in the state space (2). Once trained, $\pi_i$ is used to create a distillation dataset $D_{\pi_i}$ that acts as a memory of the learned behaviour. 
All policies are finally compressed into a single policy $\pi_{d:{1,..,i}}$ by merging the current dataset $D_{\pi_i}$ with datasets from previous tasks $D_{\pi_1} \cup ... \cup D_{\pi_{i-1}}$ and using distillation (3).} 

\label{fig:5_DiscoRL:overview}
\end{figure*}

Each task $i$ is solved by first learning a state representation encoder $E_i$ in order to compress input images into a representation of the important underlying factor of variation.
This step allows to reduce the input space for the reinforcement learning algorithm and makes it learn more efficiently~\cite{raffin2019decoupling}.
To train this encoder, as shown in Fig. \ref{fig:5_DiscoRL:overview} (left), we sample data from the environment $Env_i$ 
with a random policy. We call this dataset $D_{R,i}$.
$D_{R,i}$ is then used to train the SRL model composed of an \textit{inverse model} and an \textit{auto-encoder}. The inverse model is trained to predict the action $a_t$ that led to transition from state $s_t$ to $s_{t+1}$, both extracted from respective observations $o_t$ and $o_{t+1}$ by the auto-encoder using $E_i$. The auto-encoder is additionally trained to reconstruct the observations from the encoded states. The architecture is motivated by the results from \cite{raffin2019decoupling}, and illustrated in the Figure \ref{fig:5_DiscoRL:split-model}. 

\usetikzlibrary{arrows}
\usetikzlibrary{decorations.markings}
\usetikzlibrary{arrows,calc,fit, patterns}
\newcommand{\mygrid}{\tikz{\draw[step=0.5cm] (0,0)  grid (0.5,1.5);}}

\begin{figure}[ht!]
\centering

\resizebox{0.65\textwidth}{!}{%
\begin{tikzpicture}[
roundnode/.style={circle, draw=black!60, fill=green!0, very thick, minimum size=10mm},
roundnode2/.style={circle, draw=black!60, fill=black!20, very thick, minimum size=10mm},
squarednode_st/.style={rectangle, draw=black!60, fill=green!20, very thick, minimum size=10mm},
squarednode/.style={rectangle, draw=black!60, fill=black!20, very thick, minimum size=10mm},
squarednode_st2/.style={rectangle, draw=black!60, very thick, minimum width=10mm, minimum height = 2cm},
invisible/.style={rectangle , draw=black!0, fill=green!0, very thick, minimum size=10mm},
squarednode_img/.style={rectangle, draw=black!60, fill=black!20, very thick, minimum size=10mm},
container_AE/.style={draw, rectangle, draw=green!60, dashed, inner sep=1em},
container_Inv/.style={draw, rectangle, draw=blue, dashed, inner sep=1em},
]

\node[invisible]        (base)        {};
\node[invisible]        (base2)        [above=of base] {};

\node[invisible]        (obs)        {};
\node[invisible]        (obs2)        [above=of obs] {};

\node[roundnode]        (hidden_obs2)        [on grid,above=of base2] {$I_{t+1}$};
\node[roundnode]        (hidden_obs)        [on grid,below=of base] {$I_t$};
\node[invisible]        (hidden_obs3)        [on grid,above=of base] {};

\node[invisible]        (hidden_state)        [right=of hidden_obs3,draw] {};
\node[invisible]        (hidden_state)        [on grid,right=of hidden_state,draw] {};

\node[squarednode_st2] (anode) [right=of hidden_obs,draw, pattern=horizontal lines light gray]{};
\node[] (label) [on grid, above=0.8cm of anode]{$s_{t}$};

\node[squarednode_st2] (anode2) [right=of hidden_obs2,draw, pattern=horizontal lines light gray]{};
\node[] (label2) [on grid, above=0.8cm of anode2]{$s_{t+1}$};

\node[invisible]        (center3)        [right=of hidden_state] {};

\node[squarednode]        (recon_act)        [on grid, above=of center3, draw] {$\hat{a}_t$};
\node[squarednode]        (recon_img)        [on grid,below=of center3,draw] {$\hat{I}_t$};

\node[roundnode]        (img)        [right=of recon_img,draw] {$I_t$};
\node[roundnode]        (act)        [right=of recon_act,draw] {$a_t$};

\node[invisible]        (L_img)        [right=of img,draw] {$\mathcal{L}_{Reconstruction}$};
\node[invisible]        (L_act)        [right=of act,draw] {$\mathcal{L}_{Inverse}$};

\draw[-{Latex[length=3mm,width=2mm]}] (hidden_obs.east) -- (anode.west);
\draw[-{Latex[length=3mm,width=2mm]}] (hidden_obs2.east) -- (anode2.west);

\draw[-{Latex[length=3mm,width=2mm]}] (anode2.east) -- (recon_act.west);
\draw[-{Latex[length=3mm,width=2mm]}] (anode.east) -- (recon_act.west);

\draw[-{Latex[length=3mm,width=2mm]}] (anode.east) -- (recon_img.west);

\node[container_AE, fit=(recon_act)(act)] (fwd) {};
\node[container_Inv, fit=(recon_img)(img) ] (ae) {};
\end{tikzpicture}
}

\caption[Illustration of \textit{SRL Combination} model]{\textit{SRL Combination} model: combines the prediction of an image $I$'s reconstruction loss and an inverse dynamics model loss in a state representation $s$. Arrows represent inference, dashed frames represent losses computations, rectangles are state representations, circles are real observed data, and squares are model predictions; $t$ represents the timestep.}

\label{fig:5_DiscoRL:split-model}
\end{figure}
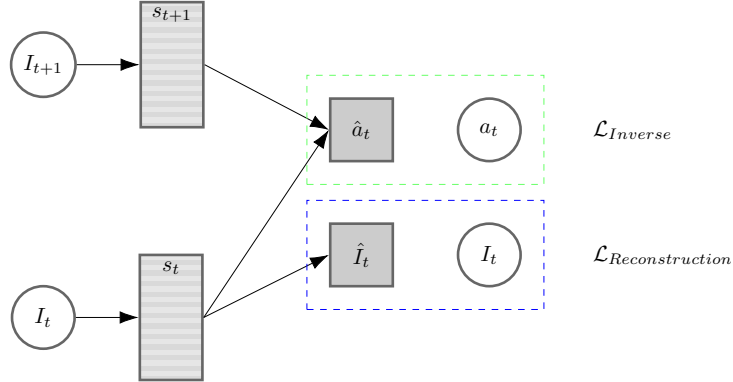

Once the SRL model is trained, we use its encoder $E_i$ to provide features as input of a policy $\pi_i$\footnote{Architecture available at: \url{https://github.com/araffin/srl-zoo/blob/438a05ab625a2c5ada573b47f73469d92de82132/models/models.py\#L179-L214}} trained using RL. We also experimented to learn the policy directly in the raw pixel space but, as shown in \cite{raffin2019decoupling}, it was less sample efficient. 

Once $\pi_i$ is learned, we use it to generate sequences of on-policy observations with associated actions, which will eventually be used for distillation (Fig. \ref{fig:5_DiscoRL:overview}, right). We call this the distillation dataset $D_{\pi_i}$. We generate $D_{\pi_i}$ the following way: we randomly sample a starting position and then let the agent generate a trajectory. At each step we save both the observation and associated action probabilities. We collect the shortest sequences maximizing the reward for an episode.
We also experiment to generate $D_{\pi_i}$ with a regular sampling and a random policy but annotated with $\pi_i$ to compare results, as detailed in Section \ref{sec:5_DiscoRL:sampling-strategies-section}.

From each task we only keep dataset $D_{\pi_i}$. As soon as we change task, $D_{R,i}$ and $Env_i$ are not available anymore. $D_{\pi_i}$ is split into a training set and a validation set.

\subsection{Learning continually}
\label{subsec:5_DiscoRL:continual_learning}

\checked{Learning policies independently ensures that learning a new policy will not degrade the previous policies. Nevertheless, it does not prevent forward transfer between tasks, since models are initialized with weights from the previous task.}

In order to learn continually, we adapt policy distillation \cite{Rusu16distillation} to a continual learning setting. The distillation consists of training a student policy to imitate a teacher policy. In our case, a student model
learns from a teacher policy the action probabilities associated to each observation. 
Each dataset $D_{\pi_i}$ allows to distill the policy $\pi_i$  (the \textbf{teacher} model) into a new network $\pi_{d:i}$ (the \textbf{student} model). 
In classic distillation, both data and models are saved 
, however saving just soft-labeled data is a lighter solution adapted to a continual setting. 

 With the aggregation of several distillation datasets $D_{\pi_i}$, we can distill several policies into the same network that can achieve all tasks (Fig.\ref{fig:5_DiscoRL:overview}, bottom right). By extension of the previous nomenclature, we denote $\pi_{d:1,..,n}$ a model where  policies $\pi_1$ ... $\pi_n$ have been distilled in. 
 When distilling all policies into the student, we select our best models with early stopping using the validation set of $D_{\pi_i}$, and test later in simulation and in real life settings.

 Since we assume that observations visually allow to recognize the current task, $\pi_{d:1,..,n}$ is able to choose the right action for the current task without a task indicator.

\medskip

The method, termed \textit{DisCoRL} for \textit{Distillation for Continual Reinforcement learning}, allows to learn continually several policies while minimizing forgetting. Regarding scalability, saving data from all past experiments may not look ideal if there is a high number of tasks. However, this solution is highly effective for remembering
and letting the reinforcement learning algorithm be absolutely free to learn a new policy without regularization. It is worth mentioning that RL is the real bottleneck in the whole process: Dataset $D_{\pi_i}$ contains approximately 10k samples per task, which allows to perform the distillation quickly, relative to how long and computationally expensive RL is (few minutes needed to learn $\pi_{d:i}$ while several hours are needed to learn $\pi_i$). Thus, in this context, it is better not to curb RL with regularization. Indeed, as will be explained in Section \ref{subsec:5_DiscoRL:neg_res}, we tried several regularization based approaches that were not successful.

\bigskip

\checked{If we describe the experiments as proposed in the framework presented in Chapter \ref{chap:2_CL}, then we are in a reinforcement learning setting, multi-task (MT) scenario. We have 3 disjoint tasks with non-iid data and we have an integer oracle task label for training but not for testing (learning labels).
Our setting fall in the NIC (New Instances and New Concepts) content update type for each task. Our approach can be classified into the rehearsal family of approaches where memory is saved as data points.
The growth of memory is linear per number of tasks and the growth of computation is less than linear growth. Indeed, the computation is mostly spent on learning the tasks using RL and not much on the continual learning model.}

\section{Experiments}
\label{sec:5_DiscoRL:experiments}

We apply our approach to learn continually three 2D navigation tasks applicable in real life. The software related to our experimental setting is available online\footnote{\url{https://github.com/kalifou/robotics-rl-srl}}.

\subsection{Robotic setup}

The experiments consists of 2D navigation tasks using a 3 wheel omni-directional robot similar to the 2D mobile navigation in \cite{Raffin18}. The input image is a top-down view of the floor and the robot is identified by a black QR code. The room where the real-life robotic experiments are performed is lighted by surroundings windows and artificial illumination and is subject to illumination changes depending on the weather and time of the day. 
The robot uses 4 high level discrete actions (move left/right, move up/down in a Cartesian plane relative to the robot) rather than motor commands. 

We simulate the experiment to increase sampling and learning speed. The simulation is performed by artificially moving the robot picture inside the background image according to the chosen actions.
We use domain randomization \cite{tobin2017domain} to improve the stability and facilitate transfer to the real world: during RL training, at each time-step, the color of the background is randomly changed.

\subsection{Continual learning setup}
\label{sub:5_DiscoRL:tasks_presentation}

Our continual learning scenario is composed of three similar environments, where the robot is rewarded according to the associated task (Fig. \ref{fig:5_DiscoRL:real-life-tasks}). 
In all environments, the robot is free to navigate for up to 250 steps, performing only discrete actions  
within the boundaries identified by a red line. Each task is associated to a visual target, which color depends on the task. This way, the controller can automatically infer which policy it needs to run and thus, does not need task labels at test time.

\begin{itemize}

\item \textbf{Task 1.} The task of environment 1 is named Target Reaching (TR). The robot gets at each time-step $t$ a positive reward $+1$ for reaching the target (red square), 
a negative reward $-1$ for bumping into the boundaries, and no reward otherwise.

\item \textbf{Task 2.} The task of environment 2 is named Target Circling (TC). The robot gets at each time-step $t$ a reward $R_t$ defined in Eq. \ref{eq:5_DiscoRL:reward-circular-task} (where $z_t$ is the 2D coordinate position with respect to the center of the circle) designed for agents to learn the task of circling around a central blue tag. This reward is highest when the agent is both on the circle (red (first) square in Eq. \ref{eq:5_DiscoRL:reward-circular-task}), and has been moving for the previous $k$ steps (blue, second square).
An additional penalty term of $-1$ is added to the reward function in case of bump with the boundaries (last, green square). A coefficient $\lambda=10$ is introduced to balance the behaviour.

\tcbset{reward-moving/.style={no shadow,colframe=blue, boxrule=1pt,frame style={opacity=0.5}}} 
\tcbset{reward-circular/.style={no shadow,colframe=red, boxrule=1pt,frame style={opacity=0.5}}} 
\tcbset{bump/.style={no shadow,colframe=green, boxrule=1pt,frame style={opacity=0.5}}} 
 
\begin{equation}
R_t =
\lambda *
\tcbhighmath[reward-circular]{(1 - \lambda (\|z_t\| -  r_{circle}) ^2)}
*  
\tcbhighmath[reward-moving]{\|z_t -z_{t-k}  \|_{2}^2}
+ 
\lambda ^ 2 *
\tcbhighmath[bump]{R_{t, bump}}
\label{eq:5_DiscoRL:reward-circular-task}
\end{equation}

\item \textbf{Task 3.} The task of environment 3 is named Target Escaping (TE). Robot A is being chased down by another robot B with an orange tag. Robot B is hard-coded to follow robot A, and robot A has to learn to escape using RL. Robot A gets at each time-step $t$ a reward of $+1$ if it's far enough from robot B, otherwise, if it is in the range of robot B, it gets a reward of $-1$. Additionally, robot A gets a negative reward $-1$ for bumping into the boundaries.

\end{itemize}

All RL tasks are learned with PPO algorithm \cite{schulman2017proximal} %
 \checked{from \textit{stable baselines} \cite{Hill2018Stable}} and the same state representation learning (SRL) model, as described in Section \ref{subsec:5_DiscoRL:oneTask}. We select the model architecture as in \cite{Raffin18} for RL and SRL.
 The input observations of all models are RGB images of size $224 * 224 * 3$.

\subsection{Dataset generation}

\checked{%
For each task, a dataset of (image, annotation) pairs is created after the teacher has been trained. The images are the agent observations starting from a random point. The trajectory are sampled using the different strategies described in the next section. The annotations are the soft-labels predicted by the teacher.}

\checked{Each tasks may use slightly different sampling methods.}
While generating on-policy datasets $D_{\pi 1}$ (see Section \ref{subsec:5_DiscoRL:oneTask})
for task 1 (TR), we allow the robot to perform a limited number of contacts with the target to reach ($N_{contacts}=10$) in order to mainly preserve the frames associated with the correct reaching behaviour. There are no such additional constraints when recording for task 2 (TC) or 3 (TE), the limit is the standard episode length, i.e. 250 time-steps. 

\checked{Different tasks will lead to different coverage of the space while sampling, therefore we adapted slightly the sampling for task 1 to maximize the samples space coverage, by modifying the sampling end criterion.}%

\section{Results}
\label{sec:5_DiscoRL:Results}

\checked{In this section, we present how we select the best strategy for sampling and distilling policy.}
Then, we use these choices to present our main result: the distillation of three tasks continually into a single policy that can achieve the three tasks both in simulation and real-life. We provide a supplementary video of this policy deployed in real-life on the robot showing the successful behaviours at \url{https://youtu.be/mzUigGWEfbU}. We also present the different strategies we tried but that did not work in our setting.

\subsection{Sampling and Distilling Methods}
\label{sec:5_DiscoRL:sampling-strategies-section}

\checked{This section present the different distillation and sampling methods experimented and the results.}

\subsubsection{Distillation strategies:} Distillation loss minimizes the difference between the student model's output and the teacher model's output for the same input. As in the policy distillation paper \cite{Rusu16distillation}, we investigate variations of the loss function: 

\begin{itemize}
\item Mean Squared Error loss:

 \begin{equation}
 \label{eq:5_DiscoRL:5_DiscoRL:mse}
     \mathcal{L}_{MSE}(x,y) = \mathbb{E} \left[||x - y||_{2}^{2} \right]
 \end{equation}

\item Kullback-Lieber divergence, and Kullback-Lieber divergence with temperature smoothing:

 \begin{equation}
 \label{eq:5_DiscoRL:kl}
     \mathcal{L}_{KL,\tau}(p|q) = \mathbb{E} 
     \left[ \text{softmax} 
     \left(
      \frac{p}{\tau} 
      \right) ln 
      \left( \frac{\text{softmax}
      (\frac{p}{\tau})}{\text{softmax}(q)}
       \right)
        \right]
 \end{equation}

\end{itemize}
 
\begin{table}
\centering
\begin{tabular}{l|c}
\textbf{Distillation loss} & \textbf{Student performance ($\pm$ std)} \\ \hline
MSE & 0.71 ($\pm$ 0.22) \\ \hline
KL ($\tau = 1$) & 0.76 ($\pm$ 0.14) \\ \hline
KL ($\tau = 0.1$) & 0.68 ($\pm$ 0.18) \\ \hline
KL ($\tau = 0.01$) & \textbf{0.77 ($\pm$ 0.13)}
\end{tabular}
\caption[Mean normalized performance of a student policies]{Mean normalized performance\footnote{Normalization is done by reducing the episode reward between 0 and the max. possible performance to a reward in [0,1].} of a student policies trained with distillation using 4 different loss functions. The student policy is trained to perform all three tasks. Kullback-Lieber divergence with $\tau=0.01$ performs best.}
\label{tab:5_DiscoRL:distillation-losses}
\end{table}

We run a performance comparison of the different losses by computing the mean normalized performance of a student policy trained to perform all three tasks (Tab.~\ref{tab:5_DiscoRL:distillation-losses}). Using the Kullback-Lieber divergence loss function with temperature smoothing with $\tau = 0.01$ is best, and optimizing the temperature parameter yields a small performance boost. This result is coherent with \cite{Rusu16distillation} where they reach the same conclusion.

\subsubsection{Data sampling strategies:}

We evaluate the effect of three different sampling strategies to create $D_{\pi_i}$ for policy distillation. Data sampling is a key component as the sampled dataset should be as small as possible but contain sufficient information for student model training. 
The strategies involved for data generation are:

\begin{itemize}

\begin{figure}
    \centering
    \begin{subfigure}[b]{0.3\textwidth}
    \centering
        \includegraphics[width=\textwidth]{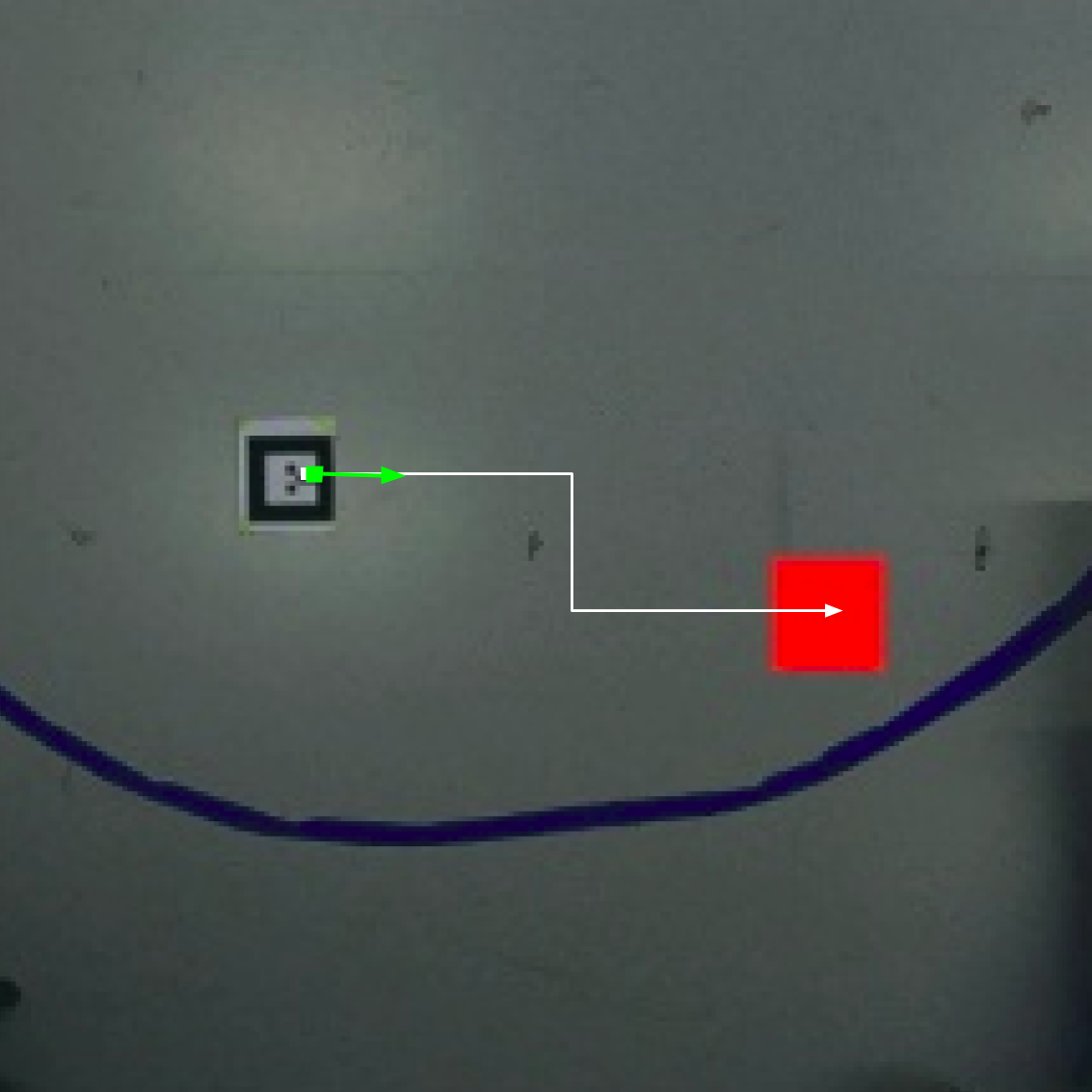}
    \end{subfigure}
\begin{subfigure}[b]{0.3\textwidth}
    \centering
        \includegraphics[width=\textwidth]{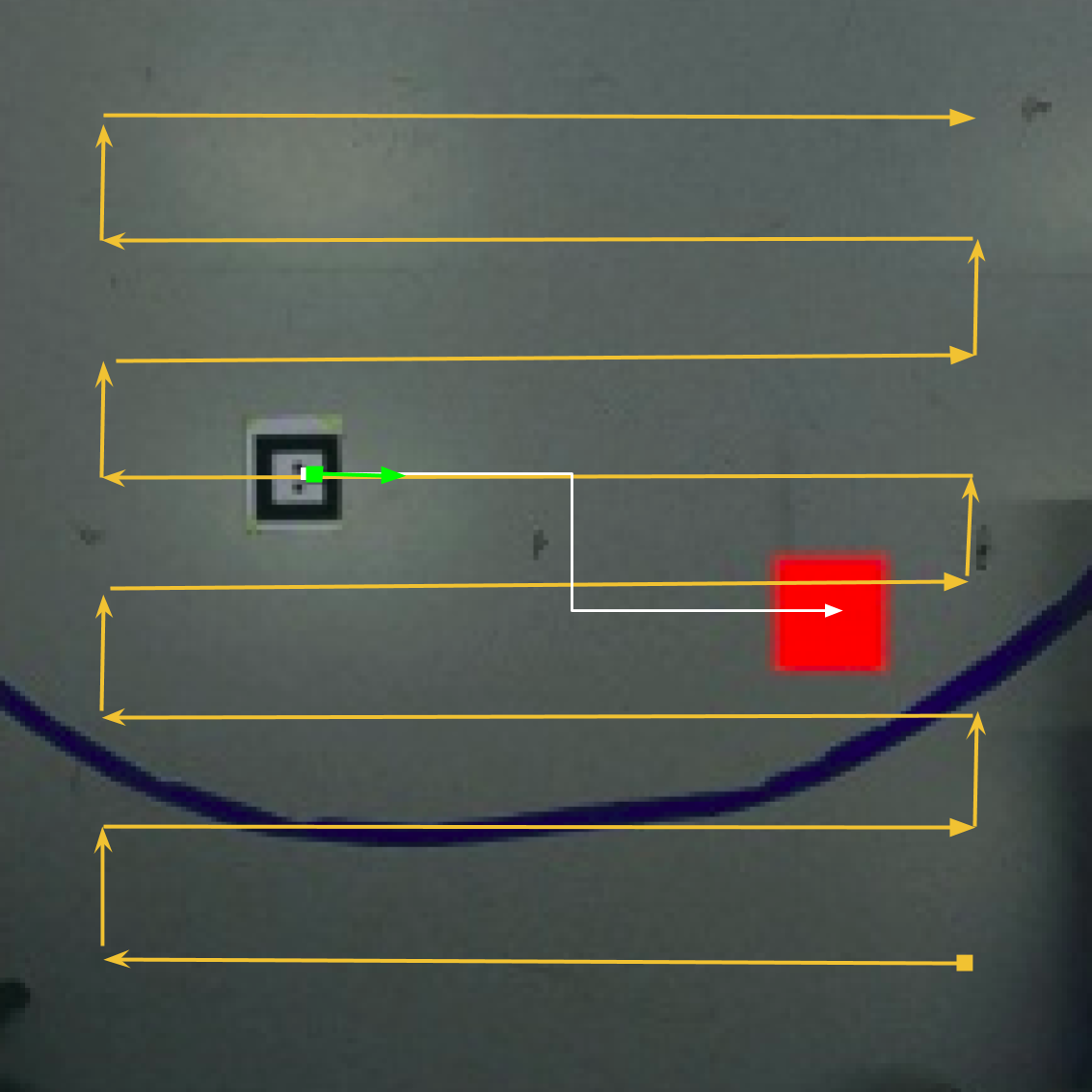}

    \end{subfigure}    
    
    \caption[Representation of data sampling strategies to distill the teacher policy.]{Representation of data sampling strategies to distill the teacher policy. Left, on policy sampling. Right, grid sampling.}
    \label{fig:5_DiscoRL:data-generation-strategies}
\end{figure}

\item  \textit{On-policy generation (Fig.\ref{fig:5_DiscoRL:data-generation-strategies}, left)}: We start an episode from a random point, then at each timestep $t$, we collect an observation $o_t$ and perform the action $a_{\pi_i, t}$ of the teacher policy. 
$D_{\pi_i}$ is thus composed of tuples ($o_t, p(a_{\pi_i, t} \mid o_t)$), with $p(a_{\pi_i, t} \mid o_t)$ the action probability associated to the action $a_{\pi_i, t}$ taken by the teacher, i.e., a \textit{soft label}, since we use the Kullback-Lieber divergence loss.

\item  \textit{Off-policy generation from a grid walker (Fig.\ref{fig:5_DiscoRL:data-generation-strategies}, right)}: at each time-step $t$, we collect an observation $o_t$ by performing an action $a_{grid, t}$ of a \textit{grid walker} exhaustively exploring the space of the arena. However, for each $o_t$ we save the probability of action $p(a_{\pi_i, t} \mid o_t)$ that would have been taken by a teacher policy. 
The goal of this strategy is to provide a more exhaustive sampling of the space of robot positions.

\item  \textit{Random Walker}: \checked{We start an episode from a random point, then at each time-step $t$, we realize random action for 200 time-steps. This method is proposed as a baseline.}

\end{itemize}

Performance of policies distilled using such strategies (see Fig. \ref{fig:5_DiscoRL:comparing-on-policy-strategies}) show that \textit{on-policy generation} (i.e., demonstrations) suffice to reproduce performance close to those of teacher policies on every task individually, with reasonable stability. In particular cases, see Fig. \ref{fig:5_DiscoRL:comparing-on-policy-strategies} for task TC, this strategy even provides a small boost in performances in the student policy over the teacher policy.

\medskip

However, using \textit{off-policy data generation from a grid walker} for distillation results in either unstable or poorly performing policies, especially in tasks defined by a reward function requiring the agent to move actively (\textit{TC} task, blue part of eq. \ref{eq:5_DiscoRL:reward-circular-task}) or anticipate the behaviour of another agent (TE task). 
In this case, the resulting policy reaches the performances of a lower-bound baseline obtained by distilling from trajectories of an untrained policy (see \textit{Student on off-policy data with a random walker} in fig. \ref{fig:5_DiscoRL:comparing-on-policy-strategies}), i.e. from a policy with random weights with input in the raw pixels' space.

\begin{figure}[ht]
    \centering
    \includegraphics[scale=0.35]{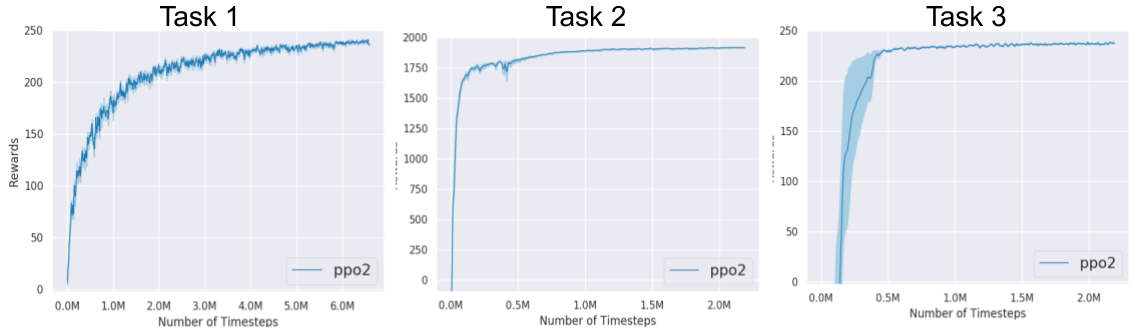}%
    \caption[Accumulated rewards on each task.]{Mean and standard error of rewards during RL learning of each task separately. Each task is learned using the same type of SRL model (SRL Combination), trained on each environment. All three tasks are mastered within roughly 2M time-steps.}
    \label{fig:5_DiscoRL:teacher_policies}
\end{figure}

\begin{figure}[ht]
    \centering
    \includegraphics[width=0.34\linewidth]{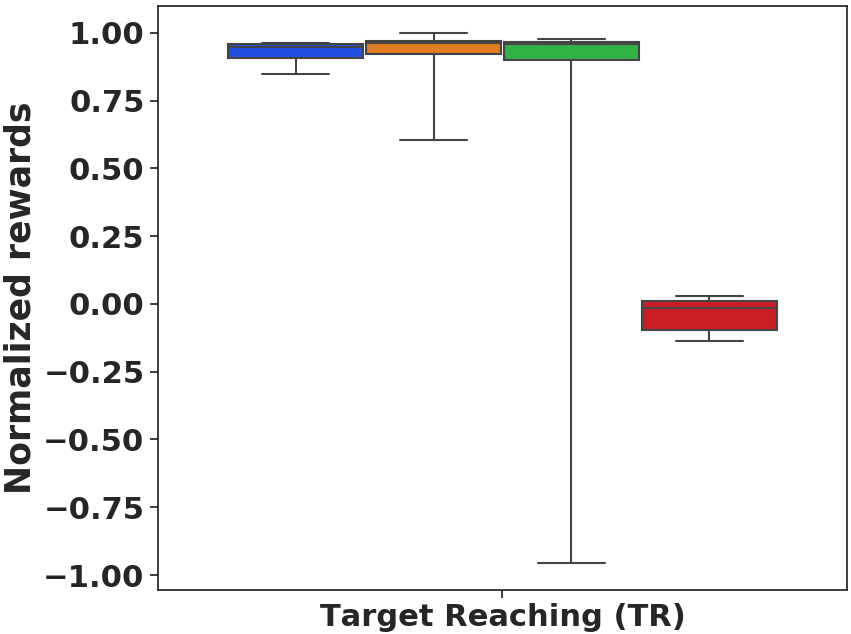}
    \includegraphics[width=0.32\linewidth]{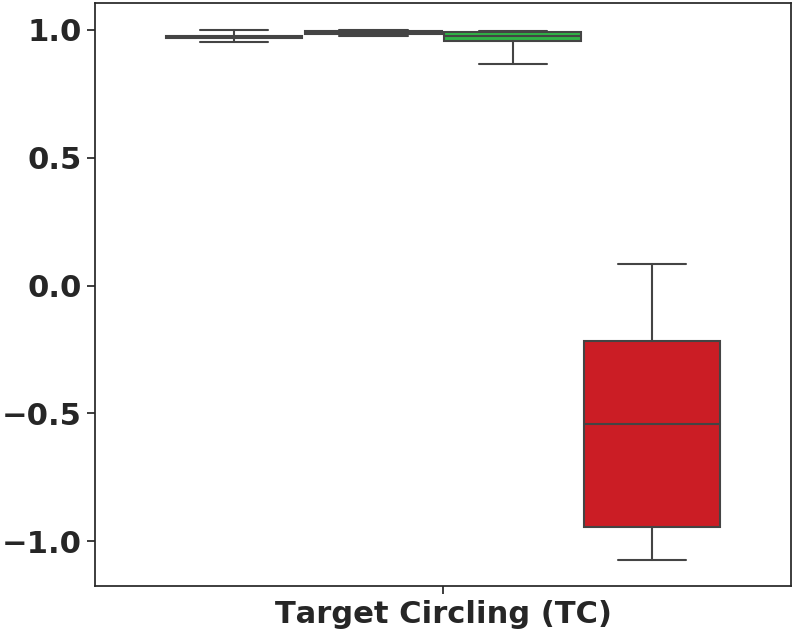}
    \includegraphics[width=0.32\linewidth]{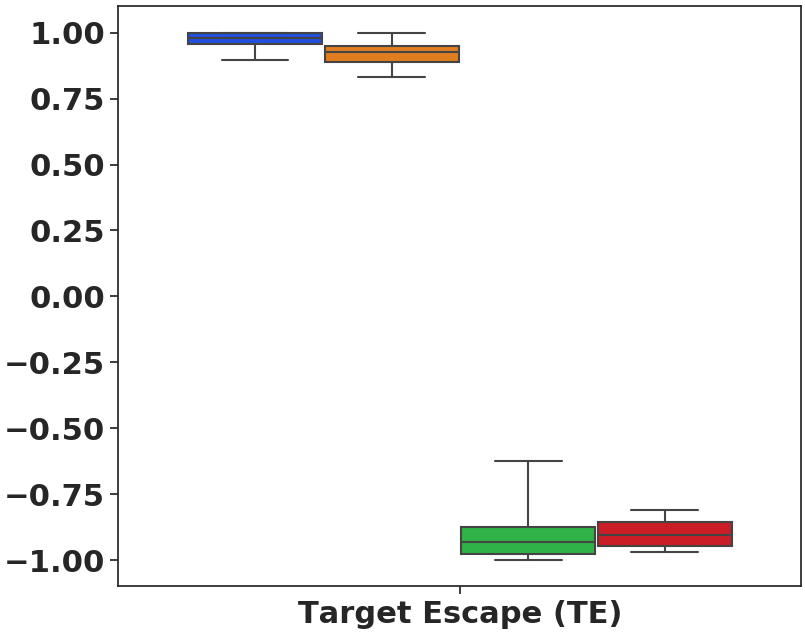}
    \includegraphics[width=.99\linewidth]{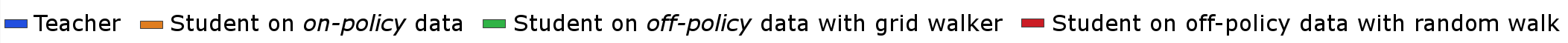}
    \caption[Performance comparison with different sampling strategies]{Efficiency (normalized rewards w.r.t the best teacher performance) of policies distilled on 8 seeds using various data sampling strategies for each task separately. Each evaluated policy is distilled on 15k tuples of sampled observations and action probabilities, for 4 epochs (see criteria of stopping in Section \ref{subsec:5_DiscoRL:continual_learning} and Figure \ref{fig:5_DiscoRL:teacher_policies}).} 
    \label{fig:5_DiscoRL:comparing-on-policy-strategies}
\end{figure}

\begin{figure}[ht]
    \centering
    \includegraphics[scale=0.25]{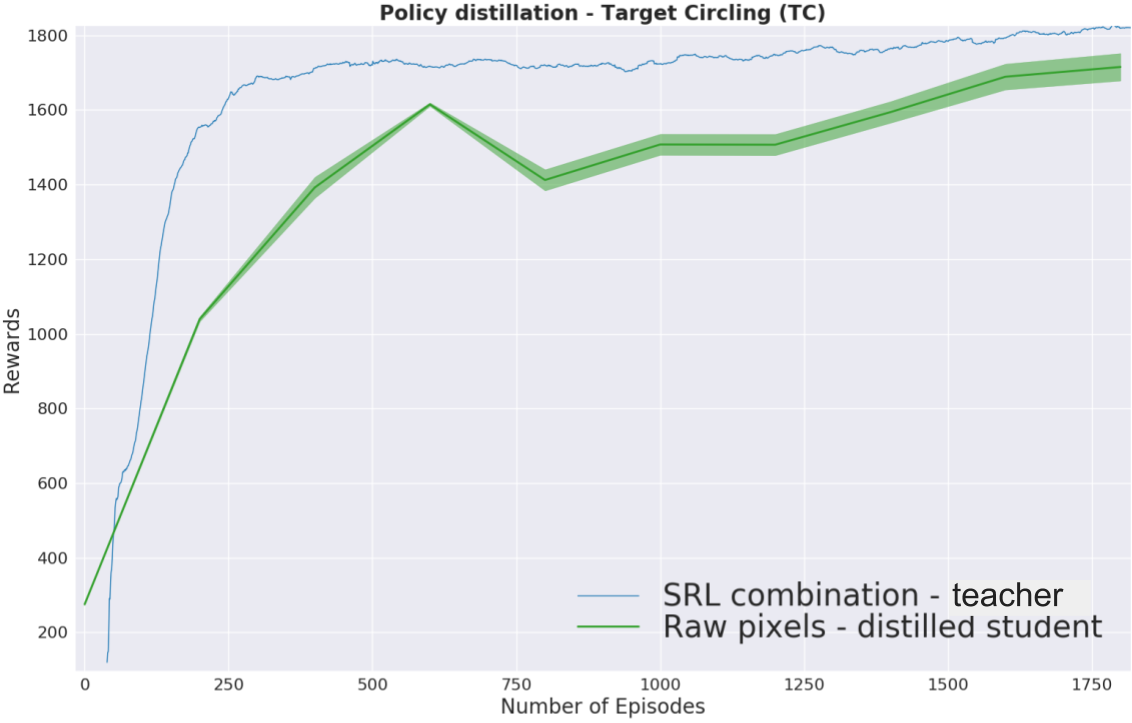}
    \caption[Demonstration of the effectiveness of distillation.]{Demonstration of the effectiveness of distillation. Blue: RL training curve of an SRL based Policy (SRL Combination) on the target circling (TC) task. Green: Mean and standard deviation performance on 8 seeds of distilled student policy. 
    The teacher policy in blue is distilled into a student policy every 200 episodes  (1 episode = 250 time-steps). }
    \label{fig:5_DiscoRL:distillation_cc}
\end{figure}

We performed a more explicit evaluation of distillation in the task 2 (Target Circling (TC)). While we train a policy using RL, we save the policy every 200 episodes (50K ), and distill it into a new student policy which we test. This is illustrated in Fig. \ref{fig:5_DiscoRL:distillation_cc}. Both curves are very close, which indicates that policy distillation enables to reproduce the skills of a teacher policy regardless of the teacher's state of convergence on the evaluated task. Moreover, distillation is able to transfer knowledge from teacher policy into a student using a small number of observations, i.e only 15k samples (w.r.t. the volume of samples required to learn the teacher policy, see Fig. \ref{fig:5_DiscoRL:teacher_policies}).

\subsection{Evaluation of each task separately}
\label{sub:5_DiscoRL:distill-sep}

\checked{We perform two evaluations on our final models.}
Our first evaluation is the performance of the final policy on the simulated environment.
This evaluation can then be compared with the performance of each teacher policy.
For the second evaluation we test if the policy is robust to the reality gap and can be adapted into a real life scenario. The simulation is voluntary close the real life setting but the reality gap is notoriously problematic.

\medskip

Before moving to a continual setup, we tested if it is possible to solve each task separately using distillation. Therefore, we evaluate distillation process for each task separately.
The challenges, is not only the ability to distill knowledge, but also the ability to know when a policy has been distilled, i.e. properly learned by the student. 
Indeed, due to the hypothesis of real continual learning settings, access to previous environments is not possible. Then, the challenges is to find a proxy task that can help indicate when early stopping of the policy distillation can be applied. 
That proxy task or signal should be different from the reward achieved in previous environments, which is no longer available. In our experiments, we found empirically that all policy could be distilled and  that limiting training to a small number of epochs, i.e $N=4$, guarantee policy learning.

\subsection{Main result}

In this section, we present our final results. We used \textit{on policy} data generation and training using KL-divergence loss with $\tau = 0.01$, as described in Section \ref{sec:5_DiscoRL:sampling-strategies-section}. 
In figure \ref{fig:5_DiscoRL:final_perf}, we show box plots over 10 episodes of reward performances for teacher policies in each task, and for the distillation of the same three teachers into a single student using DisCoRL. Each policy is evaluated in simulation and also in real-life on the robot. As a reference, we also show the performance of a random agent in each task. Our approach is effective in a continual reinforcement learning setting: the performance of teachers and student are similar.

More precisely, there are two main challenges to overcome in our setting: learning a behaviour via distillation by using only a limited number of examples, and the reality gap which can notoriously~\cite{tobin2017domain} introduce variations that may lead the policy to fail. Fig. \ref{fig:5_DiscoRL:final_perf} demonstrates the efficiency of our approach at overcoming both of these issues: only a small fraction of performance is lost from teacher to student, and from simulation to reality. We can see that the single student distilled policy achieves close to maximum rewards in all tasks, in real-life.

\begin{figure}[h]
    \centering
    \includegraphics[width=0.32\linewidth]{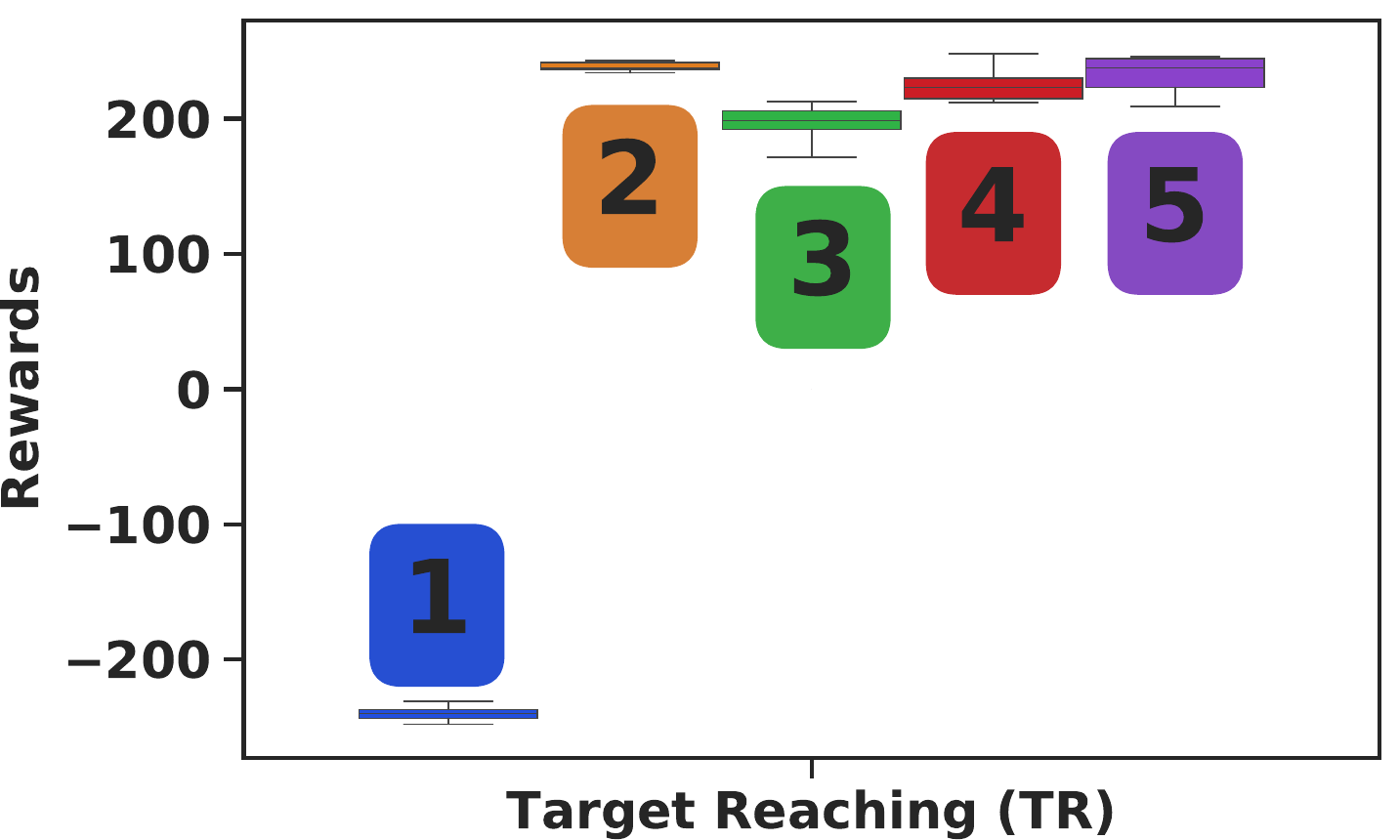}
    \includegraphics[width=0.32\linewidth]{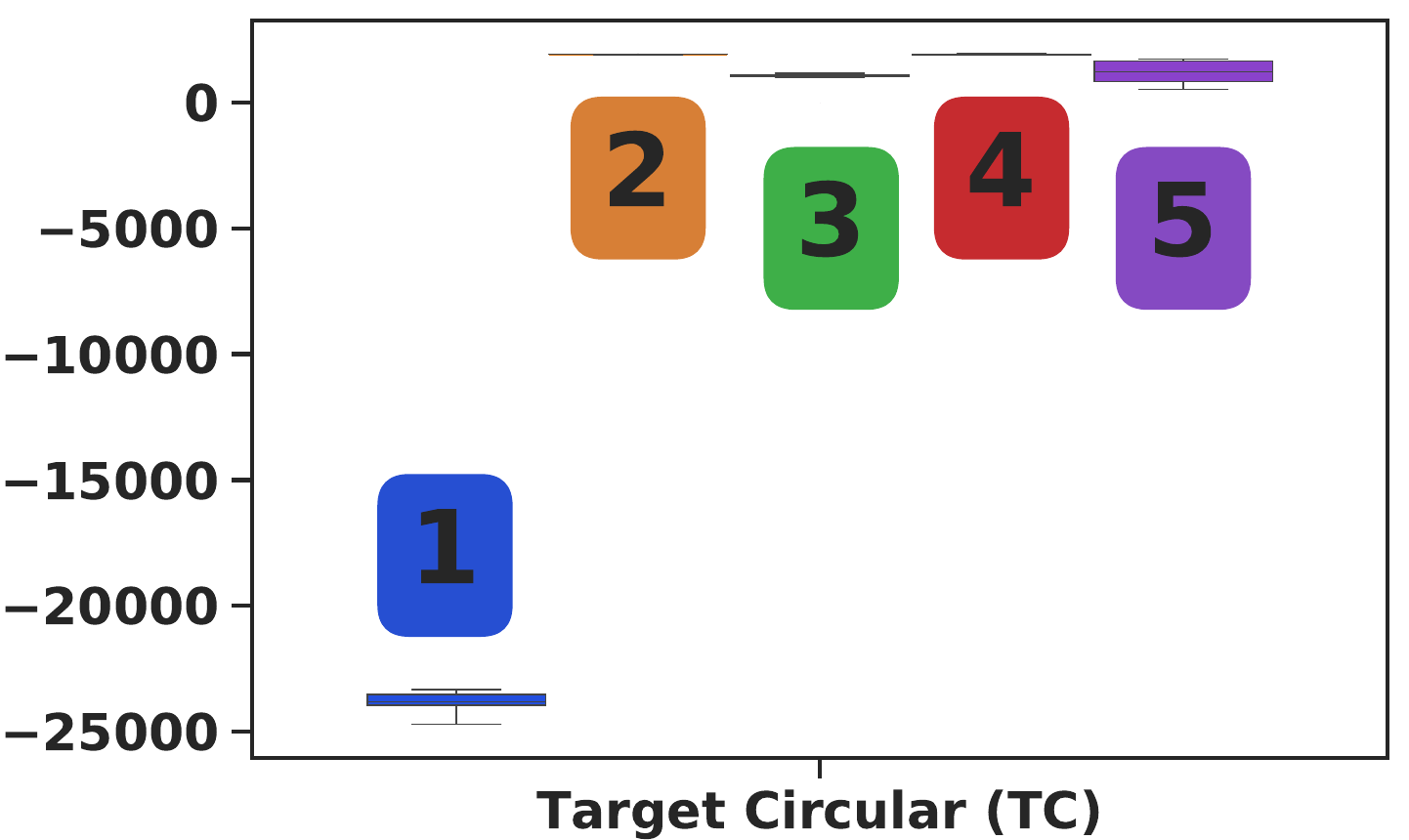}
    \includegraphics[width=0.31\linewidth]{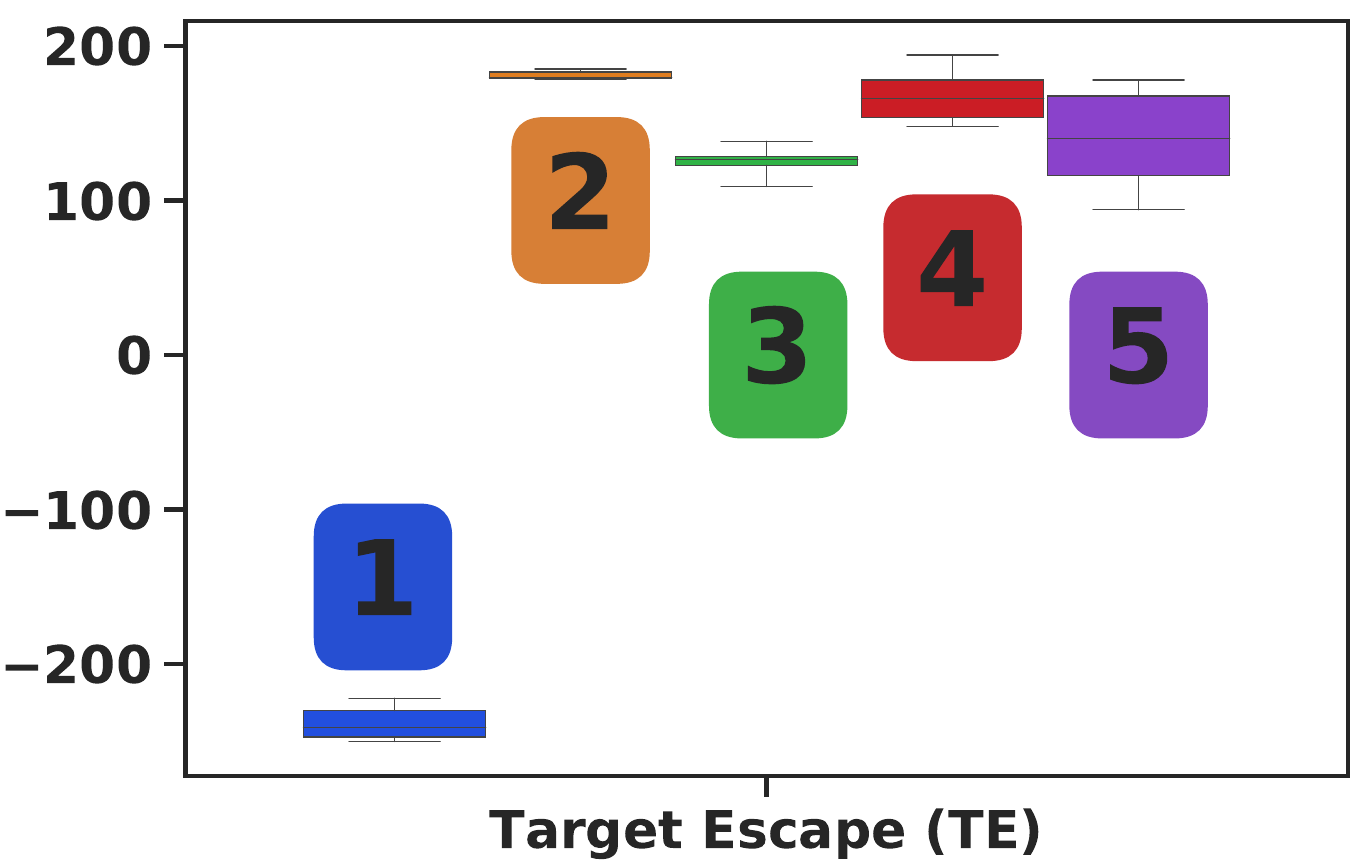}
    \includegraphics[width=.99\linewidth]{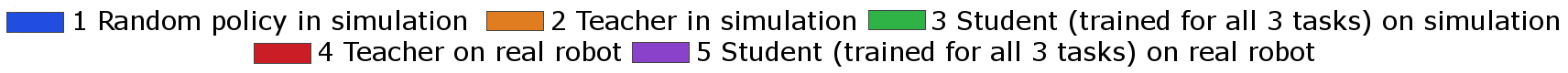}
    \caption[Distillation of teacher policies into a student network.]{Main result: distillation in a continual learning setting of three teacher policies into a single student policy. The resulting policy is able to perform all three tasks both in simulation and in the real world, while minimizing forgetting.} 
    \label{fig:5_DiscoRL:final_perf}
\end{figure}

\subsection{Negative results}
\label{subsec:5_DiscoRL:neg_res}

While distillation is effective for policy transfer, we also tested other alternatives worth mentioning. 

\subsubsection{Elastic Weight Consolidation (EWC)} 
EWC \cite{kirkpatrick2017overcoming} was implemented as a continual learning baseline to compare with the distillation method. EWC has the appealing advantage of not re-using any data from previous tasks. However, in all cases we found the method unsuccessful. 

Tuning the $\lambda$ parameter that controls the trade-off between weight protection and learning the new task showed that either $\lambda$ is too low and catastrophic forgetting happens, or $\lambda$ is too high and nothing new is learned (i.e., the full network is frozen).
A $\lambda$ value providing a proper balance in between both effects could not be found for such sequential tasks to be learned. 

\subsubsection{Progress and Compress (P\&C)} 
P\&C \cite{schwarz2018progress} was tested but as EWC, we had problems with the importance factor $\lambda$ and we where not able to learn three policies into a single model with this method.

\subsubsection{Adding task labels for distillation} Even if all tasks contain a visually differentiating identifier, they remain visually similar. In cases, we found that a distilled policy trained to perform well on several tasks can mix up tasks and thus not perform adequately. Hence, either adding tasks labels directly, or adding a module in the network that predicts the task label could be a way to improve the efficiency of distillation. However, none of the approaches were successful in practice, yielding the same results with or without task labels. 

\subsubsection{Gumbel-Softmax action sampling for the student}  This trick \cite{jang2016categorical} allows to sample from a categorical distribution using a softmax output layer. It has proven to be useful for action sampling in policy learning \cite{schulman2017proximal}. However, in our case we saw no improvement over a simple argmax strategy for action sampling when we used it on the student policy at test time.

\section{Discussion and Future Work}
\label{sec:5_DiscoRL:discussion}

Even if we believe this work proposes a stable and scalable framework for continual reinforcement learning, several possibilities for improvement exists. %
For example, we could have not only a policy learned in a continual way, but also the SRL model associated. We would need to update the SRL model as new tasks are presented sequentially. One possible approach would be to use Continual SRL methods like S-TRIGGER \cite{caselles2018continual} or VASE \cite{achille2018life}. Moreover, we would like to optimize more the memory needed to save samples by reducing their number and their size.

\medskip

Moreover, training policies on real robot experiences without the use of simulation would be desirable. However, at the moment, this is more a RL challenge than a CL challenge. One promising approach would be to use model-based RL while learning the state representation learning (SRL) model to improve sample efficiency. Though, nowadays approaches still do not offer solutions working in a reasonable amount of time.

\section{Conclusion}

In this chapter, we presented DisCoRL, an approach for continual reinforcement learning. 
\checked{DiscoRL is a simple yet efficient method for continual multi-tasks reinforcement learning. It performs independent learning of different policies without disabling forward learning transfer with a fast merging policy methodology. Moreover, at test time DiscoRL can run all the policies with a single model without the supervision of a task label.}

The method consists of summarizing sequentially learned policies into a dataset to distill them into a student model.  It allows to learn sequential tasks in a stable pipeline without forgetting. Some loss in performance may occur while transferring knowledge from teacher to student, or while transferring a policy from simulation to real life.
Nevertheless, our experiments show promising results in simulated environments and real life settings.

\newpage
\chapter{Discussion}
\label{chap:6_disc}

In the previous chapters, we presented a global overview of continual learning, the strength of replay methods and the analysis of generative replay methods. We focused on the study of generative models in a CL setting and the application of generative replay to classification problems. We also experimented with the rehearsal method for continual multi-task reinforcement learning.

In this chapter, we would first like to discuss continual learning research in a more global picture. Hence, we will discuss popular objectives in continual learning. We also try to disentangle potential use cases and present application scenarios for continual learning. 
Secondly, we discuss the work done in this thesis and the choices made. Then, we introduce a few continual learning pitfalls that should be avoided to perform healthy continual improvements.

\section{Rethinking Continual Learning Objectives}
\label{sec:6_Disc:Objectives}

Continual learning is a vast research domain with very ambitious objectives. The general aim is to learn from non-static data-sources but more particularly to learn from the real world. In this section, we first present the popular objectives of continual learning. Then we introduce several use cases we believe are the main long term objectives of continual learning. 

\subsection{Popular objectives of continual learning}
\label{sub:6_Disc:Midway}

Several continual learning papers \cite{deLange2019continual, Farquhar18, LESORT2019Continual, aljundi2019continual} present general desiderata for continual learning. Those desiderata are presented as characteristics that continual learning algorithms should have. The most common ones are:

\begin{itemize}

\item \textbf{Maximize Final performance:} Algorithm should have the best performance possible at test time.

\item \textbf{Learning without forgetting:} The ability to learn sequentially and incrementally knowledge or/and skills.

\item \textbf{Graceful forgetting:} Remembering the essential only.

\item \textbf{Detecting concept drift:} The ability to detect when the data distribution changes to avoid forgetting.

\item \textbf{Storage-Free Continual Learning:} Avoiding the storage of raw data.

\item \textbf{Efficient learning / Few shot learning :} Grasping new concepts thanks to only few data points.

\item \textbf{Applicable algorithms:} Algorithms able to solve problems with reasonable constraints such as limited memory or power consumption.

\item \textbf{Transfer between tasks:} The ability to improve a specific knowledge/skill by learning another one (forward learning and backward learning).

\item \textbf{Transfer between settings:} An algorithm able to learn in an environment/continuum A should be able to learn in a similar environment/continuum B.

\item \textbf{Execution time:} The ability to learn in a limited time, at best near to on-line.

\item \textbf{Reproducible results:} Being able to reproduce results by an independent party.

\end{itemize}

Many of those characteristics are quite easy to evaluate. The metrics proposed in Section \ref{sub:2_CL:metrics} associated with the right benchmarks make it possible to evaluate most of them or, at least, it makes it possible to compare two algorithms.
 Unfortunately, most of them 
  are not intrinsically interesting for continual learning and they should be put in a more long term perspective to understand what is important to be addressed. 
For example, the desiderata of transfer between tasks is usually not the true objective. %
 The true objective is either to reduce the execution time or improve the final accuracy. The transfer might be an answer to those objectives but the transfer is finally not exactly a goal but just a way to reach another goal. Similarly, graceful forgetting or few-shot learning are interesting, for us, only if they improve other objectives such as final performance or execution time. %

By analyzing the desiderata again we find that the most important one for continual learning are the same as for machine learning:
\begin{enumerate}
\item Maximize final performance (and generalization) 
\item Applicable algorithms (being able to adapt algorithms to specific constraint on power consumption / memory / training time / inference time)
\item Reproducible results
\end{enumerate}
Indeed, the difference between classical machine learning and continual learning is not the objective or criterion to optimize but the hypothesis on the input data (iid vs non-iid).

Nevertheless, the difference in input data might lead to supplementary constraint to fulfill the global desiderata and overcome some needs such as:

\begin{itemize}
\item Minimizing memory
\item Minimizing training time
\item Minimizing computation power
\item Autonomous inference  (label free inference)
\end{itemize}
The tools to answer those desiderata and constraints might be provided by other research domains such as:
\begin{itemize}
\item Few-Shot learning
\item Transfer (Backward  / Forward ; upon tasks / settings)
\item Knowledge distillation (for Graceful forgetting)
\item Learning from sparse labels
\item Detecting concept drift
\item Minimizing storage
\item Sparsity
\item $[~$...$~]$
\end{itemize}

Hence, progress in those research domains might be useful or necessary to leverage continual learning problems but it always depends on the targeted problem. %
Nevertheless, it is safer to keep apart desiderata from side constraints and from other research field to keep track of what we are trying to do.

We can also mention the bio-inspiration as a side objective. The biological agents provide many ideas to solves continual learning problems. However, as the other side objective, it is not what continual learning aims at solving. For example, even if biological agents don't store raw data, it does not mean that continual learning should not do it either. In many cases, saving raw data helps, as shown in Chapter   \ref{chap:5_DiscoRL}.
Creating false prohibition is deleterious for continual learning. 
The aeroplane %
is a working solution illustrating that biological agents do not necessarily have the optimal solution for our needs.

\bigskip

In the same spirit of finding the real goal of continual learning, in the next section, we present what we believe are the potential type of use cases for continual learning.

\subsection{Potential use cases}
\label{sub:6_Disc:Long}

The long term objectives are linked to the potential use cases of continual learning. Depending on them, the side objectives described in the previous section are not expected to be optimized in the same way. The application cases can be numerous and quite different. Therefore, to develop specialized approaches for different scenarios, it is necessary to identify the important differentiating factors.
We present here a possible classification of potential use case scenarios that are more interesting to target specifically than continual learning in general.

\begin{itemize}
\item \textbf{Incremental Learning}
 The objective is to learn new knowledge or/and skills without forgetting. In this case, the goal is not to improve knowledge but only to incrementally grow a set of knowledge.
 This is typically the case of a trained classification model that should learn a new class without forgetting the previous ones with very limited access to data from the past.
This case is the most straightforward to describe, the primary criterion that matters is learning without forgetting. The past learning experiences are not supposed to be accessible again and consequently, everything forgotten is lost forever.
In this case, since the forgetting should be minimized as much as possible, it looks more acceptable to not be too restrictive on the memory needed or the efficiency of learning or computation.

\underline{Example:} \textbf{Adding classes to a classifier}

\textit{A company, let's call it AItowardAGI, needs a trained model for a classification purpose. A solution is to buy a trained classification model from a company, training4U, that has access to more computation and more data. However, AItowardAGI has some personal data and would like to improve the model on this personal data without damaging initial models. Then, the company will need an incremental learning method to learn without forgetting and improve the model with new classes. We hope that AItowardAGI has read Chapter  \ref{chap:2b_Replay} of this thesis and therefore they know that they will not be able to improve the model if they don't ask for some more information about the initial training data from training4U company.}

\item \textbf{Lifelong Learning}
Lifelong learning consists of learning a never ending task, there are few variations in the tasks but the agent should always improve and be able to handle more and more of those little variations.
It could be essential for applications where smart agents need to keep improving and gain efficiency from experiences on a specific task.
In this case for example, forgetting is not unconditional, it is somehow acceptable to forget if globally it allows to improve the performance. However, as the agent learns at the same time it is used, it is expected to learn as fast as possible to adapt to changing situations. 

\underline{Example:} \textbf{Gift wrapping robot} 

\textit{A company wants to build robots that wrap gift papers. They want to send those robots in shops to wrap gifts automatically. However, each year the gifts are different and each robot needs to be updated. Moreover, shops sell objects of different sizes and shapes. Therefore, if all robots had a learning algorithm that could adapt to new sizes and shapes without selling company intervention, it would be very practical. 
  This company could use an algorithm similar to the one used in Chapter  \ref{chap:5_DiscoRL}.}

\item \textbf{Multi-task agent}
Multi-task agent is a mixture of the incremental learner and the lifelong learner. Similar to the lifelong learner, it is in a never ending situation however it has several different tasks to learn and improve in its lifelong learning curricula.
In the multi-task learning, as for lifelong-learning, forgetting can be acceptable in order to improve the global performance, however improving on one task should not deteriorate too much another one to keep progress positive.

\underline{Example:} \textbf{Periodic improvement}

\textit{The training4U company sells pre-trained classification models. So at any time, they should be able to sell a final trained model, however, they regularly receive new data about new or existing classes that could help to improve their models. Training again from scratch is too expensive and they will need a continual process to improve the models on new data without damaging current knowledge. 
Then, they need multi-task learning methods to keep improving existing classes and still be able to add new ones.
We hope they also read the Chapter  \ref{chap:2b_Replay} of this thesis, and therefore they will have developed a correct remembering process for their models.}

\item \textbf{Autonomous agent}

The autonomous agent is a multi-task agent which can additionally set its own objectives and has the capacity to explore and understand without clear tasks. This agent possess curiosity \cite{Oudeyer07, Schmidhuber10}.
  The curiosity is a self-motivated objective leading an agent to explore its environment and the possible interactions.
It helps the agent to improve on past and new tasks and enable it to eventually anticipate future tasks.  
  Moreover, the autonomous agent should be able to create its own innovative solution to problems with limited supervision, as in Open Ended Learning \cite{Doncieux18}. %
This kind of agent should make the best use of the resource that is available to it. It also has to be curious, creative and have some kind of survival reflex to evolve without auto-destruction.

\underline{Example:} \textbf{Robots on a Foreign Planet}

\textit{A space exploration program would like to send a robot to a very far exoplanet for discovery. The problem is that any communication can take weeks or years to be sent, so communication is very limited and the robot needs to be autonomous to survive in this unknown environment. Moreover, we would like  the robot to autonomously explore and adapt to the environment to complete future tasks faster. For example, the robot could automatically learn to recognize the type of grounds where it can pass to be able to move better later. The robot then needs a continual learning algorithm to complete its mission, explore and survive. }

\end{itemize}

For each case, we described what we believe are the most emblematic objective to optimize in the potential use cases.
However, depending on the use case more objectives may have to be optimized like computation power which could be restricted in certain situations like in embedded platforms. In any case, an approach would be more useful in the global research effort if it targets long-term or/and short-term objectives specifically.

\medskip

To summarize, in continual learning, the essential expected ability is to be able to improve a model based on new data, either by adding new concepts or by improving/strengthening already known concepts. 
A second important notion is adaptability to various types of supervision signals to learn, at best an algorithm should be able to learn at least \say{something} whatever the sparsity of supervision signal or its type (reward or label). 
Therefore, continual learning is somehow the science of learning autonomously and might have links to the auto-machine learning research field~\cite{NIPS2015_5872}.

\section{Discussion on the thesis choices}
\label{sec:6_Disc:Choices}

In this section, we discuss the different choices made in this thesis and the conducted experiments.

\subsection{Replay methods}
In this thesis, we presented the use of replay to deal with catastrophic forgetting. Indeed, replay methods have both the advantage to make models able to remember and learn on past tasks. We study it in particular in the context of incremental classes or tasks. %
We showed that in such settings, regularization and dynamic architecture approaches %
 rely on a task label for inference. Replay methods propose a label-free inference model that can be therefore easily deployed after training.

As presented earlier in the thesis, the assets of replay are:

\begin{itemize}
\item \textbf{A posteriori understanding:} Past learning experiences can be reinterpreted with current knowledge and favor backward transfer.
\item \textbf{Task agnostics memorization:} Some part of the memorization process is not affected by the current task and just aims at representing the learning experience data.
\item \textbf{Test labels agnostic:} The learned model does not need any supervision for inference.
\item \textbf{Model agnostic:} The memorization process is at least partially not affected by the model architecture.
\item \textbf{Manageable memory:} The memory is easy to control, if a memory is considered useless, it can be erased by just deleting it or not replaying it anymore
\item \textbf{Unconditional remembering:} (Rehearsal only) By saving raw data the memorization is theoretically robust to misinterpretation and memory modification.
\item \textbf{Good on-line capabilities:} (Rehearsal only) Saving raw data is quite fast in comparison to learning a memorization model.
\item \textbf{Auto-Memorization with self synthesizing models: }(Generative Replay Only) The generative model automatically learns to synthesize data and compress it in its weights.
\item \textbf{Supplementary understanding of data for memorization: }(Generative Replay Only) The generative models can potentially generalize processed data and share its understanding with the inference model.
\end{itemize}

The continual learning use case  that fits perfectly with the project of this thesis is \say{incremental learning} (Section \ref{sub:6_Disc:Long}), a model that learns different tasks sequentially and should and never forgets.
Studying this setting produce results that can be transferred to lifelong algorithms and multi-task continual learning. We believe, as demonstrated in Chapter  \ref{chap:2b_Replay}, that in all learning situations, two concepts need anyhow to be confronted to be distinguished and replay might be the only method able to achieve it.

\subsection{Static Deployment}

We can define two type of deployments for continual learning:

\begin{itemize}
\item \textbf{Static Deployment:} The model is trained and used frozen. The training can still be continued later.
\item \textbf{Never Ending Learning:} In this case, the training never stops; every new experience can be exploited to learn and be integrated in the model.
\end{itemize}

Moreover, it might be worth distinguishing two types of potential use. First, the \say{milestone} use case, the model is trained and it should be ready at some point to be used. We don't care about how it learns as long as at some point in time the model is ready. Secondly, the \say{always ready} use case, where at any time the model could be used and should be aware of the last data point processed. 

We target models with milestone use case and static deployment. %
Those models can learn in a continual environment and learn concept and skills sequentially until deployment start.

We considered as too ambitious the always ready scenario and
we argue that static inference is probably the use of continual algorithms that could be the safest. %
For the static inference scenario, after training, the model can be assessed and be deployed without fearing for uncontrolled or adversarial modification of the model. A model that would continue to learn after deployment would be harder to assess and the learning processes are nowadays not stable enough to be deployed.

\subsection{Task labels}

The use of training labels can be seen as a limitation of the approaches. 
The use of the task label allows targeting only the learning-without-forgetting problems without addressing concept-drift detection. 
 
We think that task labelling is cheap and may significantly help continual algorithms to learn. In many applications of continual learning, we can assume that at training time, the algorithm has a bit of assistance to learn. 
If no task label is provided, the algorithms have a higher risk to diverge and misunderstand its learning experiences.

\subsection{Data stream distribution}
In the same way we assumed task label available for training, we assume that the learning curriculum is independently and identically distributed by part (except in Chapter  \ref{chap:5_DiscoRL}). This is a clear limitation that should be solved before expecting to tackle real environment settings, however, this setting allows to better understand catastrophic forgetting behaviour and memory processes. Since we know exactly at which moment the model will start forgetting, it is easier to analyse it and address it. The iid by part settings is then very interesting for research purpose.

\subsection{Classification tasks}
\label{sub:6_Disc:Classif}

Ideally, algorithms can learn on-line new concepts from unprocessed data. Therefore, it is unlikely to have a fully annotated dataset correctly preprocessed to maximize learning easiness. 
Benchmarks using sparse labelling would therefore be more appropriate and more in the spirit of continual learning.
However, the use of fully annotated classification benchmarks makes it possible to focus on continual learning problems rather than having to deals with sparse label problems.
Nevertheless, in Chapter  \ref{chap:5_DiscoRL}, we did experiments with reinforcement learning environments that have sparse labels. And we have seen that even if learning tasks from those environments is much harder it does not make the continual learning problems significantly harder.  

\subsection{Evaluations}
\label{sub:6_Disc:Eval}

In this thesis, we evaluate essentially our algorithms with the final performance, we believe that, with the computation cost, it is the most valuable metrics to evaluate CL as discussed in Chapter   \ref{chap:2_CL}. The computation cost is however often difficult to evaluate rigorously and we did not evaluate it %
  to focus on the final performance.

The MNIST, Fashion-MNIST or KMNIST benchmarks we used are simple.  Some experiments have been conducted on Cifar10 and Core50 (not described in the manuscript), however, the results were not stable enough to conclude from the experiments. The main conclusion is that generative models are difficult to train on those datasets in a continual setting. Therefore, it is still too early to expect generative replay to work in continual real-life settings.
However, with the rapid progress of generative models we expect generative replay to became a viable solution in complex environments.
As discussed for the classification task above, the use of simple datasets makes it possible to get rid of learning shortcomings of models and only focus on continual learning shortcomings. If it is already difficult to train a model in a classical setting, studying it in a continual setting is limited by the machine learning shortcomings.

\subsection{Answers to the framework questions}
\label{sub:6_Discussion:Answers}

In Chapter \ref{chap:2_CL}, we compiled a set of questions that continual learning approaches should answer to have a proper explanation of the setting and the method.

Here, we compile the global answers for those questions concerning the thesis results. 

\begin{itemize}

\item $\bm{Q_1}$: \emph{Does some data need to be stored? If yes, how and what for? (e.g. regularization, re-training, validation)?}

Yes, we store data for model selection purposes (validation set) and in the rehearsal approaches for remembering.
\item $\bm{Q_2}$: \emph{Is the algorithm tuned based on the final performance? I.e. is it possible to go back in time to improve performance?}

Normally no, but we admit that some hyper-parameters have been tuned empirically, so we did not respect the temporal coherence perfectly in our results. 
\item $\bm{Q_3}$: \emph{Are data distributions assumed i.i.d. at any point?}

Yes, the data distribution is assumed iid by part (iid during each task). Except in reinforcement learning experiments.
\item $\bm{Q_4}$: \emph{Is each task assumed to be encountered only once? }

Yes, for unsupervised learning, no for the other experiments (even if all tasks are only encountered only once in the reported experiments for clearer evaluation purposes).
\item $\bm{Q_5}$: \emph{Is the continual learning algorithm agnostic with respect to the structure of the training data stream? (e.g. number of classes, numbers of tasks, number of learning objectives...)}

Yes, the number of classes / tasks can be dynamically changed.
\item $\bm{Q_6}$: \emph{Does the approach needs a pretrained model for the CL setting? If so, what is the new knowledge that needs to be acquired while learning continually?}

No pretrained model is used for continual learning experiments.
\item $\bm{Q_7}$: \emph{How much available memory does the algorithm require while learning? Does the memory capacity requirement change as more tasks are learned?}

The model architecture is chosen empirically based on the non-continual performance on classical benchmarks, it stays fixed. Only the output layer can be dynamically changed to add more classes. In rehearsal experiments the memory grow as there are more past tasks.
\item $\bm{Q_8}$: \emph{Is the continual learning algorithm constrained in terms of computational overhead for each learning experience? Does the computational overhead increase over the task sequence? }

There is no particular constraint, the amount of computation increases at least linearly with the number of same size tasks.
\item $\bm{Q_9}$: \emph{Is the continual learning algorithm agnostic with respect to the data type? (e.g. images, video, text,...)} 

The architecture model is designed for images and their dimension is known in advance.
\item $\bm{Q_{10}}$: \emph{Is the continual learning algorithm able to handle situations where there is not enough time to learn?}

Not yet for generative replay. For rehearsal, yes since the algorithms just need the time to save data to remember them. 
\item $\bm{Q_{11}}$: \emph{In the presence of multiple tasks, is the task label available to the algorithm during the training phase? And during evaluation?}

Task label is used for training but not for testing.
\item $\bm{Q_{12}}$: \emph{Are all the data labeled? or only the first training set? Can the user provide sparse label/feedback (e.g. active learning) to correct the system errors?}

In supervised experiments, all data are labeled, the labels are also used for evaluation in unsupervised learning, there are sparse labels for reinforcement learning.
\item $\bm{Q_{13}}$: \emph{What is expected from the algorithm to remember at the end of the full stream? Is it acceptable to forget somehow, when task, context or supervision change?}

Since the number of tasks in the experiments is not very high, the algorithm is expected to remember everything as much as possible.

\end{itemize}

\section{Continual Learning Pitfalls}
\label{sec:6_Disc:Pitfalls}

With regard to the potential use cases from Section \ref{sub:6_Disc:Long}, it is important to not fall into pitfalls that do not help to push forward true objectives. %

\subsection{The bias of the future}

In opposition to classical machine learning, the temporality in continual learning is essential. A CL algorithm should be prepared for its future without knowing it. Therefore hypothesis should be done about the future. In a chaotic world, it would not be possible to make any plan about the future, fortunately, we are not in a chaotic world and we can have reasonable hypothesis of the future to consider what can likely happen or not.

Nevertheless, we should not use the future to improve learning algorithm. Indeed, in research experiments we can virtually control the temporality of learning experiences and potentially use the future. Thus, we should be cautious to not create causal incoherence.

The main problem is that tuning the parameter at time $t=0$ based on results at time $t>0$ is aberrant. Algorithms can not be perfectly designed in a one-shot process, however, if they are designed on one curriculum they should be tested on another one to be valuable.

In the scope of continual learning, the hyper-parameter should be selected only based on the present tasks and past tasks. In a lifelong task, we believe the use of populations of models might be a good solution to deal with hyper-parameters selection. Indeed if we have a population of models with different hyper-parameters we could select the best hyper-parameter in one run.

\subsection{Spread out objectives}

As explain in the Section \ref{sec:6_Disc:Objectives}, continual learning is a wide research domain with ambitious expectations.
However, it is clear that no algorithm can tackle all learning situations. Understandably, algorithms have a limited scope of application, aiming at finding the algorithm that can solve all problems is chimeric. %

One of the pitfalls in continual learning is to not specify what kind of setting is targeted. It is probably not possible to find an approach that can solve catastrophic forgetting no matter the subject.
Then, to better address continual learning approaches and application generally it is appropriate to specify precisely the goals and scope of a research project.
The question presented in Section~\ref{sub:6_Discussion:Answers} and Chapter~\ref{chap:2_CL} should allow to specify the essential aspects of a continual learning research project such as the learning subject, the learning setting, the learning objectives and the learning tools.

\medskip

Therefore, we should target reasonable goals to expect to be able to bring useful solutions and not being lost in a far too big space of exploration.

\subsection{Scalability: a Double-Edged Sword}

In machine learning research, a comment often present in article reviews is \say{Have you tried on harder problems?}. There exist many incentives to try to tackle difficult problems in order to validate a theory or an approach.
However, those incentives are just a rule of thumb following the adage \say{who can do more can do less}. However, in machine learning, the successful transfer from one settings to another might have low correlation with the difficulties of tasks.

Then, asking for proof of scalability is probably more a mandatory obstacle to overcome to have recognition than a legitimate request for evaluation. Asking for more baseline or comparisons however might be a better request to estimate the legitimacy and the appropriateness of any approach. 

In continual learning (and in other machine learning fields), the MNIST dataset is often used, and reviewers often ask for harder settings rather than for stronger baselines. However, it is not clear if continual learning specific difficulties are correlated with machine learning difficulties. Therefore asking for harder datasets is not necessarily a service for the community. Asking for more datasets can however be legitimate to compare  approaches with similar results.

\medskip

In the next section, we gather a set of recommendations that we believe are important to have in mind for research projects in continual learning.

\section{Research recommendations}

For more concrete indications on what we consider worthwhile checking while creating a CL approach, we suggest a set of recommendations.
Those recommendation point out research topics that should be privileged or methodology that should be respected.

\begin{recommendation} 
On-line capabilities: CL algorithms should adapt to new data as soon as they are available without assumptions like the number of tasks or classes.
\end{recommendation}

\begin{recommendation} 
Autonomous inference: CL algorithms should be autonomous at least for inference, therefore they should not assume labeling information at test time.
\end{recommendation}

\begin{recommendation} 
Scalability evaluation: In order to provide a proper evaluation of the scalability and continual learning performance, we recommend, as the authors from \cite{Farquhar18}, to evaluate algorithms on more than two tasks.
\end{recommendation}

\begin{recommendation} %
Resources evaluation: To be practical, CL systems should evaluate resources consumption as for memory or computation.  %
\end{recommendation} 

\begin{recommendation} 
Reporting metrics:  We recommend reporting exactly the targeted objectives of the method and report the associated metrics. %
\end{recommendation}

\begin{recommendation} 
Ablation studies: we recommend reporting ablation studies to motivate as best as possible the different components and choices made in the CL algorithm and identify their importance for the different objectives or tools (learning without forgetting, transfer, few-shot-learning,...). %

\end{recommendation}

\begin{recommendation}
Distributional shifts: We recommend to formally describe the mechanism to handle distributional shifts, not only when tasks change, but also among batches where data points conform to different distributions.
\end{recommendation}

\begin{recommendation} 
Report precisely and clearly how an approach learns and the assumptions it make, as described in the framework (Chapter~\ref{chap:2_CL}).
\end{recommendation}

\bigskip

In this chapter, we discussed continual objectives and how specific criteria can be put in long term ambitions. We presented the global aim of the thesis regarding those objectives and justified the different choices made. We compiled a set of pitfalls that can divert continual learning from its progress and its potential applications and finally proposed a set of recommendation rules for continual learning. In the next chapter, we will conclude this thesis and provide some potential future research directions.

\newpage
\chapter{Conclusion}
\label{chap:7_ccl}

To conclude this thesis, we will first summarize the contributions presented in this manuscript. Then, we present research directions that could extend this work and improve the understanding and efficiency of replay methods for continual learning.

\section{Summary of Contributions}

The overall aim of the thesis was to study methods able to learn on incremental settings and which do not rely on any supervision to be deployable in real world applications.

First of all, we presented a framework for continual learning, built on top of \cite{Lomonaco2019ContinualLW} framework. This framework makes it possible to frame any continual learning approach systematically. It helps to rigorously set out the method, the scope and the evaluation of a CL algorithm to ease the comparison of the methods and transfer from one application to another.

Second, we demonstrate the advantages of replay methods in comparison with regularization and dynamic architecture methodology. We show that in the absence of task labels, the replay is the only method that could learn classification incrementally. 

Third, we applied replay methods in unsupervised learning (Chapter~\ref{chap:3_CL_GM}), supervised learning (Chapter~\ref{chap:4_CL_GR}) and reinforcement learning (Chapter~\ref{chap:4_CL_GR}) settings. We experimented, in particular, the generative replay methods and introduced the \say{Conditional Replay} method for continual learning. We showed that generative replay is agnostic to the test label and the number of tasks. Moreover, the generative model learns to memorize in a potentially more compact way than the initial dataset, it can generate never seen samples and offer its generalization capacity to learn downstream tasks.
We also used the rehearsal strategy for multi-task continual reinforcement learning, presenting the \textit{DiscoRL} algorithm. DiscoRL advantages are the unconditional ability to remember thanks to the hard memory process, the independence of individual policy learning and global policy learning without preventing forward transfer. We showed the effectiveness of the method on robots both in simulation and in real-life settings.

Finally, we present an extensive discussion on continual learning ultimate objectives, the choices made in the thesis experiments, the pitfalls of continual learning research and \checked{we introduce a list of recommendations which should help to push forward the limits of continual learning.}

\section{Future Research}

\checked{We presented our contributions to research in continual learning with replay methods. However, many research directions may improve these methods. We mention here several of them: improving sampling methods, improving generative models, detecting concept drift, improving hyper-parameters selection with meta-learning and estimating knowledge retention. 
Progress in those research direction will push continual learning possibilities forward.}

\subsection{Improving Replay Methods}
 
 Even if theoretically the replay method has clear advantages,
there are still potential improvements, in particular in the construction of the memory either by improving the validation of generative models for memory replay or by improving the selection %
 criterion for coreset. The data replayed should be representative enough to remember the task and general enough to be used in other tasks. %
Another research subject that would deserve to be studied is the protection against overfitting memory and insuring a good generalization. Sampling the memory should then be carefully done to find the good trade-off between the benefit of the memory without spoiling it. Smart sampling would also helps to reduce the algorithms computation consumption.

\subsection{Improving Generative Models}
Generative models are promising for continual learning. They theoretically propose a satisfying memorization solution for neural networks.
However, in the generative replay framework, the generative model is a serious bottleneck in the learning process both in computation and accuracy. 
Their training is long, computation heavy and they often suffer from instability. 

Since the experiments proposed in this thesis were performed, it seems that a lot of progress has been made in generated image quality. 
Therefore, we hope that those progresses will overcome generative model current limitations and fully exploit their potential for continual learning.

\subsection{Detecting Concept Drift}
In this thesis, we did not tackle the problem of concept drift detection. This problem is probably as crucial as tackling catastrophic forgetting for continual learning when the i.i.d assumption does not hold anywhere in the learning curricula.
It would be worth studying it more intensively, as it is essential to make lifelong learning work.

\subsection{Improving Hyper-Parameter Selection with Meta-Continual Learning}

\checked{As discussed, in Chapter \ref{chap:1b_ML}, the hyper-parameters (HP) selection is difficult in continual learning. In classical machine learning, we can select HPs that minimize validation set loss. However, in continual learning, we don't have access to the full validation set. It is then crucial to develop strategies to improve HP selection.}

\checked{Meta-learning, or learning to learn, is a training concept that aims at improving the learning process on new tasks by learning from many others.}
\checked{Meta-algorithms learn the best parameters or/and hyper-parameters to solve new tasks efficiently.
In continual learning, it could help to prepare learning algorithms for future tasks and improve existing strategies. 
It has already been used in many approaches~\cite{riemer2018learning,Javed2019Meta,beaulieu2020learning,caccia2020online}. }

\checked{For replay methods, meta-continual learning could be useful to improve memorization processes, especially by automatically learning the memorization hyper-parameters. For example, it could improve its learning of criterion for data selection or improve memory sampling to maximize remembering and minimize overfitting.}

\checked{Nevertheless, meta-learning needs to replay tasks several times to learn (or at least similar tasks) and it does not remember from one task to another. It might thus not be relevant in all continual learning settings.}

\subsection{Estimating Knowledge Retention}
In this thesis, the algorithm's goal is, at any time, to remember everything of the past in the active memory. 
\checked{Knowledge retention is only estimated as the knowledge that can be directly used for inference. For example, in classification, we only check if the neural network can correctly classify test images.
However, latent representations from past tasks can be hidden in the weights but can not be directly activated, i.e. the neural network can have memories of past tasks in inner layers without being able to use them directly because, for example, the output layer has been modified.
It would be interesting to tell the difference between a model that still has latent representations of past tasks and a model that forgot everything.}

\checked{Therefore, it would be relevant to develop methods that promote latent knowledge which could be easily reactivated and valuable again.}%

\newpage
\bibliographystyle{apalike}
\bibliography{continual,other_ref}

\clearpage
\includepdf[page=-, width=\paperwidth, height=\paperheight]{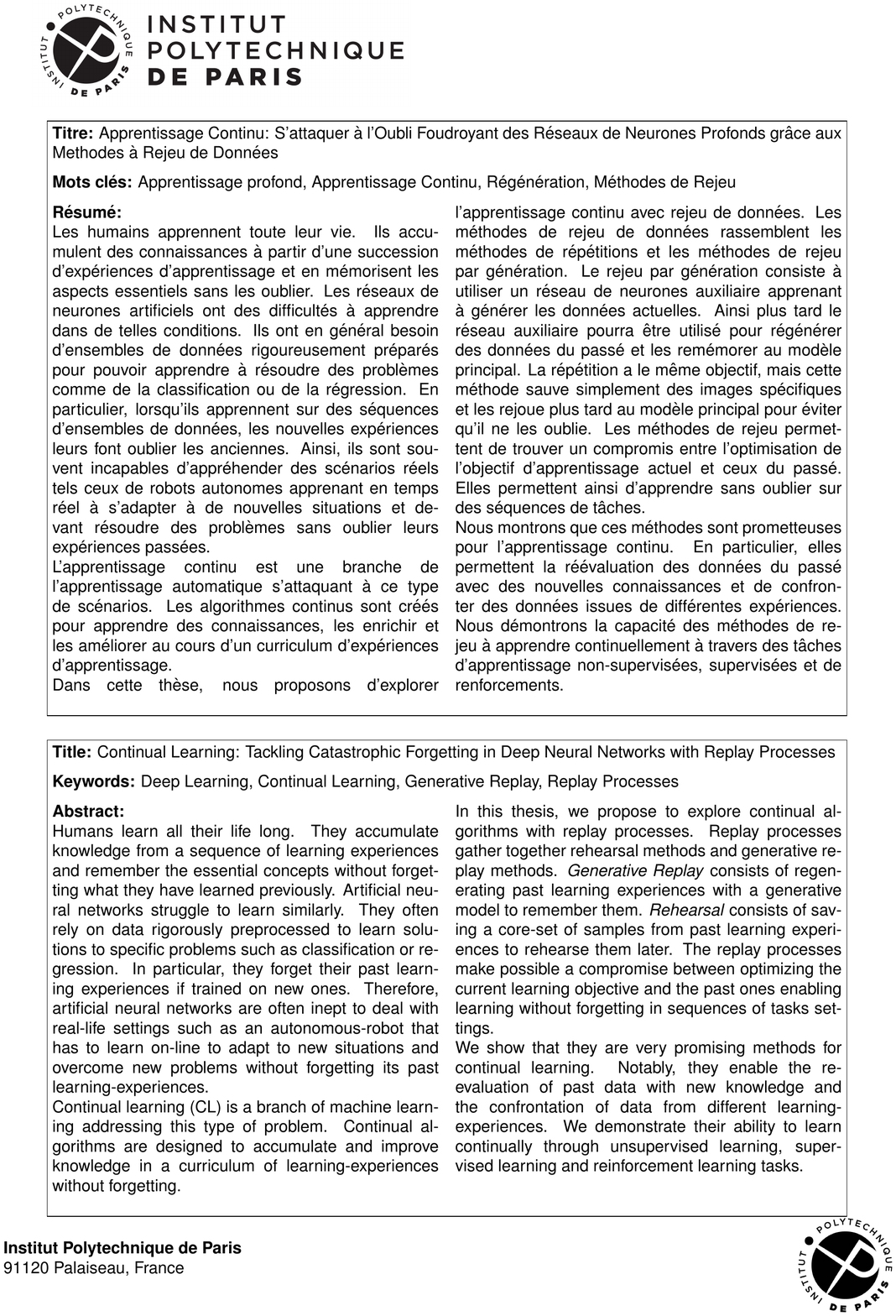}

\end{document}